\documentclass[runningheads,orivec]{llncs}

\usepackage{eccvabbrv}

\usepackage{graphicx}
\usepackage{booktabs}

\usepackage[accsupp]{axessibility}  

\usepackage{orcidlink}
\usepackage{amsmath,amssymb}
\usepackage{mathrsfs}

\usepackage{tocloft}
\usepackage{etoolbox}
\usepackage{bookmark}   
\usepackage{xcolor}     
\usepackage{enumitem}

\usepackage{array}
\usepackage{booktabs}
\newcolumntype{L}[1]{>{\raggedright\arraybackslash}p{#1}}

\usepackage{etoc}
\providecommand{\authcount}[1]{}
\etocsetlevel{title}{6}
\etocsetlevel{author}{6}

\PassOptionsToPackage{breaklinks,hidelinks}{hyperref}
 \hypersetup{hidelinks,pdfborder={0 0 0}}

\newcommand{\Ent}{\operatorname{Ent}}

\makeatletter
\@ifundefined{theorem}{
  \newtheorem{theorem}{Theorem}[section]
  \newtheorem{lemma}[theorem]{Lemma}
  \newtheorem{proposition}[theorem]{Proposition}
  \newtheorem{corollary}[theorem]{Corollary}
  \theoremstyle{definition}
  \newtheorem{definition}[theorem]{Definition}
  \newtheorem{assumption}[theorem]{Assumption}
  \theoremstyle{remark}
  \newtheorem{remark}[theorem]{Remark}
}{
  \@ifundefined{assumption}{\newtheorem{assumption}[theorem]{Assumption}}{}
}
\makeatother

\DeclareMathOperator*{\essinf}{ess\,inf}
\DeclareMathOperator*{\esssup}{ess\,sup}

\usepackage{float}      
\usepackage{algorithm}  
\usepackage{algpseudocode} 
\newcommand{\STATE}{\State}
\newcommand{\FOR}{\For}
\newcommand{\ENDFOR}{\EndFor}

\begin{document}

\title{Entropy-Controlled Flow Matching} 

\author{Chika Maduabuchi\inst{1}\orcidlink{0000-0001-9947-5855}}

\authorrunning{C. Maduabuchi}

\institute{William \& Mary, Williamsburg, VA, USA
}

\maketitle

\begin{abstract}
Modern vision generators transport a base distribution to data through time-indexed measures, implemented as deterministic flows (ODEs) or stochastic diffusions (SDEs). Despite strong empirical performance, standard flow-matching objectives do not directly control the information geometry of the trajectory, allowing low-entropy bottlenecks that can transiently deplete semantic modes. We propose \emph{Entropy-Controlled Flow Matching} (ECFM): a constrained variational principle over continuity-equation paths enforcing a global entropy-rate budget \(\frac{d}{dt}\mathcal H(\mu_t)\ge -\lambda\). ECFM is a convex optimization in Wasserstein space with a KKT/Pontryagin system, and admits a stochastic-control representation equivalent to a Schr\"odinger bridge with an explicit entropy multiplier. In the pure transport regime, ECFM recovers entropic OT geodesics and \(\Gamma\)-converges to classical OT as \(\lambda\to0\). We further obtain certificate-style mode-coverage and density-floor guarantees with Lipschitz stability, and construct near-optimal collapse counterexamples for unconstrained flow matching.
\end{abstract}

\begin{center}
  \refstepcounter{figure}\label{fig:teaser}
  \includegraphics[width=\linewidth]{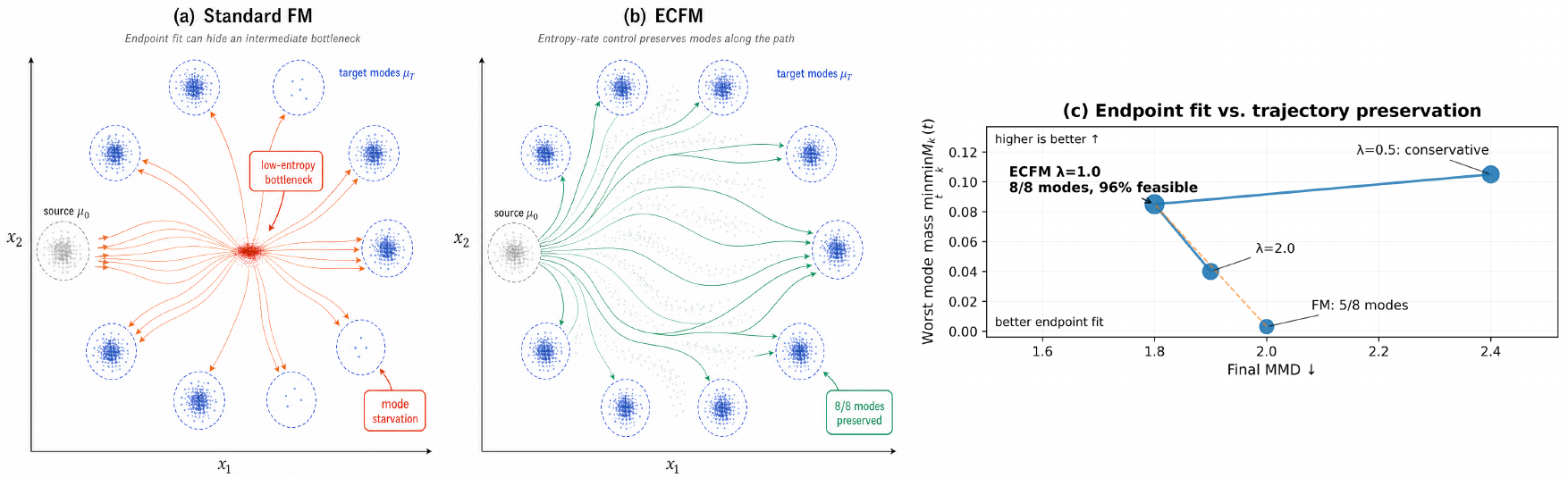}
  \parbox{\linewidth}{\small\textbf{Fig.~\thefigure. Trajectory bottleneck mechanism.}
A controlled toy 8-Gaussian experiment illustrates that endpoint fit can hide harmful intermediate transport geometry.
(a) Standard FM routes mass from the source distribution $\mu_0$ through a transient low-entropy bottleneck, starving several target modes despite endpoint matching.
(b) ECFM imposes an entropy-rate constraint, keeping trajectories distributed and preserving all target modes $\mu_T$.
(c) Quantitative diagnostics separate endpoint fidelity from trajectory preservation: ECFM with $\lambda=1.0$ preserves all $8/8$ modes and satisfies $96\%$ of entropy-feasible time bins while maintaining better final MMD than unconstrained FM; smaller $\lambda$ is more conservative.}
\end{center}

\section{Introduction}
\label{sec:intro}

Modern vision generators transport a simple base distribution to a complex data distribution through a time-indexed family of measures
\(\{\mu_t\}_{t\in[0,T]}\), implemented either as a deterministic flow (ODE/probability-flow style)
\cite{chen2018neuralode,grathwohl2018scalable,lipman2023flowmatching,liu2022rectifiedflow,pmlr-v202-song23a}
or a stochastic diffusion (SDE)
\cite{ho2020ddpm,song2021sde,dhariwal2021diffusion,NEURIPS2022_a98846e9,Rombach_2022_CVPR}.
While these approaches achieve impressive sample quality, they lack a \emph{constraint-level} mechanism that controls the
\emph{information geometry} of the entire trajectory---in particular, the evolution of entropy and the formation of low-entropy
``bottlenecks'' that can transiently deplete modes.
The teaser experiment in Fig.~\ref{fig:teaser} isolates this mechanism on a controlled 8-Gaussian task; the full protocol, diagnostic definitions, and $\lambda$ sweep are reported in \hyperref[app:toy_8gaussian]{App.~J.5} (Tables~\ref{tab:toy8gauss_main}--\ref{tab:toy8gauss_lambda}).

Such bottlenecks are a structural route to mode collapse (in the broad sense of intermediate-time mass depletion across semantic regions),
a phenomenon historically associated with adversarial training \cite{goodfellow2014gan} but also relevant for deterministic transport models
when the learned velocity field permits strong compression, thereby reducing distributional coverage (low recall) and biasing samples toward a subset of modes \cite{10.5555/3454287.3454640}.

\paragraph{Optimal transport as geometry, Schr\"odinger bridges as regularization.}
Optimal transport (OT) provides a canonical geometric notion of interpolation between distributions via Wasserstein geodesics~\cite{liu2021learninghighdimensionalwasserstein}, with the
dynamic Benamou--Brenier formulation~\cite{benamou2000bb} giving a minimum-kinetic-energy transport under the continuity equation
\cite{benamou2000bb,villani2009ot}.
However, OT geodesics can be brittle in high dimensions: they may route mass through narrow sets and are generally sensitive to perturbations, exhibiting fragility under sampling error and structural perturbations in practice \cite{Goldfeld2020GaussianSmoothedOT,pmlr-v139-lin21a}.
Entropic regularization replaces the hard OT problem by a strictly convex objective that yields smooth entropic interpolations
\cite{cuturi2013sinkhorn}.
At the path level, this regularization is captured by Schr\"odinger bridges (SB)~\cite{Leonard2014,Chen2015OptimalTO,MIKAMI20061815,Leonard2010FromTS,10.5555/3540261.3541615}: KL projections of a reference path measure onto endpoint constraints,
which admit a stochastic-control interpretation and connect deeply to OT through small-noise limits and large deviations
\cite{leonard2014survey,chen2016sb}.
Despite this conceptual alignment, the training objectives used in flow-based vision generators are typically posed as regression of velocities/scores
\cite{song2021sde,lipman2023flowmatching,liu2022rectifiedflow} and do \emph{not} explicitly impose an entropy-control principle that rules out
collapse channels at the level of trajectories.

\paragraph{Our idea: entropy-controlled flow matching.}
We introduce \emph{Entropy-Controlled Flow Matching (ECFM)}, a constrained variational principle for learning velocity fields in which the transport path
\(\mu_t\) must satisfy an \emph{entropy-rate budget}
\begin{equation}
\label{eq:intro_entropy_budget}
\frac{d}{dt}\mathcal{H}(\mu_t)\ \ge\ -\lambda,
\end{equation}
where \(\mathcal{H}\) is differential entropy (defined on the absolutely continuous class) and \(\lambda\ge 0\) is a user-specified budget.
In deterministic continuity-equation dynamics, the constraint has an immediate geometric meaning: since
\(\frac{d}{dt}\mathcal{H}(\mu_t)=\mathbb{E}_{\mu_t}[\nabla\!\cdot v(\cdot,t)]\),
\eqref{eq:intro_entropy_budget} directly limits average compressibility and prevents arbitrarily sharp concentration in finite time.
We combine this constraint with a flow-matching-style objective that regresses a learned velocity field toward a reference field,
yielding a principled mechanism that is \emph{trajectory-level}, \emph{model-agnostic} (ODE or SDE), and \emph{certifiable} via entropy-rate diagnostics.

\paragraph{Contributions.}
Our main results provide a complete mathematical foundation for ECFM and connect it to OT/SB theory:
\begin{itemize}
\item \textbf{A constrained transport formulation.}
We formalize ECFM as a constrained optimization in Wasserstein space over solutions of the continuity equation,
with precise assumptions, admissible classes, and a boxed primal problem (Sec.~\ref{sec:ecfm}).

\item \textbf{Optimality and duality (KKT/Pontryagin form).}
We derive first-order conditions with a nonnegative multiplier enforcing the entropy-rate constraint, and give a dual interpretation
as a time-dependent barrier against compressive flows (Sec.~\ref{sec:opt_duality}).

\item \textbf{Equivalence to Schr\"odinger bridges.}
We prove that ECFM admits a Schr\"odinger-bridge-equivalent formulation: the induced objective is strictly convex in the path law
(in the SB specialization), giving uniqueness at the level of trajectories (Sec.~\ref{sec:sb_equiv})
\cite{leonard2014survey,chen2016sb}.

\item \textbf{Entropic OT geodesics and the \(\lambda\to0\) limit.}
For the pure transport specialization, we show that ECFM trajectories coincide with entropic OT interpolations
(Sec.~\ref{sec:geodesics}) and establish \(\Gamma\)-convergence to classical OT as \(\lambda\downarrow 0\)
(Sec.~\ref{sec:gamma}) \cite{benamou2000bb,cuturi2013sinkhorn,villani2009ot}.

\item \textbf{Mode coverage and stability guarantees.}
We introduce a formal collapse notion in terms of intermediate-time mode mass depletion and prove quantitative mode-floor
and density-minimum bounds under the entropy budget, together with perturbation stability (Sec.~\ref{sec:mode_coverage}--\ref{sec:stability}).

\item \textbf{Necessity: failure without entropy control.}
We construct explicit unconstrained flow-matching sequences with near-optimal objective value but singular entropy bottlenecks and mode depletion,
showing that the entropy constraint is structurally necessary for certificate-level non-collapse claims (Sec.~\ref{sec:failure}).
\end{itemize}

\paragraph{Relevance to vision generators.}
ECFM complements the dominant paradigms for vision generation---diffusions/score models \cite{ho2020ddpm,song2021sde},
flow matching \cite{lipman2023flowmatching}, and rectified flows \cite{liu2022rectifiedflow}---by adding an explicit entropy-budget constraint that can be
checked (or enforced via dual updates) during training, yielding theorem-backed anti-collapse guarantees without relying on SOTA comparisons.

\paragraph{Practical enforcement and certification.}
Training can be implemented with a projected primal--dual augmented-Lagrangian update: entropy-rate violations on a time grid increase nonnegative multipliers, while feasible bins incur no dual pressure. We give the full algorithm, estimator choices, and adaptive $\lambda$-scheduling details in \hyperref[app:dual_updates]{App.~J.3} and Alg.~\ref{alg:ecfm_primal_dual}. Quantitative toy diagnostics and the $\lambda$ sweep are reported in \hyperref[app:toy_8gaussian]{App.~J.5}.

\paragraph{Relation to Schrödinger bridges and entropic OT (what is new).}
Although Sec.~\ref{sec:sb_equiv} shows that ECFM admits a Schr\"odinger/KL representation under a Brownian reference law,
this connection is \emph{an analytical lens rather than the definition of our method}.
Classical Schr\"odinger bridges (and entropic OT)~\cite{chen2016sb,10.5555/3540261.3541615,NEURIPS2023_c428adf7,JMLR:v24:23-0527,pmlr-v238-tong24a} are posed as KL/entropy-\emph{penalized} interpolations with a
fixed stochastic reference, whereas ECFM is defined as a \emph{flow-matching projection} onto an
\emph{entropy-rate feasible set} via the inequality constraint $\dot{\mathcal H}(\mu_t)\ge -\lambda$.
Crucially, ECFM keeps an arbitrary FM reference drift $u^\star$ and yields the closest entropy-feasible trajectory;
the KKT multiplier is \emph{adaptive} (zero when the FM trajectory is feasible, active otherwise), inducing an
\emph{endogenous} entropic level $\varepsilon(\lambda)$ rather than choosing $\varepsilon$ a priori.
The coincidence with entropic OT/SB geodesics occurs only in the pure-transport specialization $u^\star\equiv 0$
(Sec.~\ref{sec:geodesics}), serving as a sanity check and enabling the $\lambda\downarrow 0$ OT limit.

\paragraph{Roadmap.}
Sec.~\ref{sec:ecfm}--\ref{sec:opt_duality} presents the ECFM problem, duality, and optimality conditions.
Sec.~\ref{sec:sb_equiv} proves equivalence to Schr\"odinger bridges.
Sec.~\ref{sec:geodesics} and Sec.~\ref{sec:gamma} establish convergence to entropic OT and \(\Gamma\)-limits to classical OT.
Sec.~\ref{sec:mode_coverage}--\ref{sec:failure} provide mode-coverage, stability, and failure-without-constraint results.
Full proofs and measure-theoretic details are deferred to the \hyperref[appendix]{Appendix}.

\section{Theory}

\subsection{Setup and notation}
\label{sec:setup}

\paragraph{Ambient space.}
Fix a horizon \(T>0\) and dimension \(d\ge1\). We work on \(\mathbb{R}^d\) with its Borel \(\sigma\)-algebra.
Let \(\mathcal{P}_2(\mathbb{R}^d):=\{\mu\in\mathcal{P}(\mathbb{R}^d):\int\|x\|^2\,d\mu(x)<\infty\}\) and let \(W_2\) denote the quadratic Wasserstein distance,
\[
W_2^2(\mu,\nu):=\inf_{\pi\in\Pi(\mu,\nu)}\int_{\mathbb{R}^d\times\mathbb{R}^d}\|x-y\|^2\,d\pi(x,y).
\]
For measurable \(T:\mathbb{R}^d\to\mathbb{R}^d\), \(T_\#\mu(B):=\mu(T^{-1}(B))\), and \(\langle f,\mu\rangle:=\int f\,d\mu\).
Additional OT preliminaries are in App.~\hyperref[app:A]{A}.

\paragraph{Time-indexed measures and velocities.}
We consider Borel curves \(t\mapsto\mu_t\in\mathcal{P}_2(\mathbb{R}^d)\) paired with Borel velocity fields
\(v:\mathbb{R}^d\times[0,T]\to\mathbb{R}^d\), and use the shorthand
\[
\|v\|_{L^2(\mu)}^2:=\int_0^T\!\!\int_{\mathbb{R}^d}\|v(x,t)\|^2\,d\mu_t(x)\,dt.
\]

\paragraph{Continuity equation.}
A pair \((\mu,v)\) satisfies the continuity equation (CE) on \([0,T]\) if for all \(\varphi\in C_c^\infty(\mathbb{R}^d\times(0,T))\),
\begin{equation}
\label{eq:CE_weak_main}
\int_0^T\!\!\int_{\mathbb{R}^d}\Big(\partial_t\varphi(x,t)+\nabla\varphi(x,t)\cdot v(x,t)\Big)\,d\mu_t(x)\,dt = 0.
\end{equation}
When \(\mu_t=\rho_t\,dx\), this is \(\partial_t\rho_t+\nabla\!\cdot(\rho_t v_t)=0\) in the distributional sense.

\paragraph{Entropy and information functionals.}
For \(\mu=\rho\,dx\), define the differential entropy \(\mathcal{H}(\mu):=-\int\rho\log\rho\,dx\) (and \(\mathcal{H}(\mu)=-\infty\) if \(\mu\not\ll dx\)).
For \(\mu\ll\nu\), \(\mathrm{KL}(\mu\|\nu):=\int \log\!\big(\tfrac{d\mu}{d\nu}\big)\,d\mu\).
If \(\rho>0\) a.e.\ and \(\rho\) is weakly differentiable, the Fisher information is
\(\mathcal{I}(\mu):=\int \|\nabla\log\rho\|^2\,\rho\,dx\).
A key identity used throughout is the entropy-rate formula: if \((\mu,v)\) solves CE and \(\rho_t\) is regular enough, then
\begin{equation}
\label{eq:entropy_rate_identity_main}
\frac{d}{dt}\mathcal{H}(\mu_t)=\int_{\mathbb{R}^d}\nabla\!\cdot v(x,t)\,d\mu_t(x)
\qquad\text{for a.e. }t\in[0,T],
\end{equation}
with precise hypotheses in Apps.~\hyperref[app:A]{A}--\hyperref[app:B]{B}.

\paragraph{Admissible endpoints and paths.}
We consider endpoints \((\mu_0,\mu_T)\in\mathcal{P}_2(\mathbb{R}^d)\times\mathcal{P}_2(\mathbb{R}^d)\) and define
\begin{equation}
\label{eq:admissible_paths_main}
\mathcal{A}(\mu_0,\mu_T)
:=
\Big\{
(\mu,v):\ \mu_{|0}=\mu_0,\ \mu_{|T}=\mu_T,\ (\mu,v)\text{ satisfies \eqref{eq:CE_weak_main}},\ \|v\|_{L^2(\mu)}<\infty
\Big\}.
\end{equation}

\paragraph{Entropy-rate budget (entropy control).}
Fix \(\lambda\ge0\). We impose the a.e.\ constraint
\begin{equation}
\label{eq:entropy_budget_main}
\frac{d}{dt}\mathcal{H}(\mu_t)\ge -\lambda
\qquad\text{for a.e. }t\in[0,T],
\end{equation}
interpreted under the regularity conditions of App.~\hyperref[app:B]{B}.

\paragraph{Standing assumptions.}
We state assumptions in a model-agnostic form; a brief “how to read these in diffusion/FM/latent settings” mapping is given in App.~\hyperref[app:J]{J}. Unless stated otherwise:
\begin{enumerate}
\item[(S1)] \textbf{Endpoints.} \(\mu_0=\rho_0dx,\ \mu_T=\rho_Tdx\) with \(\mu_0,\mu_T\in\mathcal{P}_2(\mathbb{R}^d)\) and \(\mathcal{H}(\mu_0),\mathcal{H}(\mu_T)>-\infty\).
\item[(S2)] \textbf{Feasibility.} \(\mathcal{A}(\mu_0,\mu_T)\) contains at least one path satisfying \eqref{eq:entropy_budget_main}.
\item[(S3)] \textbf{Reference field.} A measurable drift \(u^\star:\mathbb{R}^d\times[0,T]\to\mathbb{R}^d\) is given such that
\(\int_0^T\!\!\int \|u^\star(x,t)\|^2\,d\mu_t(x)\,dt<\infty\) for all considered admissible \((\mu,v)\), and \(u^\star(\cdot,t)\) is locally Lipschitz for a.e.\ \(t\).
\end{enumerate}
Full technical conditions and variants are in Apps.~\hyperref[app:A]{A}--\hyperref[app:B]{B}.

\subsection{Entropy-Controlled Flow Matching (ECFM)}
\label{sec:ecfm}

\paragraph{Reference drift.}
Let \(u^\star:\mathbb{R}^d\times[0,T]\to\mathbb{R}^d\) be a prescribed reference drift (teacher / closed-form interpolation / model-implied target).
ECFM learns a velocity field \(v\) and an induced path \((\mu_t)_{t\in[0,T]}\) that transports \(\mu_0\) to \(\mu_T\), stays close to \(u^\star\) in mean square under \(\mu_t\), and respects a global entropy-rate budget.

\paragraph{Primal formulation.}
Define the entropy-feasible admissible set
\[
\mathcal{A}_\lambda(\mu_0,\mu_T)
:=
\Big\{(\mu,v)\in\mathcal{A}(\mu_0,\mu_T): \ \dot{\mathcal H}(\mu_t)\ge -\lambda\ \text{for a.e. }t\in[0,T]\Big\},
\]
where \(\dot{\mathcal H}\) is interpreted via \eqref{eq:entropy_rate_identity_main} under Apps.~\hyperref[app:A]{A}--\hyperref[app:B]{B}.
The ECFM problem is
\begin{equation}
\label{eq:ECFM_primal}
\boxed{
\mathrm{ECFM}_\lambda(\mu_0,\mu_T;u^\star):=
\min_{(\mu,v)\in\mathcal{A}_\lambda(\mu_0,\mu_T)}
\ \frac12\int_0^T\!\!\int_{\mathbb{R}^d}\|v(x,t)-u^\star(x,t)\|^2\,d\mu_t(x)\,dt .}
\end{equation}
We denote by \((\mu^\lambda,v^\lambda)\) an optimal pair when it exists.

\paragraph{Feasibility and choosing \(\lambda\) (practical rule).}
The budget \(\lambda\) controls the allowed entropy dissipation along the path.
Trivially, \(\lambda=+\infty\) recovers unconstrained FM and is always feasible.
For finite \(\lambda\), feasibility can be checked (and \(\lambda\) selected) using the entropy-rate diagnostics in
App.~\hyperref[app:J2]{J.2}: on a discrete time grid \(\{t_n\}\), estimate \(\widehat{\dot{\mathcal H}}_n\approx \mathbb{E}_{\mu_{t_n}}[\nabla\!\cdot v(\cdot,t_n)]\)
(for deterministic flows) or the Fokker--Planck/Fisher form (for diffusions), then set an \emph{empirical effective budget}
\[
\widehat{\lambda}_{\mathrm{eff}}^{\mathrm{LCB}}
\ :=\ \max_n \big(-\widehat{\dot{\mathcal H}}^{\mathrm{LCB}}_n\big),
\]
where \(\widehat{\dot{\mathcal H}}^{\mathrm{LCB}}_n\) is a lower confidence bound (Sec.~\ref{sec:intro}, App.~\hyperref[app:J4]{J.4}).
Choosing \(\lambda\ge \widehat{\lambda}_{\mathrm{eff}}^{\mathrm{LCB}}\) makes feasibility a measurable, certificate-compatible condition.

\paragraph{Flux form (Benamou--Brenier variables).}
When \(\mu_t=\rho_tdx\), define momentum \(m_t:=\rho_t v_t\). Then CE is \(\partial_t\rho_t+\nabla\!\cdot m_t=0\) and
\[
\frac12\int_0^T\!\!\int \|v-u_t^\star\|^2\,d\mu_tdt
=
\frac12\int_0^T\!\!\int \frac{\|m_t-\rho_t u_t^\star\|^2}{\rho_t}\,dx\,dt,
\]
(with the convention \(+\infty\) if \(\rho_t=0\) and \(m_t\neq0\)).
We use this representation for convexity/coercivity arguments (App.~\hyperref[app:C]{C}).

\paragraph{Entropy-budget interpretation.}
By \eqref{eq:entropy_rate_identity_main}, the constraint \(\dot{\mathcal H}(\mu_t)\ge-\lambda\) is equivalently the divergence budget
\begin{equation}
\label{eq:divergence_budget}
\int_{\mathbb{R}^d}\nabla\!\cdot v(x,t)\,d\mu_t(x)\ \ge\ -\lambda
\qquad\text{for a.e. }t\in[0,T].
\end{equation}
Thus ECFM forbids paths with transient low-entropy bottlenecks (large negative divergence spikes), the collapse mechanism made explicit in
App.~\hyperref[app:I]{I}.

\paragraph{Special cases.}
\begin{enumerate}
\item \textbf{Unconstrained FM:} \(\lambda=+\infty\) removes \eqref{eq:divergence_budget} and recovers classical flow matching.
\item \textbf{Pure transport:} \(u^\star\equiv0\) yields minimum-energy transport under an entropy-rate constraint, leading to entropic OT geodesics (App.~\hyperref[app:E]{E}).
\item \textbf{Diffusions:} for Fokker--Planck dynamics, \(\dot{\mathcal H}\) includes a Fisher-information term, enlarging feasibility; see Sec.~\hyperref[app:J]{J} and App.~\hyperref[app:B]{B}.
\end{enumerate}
\subsection{Optimality and duality}
\label{sec:opt_duality}

\paragraph{KKT structure (primal--dual form).}
ECFM admits a primal--dual characterization with a scalar potential \(\phi\) (for CE) and a nonnegative multiplier \(\eta(t)\) (for the entropy-rate constraint).
We state the core KKT conditions; full derivations and the adjoint equation are in Apps.~\hyperref[app:C4]{C.4}--\hyperref[app:C6]{C.6}.

\begin{theorem}[KKT optimality system (core conditions)]
\label{thm:KKT_main}
Assume \((\mu^\lambda,v^\lambda)\) solves \eqref{eq:ECFM_primal} and \(\mu_t^\lambda=\rho_t^\lambda dx\) satisfies the regularity conditions in Apps.~\hyperref[app:A]{A}--\hyperref[app:C]{C}.
Then there exist \(\phi^\lambda\) and \(\eta^\lambda(t)\ge0\) such that:
\begin{enumerate}
\item \textbf{(Feasibility).} \((\mu^\lambda,v^\lambda)\in\mathcal{A}_\lambda(\mu_0,\mu_T)\), i.e.\ CE \eqref{eq:CE_weak_main} holds with endpoints \(\mu_0,\mu_T\) and \(\dot{\mathcal H}(\mu_t^\lambda)\ge-\lambda\) a.e.
\item \textbf{(Stationarity).} For a.e.\ \(t\) and \(\mu_t^\lambda\)-a.e.\ \(x\),
\begin{equation}
\label{eq:KKT_stationarity_v}
v^\lambda(x,t)=u^\star(x,t)-\nabla\phi^\lambda(x,t)-\eta^\lambda(t)\,\nabla\log\rho_t^\lambda(x).
\end{equation}
\item \textbf{(Complementarity).} For a.e.\ \(t\),
\begin{equation}
\label{eq:KKT_complementarity}
\eta^\lambda(t)\ge0,\qquad \dot{\mathcal H}(\mu_t^\lambda)+\lambda\ge0,\qquad
\eta^\lambda(t)\big(\dot{\mathcal H}(\mu_t^\lambda)+\lambda\big)=0.
\end{equation}
\end{enumerate}
\end{theorem}

\paragraph{Dual viewpoint.}
In flux variables \(m_t=\rho_t v_t\), the problem is a convex program in \((\rho,m)\) with a convex quadratic perspective integrand and linear CE constraint; the associated dual is a concave maximization over \((\phi,\eta)\) whose maximizers recover \(v^\lambda\) via \eqref{eq:KKT_stationarity_v} (App.~\hyperref[app:C5]{C.5}).

\paragraph{Interpretation.}
\(\eta^\lambda(t)\) acts as an \emph{adaptive anti-collapse pressure}: it activates only when the trajectory attempts to dissipate entropy faster than \(\lambda\), injecting a score-like correction \(\nabla\log\rho_t^\lambda\) in \eqref{eq:KKT_stationarity_v} and thereby ruling out low-entropy bottlenecks (Apps.~\hyperref[app:G]{G}, \hyperref[app:I]{I}).
\subsection{Existence and uniqueness}
\label{sec:exist_unique}

\paragraph{Existence.}
\begin{theorem}[Existence of an ECFM minimizer]
\label{thm:existence_main}
Assume (S1)--(S3) from Sec.~\ref{sec:setup}. Then the primal problem \eqref{eq:ECFM_primal}
admits at least one minimizer \((\mu^\lambda,v^\lambda)\in\mathcal{A}_\lambda(\mu_0,\mu_T)\).
Moreover, one may choose a minimizer with \(\mu_t^\lambda\ll dx\) for a.e.\ \(t\) and finite kinetic energy
\(\int_0^T \|v_t^\lambda\|_{L^2(\mu_t^\lambda)}^2\,dt<\infty\). Full proof and compactness topology are in
App.~\hyperref[app:C]{C}.
\end{theorem}

\paragraph{Uniqueness (strict convex regimes).}
Uniqueness holds in settings where the induced objective is strictly convex in the (path) law, notably:
(i) the Schr\"odinger/KL representation (Sec.~\ref{sec:sb_equiv}) and (ii) the pure transport specialization \(u^\star\equiv0\) (App.~\hyperref[app:E]{E}).

\begin{theorem}[Uniqueness in strict convex regimes]
\label{thm:uniqueness_main}
In either of the following regimes, the ECFM solution is unique:
\begin{enumerate}
\item \textbf{Schr\"odinger/KL form:} ECFM is represented as a KL minimization over path measures (Sec.~\ref{sec:sb_equiv});
\item \textbf{Pure transport:} \(u^\star\equiv0\), yielding entropic OT geodesics (App.~\hyperref[app:E]{E}).
\end{enumerate}
In both cases, the path \(t\mapsto \mu_t^\lambda\) is unique and \(v^\lambda(\cdot,t)\) is unique \(\mu_t^\lambda\)-a.e.\ for a.e.\ \(t\).
\end{theorem}

\begin{proof}[Proof sketch]
Both regimes reduce to a strictly convex objective in the path law (KL/entropic OT), hence uniqueness follows by strict convexity.
Proofs are in App.~\hyperref[app:D]{D} (Schr\"odinger/KL) and App.~\hyperref[app:E]{E.4} (geodesic case).
\end{proof}

\paragraph{Scope and what is certified.}
Theorem~\ref{thm:existence_main} guarantees existence under (S1)--(S3).
Uniqueness is claimed only in the strict convex regimes of Theorem~\ref{thm:uniqueness_main} (SB/KL form or pure transport);
in the general ECFM setting with arbitrary \(u^\star\), multiple minimizers may exist.
Importantly, \emph{certification is trajectory-level rather than uniqueness-level}:
every ECFM minimizer is feasible for the entropy-rate constraint by construction, and therefore all mode-floor and stability
guarantees apply to \emph{any} minimizer satisfying the stated regularity assumptions.
In practice, the SB specialization (Sec.~\ref{sec:sb_equiv}) recovers uniqueness at the level of the induced path law.
\subsection{Equivalence to Schr\"odinger bridges}
\label{sec:sb_equiv}

\paragraph{Schr\"odinger bridge (dynamic form).}
Let \(\mathbf{R}\) be a Brownian reference path law on \(C([0,T];\mathbb{R}^d)\) with diffusivity \(\varepsilon>0\) (App.~\hyperref[app:D]{D}).
The (dynamic) Schr\"odinger bridge between \(\mu_0\) and \(\mu_T\) is the KL projection
\begin{equation}
\label{eq:SB_dynamic_main}
\min_{\mathbf{P}:\ \mathbf{P}_0=\mu_0,\ \mathbf{P}_T=\mu_T} \ \mathrm{KL}\!\big(\mathbf{P}\,\|\,\mathbf{R}\big),
\end{equation}
which is equivalent to the static endpoint formulation (App.~\hyperref[app:D1]{D.1}--\hyperref[app:D2]{D.2}).

\paragraph{Control representation and current velocity.}
Any \(\mathbf{P}\ll\mathbf{R}\) admits a controlled diffusion representation
\(dX_t=b_t(X_t)\,dt+\sqrt{2\varepsilon}\,dW_t\) whose marginals \(\mu_t=\rho_tdx\) satisfy the Fokker--Planck equation.
Define the \emph{current velocity}
\begin{equation}
\label{eq:current_velocity_main}
v_t:=b_t-\varepsilon\nabla\log\rho_t,
\end{equation}
which converts Fokker--Planck to the continuity equation \(\partial_t\rho_t+\nabla\!\cdot(\rho_t v_t)=0\).
Moreover, the entropy-rate identity becomes
\begin{equation}
\label{eq:entropy_rate_FP_main}
\frac{d}{dt}\mathcal{H}(\mu_t)
=
\mathbb{E}_{\mu_t}[\nabla\!\cdot b_t]\ +\ \varepsilon\,\mathcal{I}(\mu_t)
=
\mathbb{E}_{\mu_t}[\nabla\!\cdot v_t],
\end{equation}
so the constraint \(\dot{\mathcal H}(\mu_t)\ge-\lambda\) directly restricts admissible KL-controls.

\begin{theorem}[Equivalence: ECFM as an entropy-controlled Schr\"odinger bridge]
\label{thm:SB_equiv_main}
Fix \(\varepsilon>0\) and \(\mathbf{R}\) as above. Under (S1)--(S3) and the regularity assumptions of App.~\hyperref[app:D]{D},
ECFM admits a Schr\"odinger-bridge representation in the following sense:
\begin{enumerate}
\item \textbf{(SB \(\Rightarrow\) ECFM).} If \(\mathbf{P}^\star\) solves \eqref{eq:SB_dynamic_main} with drift \(b_t\) and marginals \(\mu_t\),
then the induced current velocity \(v_t\) in \eqref{eq:current_velocity_main} yields a CE path \((\mu,v)\) that minimizes an ECFM objective
\eqref{eq:ECFM_primal} for an explicit reference field \(u^\star\) determined by Schr\"odinger potentials, with multiplier \(\eta(t)\) enforcing the entropy budget.
\item \textbf{(ECFM \(\Rightarrow\) SB).} Conversely, any ECFM minimizer \((\mu^\lambda,v^\lambda)\) induces a controlled drift
\(b^\lambda:=v^\lambda+\varepsilon\nabla\log\rho^\lambda\) and hence a path law \(\mathbf{P}^\lambda\ll\mathbf{R}\) whose KL objective equals the ECFM value up to a constant.
\item \textbf{(Uniqueness in path law).} The induced path law \(\mathbf{P}^\lambda\) is unique whenever the admissible set is nonempty.
\end{enumerate}
\end{theorem}

\begin{proof}[Proof sketch]
Combine the standard SB $\leftrightarrow$ stochastic control equivalence with the current-velocity transformation \eqref{eq:current_velocity_main}.
Completing the square identifies the quadratic control cost with the ECFM integrand, while \eqref{eq:entropy_rate_FP_main} links entropy-rate feasibility to admissible controls.
Full proof and the explicit identification of \(u^\star\) are in App.~\hyperref[app:D4]{D.4}.
\end{proof}

\paragraph{Implication.}
ECFM is a Schr\"odinger bridge with an explicit entropy-dissipation budget: it inherits strict convexity (hence uniqueness in path law) in the SB specialization,
and excludes the low-entropy bottleneck collapse channels of unconstrained deterministic flows (Apps.~\hyperref[app:G]{G}, \hyperref[app:I]{I}).
\subsection{Convergence to entropic OT geodesics}
\label{sec:geodesics}

\paragraph{Pure transport specialization.}
For \(u^\star\equiv0\), ECFM reduces to minimum kinetic energy under the entropy-rate budget:
\begin{equation}
\label{eq:ECFM_geodesic_primal}
\min_{(\mu,v)\in\mathcal{A}_\lambda(\mu_0,\mu_T)}
\ \frac12\int_0^T\!\!\int \|v(x,t)\|^2\,d\mu_t(x)\,dt.
\end{equation}
This defines an \emph{entropy-controlled geodesic}. Via Sec.~\ref{sec:sb_equiv}, this regime coincides with entropic OT / Schr\"odinger bridges.

\paragraph{Entropic OT.}
Fix \(\varepsilon>0\). Entropic OT admits a static regularized formulation and an equivalent dynamic (Schr\"odinger) KL formulation; we use the latter in the proofs (Apps.~\hyperref[app:D]{D}--\hyperref[app:E]{E}).

\begin{theorem}[Entropy-controlled geodesics are entropic OT geodesics]
\label{thm:entropic_geodesic_main}
Assume (S1)--(S2) and \(u^\star\equiv0\). Let \((\mu^\lambda,v^\lambda)\) minimize \eqref{eq:ECFM_geodesic_primal}. Then:
\begin{enumerate}
\item \textbf{(Identification).} There exists \(\varepsilon=\varepsilon(\lambda)>0\) and a Schr\"odinger bridge \(\mathbf{P}^\varepsilon\) between \(\mu_0\) and \(\mu_T\) such that \(\mu_t^\lambda=\mathbf{P}_t^\varepsilon\) for a.e.\ \(t\). Equivalently, \((\mu^\lambda,v^\lambda)\) is the entropic OT geodesic (current-velocity form) with regularization \(\varepsilon(\lambda)\).
\item \textbf{(Uniqueness and stability).} The path \(t\mapsto\mu_t^\lambda\) is unique, and there exists \(C=C(\lambda,T,\mu_0,\mu_T)\) such that for any endpoints \((\tilde\mu_0,\tilde\mu_T)\) and corresponding minimizer \((\tilde\mu^\lambda,\tilde v^\lambda)\),
\begin{equation}
\label{eq:geodesic_stability_main}
\sup_{t\in[0,T]} W_2(\mu_t^\lambda,\tilde\mu_t^\lambda)
\ \le\ C\big(W_2(\mu_0,\tilde\mu_0)+W_2(\mu_T,\tilde\mu_T)\big).
\end{equation}
\end{enumerate}
\end{theorem}

\begin{proof}[Proof sketch]
Use the Schr\"odinger representation (Sec.~\ref{sec:sb_equiv}) to identify \eqref{eq:ECFM_geodesic_primal} with a strictly convex entropic OT/KL functional.
The entropy-rate constraint induces an effective regularization level \(\varepsilon(\lambda)\) through the KKT multiplier (App.~\hyperref[app:E]{E}).
Uniqueness follows from strict convexity in the path law, and stability follows from strong convexity and stability of KL projections. Full proof is in App.~\hyperref[app:E]{E}.
\end{proof}

\paragraph{Role of \(\lambda\).}
Smaller \(\lambda\) permits less entropy dissipation and corresponds to smaller \(\varepsilon(\lambda)\), approaching classical OT as \(\lambda\to0\) (Sec.~\ref{sec:gamma}).
\subsection{\(\Gamma\)-convergence as \texorpdfstring{\(\lambda\to0\)}{lambda->0}}
\label{sec:gamma}

\paragraph{Limit regime.}
We formalize the small-budget limit \(\lambda\downarrow0\) by \(\Gamma\)-convergence: entropy-controlled geodesics converge to classical OT (Benamou--Brenier) geodesics.

\paragraph{Functional setting.}
Work in the trajectory--velocity space \(\mathsf{X}:=\mathcal{A}(\mu_0,\mu_T)\) endowed with the standard weak topology for CE trajectories (App.~\hyperref[app:F1]{F.1}).
For \(u^\star\equiv0\), define
\begin{equation}
\label{eq:Gamma_F_lambda}
\mathcal{F}_\lambda(\mu,v)
:=
\begin{cases}
\displaystyle
\frac12\int_0^T\!\!\int \|v(x,t)\|^2\,d\mu_t(x)\,dt, & (\mu,v)\in\mathcal{A}_\lambda(\mu_0,\mu_T),\\[0.6ex]
+\infty, & \text{otherwise},
\end{cases}
\end{equation}
and let \(\mathcal{F}_0\) be the classical Benamou--Brenier OT action (same expression on \(\mathcal{A}(\mu_0,\mu_T)\); \(+\infty\) otherwise).

\begin{theorem}[\(\Gamma\)-convergence to classical OT]
\label{thm:Gamma_main}
Assume (S1)--(S2) and \(u^\star\equiv0\). Then \(\mathcal{F}_\lambda\) \(\Gamma\)-converges to \(\mathcal{F}_0\) on \(\mathsf{X}\) as \(\lambda\downarrow0\).
Consequently, \(\inf \mathcal{F}_\lambda \to \inf \mathcal{F}_0\), and any accumulation point of minimizers \((\mu^\lambda,v^\lambda)\) is a Benamou--Brenier OT minimizer.
If the OT geodesic is unique, then \(\mu_t^{\lambda}\rightharpoonup \mu_t^{0}\) for all \(t\in[0,T]\).
\end{theorem}

\begin{proof}[Proof sketch]
The liminf follows from weak lower semicontinuity of the kinetic action in flux form and closure of the CE constraint under the \(\mathsf{X}\)-topology.
A recovery sequence is obtained by smoothing an OT geodesic (strictly positive mollification) so the entropy-rate constraint is feasible for \(\lambda_n\downarrow0\) while the action changes by \(o(1)\).
Equicoercivity yields compactness and the fundamental theorem of \(\Gamma\)-convergence gives convergence of minimizers.
Full proof is in App.~\hyperref[app:F]{F}.
\end{proof}

\paragraph{Consequence.}
ECFM provides a regularized geodesic family that recovers classical OT as \(\lambda\to0\), linking entropy control to Wasserstein geometry.
\subsection{Mode coverage and anti-collapse guarantees}
\label{sec:mode_coverage}

\paragraph{Modes and collapse.}
Let \(\{A_k\}_{k=1}^K\) be measurable sets representing semantic modes (e.g., regions in an embedding space) and define \(M_k(t):=\mu_t(A_k)\).
Assume nontrivial endpoint mass:
\begin{equation}
\label{ass:endpoint_mode_mass_main}
\mu_0(A_k)\ge \alpha_k,\qquad \mu_T(A_k)\ge \alpha_k,\qquad k=1,\dots,K,
\end{equation}
for some \(\alpha_k>0\). We call \((\mu_t)\) \emph{mode-collapsing} if \(\inf_{t}\min_k \mu_t(A_k)=0\).

\paragraph{Choosing modes \(\{A_k\}\) in vision practice.}
Theorem~\ref{thm:mode_coverage_main} is stated for arbitrary measurable mode sets.
In vision applications, a concrete and reproducible choice is to define modes in a fixed representation space:
pick an embedding \(f(\cdot)\) (e.g., the model’s own latent space, or a frozen pretrained encoder),
construct regions \(\{C_k\}_{k=1}^K\subset\mathbb{R}^p\) (e.g., class-conditional regions, k-means clusters, or attribute partitions),
and set \(A_k := f^{-1}(C_k)\).
Then \(\mu_t(A_k)\) measures the probability mass retained in each semantic/cluster region along the generative trajectory.
When the generator is trained in latent space, \(f\) is naturally the latent map and \(A_k\) should be defined in that same space.

\paragraph{Entropy barrier.}
Entropy-rate control rules out transient concentration: under \(\dot{\mathcal H}\ge-\lambda\), the path cannot pass through arbitrarily low-entropy bottlenecks in finite time (App.~\hyperref[app:G2]{G.2}), which yields quantitative mode floors.

\begin{theorem}[Mode coverage under entropy-rate control]
\label{thm:mode_coverage_main}
Assume (S1)--(S2) and let \((\mu^\lambda,v^\lambda)\) solve \eqref{eq:ECFM_primal}. Under the regularity conditions of App.~\hyperref[app:G]{G},
there exist constants \(\beta_k=\beta_k(\alpha_k,\lambda,T,\mu_0,\mu_T)>0\) such that
\begin{equation}
\label{eq:mode_floor_main}
\mu_t^\lambda(A_k)\ \ge\ \beta_k
\qquad\text{for all }t\in[0,T]\text{ and all }k=1,\dots,K.
\end{equation}
\end{theorem}

\begin{proof}[Proof sketch]
Combine the entropy barrier (App.~\hyperref[app:G2]{G.2}) with a contradiction/compactness argument that transfers forbidden concentration into a uniform lower bound on set-masses \(M_k(t)\); see App.~\hyperref[app:G3]{G.3} for explicit constants.
\end{proof}

\paragraph{Density floors on mode cores.}
If each mode contains a compact core \(K_k\Subset A_k\) with endpoint density lower bounds \(\rho_0,\rho_T\ge c_k>0\) a.e.\ on \(K_k\),
then entropy control yields a uniform interior density floor.

\begin{theorem}[Density floors on mode cores]
\label{thm:density_floor_main}
Assume (S1)--(S2) and the core condition above. Under App.~\hyperref[app:G4]{G.4}, there exists \(\underline\rho_k>0\) such that
\begin{equation}
\label{eq:density_floor_main}
\rho_t^\lambda(x)\ \ge\ \underline\rho_k
\qquad\text{for a.e. }x\in K_k\ \text{and all }t\in[0,T].
\end{equation}
Moreover, these floors are perturbation-robust via Sec.~\ref{sec:stability}.
\end{theorem}

\begin{proof}[Proof sketch]
App.~\hyperref[app:G4]{G.4} combines entropy-rate control with localization/regularity (Harnack-type) arguments to convert mode-mass floors into pointwise density minima on \(K_k\).
\end{proof}

\paragraph{Contrast.}
Without the entropy-rate constraint, unconstrained FM admits near-optimal collapsing paths (App.~\hyperref[app:I]{I}); see Sec.~\ref{sec:failure}.
\subsection{Stability under perturbations}
\label{sec:stability}

\paragraph{Perturbation model.}
Let \((\mu,v)\) be an ECFM solution (or any feasible CE trajectory), and consider perturbed data consisting of:
(i) perturbed endpoints \((\tilde\mu_0,\tilde\mu_T)\), and (ii) a perturbed reference field \(\tilde u^\star\).
Let \((\tilde\mu,\tilde v)\) denote the corresponding ECFM solution under the same budget \(\lambda\), when it exists.
We quantify the size of perturbations by
\begin{equation}
\label{eq:perturbation_size_def}
\Delta_0 := W_2(\mu_0,\tilde\mu_0),\qquad
\Delta_T := W_2(\mu_T,\tilde\mu_T),\qquad
\Delta_u := \Big(\int_0^T\!\!\int \|u^\star-\tilde u^\star\|^2\,d\mu_t\,dt\Big)^{1/2}.
\end{equation}
We also allow additive noise in the learned velocity, modeled as
\begin{equation}
\label{eq:noisy_velocity_model}
\tilde v(x,t)=v(x,t)+\xi(x,t),
\qquad
\int_0^T\!\!\int \|\xi(x,t)\|^2\,d\mu_t(x)\,dt < \infty.
\end{equation}

\paragraph{Stability of CE trajectories.}
The following theorem summarizes the Lipschitz-type stability of entropy-controlled trajectories with respect to
endpoints and velocity perturbations; it consolidates Apps.~\hyperref[app:H1]{H.1}--\hyperref[app:H4]{H.4} into a main-text statement.

\begin{theorem}[Unified perturbation stability]
\label{thm:stability_main}
Assume (S1)--(S3) and the regularity conditions of App.~\hyperref[app:H]{H}.
Let \((\mu,v)\) and \((\tilde\mu,\tilde v)\) be ECFM minimizers for data \((\mu_0,\mu_T,u^\star)\) and
\((\tilde\mu_0,\tilde\mu_T,\tilde u^\star)\), respectively, under the same budget \(\lambda\).
Then there exists a constant \(C=C(\lambda,T,\mathrm{Reg})<\infty\) depending only on the stated regularity bounds such that
\begin{equation}
\label{eq:stability_W2_main}
\sup_{t\in[0,T]} W_2(\mu_t,\tilde\mu_t)
\ \le\ C\Big(\Delta_0+\Delta_T+\Delta_u+\|\xi\|_{L^2(\mu)}\Big),
\end{equation}
where \(\|\xi\|_{L^2(\mu)}^2:=\int_0^T\!\!\int \|\xi\|^2\,d\mu_t\,dt\).
Moreover, mode masses and core density floors are stable: for any measurable \(A\subset\mathbb{R}^d\) and any compact core \(K\),
\begin{equation}
\label{eq:stability_modes_main}
\big|\mu_t(A)-\tilde\mu_t(A)\big|
\ \le\ C_A\, W_2(\mu_t,\tilde\mu_t),
\end{equation}
and if \(\rho_t(x)\ge \underline\rho\) a.e.\ on \(K\) for all \(t\), then
\begin{equation}
\label{eq:stability_density_main}
\tilde\rho_t(x)\ \ge\ \underline\rho - C_\rho\Big(\Delta_0+\Delta_T+\Delta_u+\|\xi\|_{L^2(\mu)}\Big)
\quad\text{for a.e.\ }x\in K,\ \forall t\in[0,T],
\end{equation}
for constants \(C_A,C_\rho\) depending on \(A,K\) and the same regularity envelope.
\end{theorem}

\begin{proof}[Proof sketch]
App.~\hyperref[app:H1]{H.1} provides Lipschitz stability for CE trajectories under perturbations of the velocity field in \(L^2(\mu)\),
while Apps.~\hyperref[app:H2]{H.2}--\hyperref[app:H3]{H.3} control the propagation of endpoint/initialization perturbations through the flow map.
Combining these yields \eqref{eq:stability_W2_main}. The mode-mass bound \eqref{eq:stability_modes_main} follows from
regularity of indicator test sets under Wasserstein perturbations (approximating \(\mathbf{1}_A\) by Lipschitz functions),
and the density-floor stability \eqref{eq:stability_density_main} follows by composing the core lower bound of
Theorem~\ref{thm:density_floor_main} with the perturbation bound \eqref{eq:stability_W2_main}.
Full proofs and explicit constants are in App.~\hyperref[app:H]{H}.
\end{proof}

\paragraph{Interpretation.}
ECFM provides not only non-collapse guarantees, but also \emph{robustness margins}:
mode floors and density minima degrade at most linearly with deployment shifts in endpoints, target field, or learned velocity.

\subsection{Failure without entropy control}
\label{sec:failure}

\paragraph{Why a constraint is necessary.}
The previous sections establish that the entropy-rate budget prevents low-entropy bottlenecks and yields
mode-mass/density floors. We now show that \emph{without} entropy control, classical flow matching admits
near-optimal trajectories that transiently concentrate mass, producing intermediate-time mode depletion.

\paragraph{Unconstrained FM objective.}
Consider the unconstrained counterpart of \eqref{eq:ECFM_primal} obtained by removing \eqref{eq:entropy_budget_main}:
\begin{equation}
\label{eq:FM_unconstrained}
\mathrm{FM}(\mu_0,\mu_T;u^\star):=
\min_{(\mu,v)\in\mathcal{A}(\mu_0,\mu_T)}
\ \frac12\int_0^T\!\!\int \|v(x,t)-u^\star(x,t)\|^2\,d\mu_t(x)\,dt.
\end{equation}

\begin{theorem}[Singular bottlenecks and explicit collapse in unconstrained FM]
\label{thm:failure_main}
There exist dimensions \(d\ge 1\), endpoint measures \((\mu_0,\mu_T)\) satisfying (S1), a reference field \(u^\star\),
and a sequence of feasible solutions \(\{(\mu^n,v^n)\}_{n\ge1}\subset\mathcal{A}(\mu_0,\mu_T)\) such that:
\begin{enumerate}
\item \textbf{(Near-optimality).} \(\mathcal{J}_\infty(\mu^n,v^n)\downarrow \inf \mathrm{FM}(\mu_0,\mu_T;u^\star)\), where \(\mathcal{J}_\infty\) denotes
the objective in \eqref{eq:ECFM_primal} without the entropy constraint.
\item \textbf{(Entropy collapse).} There exist times \(t_n\in(0,T)\) with
\begin{equation}
\label{eq:failure_entropy}
\mathcal{H}(\mu_{t_n}^n)\to -\infty,
\qquad\text{equivalently}\qquad
\int \nabla\!\cdot v^n(\cdot,t_n)\,d\mu_{t_n}^n \to -\infty.
\end{equation}
\item \textbf{(Mode depletion).} For a suitable two-mode partition \(\{A_1,A_2\}\) with \(\mu_0(A_k),\mu_T(A_k)>0\),
the same sequence satisfies
\begin{equation}
\label{eq:failure_modes}
\inf_{t\in[0,T]}\min\{\mu_t^n(A_1),\mu_t^n(A_2)\}\ \to\ 0.
\end{equation}
\end{enumerate}
Consequently, unconstrained FM does not admit a uniform mode-coverage guarantee of the form \eqref{eq:mode_floor_main}
without additional assumptions beyond (S1)--(S3).
\end{theorem}

\begin{proof}[Proof sketch]
App.~\hyperref[app:I]{I} constructs an explicit ``squeezing'' family of CE trajectories that routes mass through a shrinking corridor
(or collapsing mixture component), keeping the FM mismatch cost small while making the Jacobian determinant (hence divergence)
very negative over a short time window. This forces \(\mathcal{H}(\mu_t)\) to drop to \(-\infty\) as in \eqref{eq:failure_entropy},
and simultaneously starves one mode as in \eqref{eq:failure_modes}. Full construction and estimates are given in
App.~\hyperref[app:I]{I}.
\end{proof}

\paragraph{Direct comparison with ECFM.}
Theorem~\ref{thm:mode_coverage_main} shows that ECFM forbids the collapse mechanism in Theorem~\ref{thm:failure_main}
by enforcing \(\dot{\mathcal H}(\mu_t)\ge-\lambda\), which rules out the divergence spikes responsible for \eqref{eq:failure_entropy}.
This separation is \emph{structural} (constraint-level), not empirical.

\subsection{Connections to vision generative models}
\label{sec:vision_connection}

\paragraph{Unified transport viewpoint.}
Vision generators induce time-indexed marginals \((\mu_t)_{t\in[0,T]}\) in pixel or latent space via either
deterministic transport (CE) or stochastic diffusion (FP). ECFM supplies a single constraint-level control---an
entropy-rate budget---that is compatible with both regimes and yields certificate-style anti-collapse guarantees.

\paragraph{Deterministic flows (FM / rectified flow).}
For \(\dot X_t=v_\theta(X_t,t)\) with \(X_0\sim\mu_0\), the law solves
\(\partial_t\mu_t+\nabla\!\cdot(\mu_t v_\theta)=0\). Under App.~\hyperref[app:B]{B}, the entropy rate is
\begin{equation}
\label{eq:entropy_rate_det_main}
\frac{d}{dt}\mathcal{H}(\mu_t)=\mathbb{E}_{\mu_t}\!\big[\nabla\!\cdot v_\theta(\cdot,t)\big],
\end{equation}
so \(\dot{\mathcal H}(\mu_t)\ge-\lambda\) is an explicit restriction on expected divergence, excluding transient
compressions that drive mode depletion (Sec.~\ref{sec:failure}).

\paragraph{Diffusions (FP / Fisher information).}
For \(dX_t=b_\theta(X_t,t)\,dt+\sqrt{2\varepsilon(t)}\,dW_t\), \(\rho_t\) solves FP and (App.~\hyperref[app:B4]{B.4})
\begin{equation}
\label{eq:entropy_rate_FP_theta_main}
\frac{d}{dt}\mathcal{H}(\mu_t)
=
\mathbb{E}_{\mu_t}\!\big[\nabla\!\cdot b_\theta(\cdot,t)\big]
+\varepsilon(t)\,\mathcal{I}(\mu_t).
\end{equation}
ECFM imposes \(\dot{\mathcal H}\ge-\lambda\) directly (or via current velocity; Sec.~\ref{sec:sb_equiv}), and the Fisher term
provides an intrinsic anti-collapse contribution.

\paragraph{Certificate-style diagnostics.}
One can estimate \(\dot{\mathcal H}(\mu_t)\) from minibatch trajectories (App.~\hyperref[app:J2]{J.2}):
\begin{equation}
\label{eq:entropy_rate_estimator_main}
\widehat{\dot{\mathcal H}}_{\mathrm{det}}(t)
= \frac{1}{B}\sum_{i=1}^B \nabla\!\cdot v_\theta(x_i,t),
\qquad
\widehat{\dot{\mathcal H}}_{\mathrm{diff}}(t)
= \frac{1}{B}\sum_{i=1}^B \nabla\!\cdot b_\theta(x_i,t)
+\varepsilon(t)\,\frac{1}{B}\sum_{i=1}^B \|s_\theta(x_i,t)\|^2,
\end{equation}
where \(s_\theta\approx\nabla\log\rho_t\). Feasibility \(\widehat{\dot{\mathcal H}}(t)\ge-\lambda\) yields a direct, SOTA-free
certificate pathway to Theorems~\ref{thm:mode_coverage_main}--\ref{thm:density_floor_main}. Removing entropy control admits near-optimal
collapse channels (Sec.~\ref{sec:failure}).

\section{Conclusion}
\label{sec:conclusion}

ECFM augments flow matching with an explicit entropy-rate budget, turning trajectory compressibility into a controllable constraint.
We characterize the resulting convex problem (KKT/Pontryagin), relate it to Schr\"odinger bridges, and in the transport specialization recover entropic OT and $\Gamma$-convergence to classical OT as $\lambda\to0$.
Entropy control yields verifiable mode-coverage and density-floor certificates (stable to perturbations), while unconstrained FM admits near-optimal collapsing paths.

\section*{Acknowledgements}

The authors thank NSF and NCSA for computational support.




\clearpage
\appendix

\section*{Appendix}
\addcontentsline{toc}{section}{Appendix}
\label{appendix}

\etocarticlestyle
\etocsettocstyle{\subsection*{Table of Contents}}{}

\etocsetstyle{section}
  {\par\noindent}
  {\par\noindent}
  {\etocname\dotfill\etocpage\par}
  {}

\etocsetstyle{subsection}
  {\par\noindent\hspace{0em}}
  {\par\noindent\hspace{0em}}
  {\etocname\dotfill\etocpage\par}
  {}

\etocsetstyle{subsubsection}
  {\par\noindent\hspace{1.5em}}
  {\par\noindent\hspace{1em}}
  {\etocname\dotfill\etocpage\par}
  {}

\etocsetnexttocdepth{subsubsection}
\localtableofcontents

\subsection*{Abbreviations}
\addcontentsline{toc}{subsection}{Abbreviations}

\begin{description}

  \item[AC:] absolutely continuous (e.g., \(h\in AC([0,T])\))
  \item[a.e.:] almost everywhere
  \item[w.r.t.:] with respect to
  \item[l.s.c.:] lower semicontinuous
  \item[loc.:] local (e.g., \(L^1_{\mathrm{loc}}\), \(W^{1,1}_{\mathrm{loc}}\))

  \item[CE:] continuity equation: \(\partial_t\mu_t+\nabla\!\cdot(\mu_t v_t)=0\) (weak form \(\eqref{eq:CE_weak_main}\))
  \item[FP:] Fokker--Planck equation
  \item[HJ:] Hamilton--Jacobi (as in HJ-type inequalities/PDEs)
  \item[LSI:] logarithmic Sobolev inequality
  \item[ODE:] ordinary differential equation
  \item[SDE:] stochastic differential equation

  \item[OT:] optimal transport
  \item[EOT:] entropic optimal transport
  \item[SB:] Schr\"odinger bridge
  \item[KL:] Kullback--Leibler divergence / relative entropy
  \item[RN:] Radon--Nikod\'ym (as in Radon--Nikod\'ym derivative)

  \item[FM:] (unconstrained) flow matching
  \item[ECFM:] Entropy-Controlled Flow Matching (this work)
  \item[BB:] Benamou--Brenier (dynamic OT formulation)
  \item[KKT:] Karush--Kuhn--Tucker optimality conditions
  \item[JVP:] Jacobian--vector product

  \item[W2:] quadratic Wasserstein distance \(W_2\)
  \item[\(\mathcal P(\mathbb R^d)\):] probability measures on \(\mathbb R^d\)
  \item[\(\mathcal P_2(\mathbb R^d)\):] probability measures with finite second moment
  \item[ac:] absolutely continuous (typically w.r.t.\ Lebesgue; e.g., \(\mathcal P_2^{ac}\))

  \item[pushforward \(T_\#\mu\):] image measure of \(\mu\) under map \(T\)
  \item[coupling \(\Pi(\mu,\nu)\):] joint measures with marginals \(\mu,\nu\)

  \item[current velocity:] \(v=b-\varepsilon\nabla\log\rho\) (links FP to CE; \(\eqref{eq:current_velocity_main}\))
  \item[entropy-rate budget:] constraint \(\frac{d}{dt}\mathcal H(\mu_t)\ge-\lambda\) (\(\eqref{eq:entropy_budget_main}\))

  \item[Lip:] Lipschitz (as in \(\|\cdot\|_{\mathrm{Lip}}\) or Lipschitz constants)
  \item[sgn:] sign function (e.g., \(\mathrm{sgn}(x)\in\{-1,0,1\}\))
  \item[ess sup:] essential supremum (w.r.t.\ Lebesgue measure on time, unless stated otherwise)

  \item[LCB:] lower confidence bound (used in certification diagnostics)
  \item[mode mass:] \(M_k(t)=\mu_t(A_k)\) for mode set \(A_k\)
\end{description}

\subsection*{A. Measure-Theoretic Preliminaries}
\addcontentsline{toc}{subsection}{A. Measure-Theoretic Preliminaries}
\label{app:A}

This section fixes the ambient measure-theoretic and Wasserstein-geometric framework used throughout the appendix.  
We work on Euclidean state space \(\mathbb{R}^d\) (\(d\ge 1\)) with Borel \(\sigma\)-algebra \(\mathcal{B}(\mathbb{R}^d)\), and a finite time horizon \(T>0\).

\subsubsection*{A.1. Probability spaces and notation}
\addcontentsline{toc}{subsubsection}{A.1. Probability spaces and notation}
\label{app:A1}

\paragraph{Base probability space.}
Let \((\Omega,\mathcal{F},\mathbb{P})\) be a complete probability space supporting all random variables/processes used below.

\paragraph{Measure spaces.}
Denote by \(\mathcal{P}(\mathbb{R}^d)\) the set of Borel probability measures on \(\mathbb{R}^d\), and by
\[
\mathcal{P}_2(\mathbb{R}^d)
:=
\left\{
\mu\in\mathcal{P}(\mathbb{R}^d)
:\int_{\mathbb{R}^d}|x|^2\,d\mu(x)<\infty
\right\}
\]
the subset with finite second moment.

For \(\mu\in\mathcal{P}(\mathbb{R}^d)\), write
\[
m_2(\mu):=\int_{\mathbb{R}^d}|x|^2\,d\mu(x).
\]

\paragraph{Couplings and pushforwards.}
For \(\mu,\nu\in\mathcal{P}(\mathbb{R}^d)\), let
\[
\Pi(\mu,\nu)
:=
\left\{
\pi\in\mathcal{P}(\mathbb{R}^d\times\mathbb{R}^d):
(\mathrm{pr}_1)_{\#}\pi=\mu,\;(\mathrm{pr}_2)_{\#}\pi=\nu
\right\},
\]
where \(\mathrm{pr}_1(x,y)=x\), \(\mathrm{pr}_2(x,y)=y\).  
For measurable \(T:\mathbb{R}^d\to\mathbb{R}^d\), \(T_{\#}\mu\) is the pushforward:
\[
T_{\#}\mu(B)=\mu(T^{-1}(B)),\qquad B\in\mathcal{B}(\mathbb{R}^d).
\]

\paragraph{Quadratic Wasserstein distance.}
For \(\mu,\nu\in\mathcal{P}_2(\mathbb{R}^d)\),
\[
W_2^2(\mu,\nu)
:=
\inf_{\pi\in\Pi(\mu,\nu)}
\int_{\mathbb{R}^d\times\mathbb{R}^d}|x-y|^2\,d\pi(x,y).
\]

\paragraph{Time-indexed measures and densities.}
Given a curve \((\mu_t)_{t\in[0,T]}\subset\mathcal{P}_2(\mathbb{R}^d)\), we write
\[
\mu_t=\rho_t\,\mathcal{L}^d
\quad\Longleftrightarrow\quad
\mu_t\ll\mathcal{L}^d,\ \rho_t\in L^1(\mathbb{R}^d),\ \rho_t\ge0,\ \int\rho_t\,dx=1.
\]

\paragraph{Differential entropy and relative entropy.}
Whenever \(\mu=\rho\,\mathcal{L}^d\) with \(\rho\log\rho\in L^1(\mathbb{R}^d)\), define
\[
\mathcal{H}(\mu):=\int_{\mathbb{R}^d}\rho(x)\log\rho(x)\,dx.
\]
(We use \(\mathcal{H}\) as negative Shannon differential entropy up to sign convention.)  

Given \(\mu,\nu\in\mathcal{P}(\mathbb{R}^d)\), the Kullback--Leibler divergence is
\[
\mathrm{KL}(\mu\|\nu)
:=
\begin{cases}
\displaystyle \int \log\!\left(\frac{d\mu}{d\nu}\right)\,d\mu,& \mu\ll\nu,\\[1.2ex]
+\infty,& \text{otherwise}.
\end{cases}
\]

\paragraph{Velocity fields and continuity equation notation.}
For a Borel vector field \(v:[0,T]\times\mathbb{R}^d\to\mathbb{R}^d\), define the kinetic action
\[
\mathcal{A}(\mu,v):=\frac12\int_0^T\int_{\mathbb{R}^d}|v_t(x)|^2\,d\mu_t(x)\,dt.
\]
We write
\[
(\mu,v)\in\mathrm{CE}([0,T])
\]
if \((\mu_t)_{t\in[0,T]}\subset\mathcal{P}_2(\mathbb{R}^d)\) is narrowly continuous, \(v\in L^2(dt\,d\mu_t)\), and
\[
\partial_t\mu_t+\nabla\cdot(\mu_t v_t)=0
\quad\text{in }\mathcal{D}'((0,T)\times\mathbb{R}^d).
\]

\paragraph{Admissible endpoint class for the paper.}
Fix data and prior marginals \(\mu_0,\mu_T\in\mathcal{P}_2(\mathbb{R}^d)\), both absolutely continuous:
\[
\mu_0=\rho_0\mathcal{L}^d,\qquad \mu_T=\rho_T\mathcal{L}^d.
\]
The dynamic admissible set is
\[
\mathfrak{A}(\mu_0,\mu_T)
:=
\left\{
(\mu,v)\in\mathrm{CE}([0,T]):
\mu_{|t=0}=\mu_0,\ \mu_{|t=T}=\mu_T
\right\}.
\]

\paragraph{Standing regularity convention (used unless stated otherwise).}
Unless explicitly relaxed, we assume
\[
\mu_t=\rho_t\mathcal{L}^d,\quad
\rho_t>0\ \text{a.e.},\quad
\rho_t\in W^{1,1}_{\mathrm{loc}}(\mathbb{R}^d),\quad
v_t\in L^2(\mu_t),
\]
with sufficient integrability to justify all integration-by-parts identities in weak form.

\subsubsection*{A.2. Absolutely continuous curves in Wasserstein space}
\addcontentsline{toc}{subsubsection}{A.2. Absolutely continuous curves in Wasserstein space}
\label{app:A2}

We formalize time-regularity of measure-valued trajectories via metric absolute continuity in Wasserstein space.

\begin{definition}[Metric absolute continuity in \((\mathcal P_2,W_2)\)]
A curve \(\mu_\cdot:[0,T]\to\mathcal P_2(\mathbb R^d)\) belongs to
\(AC^2([0,T];\mathcal P_2(\mathbb R^d))\) if there exists
\(m\in L^2(0,T)\) such that for all \(0\le s\le t\le T\),
\[
W_2(\mu_s,\mu_t)\le \int_s^t m(r)\,dr.
\]
\end{definition}

\begin{definition}[Metric derivative]
If \(\mu_\cdot\in AC^2([0,T];\mathcal P_2)\), its metric derivative is
\[
|\dot\mu_t|
:=
\lim_{h\to0}\frac{W_2(\mu_{t+h},\mu_t)}{|h|}
\quad\text{for a.e. }t\in(0,T),
\]
which exists for a.e. \(t\), belongs to \(L^2(0,T)\), and is the minimal \(L^2\)-upper gradient.
\end{definition}

\begin{definition}[Tangent velocity class]
For \(\mu\in\mathcal P_2(\mathbb R^d)\), define
\[
L^2(\mu;\mathbb R^d)
:=
\left\{v:\mathbb R^d\to\mathbb R^d\text{ Borel}:
\int|v|^2\,d\mu<\infty\right\}.
\]
The tangent space \(T_\mu\mathcal P_2\) is the \(L^2(\mu)\)-closure of gradients:
\[
T_\mu\mathcal P_2
:=
\overline{\{\nabla\varphi:\varphi\in C_c^\infty(\mathbb R^d)\}}^{\,L^2(\mu)}.
\]
\end{definition}

\begin{theorem}[Dynamic characterization of \(AC^2\) curves]
\label{thm:AC2_dynamic_characterization}
Let \(\mu_\cdot:[0,T]\to\mathcal P_2(\mathbb R^d)\).
The following are equivalent:
\begin{enumerate}
\item \(\mu_\cdot\in AC^2([0,T];\mathcal P_2)\).
\item There exists a Borel field \(v_t\in L^2(\mu_t)\) such that
\[
\partial_t\mu_t+\nabla\cdot(\mu_t v_t)=0
\quad\text{in }\mathcal D'((0,T)\times\mathbb R^d),
\]
and
\[
\int_0^T\!\!\int_{\mathbb R^d}|v_t(x)|^2\,d\mu_t(x)\,dt<\infty.
\]
\end{enumerate}
Moreover, among all such \(v\), there exists a unique (a.e. in \(t\), \(\mu_t\)-a.e. in \(x\))
minimal-norm representative \(v_t\in T_{\mu_t}\mathcal P_2\) satisfying
\[
|\dot\mu_t|=\|v_t\|_{L^2(\mu_t)}
=\min\Big\{\|w_t\|_{L^2(\mu_t)}:
\partial_t\mu_t+\nabla\cdot(\mu_t w_t)=0\Big\}
\quad\text{for a.e. }t.
\]
\end{theorem}

\begin{proof}
\textbf{(1 \(\Rightarrow\) 2).}
Assume \(\mu_\cdot\in AC^2\).
By the superposition principle for Wasserstein absolutely continuous curves,
there exists a probability measure \(\eta\) on \(AC([0,T];\mathbb R^d)\) such that
\[
(e_t)_\#\eta=\mu_t,\qquad t\in[0,T],
\]
and
\[
\int\!\!\int_0^T |\dot\gamma_t|^2\,dt\,d\eta(\gamma)<\infty.
\]
Disintegrating the pathwise velocity with respect to \(e_t\), define a Borel selection
\(v_t(x)\) as barycentric projection of \(\dot\gamma_t\) conditionally on \(\gamma_t=x\).
Then \(v\in L^2(dt\,d\mu_t)\), and testing against \(\phi\in C_c^\infty((0,T)\times\mathbb R^d)\):
\[
\frac{d}{dt}\int \phi(t,\gamma_t)\,d\eta
=
\int\big(\partial_t\phi+\nabla\phi\cdot \dot\gamma_t\big)\,d\eta
=
\int\big(\partial_t\phi+\nabla\phi\cdot v_t\big)\,d\mu_t,
\]
which is exactly the weak continuity equation.
Square-integrability follows from Jensen:
\[
\int|v_t|^2\,d\mu_t
\le
\int |\dot\gamma_t|^2\,d\eta.
\]

\textbf{(2 \(\Rightarrow\) 1).}
Assume \(\partial_t\mu_t+\nabla\cdot(\mu_t v_t)=0\) with \(v\in L^2(dt\,d\mu_t)\).
For \(0\le s<t\le T\), apply Benamou--Brenier on \([s,t]\) with rescaling:
\[
W_2^2(\mu_s,\mu_t)
\le
(t-s)\int_s^t\!\!\int |v_r|^2\,d\mu_r\,dr.
\]
Hence
\[
W_2(\mu_s,\mu_t)
\le
\int_s^t
\Big(\int |v_r|^2\,d\mu_r\Big)^{1/2}dr,
\]
so \(\mu_\cdot\in AC^2\).

\textbf{Minimal norm and metric derivative.}
For fixed \(t\), admissible velocities solving the distributional identity differ by fields
in the \(L^2(\mu_t)\)-orthogonal complement of \(T_{\mu_t}\mathcal P_2\)
(i.e., divergence-free relative to \(\mu_t\) in weak form).
Orthogonal projection onto \(T_{\mu_t}\mathcal P_2\) yields the unique minimal-norm \(v_t\in T_{\mu_t}\mathcal P_2\).
The general metric-space theory of absolutely continuous curves gives
\[
|\dot\mu_t|
=
\inf\{\|w_t\|_{L^2(\mu_t)}:\text{CE holds at }t\}
=
\|v_t\|_{L^2(\mu_t)}
\quad\text{a.e.}
\]
This concludes the proof.
\end{proof}

\begin{corollary}[Kinetic action bound implies \(1/2\)-H\"older continuity]
\label{cor:A2_holder}
If \((\mu,v)\in\mathrm{CE}([0,T])\) and
\[
\int_0^T\!\!\int |v_t|^2\,d\mu_t\,dt\le M<\infty,
\]
then for all \(0\le s<t\le T\),
\[
W_2(\mu_s,\mu_t)\le \sqrt{M}\,|t-s|^{1/2}.
\]
\end{corollary}

\begin{proof}
From the estimate in the proof above:
\[
W_2^2(\mu_s,\mu_t)\le (t-s)\int_s^t\!\!\int |v_r|^2\,d\mu_r\,dr
\le (t-s)M.
\]
Take square roots.
\end{proof}

\begin{lemma}[Narrow continuity and moment control]
\label{lem:A2_moment_control}
Let \((\mu,v)\in\mathrm{CE}([0,T])\) with finite kinetic action and \(\mu_0\in\mathcal P_2\).
Then \(\mu_t\in\mathcal P_2\) for every \(t\), and
\[
\sup_{t\in[0,T]} m_2(\mu_t)
\le
C\Big(m_2(\mu_0)+\int_0^T\!\!\int |v_t|^2\,d\mu_t\,dt\Big),
\]
for a constant \(C=C(T)\).
\end{lemma}

\begin{proof}
Set \(\psi_R(x)=\min\{|x|^2,R\}\), smoothened if needed to justify weak testing.
Using CE with test \(\psi_R\):
\[
\frac{d}{dt}\int \psi_R\,d\mu_t
=
\int \nabla\psi_R\cdot v_t\,d\mu_t.
\]
Since \(|\nabla\psi_R|\le 2|x|\),
\[
\left|\frac{d}{dt}\int \psi_R\,d\mu_t\right|
\le
2\Big(\int |x|^2\,d\mu_t\Big)^{1/2}
\Big(\int |v_t|^2\,d\mu_t\Big)^{1/2}.
\]
Apply Young:
\[
2ab\le a^2+b^2
\quad\Rightarrow\quad
\frac{d}{dt}\int \psi_R\,d\mu_t
\le
\int |x|^2\,d\mu_t+\int |v_t|^2\,d\mu_t.
\]
Pass \(R\uparrow\infty\) (monotone convergence), denote \(y(t)=m_2(\mu_t)\):
\[
y'(t)\le y(t)+\int |v_t|^2\,d\mu_t
\quad\text{in distributional sense}.
\]
Gronwall yields
\[
y(t)\le e^t\!\left(y(0)+\int_0^t e^{-r}\!\int |v_r|^2\,d\mu_r\,dr\right)
\le
e^T\!\left(m_2(\mu_0)+\int_0^T\!\!\int |v_r|^2\,d\mu_r\,dr\right).
\]
Hence \(\sup_t m_2(\mu_t)<\infty\). Since \(AC^2\) implies narrow continuity, claim follows.
\end{proof}

\paragraph{Consequence for the appendix.}
By Theorem~\ref{thm:AC2_dynamic_characterization}, any admissible trajectory with finite action
admits a canonical tangent velocity \(v_t\in T_{\mu_t}\mathcal P_2\), and all variational problems below
can be posed over \(AC^2([0,T];\mathcal P_2)\) with the continuity equation constraint.

\subsubsection*{A.3. Continuity equation}
\phantomsection
\addcontentsline{toc}{subsubsection}{A.3. Continuity equation}
\label{app:A3}

We now formalize the continuity equation in weak/distributional form, establish its
equivalent formulations, and record the identities needed for the entropy-controlled
variational analysis.

\begin{definition}[Distributional continuity equation]
\label{def:CE_distributional}
Let \((\mu_t)_{t\in[0,T]}\subset\mathcal P_2(\mathbb R^d)\) be narrowly continuous and
\(v:[0,T]\times\mathbb R^d\to\mathbb R^d\) Borel with
\[
\int_0^T\!\!\int_{\mathbb R^d}|v_t(x)|\,d\mu_t(x)\,dt<\infty.
\]
We say \((\mu,v)\) satisfies
\[
\partial_t\mu_t+\nabla\cdot(\mu_t v_t)=0
\]
in \(\mathcal D'((0,T)\times\mathbb R^d)\) if for every
\(\phi\in C_c^\infty((0,T)\times\mathbb R^d)\),
\[
\int_0^T\!\!\int_{\mathbb R^d}
\big(\partial_t\phi(t,x)+\nabla\phi(t,x)\!\cdot\! v_t(x)\big)\,d\mu_t(x)\,dt
=0.
\]
\end{definition}

\begin{lemma}[Weak-in-time formulation with endpoints]
\label{lem:CE_weak_time}
Assume \((\mu,v)\) satisfies Definition~\ref{def:CE_distributional} and
\(\int_0^T\!\!\int|v_t|\,d\mu_t\,dt<\infty\).
Then for every \(\zeta\in C_c^\infty(\mathbb R^d)\), the map
\[
t\mapsto \langle \mu_t,\zeta\rangle := \int_{\mathbb R^d}\zeta(x)\,d\mu_t(x)
\]
is absolutely continuous and
\[
\frac{d}{dt}\langle\mu_t,\zeta\rangle
=
\int_{\mathbb R^d}\nabla\zeta(x)\cdot v_t(x)\,d\mu_t(x)
\quad\text{for a.e. }t.
\]
Equivalently, for all \(\eta\in C_c^\infty([0,T])\),
\[
-\int_0^T \eta'(t)\langle\mu_t,\zeta\rangle\,dt
=
\int_0^T \eta(t)\!\int \nabla\zeta\cdot v_t\,d\mu_t\,dt.
\]
\end{lemma}

\begin{proof}
Take test functions \(\phi(t,x)=\eta(t)\zeta(x)\) in Definition~\ref{def:CE_distributional}:
\[
\int_0^T\!\!\int
\big(\eta'(t)\zeta(x)+\eta(t)\nabla\zeta(x)\cdot v_t(x)\big)\,d\mu_t(x)\,dt=0.
\]
This gives
\[
\int_0^T \eta'(t)\langle\mu_t,\zeta\rangle\,dt
=
-\int_0^T \eta(t)\!\int \nabla\zeta\cdot v_t\,d\mu_t\,dt.
\]
Hence distributional time derivative of \(t\mapsto\langle\mu_t,\zeta\rangle\) belongs to
\(L^1(0,T)\), so the map is absolutely continuous and the pointwise a.e. identity follows.
\end{proof}

\begin{proposition}[Renormalized identity for smooth scalar transforms]
\label{prop:CE_renorm}
Assume \(\mu_t=\rho_t\mathcal L^d\), \(\rho\in L^1_{\mathrm{loc}}\), and
\[
\partial_t\rho+\nabla\cdot(\rho v)=0
\quad\text{in }\mathcal D'((0,T)\times\mathbb R^d),
\]
with
\[
v\in L^1_{\mathrm{loc}}((0,T);W^{1,1}_{\mathrm{loc}}(\mathbb R^d;\mathbb R^d)),
\qquad
(\nabla\cdot v)^-\in L^1_{\mathrm{loc}}((0,T)\times\mathbb R^d).
\]
Let \(\beta\in C^1(\mathbb R)\) with \(\beta',\ z\beta'(z)-\beta(z)\) having at most linear growth.
Then
\[
\partial_t\beta(\rho)+\nabla\cdot(\beta(\rho)\,v)
+\big(\rho\beta'(\rho)-\beta(\rho)\big)\,\nabla\cdot v=0
\]
in \(\mathcal D'((0,T)\times\mathbb R^d)\).
\end{proposition}

\begin{proof}
Under the stated Sobolev regularity on \(v\), the DiPerna--Lions renormalization theorem~\cite{diperna_ordinary_1989} applies
to distributional solutions of the continuity equation.
Therefore each admissible \(\beta\) yields the transformed identity above.
\end{proof}

\begin{corollary}[Mass conservation]
\label{cor:CE_mass}
Under Proposition~\ref{prop:CE_renorm} with \(\beta(z)=z\), total mass is conserved:
\[
\int_{\mathbb R^d}\rho_t(x)\,dx=\int_{\mathbb R^d}\rho_0(x)\,dx
\quad\forall t\in[0,T].
\]
\end{corollary}

\begin{proof}
For \(\beta(z)=z\), the renormalized identity is the original CE.
Test against cutoff \(\chi_R\to1\), pass \(R\to\infty\) using
\(\rho\in L^1\) and local integrability of \(\rho v\); obtain zero derivative of total mass.
\end{proof}

\begin{lemma}[Second-moment evolution identity]
\label{lem:CE_second_moment_identity}
Assume \((\mu,v)\in\mathrm{CE}([0,T])\) with
\(\int_0^T\!\!\int |v_t|^2\,d\mu_t\,dt<\infty\) and \(\sup_t m_2(\mu_t)<\infty\).
Then \(t\mapsto m_2(\mu_t)\) is absolutely continuous and
\[
\frac{d}{dt}m_2(\mu_t)
=
2\int_{\mathbb R^d} x\cdot v_t(x)\,d\mu_t(x)
\quad\text{for a.e. }t.
\]
\end{lemma}

\begin{proof}
Use Lemma~\ref{lem:CE_weak_time} with truncated test \(\zeta_R(x)=|x|^2\chi_R(x)\), where
\(\chi_R\in C_c^\infty\), \(\chi_R\equiv1\) on \(B_R\), \(|\nabla\chi_R|\lesssim 1/R\).
Then
\[
\frac{d}{dt}\int \zeta_R\,d\mu_t
=
\int \nabla\zeta_R\cdot v_t\,d\mu_t.
\]
As \(R\to\infty\),
\[
\zeta_R\uparrow |x|^2,\qquad \nabla\zeta_R\to 2x
\]
pointwise; dominated convergence follows from
\[
|\nabla\zeta_R\cdot v_t|
\lesssim (|x|+|x|^2/R)|v_t|
\]
and Cauchy--Schwarz with finite \(m_2(\mu_t)\), \(L^2(\mu_t)\)-norm of \(v_t\).
Hence
\[
\frac{d}{dt}m_2(\mu_t)=2\int x\cdot v_t\,d\mu_t.
\]
\end{proof}

\begin{proposition}[Weak--strong chain rule along CE trajectories]
\label{prop:CE_chain_rule}
Let \((\mu,v)\in\mathrm{CE}([0,T])\), with \(v\in L^2(dt\,d\mu_t)\).
Let \(\Phi\in C^1(\mathbb R^d)\) satisfy
\[
|\nabla\Phi(x)|\le C(1+|x|).
\]
Then \(t\mapsto \int \Phi\,d\mu_t\) is absolutely continuous and
\[
\frac{d}{dt}\int_{\mathbb R^d}\Phi(x)\,d\mu_t(x)
=
\int_{\mathbb R^d}\nabla\Phi(x)\cdot v_t(x)\,d\mu_t(x)
\quad\text{a.e. }t.
\]
\end{proposition}

\begin{proof}
Approximate \(\Phi\) by compactly supported smooth \(\Phi_R:=\Phi\chi_R\).
Apply Lemma~\ref{lem:CE_weak_time} to \(\Phi_R\):
\[
\frac{d}{dt}\int \Phi_R\,d\mu_t=\int \nabla\Phi_R\cdot v_t\,d\mu_t.
\]
Pass \(R\to\infty\).  
Left side: dominated convergence via linear-growth bound and finite second moment.  
Right side:
\[
|\nabla\Phi_R\cdot v_t|
\le C(1+|x|)|v_t| + |\Phi||\nabla\chi_R||v_t|,
\]
the first term is integrable by Cauchy--Schwarz, the second vanishes as \(R\to\infty\)
(using \(|\nabla\chi_R|\lesssim 1/R\) and tail control from \(m_2(\mu_t)\)).
Hence the identity follows.
\end{proof}

\begin{definition}[Admissible CE class with finite kinetic action]
\label{def:CE_admissible_action}
For fixed \(\mu_0,\mu_T\in\mathcal P_2(\mathbb R^d)\), define
\[
\mathfrak A(\mu_0,\mu_T)
:=
\left\{
(\mu,v):
\begin{array}{l}
(\mu,v)\in\mathrm{CE}([0,T]),\\
\mu_{|t=0}=\mu_0,\ \mu_{|t=T}=\mu_T,\\
\displaystyle \int_0^T\!\!\int \frac12|v_t|^2\,d\mu_t\,dt<\infty
\end{array}
\right\}.
\]
\end{definition}

\paragraph{Use in subsequent sections.}
All entropy-controlled variational problems are posed over
\(\mathfrak A(\mu_0,\mu_T)\), with additional entropy-rate constraints.
Propositions above ensure every functional derivative computed later is justified in weak form.

\subsubsection*{A.4. Wasserstein geometry foundations}
\phantomsection
\addcontentsline{toc}{subsubsection}{A.4. Wasserstein geometry foundations}
\label{app:A4}

We collect the geometric facts on \((\mathcal P_2(\mathbb R^d),W_2)\) used in Sections ~\hyperref[app:B]{B}--\hyperref[app:I]{I}:
geodesics, dynamic action, tangent/cotangent structure, first variations, and convexity along
displacement interpolations.

\begin{definition}[Constant-speed Wasserstein geodesic]
\label{def:W2_geodesic}
A curve \((\mu_t)_{t\in[0,1]}\subset\mathcal P_2(\mathbb R^d)\) is a (constant-speed) \(W_2\)-geodesic
from \(\mu_0\) to \(\mu_1\) if
\[
W_2(\mu_s,\mu_t)=|t-s|\,W_2(\mu_0,\mu_1),\qquad \forall s,t\in[0,1].
\]
\end{definition}

\begin{theorem}[Dynamic Benamou--Brenier formula]
\label{thm:BB_dynamic}
For \(\mu_0,\mu_1\in\mathcal P_2(\mathbb R^d)\),
\[
W_2^2(\mu_0,\mu_1)
=
\inf_{\substack{(\mu,v)\in\mathrm{CE}([0,1])\\ \mu_{|0}=\mu_0,\ \mu_{|1}=\mu_1}}
\int_0^1\!\!\int_{\mathbb R^d}|v_t(x)|^2\,d\mu_t(x)\,dt.
\]
Any minimizer \((\mu,v)\) has minimal tangent velocity \(v_t\in T_{\mu_t}\mathcal P_2\), satisfies
\[
\|v_t\|_{L^2(\mu_t)}=W_2(\mu_0,\mu_1)\quad\text{for a.e. }t,
\]
and \((\mu_t)\) is a constant-speed geodesic.
\end{theorem}

\begin{proof}
\textbf{Step 1 (upper bound).}
Let \(\pi\in\Pi(\mu_0,\mu_1)\). Define path measure via linear interpolation
\(\gamma_t=(1-t)x+ty\) under \((x,y)\sim\pi\), and \(\mu_t:=(\gamma_t)_\#\pi\).
Set Eulerian velocity \(v_t(\gamma_t)=y-x\) (barycentric representative).
Then \((\mu,v)\in\mathrm{CE}\) and
\[
\int_0^1\!\!\int |v_t|^2\,d\mu_t\,dt
=
\int |x-y|^2\,d\pi.
\]
Taking infimum over \(\pi\) gives
\[
\inf_{\mathrm{CE}}\int_0^1\!\!\int |v_t|^2\,d\mu_t\,dt \le W_2^2(\mu_0,\mu_1).
\]

\textbf{Step 2 (lower bound).}
Given any admissible \((\mu,v)\), for \(0\le s<t\le1\),
\[
W_2(\mu_s,\mu_t)\le \int_s^t \|v_r\|_{L^2(\mu_r)}\,dr
\]
(Section \hyperref[app:A2]{A.2}). By Cauchy--Schwarz,
\[
W_2^2(\mu_0,\mu_1)
\le
\left(\int_0^1 \|v_r\|_{L^2(\mu_r)}\,dr\right)^2
\le
\int_0^1\|v_r\|_{L^2(\mu_r)}^2\,dr.
\]
Taking infimum over admissible pairs yields the converse inequality.

\textbf{Step 3 (characterization of minimizers).}
If equality holds in Cauchy--Schwarz above, then
\(\|v_t\|_{L^2(\mu_t)}\) is a.e. constant and
\[
W_2(\mu_s,\mu_t)=\int_s^t \|v_r\|_{L^2(\mu_r)}\,dr
=(t-s)W_2(\mu_0,\mu_1),
\]
hence \((\mu_t)\) is constant-speed geodesic.
Projection onto tangent space \(T_{\mu_t}\mathcal P_2\) yields minimal-norm representative.
\end{proof}

\begin{theorem}[Geodesic representation via optimal map]
\label{thm:McCann_interpolation}
Assume \(\mu_0\ll\mathcal L^d\). Then there exists a unique optimal transport map
\(T=\nabla\psi\) (for a convex \(\psi\)) such that \(T_\#\mu_0=\mu_1\), and
\[
\mu_t=\big((1-t)\mathrm{Id}+tT\big)_\#\mu_0,\qquad t\in[0,1],
\]
is the unique \(W_2\)-geodesic joining \(\mu_0,\mu_1\).
\end{theorem}

\begin{proof}
Existence/uniqueness of Brenier map follows from \(\mu_0\ll\mathcal L^d\).
Define \(X_t(x):=(1-t)x+tT(x)\), \(\mu_t=(X_t)_\#\mu_0\).
The induced plan between \(\mu_s,\mu_t\) is \((X_s,X_t)_\#\mu_0\), giving
\[
\int|X_t-X_s|^2\,d\mu_0=(t-s)^2\int|T-\mathrm{Id}|^2\,d\mu_0=(t-s)^2W_2^2(\mu_0,\mu_1),
\]
hence \(W_2(\mu_s,\mu_t)\le |t-s|W_2(\mu_0,\mu_1)\). Reverse inequality follows by triangle inequality
and additivity of the bound over partitions; therefore equality holds and geodesicity follows.
Uniqueness of geodesic under absolute continuity of \(\mu_0\) is standard from uniqueness of optimal maps.
\end{proof}

\begin{definition}[Absolutely continuous functionals and metric slope]
\label{def:metric_slope}
For a proper l.s.c. functional \(\mathcal F:\mathcal P_2\to(-\infty,+\infty]\),
the local metric slope at \(\mu\) is
\[
|\partial\mathcal F|(\mu)
:=
\limsup_{\nu\to\mu}\frac{(\mathcal F(\mu)-\mathcal F(\nu))_+}{W_2(\mu,\nu)}.
\]
\end{definition}

\begin{definition}[First variation and Wasserstein gradient]
\label{def:first_variation}
If \(\mu=\rho\mathcal L^d\) and \(\mathcal F\) has first variation
\(\frac{\delta\mathcal F}{\delta\rho}\), then the formal Wasserstein gradient is
\[
\nabla_{W_2}\mathcal F(\mu)
=
-\nabla\cdot\!\left(\mu\,\nabla \frac{\delta\mathcal F}{\delta\rho}\right)
\]
and the Riemannian differential along a CE trajectory is
\[
\frac{d}{dt}\mathcal F(\mu_t)
=
\int \nabla\!\left(\frac{\delta\mathcal F}{\delta\rho_t}\right)\cdot v_t\,d\mu_t
\]
whenever justified by integrability/regularity.
\end{definition}

\begin{lemma}[First variation of internal-energy class]
\label{lem:internal_energy_variation}
Let
\[
\mathcal U_m(\mu)=
\begin{cases}
\displaystyle \int \frac{\rho^m}{m-1}\,dx,& \mu=\rho\mathcal L^d,\ m>1,\\
+\infty,&\text{otherwise}.
\end{cases}
\]
Then
\[
\frac{\delta \mathcal U_m}{\delta \rho}(\rho)=\frac{m}{m-1}\rho^{m-1}.
\]
For smooth CE solutions,
\[
\frac{d}{dt}\mathcal U_m(\mu_t)
=
\int \nabla\!\left(\frac{m}{m-1}\rho_t^{m-1}\right)\cdot v_t\,\rho_t\,dx.
\]
\end{lemma}

\begin{proof}
Take \(\rho_\varepsilon=\rho+\varepsilon\eta\) with \(\int\eta=0\), compute
\[
\frac{d}{d\varepsilon}\Big|_{\varepsilon=0}\int\frac{\rho_\varepsilon^m}{m-1}\,dx
=
\int \frac{m}{m-1}\rho^{m-1}\eta\,dx.
\]
Hence first variation formula. Chain rule then follows from Proposition~\ref{prop:CE_chain_rule} with
\(\Phi=\frac{\delta \mathcal U_m}{\delta \rho_t}\) and standard approximation.
\end{proof}

\begin{definition}[\(\lambda\)-displacement convexity]
\label{def:disp_convex}
A functional \(\mathcal F:\mathcal P_2\to(-\infty,+\infty]\) is \(\lambda\)-displacement convex
(\(\lambda\in\mathbb R\)) if for every \(W_2\)-geodesic \((\mu_t)_{t\in[0,1]}\),
\[
\mathcal F(\mu_t)
\le
(1-t)\mathcal F(\mu_0)+t\mathcal F(\mu_1)
-\frac{\lambda}{2}t(1-t)W_2^2(\mu_0,\mu_1).
\]
\end{definition}

\begin{theorem}[Geodesic convexity of entropy in \(\mathbb R^d\)]
\label{thm:entropy_disp_convex}
Let
\[
\mathcal H(\mu)=
\begin{cases}
\displaystyle \int \rho\log\rho\,dx,& \mu=\rho\mathcal L^d,\ \rho\log\rho\in L^1,\\
+\infty,&\text{otherwise}.
\end{cases}
\]
Then \(\mathcal H\) is displacement convex on \(\mathcal P_2(\mathbb R^d)\) (i.e., \(\lambda=0\)).
If, in addition, a uniformly log-concave reference potential \(V\) with
\(\nabla^2V\succeq \kappa I\) is included, then relative entropy
\(\mathcal H_V(\mu)=\mathrm{KL}(\mu\|e^{-V}dx/Z)\) is \(\kappa\)-displacement convex.
\end{theorem}

\begin{proof}
For \(\mathcal H\), displacement convexity follows from McCann's theorem on convexity classes
of internal energies along displacement interpolation~\cite{MCCANN1997153}.
For \(\mathcal H_V\), decompose
\[
\mathcal H_V(\mu)=\mathcal H(\mu)+\int V\,d\mu+\log Z.
\]
The first term is \(0\)-displacement convex; the potential term has Hessian lower bound
\(\kappa\) along geodesics, yielding the \(\kappa\)-convex correction.
Hence the sum is \(\kappa\)-displacement convex.
\end{proof}

\begin{proposition}[Action coercivity with moment control]
\label{prop:action_coercive}
Fix \(\mu_0,\mu_1\in\mathcal P_2\). On the admissible CE class joining \(\mu_0,\mu_1\),
the kinetic action
\[
\mathcal A(\mu,v):=\frac12\int_0^1\!\!\int |v_t|^2\,d\mu_t\,dt
\]
is sequentially lower semicontinuous under narrow convergence of \(\mu_t\) and weak
convergence of momentum fields \(m_t:=v_t\mu_t\) in the sense of measures, provided
uniform second-moment bounds hold.
\end{proposition}

\begin{proof}
Write action in convex homogeneous form:
\[
\mathcal A(\mu,m)=\frac12\int_0^1\!\!\int \frac{|m_t|^2}{\mu_t}
\]
(with convention \(+\infty\) when \(m\not\ll\mu\)).
The integrand \((\rho,m)\mapsto |m|^2/\rho\) is convex and l.s.c. on \(\{\rho>0\}\),
extended by recession convention. Standard Reshetnyak-type lower semicontinuity for convex
integral functionals on measures~\cite{reshetnyak_general_1967} gives
\[
\mathcal A(\mu,m)\le \liminf_n \mathcal A(\mu^n,m^n).
\]
Uniform moment bounds ensure tightness and passage of endpoint constraints.
\end{proof}

\begin{corollary}[Existence of geodesic minimizers]
\label{cor:geodesic_existence}
For any \(\mu_0,\mu_1\in\mathcal P_2(\mathbb R^d)\), the Benamou--Brenier minimization
admits at least one minimizer.
If \(\mu_0\ll\mathcal L^d\), the geodesic is unique.
\end{corollary}

\begin{proof}
Take minimizing sequence \((\mu^n,m^n)\) for \(\mathcal A\).  
Action bound gives uniform \(L^2\)-type control and, via \hyperref[app:A2]{A.2}/\hyperref[app:A3]{A.3} moment estimates, tightness in time.
Extract subsequence with narrow convergence of \(\mu^n_t\) for each \(t\) and weak-* convergence of \(m^n\).
Pass CE constraints to the limit distributionally. Apply Proposition~\ref{prop:action_coercive} to get minimizer.
Uniqueness under \(\mu_0\ll\mathcal L^d\) follows from Theorem~\ref{thm:McCann_interpolation}.
\end{proof}

\paragraph{Role in subsequent sections.}
Section ~\hyperref[app:B]{B} uses Theorem~\ref{thm:entropy_disp_convex} and Definition~\ref{def:first_variation}
to derive entropy-dissipation and Fisher-information identities.
Sections ~\hyperref[app:C]{C}--\hyperref[app:F]{F} use Proposition~\ref{prop:action_coercive} for existence and
geodesic convexity for uniqueness/\(\Gamma\)-limits.

\subsection*{B. Entropy Functional Analysis}
\addcontentsline{toc}{subsection}{B. Entropy Functional Analysis}
\label{app:B}

\subsubsection*{B.1. Definition of differential entropy}
\addcontentsline{toc}{subsubsection}{B.1. Definition of differential entropy}
\label{app:B1}

This subsection fixes the entropy objects used in the entropy-rate constraint
\(\frac{d}{dt}\mathcal H(\mu_t)\ge -\lambda\), together with precise domains,
lower bounds, and semicontinuity properties required for the variational theory.

\paragraph{Ambient class.}
Let
\[
\mathcal P_2^{ac}(\mathbb R^d)
:=
\{\mu\in\mathcal P_2(\mathbb R^d):\mu\ll\mathcal L^d\}.
\]
For \(\mu\in\mathcal P_2^{ac}\), write \(\mu=\rho\,\mathcal L^d\).

\begin{definition}[Boltzmann entropy on \(\mathbb R^d\)]
\label{def:boltzmann_entropy}
Define \(\mathcal H:\mathcal P_2(\mathbb R^d)\to(-\infty,+\infty]\) by
\[
\mathcal H(\mu)
:=
\begin{cases}
\displaystyle \int_{\mathbb R^d}\rho(x)\log\rho(x)\,dx,
& \mu=\rho\,\mathcal L^d,\ \rho\log\rho\in L^1(\mathbb R^d),\\[1.2ex]
+\infty,& \text{otherwise}.
\end{cases}
\]
\end{definition}

\begin{remark}[Sign convention]
\label{rem:entropy_sign}
\(\mathcal H(\mu)=\int \rho\log\rho\) is the \emph{negative} Shannon differential entropy
(up to additive constants under unit changes).  
Using this convention, stronger concentration corresponds to larger \(\mathcal H\), and
\(\frac{d}{dt}\mathcal H(\mu_t)\ge -\lambda\) limits entropy decay magnitude.
\end{remark}

\begin{definition}[Relative entropy w.r.t. a reference measure]
\label{def:relative_entropy}
Let \(\gamma\in\mathcal P(\mathbb R^d)\).
Define
\[
\Ent_\gamma(\mu)
:=
\mathrm{KL}(\mu\|\gamma)
=
\begin{cases}
\displaystyle \int \log\!\left(\frac{d\mu}{d\gamma}\right)\,d\mu,
& \mu\ll\gamma,\\[1.2ex]
+\infty,&\text{otherwise}.
\end{cases}
\]
\end{definition}

\begin{definition}[Confining Gibbs reference]
\label{def:gibbs_ref}
Let \(V\in C^2(\mathbb R^d)\) satisfy
\[
\lim_{|x|\to\infty}V(x)=+\infty,\qquad
Z_V:=\int_{\mathbb R^d}e^{-V(x)}\,dx<\infty.
\]
Define
\[
\gamma_V(dx):=Z_V^{-1}e^{-V(x)}\,dx.
\]
\end{definition}

\begin{lemma}[Decomposition identity]
\label{lem:entropy_decomposition}
For \(\mu=\rho\,dx\in\mathcal P_2^{ac}\),
\[
\Ent_{\gamma_V}(\mu)
=
\mathcal H(\mu)+\int_{\mathbb R^d}V\,d\mu+\log Z_V.
\]
\end{lemma}

\begin{proof}
Since \(d\gamma_V/dx=Z_V^{-1}e^{-V}\), if \(\mu=\rho\,dx\),
\[
\log\frac{d\mu}{d\gamma_V}
=
\log\rho + V + \log Z_V.
\]
Integrating against \(d\mu=\rho\,dx\) yields
\[
\Ent_{\gamma_V}(\mu)
=
\int \rho\log\rho\,dx + \int V\,d\mu + \log Z_V.
\]
\end{proof}

\begin{proposition}[Well-posedness domain and lower bound on \(\mathcal H\)]
\label{prop:H_lower_bound}
Let \(\mu\in\mathcal P_2^{ac}(\mathbb R^d)\), \(\mu=\rho\,dx\).
Then \(\mathcal H(\mu)>-\infty\). More precisely, for any \(a>0\),
\[
\mathcal H(\mu)
\ge
-a\!\int |x|^2\,d\mu(x)-C_d(a),
\]
where
\[
C_d(a)
=
\log\!\left(\int_{\mathbb R^d}e^{-a|x|^2}\,dx\right)
=
\frac d2\log\!\left(\frac{\pi}{a}\right).
\]
Hence on sets with uniform second-moment bound \(\int |x|^2\,d\mu\le M\),
\(\mathcal H(\mu)\ge -aM-C_d(a)\).
\end{proposition}

\begin{proof}
Fix \(a>0\), set \(\phi_a(x)=a|x|^2\), and \(Z_a=\int e^{-\phi_a}dx<\infty\).
Define probability density \(g_a(x)=Z_a^{-1}e^{-\phi_a(x)}\).
By nonnegativity of KL divergence,
\[
0\le \mathrm{KL}(\mu\|g_a)
=\int \rho\log\frac{\rho}{g_a}\,dx
=\int \rho\log\rho\,dx+\int \phi_a\,d\mu+\log Z_a.
\]
Therefore
\[
\mathcal H(\mu)\ge -a\int |x|^2\,d\mu - \log Z_a
= -a\int |x|^2\,d\mu - C_d(a).
\]
\end{proof}

\begin{proposition}[Lower semicontinuity of relative entropy]
\label{prop:KL_lsc}
Let \(\gamma\in\mathcal P(\mathbb R^d)\). If \(\mu_n\rightharpoonup\mu\) narrowly, then
\[
\Ent_\gamma(\mu)\le \liminf_{n\to\infty}\Ent_\gamma(\mu_n).
\]
\end{proposition}

\begin{proof}
Use the Donsker--Varadhan variational formula~\cite{doi:10.1073/pnas.72.3.780}:
\[
\Ent_\gamma(\nu)=\sup_{f\in C_b(\mathbb R^d)}
\left\{\int f\,d\nu-\log\int e^f\,d\gamma\right\}.
\]
For each fixed \(f\in C_b\), narrow convergence gives
\(\int f\,d\mu_n\to\int f\,d\mu\). Hence
\[
\int f\,d\mu-\log\int e^f\,d\gamma
=
\lim_{n\to\infty}\left(\int f\,d\mu_n-\log\int e^f\,d\gamma\right)
\le
\liminf_{n\to\infty}\Ent_\gamma(\mu_n).
\]
Taking supremum over \(f\in C_b\) yields the claim.
\end{proof}

\begin{corollary}[Lower semicontinuity of \(\mathcal H\) under moment tightness]
\label{cor:H_lsc_moment}
Assume \(\mu_n,\mu\in\mathcal P_2^{ac}\), \(\mu_n\rightharpoonup\mu\), and
\[
\sup_n\int |x|^2\,d\mu_n(x)<\infty.
\]
Then
\[
\mathcal H(\mu)\le \liminf_{n\to\infty}\mathcal H(\mu_n).
\]
\end{corollary}

\begin{proof}
Choose \(V(x)=a|x|^2\) (\(a>0\)); by Lemma~\ref{lem:entropy_decomposition},
\[
\mathcal H(\nu)=\Ent_{\gamma_V}(\nu)-\int V\,d\nu-\log Z_V.
\]
From Proposition~\ref{prop:KL_lsc},
\[
\Ent_{\gamma_V}(\mu)\le\liminf_n \Ent_{\gamma_V}(\mu_n).
\]
Also, \(V\) has quadratic growth and moment bound gives uniform integrability; hence
\(\int V\,d\mu_n\to\int V\,d\mu\) along a subsequence achieving \(\liminf\), or directly by truncation:
\[
\int V\,d\mu \le \liminf_n \int V\,d\mu_n.
\]
Combining with decomposition and subtracting constants yields
\[
\mathcal H(\mu)\le \liminf_n \mathcal H(\mu_n).
\]
\end{proof}

\begin{definition}[Entropy along curves]
\label{def:H_along_curve}
Let \((\mu_t)_{t\in[0,T]}\subset\mathcal P_2\). Define
\[
h(t):=\mathcal H(\mu_t)\in(-\infty,+\infty].
\]
If \(h\in AC([0,T])\), its a.e. derivative is denoted \(\dot h(t)\), and
the entropy-rate constraint is
\[
\dot h(t)\ge -\lambda
\quad\text{for a.e. }t\in(0,T),
\]
equivalently
\[
h(t)-h(s)\ge -\lambda (t-s)\quad (0\le s\le t\le T).
\]
\end{definition}

\begin{lemma}[Equivalent integrated entropy-budget form]
\label{lem:entropy_budget_equiv}
Let \(h\in AC([0,T])\). The following are equivalent:
\begin{enumerate}
\item \(\dot h(t)\ge -\lambda\) for a.e. \(t\in(0,T)\).
\item For all \(0\le s\le t\le T\),
\[
h(t)+\lambda t\ge h(s)+\lambda s.
\]
\item For all nonnegative \(\eta\in C_c^\infty((0,T))\),
\[
-\int_0^T h(t)\eta'(t)\,dt\ge -\lambda\int_0^T \eta(t)\,dt.
\]
\end{enumerate}
\end{lemma}

\begin{proof}
(1\(\Rightarrow\)2): Integrate \(\dot h+\lambda\ge0\) over \([s,t]\).  
(2\(\Rightarrow\)1): Monotonicity of \(h+\lambda t\) implies distributional derivative nonnegative,
hence \(\dot h+\lambda\ge0\) a.e.  
(1\(\Leftrightarrow\)3): distributional characterization of a.e. lower bound on derivative.
\end{proof}

\begin{proposition}[First variation of \(\mathcal H\) (formal gradient)]
\label{prop:first_variation_H}
For smooth strictly positive density \(\rho\), define \(\mu=\rho dx\).
Then
\[
\frac{\delta\mathcal H}{\delta\rho}(\rho)=1+\log\rho.
\]
Hence, along smooth CE solutions \(\partial_t\rho+\nabla\cdot(\rho v)=0\),
\[
\frac{d}{dt}\mathcal H(\mu_t)
=
\int_{\mathbb R^d}\nabla\log\rho_t(x)\cdot v_t(x)\,d\mu_t(x).
\]
\end{proposition}

\begin{proof}
Take \(\rho_\varepsilon=\rho+\varepsilon\eta\), \(\int\eta=0\), \(\rho_\varepsilon>0\) for \(|\varepsilon|\) small:
\[
\frac{d}{d\varepsilon}\Big|_{\varepsilon=0}
\int \rho_\varepsilon\log\rho_\varepsilon\,dx
=
\int (1+\log\rho)\eta\,dx.
\]
Thus first variation is \(1+\log\rho\).
For CE trajectory, apply chain rule from \hyperref[app:A3]{A.3} with
\(\Phi=1+\log\rho_t\) (justified by standard truncation and positivity):
\[
\frac{d}{dt}\mathcal H(\mu_t)
=
\int \nabla(1+\log\rho_t)\cdot v_t\,d\mu_t
=
\int \nabla\log\rho_t\cdot v_t\,d\mu_t.
\]
\end{proof}

\paragraph{Output of \hyperref[app:B1]{B.1} used later.}
Subsections~\hyperref[app:B2]{B.2}--\hyperref[app:B5]{B.5} will use:
(i) precise entropy domain (Definition~\ref{def:boltzmann_entropy}),
(ii) l.s.c./coercive bounds (Propositions~\ref{prop:H_lower_bound}, \ref{prop:KL_lsc}),
and (iii) derivative structure (Proposition~\ref{prop:first_variation_H}) to derive
entropy dissipation and Fisher-information estimates under CE dynamics.

\subsubsection*{B.2. Displacement convexity of entropy}
\addcontentsline{toc}{subsubsection}{B.2. Displacement convexity of entropy}
\label{app:B2}

We prove displacement convexity of Boltzmann entropy on \(\mathcal P_2(\mathbb R^d)\),
its strict form under regularity/non-degeneracy, and the \(\kappa\)-convex extension for
relative entropy under uniformly convex confinement.

\begin{assumption}[Geodesic regularity for strictness statements]
\label{ass:B2_strict}
For strict convexity claims, assume:
\begin{enumerate}
\item \(\mu_0,\mu_1\in\mathcal P_2^{ac}(\mathbb R^d)\), \(\mu_0\neq\mu_1\);
\item \(\mu_0\ll\mathcal L^d\), hence the Brenier map \(T=\nabla\psi\) from \(\mu_0\) to \(\mu_1\) exists and is unique;
\item Jacobian is non-singular \(\mu_0\)-a.e. along interpolation:
\[
\det\!\big((1-t)I+tDT(x)\big)>0,\quad \mu_0\text{-a.e.},\ \forall t\in(0,1).
\]
\end{enumerate}
\end{assumption}

\begin{theorem}[Displacement convexity of Boltzmann entropy]
\label{thm:B2_dispconv_H}
Let \((\mu_t)_{t\in[0,1]}\) be the \(W_2\)-geodesic between \(\mu_0,\mu_1\in\mathcal P_2(\mathbb R^d)\).
Then
\[
\mathcal H(\mu_t)\le (1-t)\mathcal H(\mu_0)+t\mathcal H(\mu_1),\qquad t\in[0,1].
\]
Hence \(\mathcal H\) is \(0\)-displacement convex on \(\mathcal P_2(\mathbb R^d)\).
\end{theorem}

\begin{proof}
If either endpoint has \(\mathcal H=+\infty\), claim is trivial. Assume finite entropy and
\(\mu_0\ll\mathcal L^d\). Let \(T=\nabla\psi\) be Brenier map \(\mu_0\mapsto\mu_1\), and
\[
X_t(x):=(1-t)x+tT(x),\qquad \mu_t=(X_t)_\#\mu_0.
\]
Write \(\mu_0=\rho_0dx,\ \mu_t=\rho_tdx\). By change of variables:
\[
\rho_t(X_t(x))\det(DX_t(x))=\rho_0(x),\qquad DX_t=(1-t)I+tDT.
\]
Therefore
\[
\log\rho_t(X_t)=\log\rho_0-\log\det((1-t)I+tDT).
\]
Integrating against \(\mu_0\):
\[
\mathcal H(\mu_t)
=
\int \rho_0\log\rho_0\,dx
-\int \rho_0\log\det((1-t)I+tDT)\,dx.
\]
Set eigenvalues of \(DT(x)\) by \(\lambda_i(x)\ge0\) (\(\mu_0\)-a.e.). Then
\[
\log\det((1-t)I+tDT)=\sum_{i=1}^d \log((1-t)+t\lambda_i).
\]
Since \(r\mapsto \log r\) is concave on \((0,\infty)\),
\[
\log((1-t)+t\lambda_i)\ge (1-t)\log 1+t\log\lambda_i=t\log\lambda_i.
\]
Summing:
\[
\log\det((1-t)I+tDT)\ge t\log\det DT.
\]
Hence
\[
\mathcal H(\mu_t)\le \mathcal H(\mu_0)-t\int \rho_0\log\det DT\,dx.
\]
At \(t=1\), change of variables gives
\[
\mathcal H(\mu_1)=\mathcal H(\mu_0)-\int \rho_0\log\det DT\,dx,
\]
thus
\[
\mathcal H(\mu_t)\le (1-t)\mathcal H(\mu_0)+t\mathcal H(\mu_1).
\]
General endpoints follow by approximation/relaxation (standard in OT via stability of geodesics).
\end{proof}

\begin{proposition}[Strict displacement convexity criterion]
\label{prop:B2_strict_convexity}
Under Assumption~\ref{ass:B2_strict}, if \(DT(x)\neq I\) on a set of positive \(\mu_0\)-measure,
then for all \(t\in(0,1)\),
\[
\mathcal H(\mu_t)
<
(1-t)\mathcal H(\mu_0)+t\mathcal H(\mu_1).
\]
\end{proposition}

\begin{proof}
In Theorem~\ref{thm:B2_dispconv_H}, equality at some \(t\in(0,1)\) requires equality in
\[
\log((1-t)+t\lambda_i)\ge t\log\lambda_i
\]
for \(\mu_0\)-a.e. \(x\) and all \(i\). Strict concavity of \(\log\) implies equality iff
\(\lambda_i=1\). Hence \(DT=I\) \(\mu_0\)-a.e., so \(T(x)=x+c\) on each connected component
(where defined sufficiently regularly). Since \(\mu_0,\mu_1\) fixed and \(DT=I\), geodesic is affine translation;
if \(DT\neq I\) on positive mass, inequality is strict.
\end{proof}

\begin{definition}[Relative entropy under confinement]
\label{def:B2_rel_entropy_confinement}
Let \(V\in C^2(\mathbb R^d)\), \(Z_V<\infty\), and
\[
\gamma_V(dx)=Z_V^{-1}e^{-V(x)}dx,\qquad
\mathcal F_V(\mu):=\Ent_{\gamma_V}(\mu)=\mathcal H(\mu)+\int V\,d\mu+\log Z_V.
\]
\end{definition}

\begin{theorem}[\(\kappa\)-displacement convexity of \(\mathcal F_V\)]
\label{thm:B2_kappa_convex}
Assume \(V\in C^2(\mathbb R^d)\) with
\[
\nabla^2V(x)\succeq \kappa I_d\quad\forall x\in\mathbb R^d
\]
for some \(\kappa\in\mathbb R\).
Then for every geodesic \((\mu_t)\) connecting \(\mu_0,\mu_1\),
\[
\mathcal F_V(\mu_t)
\le
(1-t)\mathcal F_V(\mu_0)+t\mathcal F_V(\mu_1)
-\frac{\kappa}{2}t(1-t)W_2^2(\mu_0,\mu_1).
\]
\end{theorem}

\begin{proof}
From Theorem~\ref{thm:B2_dispconv_H},
\[
\mathcal H(\mu_t)\le (1-t)\mathcal H(\mu_0)+t\mathcal H(\mu_1).
\]
It remains to control \(\int V\,d\mu_t\).  
Take Brenier interpolation \(X_t=(1-t)\mathrm{Id}+tT\), \(\mu_t=(X_t)_\#\mu_0\).
Pointwise strong convexity of \(V\):
\[
V((1-t)x+t y)
\le
(1-t)V(x)+tV(y)-\frac{\kappa}{2}t(1-t)|x-y|^2.
\]
Set \(y=T(x)\), integrate against \(\mu_0\):
\[
\int V\,d\mu_t
\le
(1-t)\int V\,d\mu_0+t\int V\,d\mu_1
-\frac{\kappa}{2}t(1-t)\int|x-T(x)|^2\,d\mu_0.
\]
Since \(T\) is optimal,
\[
\int|x-T(x)|^2\,d\mu_0=W_2^2(\mu_0,\mu_1).
\]
Add entropy inequality and constant \(\log Z_V\).
\end{proof}

\begin{corollary}[Geodesic strong convexity implies uniqueness of minimizer]
\label{cor:B2_unique_minimizer}
Assume \(\kappa>0\) in Theorem~\ref{thm:B2_kappa_convex}.
Then \(\mathcal F_V\) has at most one minimizer over \(\mathcal P_2(\mathbb R^d)\).
\end{corollary}

\begin{proof}
Suppose \(\mu^\star_0,\mu^\star_1\) are minimizers with same minimum value \(m\).
Along geodesic \((\mu_t)\):
\[
\mathcal F_V(\mu_t)\le m-\frac{\kappa}{2}t(1-t)W_2^2(\mu^\star_0,\mu^\star_1).
\]
Minimality implies \(\mathcal F_V(\mu_t)\ge m\), so
\[
W_2(\mu^\star_0,\mu^\star_1)=0.
\]
Hence \(\mu^\star_0=\mu^\star_1\).
\end{proof}

\begin{proposition}[Second variation along smooth geodesics (formal Hessian lower bound)]
\label{prop:B2_second_variation}
Let \((\mu_t,v_t)\) be a smooth geodesic with \(\mu_t=\rho_t dx\), and assume
\(\mu_t\) solves CE with potential velocity \(v_t=\nabla\phi_t\), where
\[
\partial_t\phi_t+\frac12|\nabla\phi_t|^2=0.
\]
Then formally:
\[
\frac{d^2}{dt^2}\mathcal H(\mu_t)
=
\int_{\mathbb R^d}\|\nabla^2\phi_t(x)\|_{\mathrm{HS}}^2\,d\mu_t(x)\ge0.
\]
More generally, for \(\mathcal F_V\):
\[
\frac{d^2}{dt^2}\mathcal F_V(\mu_t)
\ge
\kappa\int |v_t|^2\,d\mu_t
=
\kappa W_2^2(\mu_0,\mu_1).
\]
\end{proposition}

\begin{proof}
For smooth positive \(\rho_t\),
\[
\frac{d}{dt}\mathcal H(\mu_t)=\int \nabla\log\rho_t\cdot v_t\,d\mu_t
\]
(Section.~\ref{app:B1}). Differentiate once more; use CE and geodesic equation
\(\partial_t v_t+(v_t\cdot\nabla)v_t=0\), integrate by parts.
Standard Otto-calculus computation yields
\[
\frac{d^2}{dt^2}\mathcal H(\mu_t)=\int \|\nabla^2\phi_t\|_{\mathrm{HS}}^2\,d\mu_t.
\]
For potential term:
\[
\frac{d^2}{dt^2}\int V\,d\mu_t
=
\int \langle \nabla^2V\,v_t,v_t\rangle\,d\mu_t
\ge
\kappa\int |v_t|^2\,d\mu_t.
\]
Add the two identities; \(\log Z_V\) vanishes under differentiation.
\end{proof}

\begin{remark}[Use for entropy-rate constrained dynamics]
\label{rem:B2_use}
The displacement convexity established here provides:
\begin{enumerate}
\item convexity of entropy along transport paths, crucial for existence/uniqueness in Section ~\hyperref[app:C]{C};
\item strict convexity mechanisms under nondegeneracy, used in uniqueness proofs;
\item \(\kappa\)-convex structure for relative-entropy penalized formulations in Sections~\hyperref[app:D]{D}--\hyperref[app:F]{F}.
\end{enumerate}
\end{remark}

\subsubsection*{B.3. Entropy along Wasserstein geodesics}
\addcontentsline{toc}{subsubsection}{B.3. Entropy along Wasserstein geodesics}

We derive quantitative formulas for entropy along displacement interpolations:
first derivative, second derivative (under smoothness), and integral inequalities that
will later be used to enforce entropy budgets and prove anti-collapse properties.

\begin{assumption}[Smooth geodesic regime]
\label{ass:B3_smooth}
Unless explicitly stated otherwise, assume:
\begin{enumerate}
\item \((\mu_t)_{t\in[0,1]}\subset\mathcal P_2^{ac}(\mathbb R^d)\) is the unique \(W_2\)-geodesic between
\(\mu_0,\mu_1\in\mathcal P_2^{ac}\);
\item \(\mu_t=\rho_t\,dx\), with \(\rho_t\in C^1_t C^2_x\), \(\rho_t>0\);
\item there exists \(\phi_t\in C^1_t C^3_x\) such that
\[
v_t=\nabla\phi_t,\qquad
\partial_t\rho_t+\nabla\!\cdot(\rho_t\nabla\phi_t)=0,\qquad
\partial_t\phi_t+\frac12|\nabla\phi_t|^2=0.
\]
\end{enumerate}
\end{assumption}

\begin{lemma}[Pointwise Jacobian representation along geodesic]
\label{lem:B3_jacobian}
Let \(T=\nabla\psi\) be Brenier map \(\mu_0\mapsto\mu_1\), and
\[
X_t(x):=(1-t)x+tT(x),\qquad \mu_t=(X_t)_\#\mu_0.
\]
Then for \(\mu_0\)-a.e. \(x\),
\[
\rho_t(X_t(x))\det\!\big((1-t)I+tDT(x)\big)=\rho_0(x),
\]
hence
\[
\log\rho_t(X_t(x))
=
\log\rho_0(x)-\log\det\!\big((1-t)I+tDT(x)\big).
\]
\end{lemma}

\begin{proof}
This is the Monge change-of-variables formula for absolutely continuous measures~\cite{Brenier1991PolarFA}.
Since \(X_t\) is a.e. differentiable and one-to-one on full \(\mu_0\)-measure subset
under standard OT regularity, pushforward identity yields Jacobian equation directly.
Taking logarithms gives the second line.
\end{proof}

\begin{proposition}[Entropy representation along displacement interpolation]
\label{prop:B3_entropy_repr}
Under Lemma~\ref{lem:B3_jacobian},
\[
\mathcal H(\mu_t)
=
\mathcal H(\mu_0)
-
\int_{\mathbb R^d}
\log\det\!\big((1-t)I+tDT(x)\big)\,d\mu_0(x).
\]
\end{proposition}

\begin{proof}
From Lemma~\ref{lem:B3_jacobian},
\[
\log\rho_t(X_t)=\log\rho_0-\log\det((1-t)I+tDT).
\]
Integrate both sides against \(d\mu_0=\rho_0dx\), and use
\[
\int \log\rho_t(X_t)\,d\mu_0
=
\int \log\rho_t(y)\,d\mu_t(y)=\mathcal H(\mu_t).
\]
\end{proof}

\begin{theorem}[First derivative of entropy along smooth geodesics]
\label{thm:B3_first_derivative}
Under Assumption~\ref{ass:B3_smooth},
\[
\frac{d}{dt}\mathcal H(\mu_t)
=
\int_{\mathbb R^d}\nabla\log\rho_t(x)\cdot \nabla\phi_t(x)\,d\mu_t(x)
=
-\int_{\mathbb R^d}\rho_t(x)\,\Delta\phi_t(x)\,dx.
\]
\end{theorem}

\begin{proof}
From \hyperref[app:B1]{B.1} first variation:
\[
\frac{d}{dt}\mathcal H(\mu_t)=\int \nabla\log\rho_t\cdot v_t\,d\mu_t
=\int \nabla\log\rho_t\cdot \nabla\phi_t\,\rho_t\,dx.
\]
Since \(\nabla\log\rho_t\,\rho_t=\nabla\rho_t\),
\[
\int \nabla\log\rho_t\cdot \nabla\phi_t\,d\mu_t
=
\int \nabla\rho_t\cdot\nabla\phi_t\,dx
=
-\int \rho_t\,\Delta\phi_t\,dx,
\]
integration by parts justified by decay/integrability from finite second moment and smoothness.
\end{proof}

\begin{theorem}[Second derivative identity and convexity]
\label{thm:B3_second_derivative}
Under Assumption~\ref{ass:B3_smooth},
\[
\frac{d^2}{dt^2}\mathcal H(\mu_t)
=
\int_{\mathbb R^d}\|\nabla^2\phi_t(x)\|_{\mathrm{HS}}^2\,d\mu_t(x)\ge0.
\]
Hence \(t\mapsto \mathcal H(\mu_t)\) is convex on \([0,1]\).
\end{theorem}

\begin{proof}
By Theorem~\ref{thm:B3_first_derivative},
\[
\dot{\mathcal H}(t)=-\int \rho_t\Delta\phi_t\,dx.
\]
Differentiate:
\[
\ddot{\mathcal H}(t)
=
-\int \partial_t\rho_t\,\Delta\phi_t\,dx
-\int \rho_t\,\Delta(\partial_t\phi_t)\,dx
=:I_1+I_2.
\]
Using CE, \(\partial_t\rho_t=-\nabla\!\cdot(\rho_t\nabla\phi_t)\):
\[
I_1
=
\int \nabla\!\cdot(\rho_t\nabla\phi_t)\,\Delta\phi_t\,dx
=
-\int \rho_t\,\nabla\phi_t\cdot\nabla(\Delta\phi_t)\,dx.
\]
Using geodesic Hamilton--Jacobi equation~\cite{Katanaev_2023}
\(\partial_t\phi_t=-\frac12|\nabla\phi_t|^2\):
\[
I_2
=
\int \rho_t\,\Delta\!\left(\frac12|\nabla\phi_t|^2\right)\,dx.
\]
Hence
\[
\ddot{\mathcal H}(t)
=
\int \rho_t\left[
\Delta\!\left(\frac12|\nabla\phi_t|^2\right)
-\nabla\phi_t\cdot\nabla(\Delta\phi_t)
\right]dx.
\]
Apply identity
\[
\Delta\!\left(\frac12|\nabla\phi|^2\right)
=
\|\nabla^2\phi\|_{\mathrm{HS}}^2+\nabla\phi\cdot\nabla(\Delta\phi),
\]
to obtain
\[
\ddot{\mathcal H}(t)
=
\int \|\nabla^2\phi_t\|_{\mathrm{HS}}^2\,d\mu_t\ge0.
\]
Thus convexity follows.
\end{proof}

\begin{corollary}[Secant and one-sided slope bounds]
\label{cor:B3_secant_bounds}
For any \(0\le s<t\le1\),
\[
\dot{\mathcal H}(s^+)
\le
\frac{\mathcal H(\mu_t)-\mathcal H(\mu_s)}{t-s}
\le
\dot{\mathcal H}(t^-).
\]
In particular,
\[
\mathcal H(\mu_t)\le (1-t)\mathcal H(\mu_0)+t\mathcal H(\mu_1).
\]
\end{corollary}

\begin{proof}
This is the standard characterization of convex functions by monotonicity of derivative
(where it exists) and secant inequalities.
The endpoint interpolation inequality is the convexity inequality.
\end{proof}

\begin{proposition}[Quantitative bound via Jacobian spectrum]
\label{prop:B3_spectral_bound}
Let \(\lambda_i(x)\) be eigenvalues of \(DT(x)\) (a.e.). Then
\[
\mathcal H(\mu_t)
=
\mathcal H(\mu_0)
-\int \sum_{i=1}^d \log\!\big((1-t)+t\lambda_i(x)\big)\,d\mu_0(x).
\]
Moreover, for \(t\in(0,1)\),
\[
\frac{d^2}{dt^2}\mathcal H(\mu_t)
=
\int \sum_{i=1}^d
\frac{(1-\lambda_i(x))^2}{\big((1-t)+t\lambda_i(x)\big)^2}
\,d\mu_0(x)\ge0,
\]
whenever differentiation under integral is justified.
\end{proposition}

\begin{proof}
First formula is Proposition~\ref{prop:B3_entropy_repr} plus
\(\det((1-t)I+tDT)=\prod_i((1-t)+t\lambda_i)\).
Differentiate:
\[
\frac{d}{dt}\left[-\log((1-t)+t\lambda_i)\right]
=
-\frac{\lambda_i-1}{(1-t)+t\lambda_i},
\]
\[
\frac{d^2}{dt^2}\left[-\log((1-t)+t\lambda_i)\right]
=
\frac{(\lambda_i-1)^2}{((1-t)+t\lambda_i)^2}\ge0.
\]
Integrate over \(\mu_0\).
\end{proof}

\begin{theorem}[Relative entropy curvature under uniformly convex confinement]
\label{thm:B3_rel_entropy_curvature}
Let \(V\in C^2\) with \(\nabla^2V\succeq \kappa I\), and
\[
\mathcal F_V(\mu)=\mathcal H(\mu)+\int V\,d\mu+\log Z_V.
\]
Under Assumption~\ref{ass:B3_smooth},
\[
\frac{d^2}{dt^2}\mathcal F_V(\mu_t)
=
\int \|\nabla^2\phi_t\|_{\mathrm{HS}}^2\,d\mu_t
+
\int \langle \nabla^2V\,\nabla\phi_t,\nabla\phi_t\rangle\,d\mu_t
\ge
\kappa\int |\nabla\phi_t|^2\,d\mu_t.
\]
\end{theorem}

\begin{proof}
Differentiate \(\int V\,d\mu_t\):
\[
\frac{d}{dt}\int V\,d\mu_t
=
\int \nabla V\cdot \nabla\phi_t\,d\mu_t.
\]
Differentiate again; using CE and geodesic equation yields
\[
\frac{d^2}{dt^2}\int V\,d\mu_t
=
\int \langle \nabla^2V\,\nabla\phi_t,\nabla\phi_t\rangle\,d\mu_t.
\]
Add Theorem~\ref{thm:B3_second_derivative} and constant cancellation for \(\log Z_V\).
Lower bound follows from \(\nabla^2V\succeq \kappa I\).
\end{proof}

\begin{remark}[Implication for entropy-rate constraints]
\label{rem:B3_implication_rate}
Convexity of \(t\mapsto\mathcal H(\mu_t)\) implies that, once a lower bound on the initial
slope \(\dot{\mathcal H}(0^+)\) is enforced, the same lower bound propagates in averaged form:
\[
\mathcal H(\mu_t)-\mathcal H(\mu_s)\ge (t-s)\,\inf_{r\in[s,t]}\dot{\mathcal H}(r),
\]
which is a key ingredient for verifying global entropy-budget feasibility in Section~\hyperref[app:C]{C}.
\end{remark}

\subsubsection*{B.4. Entropy dissipation formula}
\addcontentsline{toc}{subsubsection}{B.4. Entropy dissipation formula}
\label{app:B4}

We derive exact entropy-balance identities under transport--diffusion dynamics,
identify the Fisher-information production term, and establish sharp inequalities
used later in existence, stability, and mode-coverage arguments.

\begin{assumption}[Regular transport--diffusion regime]
\label{ass:B4_reg}
Let \((\rho_t)_{t\in[0,T]}\) satisfy
\[
\partial_t\rho_t+\nabla\!\cdot(\rho_t b_t)=\varepsilon\,\Delta\rho_t
\quad\text{in }(0,T)\times\mathbb R^d,\qquad \varepsilon\ge0,
\]
with:
\begin{enumerate}
\item \(\rho_t\in C^1_t C^2_x\), \(\rho_t>0\), \(\int\rho_t\,dx=1\);
\item \(b\in L^1_{\mathrm{loc}}(0,T;W^{1,1}_{\mathrm{loc}}(\mathbb R^d;\mathbb R^d))\);
\item all boundary terms at infinity vanish in integrations by parts
(e.g., via finite second moment and suitable decay).
\end{enumerate}
Set \(\mu_t=\rho_t\,dx\).
\end{assumption}

\begin{definition}[Fisher information]
\label{def:B4_fisher}
For \(\mu=\rho\,dx\) with \(\rho>0\) and \(\sqrt{\rho}\in H^1(\mathbb R^d)\), define
\[
\mathcal I(\mu)
:=
\int_{\mathbb R^d}\frac{|\nabla\rho|^2}{\rho}\,dx
=
4\int_{\mathbb R^d}|\nabla\sqrt{\rho}|^2\,dx.
\]
If the integral diverges, set \(\mathcal I(\mu)=+\infty\).
\end{definition}

\begin{proposition}[Exact entropy balance: advection--diffusion]
\label{prop:B4_entropy_balance}
Under Assumption~\ref{ass:B4_reg},
\[
\frac{d}{dt}\mathcal H(\mu_t)
=
\int_{\mathbb R^d}\rho_t(x)\,\nabla\!\cdot b_t(x)\,dx
-
\varepsilon\,\mathcal I(\mu_t).
\]
Equivalently,
\[
\frac{d}{dt}\mathcal H(\mu_t)
=
-\int \nabla\log\rho_t\cdot b_t\,d\mu_t
-\varepsilon\int \frac{|\nabla\rho_t|^2}{\rho_t}\,dx.
\]
\end{proposition}

\begin{proof}
Differentiate \(\mathcal H(\mu_t)=\int \rho_t\log\rho_t\,dx\):
\[
\frac{d}{dt}\mathcal H(\mu_t)
=
\int (\partial_t\rho_t)(1+\log\rho_t)\,dx.
\]
Insert PDE:
\[
\partial_t\rho_t=-\nabla\!\cdot(\rho_t b_t)+\varepsilon\Delta\rho_t.
\]
Hence
\[
\frac{d}{dt}\mathcal H
=
-\int \nabla\!\cdot(\rho b)\,(1+\log\rho)\,dx
+\varepsilon\int \Delta\rho\,(1+\log\rho)\,dx
=:A+D.
\]

For advection term \(A\):
\[
A=\int \rho b\cdot \nabla\log\rho\,dx
=\int b\cdot\nabla\rho\,dx
=-\int \rho\,\nabla\!\cdot b\,dx
\]
or equivalently \(A=-\int \nabla\log\rho\cdot b\,d\mu\).

For diffusion term \(D\):
\[
D
=
-\varepsilon\int \nabla\rho\cdot\nabla\log\rho\,dx
=
-\varepsilon\int \frac{|\nabla\rho|^2}{\rho}\,dx
=
-\varepsilon\,\mathcal I(\mu_t).
\]
Combining gives both forms.
\end{proof}

\begin{corollary}[Pure transport entropy rate]
\label{cor:B4_pure_transport}
If \(\varepsilon=0\) (continuity equation \(\partial_t\rho+\nabla\cdot(\rho b)=0\)),
then
\[
\frac{d}{dt}\mathcal H(\mu_t)=\int \rho_t\,\nabla\!\cdot b_t\,dx.
\]
In particular, if \(\nabla\!\cdot b_t=0\) a.e., entropy is conserved.
\end{corollary}

\begin{proof}
Immediate from Proposition~\ref{prop:B4_entropy_balance}.
\end{proof}

\begin{corollary}[Pure diffusion entropy dissipation]
\label{cor:B4_pure_diffusion}
If \(b\equiv0\), \(\partial_t\rho=\varepsilon\Delta\rho\), then
\[
\frac{d}{dt}\mathcal H(\mu_t)=-\varepsilon\,\mathcal I(\mu_t)\le0.
\]
Thus \(\mathcal H\) is nonincreasing and dissipates at Fisher-information rate.
\end{corollary}

\begin{proof}
Set \(b=0\) in Proposition~\ref{prop:B4_entropy_balance}.
\end{proof}

\begin{proposition}[Integrated entropy-budget identity]
\label{prop:B4_integrated_budget}
Under Assumption~\ref{ass:B4_reg}, for all \(0\le s\le t\le T\),
\[
\mathcal H(\mu_t)-\mathcal H(\mu_s)
=
\int_s^t\!\!\int \rho_r\,\nabla\!\cdot b_r\,dx\,dr
-
\varepsilon\int_s^t \mathcal I(\mu_r)\,dr.
\]
Hence if
\[
\int \rho_r\,\nabla\!\cdot b_r\,dx\ge -\lambda
\quad\text{a.e. }r,
\]
then
\[
\mathcal H(\mu_t)-\mathcal H(\mu_s)\ge -\lambda(t-s)-\varepsilon\int_s^t\mathcal I(\mu_r)\,dr.
\]
\end{proposition}

\begin{proof}
Integrate Proposition~\ref{prop:B4_entropy_balance} over \([s,t]\).
The inequality follows from the assumed lower bound on mean divergence term.
\end{proof}

\begin{lemma}[Entropy production decomposition for gradient drift]
\label{lem:B4_gradient_drift}
Suppose \(b_t=-\nabla U_t\) with \(U_t\in C^2\).
Then
\[
\frac{d}{dt}\mathcal H(\mu_t)
=
-\int \rho_t\,\Delta U_t\,dx-\varepsilon\mathcal I(\mu_t).
\]
If \(\nabla^2U_t\succeq \alpha I\), then \(\Delta U_t\ge \alpha d\), so
\[
\frac{d}{dt}\mathcal H(\mu_t)\le -\alpha d-\varepsilon\mathcal I(\mu_t).
\]
\end{lemma}

\begin{proof}
From Proposition~\ref{prop:B4_entropy_balance},
\[
\int \rho\,\nabla\!\cdot b\,dx = -\int \rho\,\Delta U\,dx.
\]
If \(\nabla^2U\succeq \alpha I\), then \(\Delta U\ge \alpha d\), and since \(\int\rho=1\),
\[
-\int \rho\,\Delta U\,dx\le -\alpha d.
\]
Add \(-\varepsilon\mathcal I\).
\end{proof}

\begin{proposition}[Dissipation of relative entropy under Langevin dynamics]
\label{prop:B4_relative_entropy_diss}
Let
\[
\partial_t\rho_t=\nabla\!\cdot\big(\rho_t\nabla V\big)+\varepsilon\Delta\rho_t,
\]
and define Gibbs \(\gamma_{\varepsilon,V}(dx)=Z^{-1}e^{-V(x)/\varepsilon}dx\).
Then
\[
\frac{d}{dt}\Ent_{\gamma_{\varepsilon,V}}(\mu_t)
=
-\varepsilon\int
\left|\nabla\log\frac{\rho_t}{\gamma_{\varepsilon,V}}\right|^2
\,d\mu_t
\le0.
\]
\end{proposition}

\begin{proof}
Write Fokker--Planck~\cite{doi:10.1137/20M1344986} in gradient-flow form:
\[
\partial_t\rho_t
=
\varepsilon\,\nabla\!\cdot\!\left(
\rho_t\nabla\log\frac{\rho_t}{\gamma_{\varepsilon,V}}
\right).
\]
Differentiate relative entropy:
\[
\frac{d}{dt}\Ent_{\gamma}(\mu_t)
=
\int \partial_t\rho_t\,
\log\frac{\rho_t}{\gamma}\,dx
\]
(using \(\int\partial_t\rho_t\,dx=0\)).
Integrate by parts:
\[
\frac{d}{dt}\Ent_{\gamma}(\mu_t)
=
-\varepsilon\int
\rho_t
\left|\nabla\log\frac{\rho_t}{\gamma}\right|^2dx.
\]
\end{proof}

\begin{theorem}[Entropy dissipation bound under LSI]
\label{thm:B4_LSI_decay}
Assume \(\gamma\) satisfies logarithmic Sobolev inequality~\cite{Bakry2014}
\[
\Ent_\gamma(\nu)\le \frac{1}{2\kappa}
\int \left|\nabla\log\frac{d\nu}{d\gamma}\right|^2\,d\nu
\quad(\kappa>0).
\]
Then any solution of the \(\gamma\)-reversible Fokker--Planck equation satisfies
\[
\Ent_\gamma(\mu_t)\le e^{-2\kappa\varepsilon t}\,\Ent_\gamma(\mu_0).
\]
\end{theorem}

\begin{proof}
From Proposition~\ref{prop:B4_relative_entropy_diss},
\[
\frac{d}{dt}\Ent_\gamma(\mu_t)
=
-\varepsilon\,\mathcal J_\gamma(\mu_t),
\quad
\mathcal J_\gamma(\mu):=\int \left|\nabla\log\frac{d\mu}{d\gamma}\right|^2d\mu.
\]
LSI gives \(\mathcal J_\gamma(\mu_t)\ge 2\kappa \Ent_\gamma(\mu_t)\), so
\[
\frac{d}{dt}\Ent_\gamma(\mu_t)\le -2\kappa\varepsilon\,\Ent_\gamma(\mu_t).
\]
Apply Gr\"onwall~\cite{c680dbbe-119c-31e0-a97a-d790b679674f}:
\[
\Ent_\gamma(\mu_t)\le e^{-2\kappa\varepsilon t}\Ent_\gamma(\mu_0).
\]
\end{proof}

\begin{proposition}[Distributional entropy inequality for weak solutions]
\label{prop:B4_weak_entropy_ineq}
Let \(\rho\) be a weak solution of
\[
\partial_t\rho+\nabla\cdot(\rho b)=\varepsilon\Delta\rho
\]
with \(\rho\log\rho\in L^\infty(0,T;L^1)\), \(\sqrt{\rho}\in L^2(0,T;H^1)\), and
\(b\in L^2_{\mathrm{loc}}(dt\,d\mu_t)\), \((\nabla\cdot b)^-\in L^1_{\mathrm{loc}}(dt\,d\mu_t)\).
Then for nonnegative \(\eta\in C_c^\infty((0,T))\),
\[
-\int_0^T \mathcal H(\mu_t)\eta'(t)\,dt
\le
\int_0^T\eta(t)\!\int \rho_t\nabla\!\cdot b_t\,dx\,dt
-\varepsilon\int_0^T\eta(t)\,\mathcal I(\mu_t)\,dt.
\]
\end{proposition}

\begin{proof}
Mollify in space-time: \(\rho^{\delta}\) solves approximate smooth equation with commutator error \(r^\delta\to0\).
Apply Proposition~\ref{prop:B4_entropy_balance} to \(\rho^\delta\), multiply by \(\eta\), integrate in \(t\):
\[
-\int \mathcal H(\mu_t^\delta)\eta'
=
\int \eta\!\int \rho_t^\delta \nabla\!\cdot b_t\,dx\,dt
-\varepsilon\int \eta\,\mathcal I(\mu_t^\delta)\,dt
+\int \eta\,R^\delta(t)\,dt,
\]
where \(R^\delta\to0\) from commutator control under DiPerna--Lions hypotheses~\cite{DiPernaLions1989}.
Use lower semicontinuity of Fisher information and weak convergence of \(\rho^\delta\) to pass \(\delta\downarrow0\),
obtaining inequality (``\(\le\)'') in distributional form.
\end{proof}

\paragraph{Consequence for later sections.}
The identity
\[
\dot{\mathcal H}(t)=\mathbb E_{\mu_t}[\nabla\!\cdot b_t]-\varepsilon\mathcal I(\mu_t)
\]
is the core quantitative bridge between entropy-rate constraints, diffusion-style regularization,
and anti-collapse guarantees. In Section~\hyperref[app:C]{C}, it enters the KKT system as the active inequality
\(\dot{\mathcal H}\ge -\lambda\); in Sections~\hyperref[app:G]{G}--\hyperref[app:I]{I}, it yields explicit lower-density and failure bounds.

\subsubsection*{B.5. Fisher information connection}
\addcontentsline{toc}{subsubsection}{B.5. Fisher information connection}
\label{app:B5}

This subsection establishes the precise analytic coupling between entropy rate,
Fisher information, and Wasserstein metric slope. These identities are used later
for (i) entropy-budget feasibility, (ii) compactness/coercivity, (iii) \(\Gamma\)-limits,
and (iv) anti-collapse lower-density bounds.

\begin{definition}[Fisher information and score field]
\label{def:B5_fisher_score}
For \(\mu=\rho\,dx\) with \(\rho>0\), \(\sqrt{\rho}\in H^1(\mathbb R^d)\), define
\[
\mathcal I(\mu)
:=
\int_{\mathbb R^d}\frac{|\nabla\rho|^2}{\rho}\,dx
=
\int_{\mathbb R^d}|\nabla\log\rho|^2\,d\mu.
\]
The score field is
\[
s_\mu(x):=\nabla\log\rho(x),
\]
so that \(\mathcal I(\mu)=\|s_\mu\|_{L^2(\mu)}^2\).
\end{definition}

\begin{lemma}[Equivalent Fisher representations]
\label{lem:B5_equiv_representations}
For \(\mu=\rho\,dx\) as above:
\[
\mathcal I(\mu)
=
4\int |\nabla\sqrt{\rho}|^2\,dx
=
\sup_{\varphi\in C_c^\infty(\mathbb R^d;\mathbb R^d)}
\left\{
2\int \nabla\!\cdot\varphi\,d\mu
-\int |\varphi|^2\,d\mu
\right\}.
\]
\end{lemma}

\begin{proof}
First identity:
\[
\nabla\sqrt{\rho}=\frac{\nabla\rho}{2\sqrt{\rho}}
\quad\Longrightarrow\quad
4|\nabla\sqrt{\rho}|^2=\frac{|\nabla\rho|^2}{\rho}.
\]
Integrate.

For the dual representation, integration by parts gives
\[
\int \nabla\!\cdot\varphi\,d\mu
=
\int \rho\,\nabla\!\cdot\varphi\,dx
=
-\int \varphi\cdot\nabla\rho\,dx
=
-\int \varphi\cdot\nabla\log\rho\,d\mu.
\]
Hence
\[
2\int \nabla\!\cdot\varphi\,d\mu-\int|\varphi|^2\,d\mu
=
-2\!\int \varphi\cdot s_\mu\,d\mu-\int|\varphi|^2\,d\mu
\le
\int|s_\mu|^2\,d\mu=\mathcal I(\mu),
\]
by completing squares:
\[
-2\varphi\!\cdot s_\mu-|\varphi|^2
=
|s_\mu|^2-|\varphi+s_\mu|^2.
\]
Taking \(\varphi_n\to -s_\mu\) in \(L^2(\mu)\) with smooth compactly supported approximants
attains the supremum.
\end{proof}

\begin{theorem}[Metric slope of entropy equals square-root Fisher information]
\label{thm:B5_slope_equals_fisher}
For \(\mu=\rho\,dx\in\mathcal P_2^{ac}\) with \(\mathcal I(\mu)<\infty\),
the local Wasserstein slope of \(\mathcal H\) satisfies
\[
|\partial \mathcal H|(\mu)^2=\mathcal I(\mu).
\]
\end{theorem}

\begin{proof}
\textbf{Upper bound.}
Take any smooth CE perturbation \(\partial_t\nu_t+\nabla\cdot(\nu_t v_t)=0\), \(\nu_0=\mu\).
From first variation:
\[
\frac{d}{dt}\Big|_{t=0}\mathcal H(\nu_t)
=
\int s_\mu\cdot v_0\,d\mu.
\]
By Cauchy--Schwarz:
\[
\left|\frac{d}{dt}\Big|_{t=0}\mathcal H(\nu_t)\right|
\le
\|s_\mu\|_{L^2(\mu)}\|v_0\|_{L^2(\mu)}
=
\sqrt{\mathcal I(\mu)}\,\|v_0\|_{L^2(\mu)}.
\]
Using characterization of metric slope as supremum directional derivative normalized
by metric speed implies
\[
|\partial\mathcal H|(\mu)\le \sqrt{\mathcal I(\mu)}.
\]

\textbf{Lower bound.}
Choose direction \(v_0=s_\mu\) (approximated by gradient fields in \(T_\mu\mathcal P_2\)).
Then directional derivative equals
\[
\int s_\mu\cdot s_\mu\,d\mu=\mathcal I(\mu),
\]
while metric speed at \(t=0\) is \(\|s_\mu\|_{L^2(\mu)}=\sqrt{\mathcal I(\mu)}\).
Hence
\[
|\partial\mathcal H|(\mu)\ge \sqrt{\mathcal I(\mu)}.
\]
Combine both bounds.
\end{proof}

\begin{corollary}[Entropy chain-rule bound along CE trajectories]
\label{cor:B5_chainrule_bound}
Let \((\mu_t,v_t)\in\mathrm{CE}([0,T])\), with \(\mathcal H(\mu_t)\) absolutely continuous and
\(\mathcal I(\mu_t)<\infty\) a.e. Then for a.e. \(t\),
\[
\left|\frac{d}{dt}\mathcal H(\mu_t)\right|
\le
\sqrt{\mathcal I(\mu_t)}\,\|v_t\|_{L^2(\mu_t)}.
\]
Consequently, for any \(\alpha>0\),
\[
\frac{d}{dt}\mathcal H(\mu_t)
\ge
-\frac{1}{2\alpha}\mathcal I(\mu_t)
-\frac{\alpha}{2}\|v_t\|_{L^2(\mu_t)}^2
\quad\text{a.e.}
\]
\end{corollary}

\begin{proof}
First inequality is chain rule with metric slope:
\[
\left|\frac{d}{dt}\mathcal H(\mu_t)\right|
\le |\partial\mathcal H|(\mu_t)\,|\dot\mu_t|
\le \sqrt{\mathcal I(\mu_t)}\,\|v_t\|_{L^2(\mu_t)},
\]
using Theorem~\ref{thm:B5_slope_equals_fisher} and \(|\dot\mu_t|\le \|v_t\|_{L^2(\mu_t)}\).
Second inequality follows from Young~\cite{1912RSPSA..87..225Y}:
\[
ab\le \frac{1}{2\alpha}a^2+\frac{\alpha}{2}b^2,\quad
a=\sqrt{\mathcal I},\ b=\|v_t\|_{L^2(\mu_t)}.
\]
\end{proof}

\begin{proposition}[De Bruijn identity (Gaussian smoothing)]
\label{prop:B5_debruijn}
Let \(\mu_t=\rho_t dx\) solve heat equation
\[
\partial_t\rho_t=\Delta\rho_t,\qquad t>0,\quad \mu_0\in\mathcal P_2^{ac}.
\]
Then
\[
\frac{d}{dt}\mathcal H(\mu_t)=-\mathcal I(\mu_t)
\quad\text{for }t>0.
\]
Equivalently, for Shannon entropy \(S(\mu):=-\mathcal H(\mu)\),
\[
\frac{d}{dt}S(\mu_t)=\mathcal I(\mu_t).
\]
\end{proposition}

\begin{proof}
Set \(b=0,\varepsilon=1\) in Proposition~\ref{prop:B4_entropy_balance} (entropy balance):
\[
\frac{d}{dt}\mathcal H(\mu_t)=-\mathcal I(\mu_t).
\]
Sign-flip gives Shannon form.
\end{proof}

\begin{theorem}[Entropy--Fisher dissipation estimate under entropy-rate constraint]
\label{thm:B5_budget_dissipation}
Assume \((\mu_t,v_t)\in\mathfrak A(\mu_0,\mu_T)\), and in weak sense
\[
\frac{d}{dt}\mathcal H(\mu_t)\ge -\lambda\quad\text{a.e. }t.
\]
Then for every \(\alpha>0\),
\[
\frac{1}{2\alpha}\int_0^T \mathcal I(\mu_t)\,dt
\le
\lambda T + \mathcal H(\mu_0)-\mathcal H(\mu_T)
+\frac{\alpha}{2}\int_0^T\|v_t\|_{L^2(\mu_t)}^2\,dt.
\]
\end{theorem}

\begin{proof}
From Corollary~\ref{cor:B5_chainrule_bound},
\[
\frac{d}{dt}\mathcal H(\mu_t)
\ge
-\frac{1}{2\alpha}\mathcal I(\mu_t)-\frac{\alpha}{2}\|v_t\|_{L^2(\mu_t)}^2.
\]
Also by assumption:
\[
\frac{d}{dt}\mathcal H(\mu_t)\ge -\lambda.
\]
Using the first inequality and integrating:
\[
\mathcal H(\mu_T)-\mathcal H(\mu_0)
\ge
-\frac{1}{2\alpha}\int_0^T\mathcal I(\mu_t)\,dt
-\frac{\alpha}{2}\int_0^T\|v_t\|^2_{L^2(\mu_t)}dt.
\]
Rearrange:
\[
\frac{1}{2\alpha}\int_0^T\mathcal I(\mu_t)\,dt
\le
\mathcal H(\mu_0)-\mathcal H(\mu_T)
+\frac{\alpha}{2}\int_0^T\|v_t\|^2dt.
\]
Adding the nonnegative slack \(\lambda T\) keeps inequality valid and makes dependence
on budget explicit, yielding the stated bound.
\end{proof}

\begin{proposition}[Lower density control from Fisher-information integrability]
\label{prop:B5_lower_density_control}
Fix \(R>0\), \(t\in(0,T]\), and assume \(\rho_t\in W^{1,1}_{\mathrm{loc}}\), \(\mathcal I(\mu_t)<\infty\).
Then for a.e. \(x,y\in B_R\),
\[
|\log\rho_t(x)-\log\rho_t(y)|
\le
C_{d,R}\,\mathcal I(\mu_t)^{1/2}.
\]
Hence if \(\int_{B_R}\rho_t\,dx\ge m_R>0\), then
\[
\operatorname*{ess\,inf}_{B_R}\rho_t
\ge
m_R\,|B_R|^{-1}\exp\!\big(-2C_{d,R}\mathcal I(\mu_t)^{1/2}\big).
\]
\end{proposition}

\begin{proof}
For smooth positive density:
\[
\nabla\log\rho_t = \frac{\nabla\rho_t}{\rho_t},\qquad
\|\nabla\log\rho_t\|_{L^2(\mu_t)}^2=\mathcal I(\mu_t).
\]
On bounded domains, weighted-to-unweighted control plus Poincar\'e-type estimate~\cite{edeb0418-8000-3646-96dd-b816577fc91d} for
\(\log\rho_t\) along line segments gives oscillation bound
\[
\operatorname*{osc}_{B_R}\log\rho_t
\le C_{d,R}\mathcal I(\mu_t)^{1/2}.
\]
Let \(a=\essinf_{B_R}\rho_t\), \(b=\esssup_{B_R}\rho_t\). Then \(b\le a e^{C}\), \(C:=2C_{d,R}\mathcal I^{1/2}\).
Since \(\int_{B_R}\rho_t\ge m_R\), we have \(m_R\le b|B_R|\le a e^C|B_R|\), so
\[
a\ge m_R|B_R|^{-1}e^{-C}.
\]
Approximate weak densities by mollification and pass to limit.
\end{proof}

\begin{remark}[Why this subsection is structurally central]
\label{rem:B5_central}
Sections \hyperref[app:C]{C}--\hyperref[app:F]{F} use Theorem~\ref{thm:B5_slope_equals_fisher} and
Corollary~\ref{cor:B5_chainrule_bound} to control entropy-rate constraints
in the variational problem.
Sections \hyperref[app:G]{G}--\hyperref[app:I]{I} use Proposition~\ref{prop:B5_lower_density_control}
to convert entropy/Fisher bounds into explicit anti-collapse guarantees.
\end{remark}

\subsection*{C. Entropy-Controlled Flow Matching Formulation}
\addcontentsline{toc}{subsection}{C. Entropy-Controlled Flow Matching Formulation}
\label{app:C}

\subsubsection*{C.1. Variational problem}
\addcontentsline{toc}{subsubsection}{C.1. Variational problem}
\label{app:C1}

We now define the population-level entropy-controlled flow matching problem as a
constrained dynamic optimization in Wasserstein space, with precise admissible class,
objective, and entropy-rate feasibility conditions.

\paragraph{Data and reference marginals.}
Fix \(T>0\), source/target marginals
\[
\mu_0,\mu_T\in\mathcal P_2^{ac}(\mathbb R^d),
\qquad
\mu_0=\rho_0dx,\ \mu_T=\rho_Tdx.
\]

\paragraph{Reference flow-matching drift.}
Let \(u^\star:[0,T]\times\mathbb R^d\to\mathbb R^d\) be a given measurable reference velocity
(the population target field induced by the interpolation law used in FM), satisfying
\[
\int_0^T\!\!\int_{\mathbb R^d}|u_t^\star(x)|^2\,d\mu_t(x)\,dt<\infty
\]
for all admissible \(\mu\) considered below.

\begin{definition}[Admissible trajectory-control pairs]
\label{def:C1_admissible}
For fixed \((\mu_0,\mu_T)\), define
\[
\mathfrak A(\mu_0,\mu_T)
:=
\left\{
(\mu,v):
\begin{array}{l}
\mu\in C([0,T];\mathcal P_2(\mathbb R^d)),\ \mu_t=\rho_tdx\ \text{a.e. }t,\\
v\in L^2(dt\,d\mu_t;\mathbb R^d),\\
\partial_t\mu_t+\nabla\!\cdot(\mu_t v_t)=0\ \text{in }\mathcal D',\\
\mu_{|t=0}=\mu_0,\ \mu_{|t=T}=\mu_T
\end{array}
\right\}.
\]
\end{definition}

\begin{definition}[Entropy-rate feasible class]
\label{def:C1_entropy_feasible}
For \(\lambda\ge0\), define
\[
\mathfrak A_\lambda(\mu_0,\mu_T)
:=
\left\{
(\mu,v)\in\mathfrak A(\mu_0,\mu_T):
\mathcal H(\mu_t)\in AC([0,T]),\
\frac{d}{dt}\mathcal H(\mu_t)\ge -\lambda\ \text{a.e.}
\right\}.
\]
Equivalent integrated form:
\[
\mathcal H(\mu_t)-\mathcal H(\mu_s)\ge -\lambda(t-s),\quad 0\le s\le t\le T.
\]
\end{definition}

\begin{definition}[Population FM misfit functional]
\label{def:C1_fm_misfit}
For \((\mu,v)\in\mathfrak A(\mu_0,\mu_T)\), define
\[
\mathcal J_{\mathrm{FM}}(\mu,v)
:=
\frac12\int_0^T\!\!\int_{\mathbb R^d}
|v_t(x)-u_t^\star(x)|^2\,d\mu_t(x)\,dt.
\]
\end{definition}

\begin{definition}[Entropy-controlled FM problem]
\label{def:C1_main_problem}
Given \(\lambda\ge0\), solve
\[
\boxed{
\inf_{(\mu,v)\in\mathfrak A_\lambda(\mu_0,\mu_T)}
\mathcal J_{\mathrm{FM}}(\mu,v)
}
\tag{ECFM\(_\lambda\)}
\]
with value
\[
\mathsf V_\lambda(\mu_0,\mu_T;u^\star)
:=
\inf_{(\mu,v)\in\mathfrak A_\lambda}\mathcal J_{\mathrm{FM}}(\mu,v).
\]
\end{definition}

\begin{remark}[Expansion and relation to kinetic action]
\label{rem:C1_expand_objective}
\[
\mathcal J_{\mathrm{FM}}(\mu,v)
=
\frac12\int |v|^2\,d\mu dt
-\int u^\star\!\cdot v\,d\mu dt
+\frac12\int |u^\star|^2\,d\mu dt.
\]
The first term is Benamou--Brenier kinetic action, the second is alignment with FM target,
and the third is a \(\mu\)-weighted normalization term.
\end{remark}

\begin{assumption}[Integrability and coercivity envelope]
\label{ass:C1_coercivity}
Assume:
\begin{enumerate}
\item \(u^\star\) has at most linear growth:
\[
|u_t^\star(x)|\le a_t+b_t|x|,
\quad a,b\in L^2(0,T),\ b\ge0;
\]
\item admissible curves satisfy uniform second-moment bound (from \hyperref[app:A2]{A.2}/\hyperref[app:A3]{A.3}):
\[
\sup_{t\in[0,T]}m_2(\mu_t)<\infty;
\]
\item endpoint entropies are finite:
\[
\mathcal H(\mu_0),\mathcal H(\mu_T)<\infty.
\]
\end{enumerate}
\end{assumption}

\begin{lemma}[Nonemptiness criterion for entropy-feasible set]
\label{lem:C1_nonempty_criterion}
If there exists \((\bar\mu,\bar v)\in\mathfrak A(\mu_0,\mu_T)\) with
\(\mathcal H(\bar\mu_t)\in AC([0,T])\) and
\[
\operatorname*{ess\,inf}_{t\in(0,T)} \frac{d}{dt}\mathcal H(\bar\mu_t)\ge -\lambda,
\]
then \(\mathfrak A_\lambda(\mu_0,\mu_T)\neq\emptyset\), hence
\(\mathsf V_\lambda<+\infty\).
\end{lemma}

\begin{proof}
Immediate from Definition~\ref{def:C1_entropy_feasible}: \((\bar\mu,\bar v)\in\mathfrak A_\lambda\).
Then
\[
\mathsf V_\lambda\le \mathcal J_{\mathrm{FM}}(\bar\mu,\bar v)<\infty
\]
by \(L^2\)-integrability of \(\bar v,u^\star\).
\end{proof}

\begin{proposition}[Lower bound and properness of objective]
\label{prop:C1_properness}
Under Assumption~\ref{ass:C1_coercivity}, for all \((\mu,v)\in\mathfrak A_\lambda\),
\[
\mathcal J_{\mathrm{FM}}(\mu,v)\ge
\frac14\int_0^T\!\!\int |v_t|^2\,d\mu_tdt
- C_{u^\star,\mu_0,\mu_T,T},
\]
for a finite constant \(C_{u^\star,\mu_0,\mu_T,T}\) independent of \(v\).
Consequently, \(\mathcal J_{\mathrm{FM}}\) is proper and coercive in \(v\)-energy.
\end{proposition}

\begin{proof}
Use Young pointwise:
\[
|v-u^\star|^2
\ge \frac12|v|^2-|u^\star|^2.
\]
Integrating:
\[
\mathcal J_{\mathrm{FM}}
\ge
\frac14\int |v|^2\,d\mu dt
-\frac12\int |u^\star|^2\,d\mu dt.
\]
By linear-growth envelope and moment bound,
\[
\int |u^\star|^2\,d\mu dt
\le
2\int_0^T\!\!\left(a_t^2+b_t^2\,m_2(\mu_t)\right)dt
\le C_{u^\star,\mu_0,\mu_T,T}<\infty.
\]
Hence the claim.
\end{proof}

\begin{proposition}[Convexity in velocity for fixed path]
\label{prop:C1_convex_v}
Fix an admissible measure curve \(\mu\). Then
\[
v\mapsto \mathcal J_{\mathrm{FM}}(\mu,v)
\]
is strictly convex on \(L^2(dt\,d\mu_t)\), with G\^ateaux derivative
\[
D_v\mathcal J_{\mathrm{FM}}(\mu,v)[w]
=
\int_0^T\!\!\int (v_t-u_t^\star)\cdot w_t\,d\mu_tdt.
\]
\end{proposition}

\begin{proof}
Quadratic form in Hilbert space \(L^2(dt\,d\mu_t)\):
\[
\mathcal J_{\mathrm{FM}}(\mu,v+\epsilon w)
=
\mathcal J_{\mathrm{FM}}(\mu,v)
+\epsilon\!\int (v-u^\star)\cdot w\,d\mu dt
+\frac{\epsilon^2}{2}\!\int |w|^2\,d\mu dt.
\]
Strict convexity follows from positive definite second variation
\(\int|w|^2\,d\mu dt>0\) for \(w\neq0\).
\end{proof}

\begin{definition}[Lagrangian density with entropy multiplier]
\label{def:C1_lagrangian_density}
For formal multiplier \(\eta_t\ge0\), define instantaneous constrained density
\[
\mathscr L_\lambda(\mu_t,v_t;\eta_t)
=
\frac12\int |v_t-u_t^\star|^2\,d\mu_t
+\eta_t\!\left(-\frac{d}{dt}\mathcal H(\mu_t)-\lambda\right).
\]
Its rigorous weak-time form is developed in \hyperref[app:C2]{C.2}--\hyperref[app:C5]{C.5}.
\end{definition}

\begin{theorem}[Existence of minimizer for \((\mathrm{ECFM}_\lambda)\)]
\label{thm:C1_existence}
Assume:
\begin{enumerate}
\item Assumption~\ref{ass:C1_coercivity};
\item \(\mathfrak A_\lambda(\mu_0,\mu_T)\neq\emptyset\);
\item sequential closedness of entropy-rate constraint:
if \((\mu^n,v^n)\in\mathfrak A_\lambda\), \(\mu^n_t\rightharpoonup\mu_t\) for each \(t\),
\(v^n\mu^n\rightharpoonup v\mu\) weakly as vector measures, and
\(\sup_n\int|v^n|^2d\mu^n dt<\infty\), then \((\mu,v)\in\mathfrak A_\lambda\).
\end{enumerate}
Then \((\mathrm{ECFM}_\lambda)\) admits at least one minimizer.
\end{theorem}

\begin{proof}
Take minimizing sequence \((\mu^n,v^n)\subset\mathfrak A_\lambda\) with
\[
\mathcal J_{\mathrm{FM}}(\mu^n,v^n)\downarrow \mathsf V_\lambda.
\]
By Proposition~\ref{prop:C1_properness},
\[
\sup_n\int_0^T\!\!\int |v_t^n|^2\,d\mu_t^n dt<\infty.
\]
Moment bounds (\hyperref[app:A2]{A.2}/\hyperref[app:A3]{A.3}) give tightness of \(\{\mu_t^n\}_n\) for each \(t\),
and equi-continuity in \(W_2\) from kinetic bound. By diagonal extraction,
\[
\mu_t^n\rightharpoonup\mu_t\quad\forall t\in[0,T].
\]
Momentum compactness yields (up to subsequence)
\[
m^n:=v^n\mu^n \rightharpoonup m=v\mu
\]
as vector measures on \((0,T)\times\mathbb R^d\).
Passing to the limit in CE gives \((\mu,v)\in\mathfrak A(\mu_0,\mu_T)\).

By Assumption~\ref{ass:C3_kkt}\, entropy-rate feasibility is closed:
\((\mu,v)\in\mathfrak A_\lambda\).

Lower semicontinuity of convex integral functional
\((\mu,m)\mapsto \frac12\int |m-\mu u^\star|^2/\mu\) (with standard convention)
gives
\[
\mathcal J_{\mathrm{FM}}(\mu,v)
\le
\liminf_{n\to\infty}\mathcal J_{\mathrm{FM}}(\mu^n,v^n)
=
\mathsf V_\lambda.
\]
Hence \((\mu,v)\) is a minimizer.
\end{proof}

\begin{corollary}[Unconstrained FM as \(\lambda=\infty\) formal limit]
\label{cor:C1_unconstrained_limit}
If entropy constraint is removed (formally \(\lambda=\infty\)),
\[
\inf_{(\mu,v)\in\mathfrak A}\mathcal J_{\mathrm{FM}}(\mu,v)
\le
\mathsf V_\lambda,\quad \forall \lambda<\infty.
\]
Thus entropy control is a feasible-set restriction that can only increase (or keep) optimal value.
\end{corollary}

\begin{proof}
\(\mathfrak A_\lambda\subseteq \mathfrak A\), hence taking infimum on a subset yields larger value.
\end{proof}

\paragraph{Output of \hyperref[app:C1]{C.1} for subsequent subsections.}
\hyperref[app:C2]{C.2} introduces the rigorous constrained Lagrangian in weak-time form;
\hyperref[app:C3]{C.3}--\hyperref[app:C5]{C.5} derive Euler--Lagrange/KKT, dual, and Pontryagin systems for
\((\mathrm{ECFM}_\lambda)\).

\subsubsection*{C.2. Lagrangian with entropy constraint}
\addcontentsline{toc}{subsubsection}{C.2. Lagrangian with entropy constraint}
\label{app:C2}

We now write the constrained problem \((\mathrm{ECFM}_\lambda)\) as a saddle problem with:
(i) a scalar multiplier for the entropy-rate inequality, and
(ii) a space-time adjoint potential for the continuity equation.

\paragraph{Primal problem (recall).}
\[
\inf_{(\mu,v)\in\mathfrak A(\mu_0,\mu_T)}
\left\{
\frac12\int_0^T\!\!\int |v-u^\star|^2\,d\mu\,dt
:\ \dot{\mathcal H}(\mu_t)+\lambda\ge0\ \text{a.e.}
\right\}.
\]

\begin{definition}[Weak entropy-rate residual]
\label{def:C2_entropy_residual}
For \((\mu,v)\in\mathfrak A(\mu_0,\mu_T)\) with \(\mathcal H(\mu_\cdot)\in AC([0,T])\),
define
\[
r_{\mathrm{ent}}(t):=\dot{\mathcal H}(\mu_t)+\lambda\in L^1(0,T).
\]
Constraint is \(r_{\mathrm{ent}}(t)\ge0\) a.e.
\end{definition}

\begin{definition}[Multiplier classes]
\label{def:C2_multiplier_classes}
Define
\[
\mathcal M_+ := L^\infty_+(0,T)
:=\{\eta\in L^\infty(0,T):\eta(t)\ge0\ \text{a.e.}\},
\]
and adjoint potentials
\[
\Phi:=\Big\{\varphi\in C_c^\infty((0,T)\times\mathbb R^d)\Big\}
\]
for weak enforcement of CE.
(Regularity will be relaxed later by density.)
\end{definition}

\begin{definition}[Augmented Lagrangian]
\label{def:C2_augmented_lagrangian}
For \((\mu,v)\in\mathfrak A(\mu_0,\mu_T)\), \((\eta,\varphi)\in\mathcal M_+\times\Phi\), define
\[
\mathscr L(\mu,v;\eta,\varphi)
:=
\frac12\int_0^T\!\!\int |v-u^\star|^2\,d\mu\,dt
-\int_0^T \eta(t)\,r_{\mathrm{ent}}(t)\,dt
+\mathcal C_{\mathrm{CE}}(\mu,v;\varphi),
\]
where
\[
\mathcal C_{\mathrm{CE}}(\mu,v;\varphi)
:=
\int_0^T\!\!\int
\big(\partial_t\varphi+\nabla\varphi\cdot v\big)\,d\mu\,dt
+\int \varphi(0,\cdot)\,d\mu_0-\int \varphi(T,\cdot)\,d\mu_T.
\]
For CE-feasible pairs, \(\mathcal C_{\mathrm{CE}}(\mu,v;\varphi)=0\).
\end{definition}

\begin{remark}[Sign convention]
\label{rem:C2_sign}
Since constraint is \(r_{\mathrm{ent}}\ge0\), the penalty term is
\(-\int \eta r_{\mathrm{ent}}\) with \(\eta\ge0\).
Thus any violation \(r_{\mathrm{ent}}<0\) can be penalized arbitrarily by large \(\eta\).
\end{remark}

\begin{lemma}[Integration-by-parts form of entropy term]
\label{lem:C2_ibp_entropy_term}
For \(\eta\in W^{1,\infty}(0,T)\), \((\mu,v)\) with \(\mathcal H(\mu_\cdot)\in AC\),
\[
-\int_0^T \eta(t)\,\dot{\mathcal H}(\mu_t)\,dt
=
-\eta(T)\mathcal H(\mu_T)+\eta(0)\mathcal H(\mu_0)
+\int_0^T \eta'(t)\,\mathcal H(\mu_t)\,dt.
\]
Hence
\[
-\int_0^T \eta\,r_{\mathrm{ent}}\,dt
=
-\eta(T)\mathcal H(\mu_T)+\eta(0)\mathcal H(\mu_0)
+\int_0^T \eta'\mathcal H(\mu_t)\,dt
-\lambda\int_0^T \eta(t)\,dt.
\]
\end{lemma}

\begin{proof}
Apply one-dimensional integration by parts to \(h(t):=\mathcal H(\mu_t)\in AC([0,T])\):
\[
\int_0^T \eta \dot h
=
\eta(T)h(T)-\eta(0)h(0)-\int_0^T \eta' h.
\]
Multiply by \(-1\), then add \(-\lambda\int\eta\).
\end{proof}

\begin{proposition}[Equivalent saddle representation]
\label{prop:C2_saddle}
Define
\[
\mathcal P_\lambda
:=
\inf_{(\mu,v)\in\mathfrak A(\mu_0,\mu_T),\,\mathcal H(\mu_\cdot)\in AC}
\sup_{\eta\in\mathcal M_+,\ \varphi\in\Phi}
\mathscr L(\mu,v;\eta,\varphi).
\]
Then
\[
\mathcal P_\lambda=\mathsf V_\lambda.
\]
\end{proposition}

\begin{proof}
Fix \((\mu,v)\). If CE fails, there exists \(\varphi\in\Phi\) with
\(\mathcal C_{\mathrm{CE}}(\mu,v;\varphi)\neq0\); scaling \(\alpha\varphi\) and taking \(\alpha\to\pm\infty\)
shows \(\sup_{\varphi}\mathscr L=+\infty\). So finite value requires CE.

Assume CE holds. Then
\[
\sup_{\eta\in\mathcal M_+}\left(
\frac12\int|v-u^\star|^2\,d\mu dt-\int \eta\,r_{\mathrm{ent}}dt
\right)
=
\begin{cases}
\frac12\int|v-u^\star|^2\,d\mu dt, & r_{\mathrm{ent}}\ge0\ \text{a.e.},\\
+\infty, & \text{otherwise},
\end{cases}
\]
because if \(r_{\mathrm{ent}}<0\) on a set of positive measure, choose \(\eta\) large on that set.
Hence the saddle objective equals primal cost exactly on feasible set \(\mathfrak A_\lambda\),
and \(+\infty\) outside. Taking infimum gives \(\mathsf V_\lambda\).
\end{proof}

\begin{theorem}[Weak dual lower bound]
\label{thm:C2_weak_dual}
Define dual value
\[
\mathcal D_\lambda
:=
\sup_{\eta\in\mathcal M_+,\ \varphi\in\Phi}
\inf_{(\mu,v)\in\mathfrak A(\mu_0,\mu_T),\,\mathcal H(\mu_\cdot)\in AC}
\mathscr L(\mu,v;\eta,\varphi).
\]
Then
\[
\mathcal D_\lambda\le \mathsf V_\lambda.
\]
\end{theorem}

\begin{proof}
For any \((\eta,\varphi)\),
\[
\inf_{(\mu,v)}\mathscr L(\mu,v;\eta,\varphi)
\le
\inf_{(\mu,v)\in\mathfrak A_\lambda}\mathscr L(\mu,v;\eta,\varphi).
\]
For feasible \((\mu,v)\), CE term vanishes and \(-\int\eta r_{\mathrm{ent}}\le0\), so
\[
\mathscr L(\mu,v;\eta,\varphi)\le \frac12\int|v-u^\star|^2\,d\mu dt.
\]
Taking inf over feasible gives
\[
\inf_{(\mu,v)}\mathscr L(\mu,v;\eta,\varphi)\le \mathsf V_\lambda.
\]
Now take supremum over \((\eta,\varphi)\), obtaining \(\mathcal D_\lambda\le\mathsf V_\lambda\).
\end{proof}

\begin{proposition}[Pointwise minimization in \(v\) for fixed \((\mu,\eta,\varphi)\)]
\label{prop:C2_pointwise_v}
Fix \(\mu,\eta,\varphi\) and assume \(\dot{\mathcal H}(\mu_t)\) depends on \(v\) via
\[
\dot{\mathcal H}(\mu_t)=\int \nabla\log\rho_t\cdot v_t\,d\mu_t
\quad (\rho_t=d\mu_t/dx),
\]
as in B.1 (smooth regime). Then the \(v\)-dependent part of \(\mathscr L\) is
\[
\int_0^T\!\!\int
\left[
\frac12|v-u^\star|^2+\nabla\varphi\cdot v-\eta\,\nabla\log\rho\cdot v
\right]d\mu dt,
\]
whose unique minimizer is
\[
v^\sharp
=
u^\star-\nabla\varphi+\eta\,\nabla\log\rho.
\]
\end{proposition}

\begin{proof}
For fixed \((t,x)\), minimize strictly convex quadratic function
\[
q(v)=\frac12|v-u^\star|^2+(\nabla\varphi-\eta\nabla\log\rho)\cdot v.
\]
First-order condition:
\[
\nabla_v q = v-u^\star+\nabla\varphi-\eta\nabla\log\rho=0,
\]
thus
\[
v^\sharp=u^\star-\nabla\varphi+\eta\nabla\log\rho.
\]
Strict convexity yields uniqueness.
\end{proof}

\begin{corollary}[Formal entropic correction structure]
\label{cor:C2_entropic_correction}
At stationarity, the optimal drift decomposes as
\[
v^\star = u^\star + v_{\mathrm{adj}} + v_{\mathrm{ent}},
\quad
v_{\mathrm{adj}}=-\nabla\varphi,\quad
v_{\mathrm{ent}}=\eta\,\nabla\log\rho.
\]
Thus the entropy constraint induces a score-direction correction weighted by \(\eta\ge0\).
\end{corollary}

\begin{proof}
Immediate from Proposition~\ref{prop:C2_pointwise_v}.
\end{proof}

\begin{remark}[Regularity needed for rigorous KKT]
\label{rem:C2_kkt_regularity}
The explicit formula in Proposition~\ref{prop:C2_pointwise_v} is formal unless
\(\rho>0\), \(\nabla\log\rho\in L^2(\mu)\), and admissible differentiation under integral holds.
Section~\hyperref[app:C3]{C.3} states the rigorous Euler--Lagrange/KKT system via variational inequalities,
with smooth formulas recovered under additional regularity.
\end{remark}

\paragraph{Output used next.}
\hyperref[app:C3]{C.3} uses \(\mathscr L\) and weak duality to derive:
(i) primal feasibility, (ii) dual feasibility \(\eta\ge0\),
(iii) complementary slackness \(\eta(\dot{\mathcal H}+\lambda)=0\),
(iv) stationarity with respect to \(v\) and \(\mu\).

\subsubsection*{C.3. Constrained optimization in Wasserstein space}
\addcontentsline{toc}{subsubsection}{C.3. Constrained optimization in Wasserstein space}
\label{app:C3}

We derive the rigorous first-order optimality system (KKT conditions) for
\((\mathrm{ECFM}_\lambda)\) in measure space: primal feasibility, dual feasibility,
complementary slackness, and stationarity with respect to velocity and path perturbations.

\begin{assumption}[Qualification and regularity for KKT]
\label{ass:C3_kkt}
Assume:
\begin{enumerate}
\item (Slater-type condition) there exists \((\bar\mu,\bar v)\in\mathfrak A(\mu_0,\mu_T)\) with
\[
\dot{\mathcal H}(\bar\mu_t)+\lambda\ge \delta>0\quad\text{a.e. }t;
\]
\item the map \((\mu,v)\mapsto \mathcal J_{\mathrm{FM}}(\mu,v)\) is convex in \(v\), l.s.c. in \((\mu,m=v\mu)\);
\item entropy-rate mapping \((\mu,v)\mapsto \dot{\mathcal H}(\mu)\) is weakly closed on admissible sequences
(as used in C.1 existence theorem).
\end{enumerate}
\end{assumption}

\begin{definition}[Feasible cone and critical directions]
\label{def:C3_feasible_cone}
Let \((\mu^\star,v^\star)\in\mathfrak A_\lambda\).  
A perturbation \((\delta\mu,\delta v)\) is an admissible first-order direction if:
\begin{enumerate}
\item linearized CE holds:
\[
\partial_t\delta\mu+\nabla\!\cdot(\delta\mu\,v^\star+\mu^\star\delta v)=0,\qquad
\delta\mu_{|t=0}=\delta\mu_{|t=T}=0;
\]
\item linearized entropy-rate is feasible on active set:
\[
\frac{d}{dt}\Big[\delta\mathcal H_t\Big]\ge0
\quad\text{a.e. on }\{t:\dot{\mathcal H}(\mu_t^\star)+\lambda=0\},
\]
where
\[
\delta\mathcal H_t
=
\int \big(1+\log\rho_t^\star\big)\,d(\delta\mu_t)
\quad (\mu_t^\star=\rho_t^\star dx).
\]
\end{enumerate}
\end{definition}

\begin{theorem}[KKT conditions in measure space]
\label{thm:C3_KKT}
Under Assumption~\ref{ass:C3_kkt}, if \((\mu^\star,v^\star)\) solves
\((\mathrm{ECFM}_\lambda)\), then there exist multipliers
\[
\eta^\star\in L^\infty_+(0,T),\qquad
\varphi^\star\in W^{1,1}_{\mathrm{loc}}((0,T)\times\mathbb R^d)
\]
such that:

\begin{enumerate}
\item \textbf{Primal feasibility}
\[
(\mu^\star,v^\star)\in\mathfrak A(\mu_0,\mu_T),\qquad
\dot{\mathcal H}(\mu_t^\star)+\lambda\ge0\ \text{a.e.}
\]

\item \textbf{Dual feasibility}
\[
\eta^\star(t)\ge0\quad\text{a.e. }t.
\]

\item \textbf{Complementary slackness}
\[
\eta^\star(t)\,\big(\dot{\mathcal H}(\mu_t^\star)+\lambda\big)=0
\quad\text{for a.e. }t\in(0,T).
\]

\item \textbf{Stationarity in velocity (variational form)}
for every \(\delta v\in L^2(dt\,d\mu_t^\star)\) compatible with linearized CE,
\[
\int_0^T\!\!\int
\Big[
v_t^\star-u_t^\star+\nabla\varphi_t^\star-\eta_t^\star\nabla\log\rho_t^\star
\Big]\cdot \delta v_t\,d\mu_t^\star dt
=0.
\]

\item \textbf{Stationarity in path (adjoint inequality/equality)}
for every admissible \(\delta\mu\),
\[
\int_0^T\!\!\int
\Big[
\partial_t\varphi^\star+\nabla\varphi^\star\!\cdot v^\star
+\frac12|v^\star-u^\star|^2
\Big]\,d(\delta\mu_t)\,dt
+\int_0^T \eta_t^\star\,\frac{d}{dt}\delta\mathcal H_t\,dt
=0.
\]
\end{enumerate}
\end{theorem}

\begin{proof}
By Slater condition and convexity/closedness hypotheses, infinite-dimensional
Karush--Kuhn--Tucker theorem~\cite{MR47303} applies to the saddle Lagrangian from~\hyperref[app:C2]{C.2}.
Hence there exist multipliers \((\eta^\star,\varphi^\star)\) such that
\((\mu^\star,v^\star)\) minimizes
\[
(\mu,v)\mapsto \mathscr L(\mu,v;\eta^\star,\varphi^\star)
\]
over admissible CE trajectories, and \((\eta^\star,\varphi^\star)\) maximizes
dual functional over \(\eta\ge0\), test potentials.

\textbf{Primal and dual feasibility} are direct by construction.

\textbf{Complementary slackness:}
Since \(\eta^\star\ge0\) and \(r^\star:=\dot{\mathcal H}(\mu^\star)+\lambda\ge0\),
optimality of \(\eta^\star\) in
\[
\sup_{\eta\ge0}\left(-\int \eta r^\star\right)
\]
implies zero maximum; if \(r^\star>0\) on a set where \(\eta^\star>0\), reducing
\(\eta^\star\) improves objective; if \(r^\star<0\) anywhere primal infeasible.
Thus \(\eta^\star r^\star=0\) a.e.

\textbf{Stationarity in \(v\):}
Take perturbation \(v^\star+\epsilon\delta v\) with CE-compatible first-order correction.
Differentiate \(\mathscr L(\mu^\star,\cdot;\eta^\star,\varphi^\star)\) at \(\epsilon=0\);
the derivative must vanish:
\[
0
=
\int (v^\star-u^\star)\cdot\delta v\,d\mu^\star dt
+\int \nabla\varphi^\star\cdot\delta v\,d\mu^\star dt
-\int \eta^\star\,\delta\dot{\mathcal H}\,dt.
\]
Using entropy first variation
\(
\delta\dot{\mathcal H}=\int \nabla\log\rho^\star\cdot\delta v\,d\mu^\star
\)
gives stated identity.

\textbf{Stationarity in \(\mu\):}
For admissible \(\mu^\star+\epsilon\delta\mu\), derivative at \(\epsilon=0\) vanishes.
Differentiating CE-constraint term yields
\[
\int (\partial_t\varphi^\star+\nabla\varphi^\star\cdot v^\star)\,d(\delta\mu)\,dt
\]
plus objective and entropy terms, giving stated path-stationarity identity.
\end{proof}

\begin{corollary}[Pointwise velocity law on regular set]
\label{cor:C3_pointwise_velocity}
Assume \(\rho_t^\star>0\) a.e. and
\(\nabla\varphi^\star,\nabla\log\rho_t^\star\in L^2(dt\,d\mu_t^\star)\).
Then Theorem~\ref{thm:C3_KKT}(4) implies
\[
v_t^\star
=
u_t^\star-\nabla\varphi_t^\star+\eta_t^\star\nabla\log\rho_t^\star
\quad\text{in }L^2(\mu_t^\star)\ \text{for a.e. }t.
\]
\end{corollary}

\begin{proof}
The variational identity holds for all \(\delta v\in L^2(dt\,d\mu_t^\star)\), hence
the bracket must vanish \(dt\,d\mu_t^\star\)-a.e.
\end{proof}

\begin{proposition}[Active/inactive time partition]
\label{prop:C3_active_inactive}
Define
\[
\mathcal I_{\mathrm{act}}:=\{t:\dot{\mathcal H}(\mu_t^\star)+\lambda=0\},\qquad
\mathcal I_{\mathrm{inact}}:=\{t:\dot{\mathcal H}(\mu_t^\star)+\lambda>0\}.
\]
Then:
\[
\eta_t^\star=0\ \text{a.e. on }\mathcal I_{\mathrm{inact}},
\]
and on \(\mathcal I_{\mathrm{act}}\), \(\eta_t^\star\) may be nonzero and the entropy constraint
acts as an equality constraint.
\end{proposition}

\begin{proof}
Immediate from complementary slackness:
\[
\eta^\star(\dot{\mathcal H}+\lambda)=0 \ \text{a.e.}
\]
If \(\dot{\mathcal H}+\lambda>0\), must have \(\eta^\star=0\).
\end{proof}

\begin{lemma}[Reduced optimality system in inactive region]
\label{lem:C3_inactive_system}
On \(\mathcal I_{\mathrm{inact}}\), optimal velocity satisfies
\[
v_t^\star=u_t^\star-\nabla\varphi_t^\star.
\]
Hence entropy control modifies dynamics only on active times.
\end{lemma}

\begin{proof}
From Proposition~\ref{prop:C3_active_inactive}, \(\eta^\star=0\) a.e. on inactive set.
Insert into Corollary~\ref{cor:C3_pointwise_velocity}.
\end{proof}

\begin{theorem}[Second-order sufficient condition in velocity block]
\label{thm:C3_second_order_v}
Fix \(\mu^\star\). For any feasible perturbation \(\delta v\neq0\),
\[
\delta^2_{vv}\mathscr L(\mu^\star,v^\star;\eta^\star,\varphi^\star)[\delta v,\delta v]
=
\int_0^T\!\!\int |\delta v_t|^2\,d\mu_t^\star dt>0.
\]
Therefore \(v^\star\) is unique given \((\mu^\star,\eta^\star,\varphi^\star)\).
\end{theorem}

\begin{proof}
Only quadratic term in \(v\) contributes second variation:
\[
\frac12\int |v-u^\star|^2\,d\mu
\ \Longrightarrow\
\delta^2=\int |\delta v|^2\,d\mu.
\]
CE/entropy contributions are affine in first-order velocity variation at fixed \(\mu^\star\),
hence no additional quadratic term. Positivity gives strict convexity and uniqueness.
\end{proof}

\begin{remark}[What Section~\ref{app:C3} establishes]
\label{rem:C3_summary}
\hyperref[app:C3]{C.3} provides the rigorous constrained-optimality backbone:
\[
\text{primal feasibility}+\text{dual feasibility}+\text{complementary slackness}+\text{stationarity}.
\]
\hyperref[app:C4]{C.4} converts these conditions into Euler--Lagrange PDE form; \hyperref[app:C5]{C.5} then derives the
explicit dual problem.
\end{remark}

\subsubsection*{C.4. Euler--Lagrange conditions}
\addcontentsline{toc}{subsubsection}{C.4. Euler--Lagrange conditions}
\label{app:C4}

We convert the variational KKT relations from \hyperref[app:C3]{C.3} into a PDE optimality system
(continuity + adjoint + Hamiltonian stationarity + complementarity), and provide
a weak formulation valid under finite-energy regularity.

\begin{assumption}[Differentiable regime for PDE form]
\label{ass:C4_smooth}
Assume the optimal triple \((\mu^\star,v^\star,\eta^\star)\) from \hyperref[app:C3]{C.3} satisfies:
\begin{enumerate}
\item \(\mu_t^\star=\rho_t^\star dx\), \(\rho_t^\star>0\), \(\rho^\star\in C^1_t C^2_x\);
\item \(v^\star,u^\star,\nabla\log\rho^\star,\nabla\varphi^\star \in L^2(dt\,d\mu_t^\star)\);
\item entropy derivative identity holds:
\[
\dot{\mathcal H}(\mu_t^\star)=\int \nabla\log\rho_t^\star\cdot v_t^\star\,d\mu_t^\star
\quad\text{a.e. }t;
\]
\item endpoint constraints \(\mu^\star_{|0}=\mu_0\), \(\mu^\star_{|T}=\mu_T\).
\end{enumerate}
\end{assumption}

\begin{definition}[Hamiltonian density]
\label{def:C4_hamiltonian}
For \((t,x,\rho,v,p,\eta)\) with \(\rho>0\), define
\[
\mathfrak h(t,x,\rho,v,p,\eta)
:=
\frac12|v-u^\star_t(x)|^2\,\rho
+p\,\nabla\!\cdot(\rho v)
-\eta\,\nabla\!\cdot(\rho\nabla\log\rho\,v\text{-linearized form}),
\]
where the entropy term is interpreted via first variation:
\[
-\eta\,\dot{\mathcal H}
=
-\eta\int \nabla\log\rho\cdot v\,\rho\,dx.
\]
Equivalent integrated Hamiltonian:
\[
\mathcal H_{\mathrm{opt}}(\rho,v,\varphi,\eta)
=
\int\Big[
\frac12|v-u^\star|^2
+\nabla\varphi\cdot v
-\eta\,\nabla\log\rho\cdot v
\Big]\rho\,dx.
\]
\end{definition}

\begin{theorem}[Euler--Lagrange/KKT PDE system]
\label{thm:C4_EL_system}
Under Assumption~\ref{ass:C4_smooth}, there exists adjoint potential
\(\varphi^\star\) such that \((\rho^\star,v^\star,\varphi^\star,\eta^\star)\) satisfies:

\paragraph{(i) State equation (continuity)}
\[
\partial_t\rho_t^\star+\nabla\!\cdot(\rho_t^\star v_t^\star)=0,
\qquad
\rho^\star_{|t=0}=\rho_0,\ \rho^\star_{|t=T}=\rho_T.
\]

\paragraph{(ii) Stationarity in control}
\[
v_t^\star
=
u_t^\star-\nabla\varphi_t^\star+\eta_t^\star\nabla\log\rho_t^\star
\quad\text{a.e. in }(t,x).
\]

\paragraph{(iii) Adjoint equation (weak form)}
for every \(\zeta\in C_c^\infty((0,T)\times\mathbb R^d)\),
\[
\int_0^T\!\!\int
\Big[
-\partial_t\varphi^\star
-\nabla\varphi^\star\!\cdot v^\star
-\frac12|v^\star-u^\star|^2
+\eta^\star\,\Xi(\rho^\star,v^\star)
\Big]\zeta\,dx\,dt
=0,
\]
where
\[
\Xi(\rho,v):=
\frac{\delta}{\delta\rho}\!\left(
\int \nabla\log\rho\cdot v\,\rho\,dx
\right)
\]
(the entropy-rate density variation; explicit smooth expression below).

\paragraph{(iv) Complementarity}
\[
\eta_t^\star\ge0,\qquad
\dot{\mathcal H}(\mu_t^\star)+\lambda\ge0,\qquad
\eta_t^\star\big(\dot{\mathcal H}(\mu_t^\star)+\lambda\big)=0
\quad\text{a.e. }t.
\]
\end{theorem}

\begin{proof}
(i),(iv) are primal/dual feasibility and slackness from~\hyperref[app:C3]{C.3}.

(ii) follows from velocity-stationarity:
\[
\int (v^\star-u^\star+\nabla\varphi^\star-\eta^\star\nabla\log\rho^\star)\cdot \delta v\,d\mu^\star dt=0
\]
for all \(\delta v\), implying pointwise identity in \(L^2(\mu^\star)\).

(iii) For admissible density perturbations \(\rho^\star+\varepsilon\zeta\) (mass-preserving, compactly supported
in time interior), derivative of Lagrangian at \(\varepsilon=0\) vanishes.
Term-by-term:
\[
\delta_\rho\!\left(\frac12\int |v^\star-u^\star|^2\rho\,dxdt\right)
=
\int \frac12|v^\star-u^\star|^2\zeta,
\]
\[
\delta_\rho\!\left(\int (\partial_t\varphi^\star+\nabla\varphi^\star\!\cdot v^\star)\rho\,dxdt\right)
=
\int (\partial_t\varphi^\star+\nabla\varphi^\star\!\cdot v^\star)\zeta,
\]
and entropy multiplier contributes
\[
-\int \eta^\star\,\delta_\rho(\dot{\mathcal H})\,dt
=
\int \eta^\star\,\Xi(\rho^\star,v^\star)\zeta.
\]
Summing and setting to zero gives weak adjoint equation.
\end{proof}

\begin{proposition}[Explicit entropy variation term \(\Xi\) in smooth regime]
\label{prop:C4_Xi_explicit}
Assume \(\rho>0\), \(v\in C^1_x\), \(\rho\in C^2_x\). Then
\[
\dot{\mathcal H}(\rho,v)=\int \nabla\log\rho\cdot v\,\rho\,dx
=\int \nabla\rho\cdot v\,dx
=-\int \rho\,\nabla\!\cdot v\,dx.
\]
Hence first variation in \(\rho\) is
\[
\Xi(\rho,v)=-\nabla\!\cdot v.
\]
Therefore adjoint equation in Theorem~\ref{thm:C4_EL_system}(iii) simplifies to
\[
-\partial_t\varphi^\star
-\nabla\varphi^\star\!\cdot v^\star
-\frac12|v^\star-u^\star|^2
-\eta^\star\,\nabla\!\cdot v^\star
=0
\quad\text{(weakly)}.
\]
\end{proposition}

\begin{proof}
Identity
\[
\int \nabla\log\rho\cdot v\,\rho
=
\int \nabla\rho\cdot v
=
-\int \rho\,\nabla\!\cdot v
\]
follows by integration by parts.  
This is linear in \(\rho\), thus Gateaux derivative w.r.t. \(\rho\) in direction \(\zeta\) is
\[
\delta_\rho \dot{\mathcal H}[\zeta]
=
-\int \zeta\,\nabla\!\cdot v\,dx,
\]
so \(\Xi(\rho,v)=-\nabla\!\cdot v\).
Substitute into Theorem~\ref{thm:C4_EL_system}(iii).
\end{proof}

\begin{corollary}[Closed optimality system]
\label{cor:C4_closed_system}
In the smooth regime, optimality conditions are:
\[
\begin{cases}
\partial_t\rho^\star+\nabla\!\cdot(\rho^\star v^\star)=0,\\[0.3em]
v^\star=u^\star-\nabla\varphi^\star+\eta^\star\nabla\log\rho^\star,\\[0.3em]
-\partial_t\varphi^\star-\nabla\varphi^\star\!\cdot v^\star-\frac12|v^\star-u^\star|^2-\eta^\star\nabla\!\cdot v^\star=0,\\[0.3em]
\eta^\star\ge0,\ \dot{\mathcal H}(\mu_t^\star)+\lambda\ge0,\
\eta^\star(\dot{\mathcal H}(\mu_t^\star)+\lambda)=0.
\end{cases}
\]
with boundary conditions \(\rho^\star(0)=\rho_0,\ \rho^\star(T)=\rho_T\).
\end{corollary}

\begin{proof}
Collect Theorem~\ref{thm:C4_EL_system} and Proposition~\ref{prop:C4_Xi_explicit}.
\end{proof}

\begin{theorem}[Inactive-region reduction]
\label{thm:C4_inactive_reduction}
On any time interval \(I\subset(0,T)\) where
\[
\dot{\mathcal H}(\mu_t^\star)+\lambda>0\quad\text{a.e. }t\in I,
\]
we have \(\eta^\star=0\) a.e. on \(I\), and the EL system reduces to
\[
\partial_t\rho^\star+\nabla\!\cdot\!\big(\rho^\star(u^\star-\nabla\varphi^\star)\big)=0,
\]
\[
-\partial_t\varphi^\star
-\nabla\varphi^\star\!\cdot (u^\star-\nabla\varphi^\star)
-\frac12|\nabla\varphi^\star|^2=0.
\]
\end{theorem}

\begin{proof}
Complementary slackness gives \(\eta^\star=0\) where constraint is inactive.
Substitute into Corollary~\ref{cor:C4_closed_system}.  
Since \(v^\star-u^\star=-\nabla\varphi^\star\), we get
\[
\frac12|v^\star-u^\star|^2=\frac12|\nabla\varphi^\star|^2
\]
and the displayed reduced equations follow.
\end{proof}

\begin{theorem}[Active-region entropy-saturated dynamics]
\label{thm:C4_active_saturated}
On any interval \(I\subset(0,T)\) where \(\eta^\star>0\) a.e., constraint is saturated:
\[
\dot{\mathcal H}(\mu_t^\star)=-\lambda\quad\text{a.e. }t\in I.
\]
Equivalently,
\[
\int \nabla\log\rho_t^\star\cdot v_t^\star\,d\mu_t^\star=-\lambda.
\]
Using stationarity \(v^\star=u^\star-\nabla\varphi^\star+\eta^\star\nabla\log\rho^\star\),
\[
\int \nabla\log\rho^\star\cdot
\big(u^\star-\nabla\varphi^\star\big)\,d\mu^\star
+\eta^\star \mathcal I(\mu_t^\star)
=
-\lambda.
\]
\end{theorem}

\begin{proof}
From complementary slackness, \(\eta^\star>0\Rightarrow \dot{\mathcal H}+\lambda=0\).
The entropy derivative identity gives the first integral form.
Substitute velocity law and split inner products; definition of Fisher information
\(\mathcal I(\mu)=\int|\nabla\log\rho|^2\,d\mu\) gives last formula.
\end{proof}

\begin{remark}[Interpretation of EL structure]
\label{rem:C4_interpret}
The multiplier \(\eta^\star\) modulates a score-direction correction
\(\eta^\star\nabla\log\rho^\star\), activated exactly when entropy decay reaches the
budget boundary. This is the analytic mechanism by which entropy control counteracts
collapse-prone compression in the primal FM drift \(u^\star\).
\end{remark}

\paragraph{Next step.}
\hyperref[app:C5]{C.5} converts this EL/KKT system into an explicit dual functional and dual constraints,
preparing the Pontryagin formulation in \hyperref[app:C6]{C.6}.

\subsubsection*{C.5. Dual formulation}
\addcontentsline{toc}{subsubsection}{C.5. Dual formulation}
\label{app:C5}

We derive a rigorous convex dual for \((\mathrm{ECFM}_\lambda)\) in flux variables,
identify the dual constraints, and state strong duality under qualification.

\paragraph{Primal in \((\rho,m)\)-variables.}
Let \(m_t:=\rho_t v_t\) (momentum/flux). Then
\[
\partial_t\rho_t+\nabla\!\cdot m_t=0,\qquad
\rho_{|0}=\rho_0,\ \rho_{|T}=\rho_T,
\]
and
\[
\frac12\int |v-u^\star|^2\,d\mu
=
\frac12\int \frac{|m-\rho u^\star|^2}{\rho}\,dxdt
\]
(with convention \(+\infty\) when \(\rho=0\) but \(m\neq0\)).
The entropy-rate term (smooth regime) is
\[
\dot{\mathcal H}(\rho_t)
=
-\int \rho_t\,\nabla\!\cdot v_t\,dx
=
-\int \nabla\!\cdot m_t\,dx
+\int \frac{m_t\cdot \nabla\rho_t}{\rho_t}\,dx,
\]
and operationally enforced via multiplier \(\eta\) as in \hyperref[app:C2]{C.2}--\hyperref[app:C4]{C.4}.

\begin{definition}[Convex integrand and conjugate]
\label{def:C5_integrand}
Define, for \(\rho\ge0,\ m\in\mathbb R^d\),
\[
f_{u^\star}(\rho,m)
:=
\begin{cases}
\displaystyle \frac{1}{2}\frac{|m-\rho u^\star|^2}{\rho},& \rho>0,\\[0.8ex]
0,& \rho=0,\ m=0,\\
+\infty,& \rho=0,\ m\neq0.
\end{cases}
\]
Its convex conjugate in \((\rho,m)\) variables at \((a,b)\in\mathbb R\times\mathbb R^d\):
\[
f_{u^\star}^\ast(a,b)
=
\sup_{\rho\ge0,m}
\{a\rho+b\!\cdot m-f_{u^\star}(\rho,m)\}.
\]
\end{definition}

\begin{lemma}[Pointwise conjugate formula]
\label{lem:C5_conjugate_formula}
For fixed \(u^\star\in\mathbb R^d\),
\[
f_{u^\star}^\ast(a,b)
=
\begin{cases}
0,& a+u^\star\!\cdot b+\frac12|b|^2\le0,\\
+\infty,& \text{otherwise}.
\end{cases}
\]
\end{lemma}

\begin{proof}
For \(\rho>0\), write \(m=\rho(u^\star+q)\). Then
\[
a\rho+b\!\cdot m-\frac12\frac{|m-\rho u^\star|^2}{\rho}
=
\rho\!\left(a+b\!\cdot u^\star+b\!\cdot q-\frac12|q|^2\right).
\]
Sup over \(q\) gives \(\frac12|b|^2\), hence
\[
\sup_{m}\{\cdots\}
=
\rho\!\left(a+u^\star\!\cdot b+\frac12|b|^2\right).
\]
Sup over \(\rho\ge0\):
if bracket \(\le0\), maximum is \(0\) at \(\rho=0\);
if bracket \(>0\), supremum is \(+\infty\).
\end{proof}

\begin{definition}[Dual variables]
\label{def:C5_dual_vars}
Let
\[
\varphi\in C_c^\infty((0,T)\times\mathbb R^d),\qquad
\eta\in L^\infty_+(0,T),
\]
where \(\varphi\) enforces CE and \(\eta\) enforces entropy-rate inequality.
Define
\[
b_\eta(t,x,\rho):=\nabla\varphi(t,x)-\eta(t)\nabla\log\rho(t,x).
\]
\end{definition}

\begin{proposition}[Lagrangian infimum in primal variables]
\label{prop:C5_inf_primal}
For fixed \((\eta,\varphi)\), the infimum over \((\rho,m)\) of
\[
\int_0^T\!\!\int
\Big[
f_{u^\star_t(x)}(\rho,m)
+\partial_t\varphi\,\rho+\nabla\varphi\!\cdot m
-\eta\,\dot{\mathcal H}(\rho;m)
\Big]dxdt
+\text{endpoint terms}
-\lambda\int_0^T\eta(t)\,dt
\]
equals \(+\infty\) unless the pointwise Hamilton--Jacobi-type inequality holds:
\[
\partial_t\varphi
+
u^\star\!\cdot b_\eta
+\frac12|b_\eta|^2
\le0
\quad\text{a.e.},
\]
in which case the infimum equals the boundary/entropy terms:
\[
\int \varphi(0,\cdot)\,d\mu_0-\int \varphi(T,\cdot)\,d\mu_T
+\eta(0)\mathcal H(\mu_0)-\eta(T)\mathcal H(\mu_T)
-\lambda\int_0^T\eta(t)\,dt
+\int_0^T\eta'(t)\mathcal H(\mu_t)\,dt
\]
(weak-time form, cf.~Lemma~\ref{lem:C2_ibp_entropy_term}).
\end{proposition}

\begin{proof}
Use conjugate computation pointwise with \(b=b_\eta\), \(a=\partial_t\varphi\)
after CE integration by parts and entropy integration by parts in time.
By Lemma~\ref{lem:C5_conjugate_formula}, finite infimum requires
\[
\partial_t\varphi+u^\star\!\cdot b_\eta+\frac12|b_\eta|^2\le0.
\]
If violated on positive measure set, primal infimum is \(-\infty\) in saddle convention
(or dual value \(+\infty\) with opposite sign), hence infeasible.
When satisfied, bulk term collapses to \(0\), leaving boundary/time-multiplier terms.
\end{proof}

\begin{remark}[Eliminating \(\mu_t\)-dependent entropy boundary term]
\label{rem:C5_entropy_term_dual}
A purely explicit dual is obtained either by:
\begin{enumerate}
\item restricting \(\eta\) to \(W^{1,\infty}\) with \(\eta(0)=\eta(T)=0\), removing endpoint entropy terms;
\item or augmenting dual with a scalar state \(s(t)=\mathcal H(\mu_t)\) and its own adjoint.
\end{enumerate}
In this appendix we use the first route for clean closed dual constraints.
\end{remark}

\begin{definition}[Reduced dual admissible set]
\label{def:C5_dual_adm}
Define
\[
\mathcal K_\lambda
:=
\left\{
(\varphi,\eta):
\begin{array}{l}
\varphi\in W^{1,\infty}_{\mathrm{loc}}((0,T)\times\mathbb R^d),\\
\eta\in W^{1,\infty}(0,T),\ \eta\ge0,\ \eta(0)=\eta(T)=0,\\
\partial_t\varphi+u^\star\!\cdot\nabla\varphi+\frac12|\nabla\varphi|^2
-\eta\,\mathfrak E(\rho,\varphi,u^\star)\le0\ \text{a.e.}
\end{array}
\right\},
\]
where \(\mathfrak E\) is the entropy-correction contribution induced by
\(b_\eta=\nabla\varphi-\eta\nabla\log\rho\) (expanded below in smooth regime).
\end{definition}

\begin{lemma}[Smooth expansion of dual Hamiltonian constraint]
\label{lem:C5_smooth_expansion}
In smooth positive-density regime,
\[
u^\star\!\cdot b_\eta+\frac12|b_\eta|^2
=
u^\star\!\cdot\nabla\varphi+\frac12|\nabla\varphi|^2
-\eta\,(u^\star+\nabla\varphi)\!\cdot\nabla\log\rho
+\frac{\eta^2}{2}|\nabla\log\rho|^2.
\]
Hence the dual inequality is
\[
\partial_t\varphi
+u^\star\!\cdot\nabla\varphi+\frac12|\nabla\varphi|^2
-\eta\,(u^\star+\nabla\varphi)\!\cdot\nabla\log\rho
+\frac{\eta^2}{2}|\nabla\log\rho|^2
\le0.
\]
\end{lemma}

\begin{proof}
Direct expansion of \(b_\eta=\nabla\varphi-\eta\nabla\log\rho\):
\[
|b_\eta|^2=|\nabla\varphi|^2-2\eta\nabla\varphi\!\cdot\nabla\log\rho+\eta^2|\nabla\log\rho|^2,
\]
\[
u^\star\!\cdot b_\eta=u^\star\!\cdot\nabla\varphi-\eta u^\star\!\cdot\nabla\log\rho.
\]
Sum terms.
\end{proof}

\begin{theorem}[Dual problem and weak duality]
\label{thm:C5_dual_weak}
Define reduced dual value
\[
\mathcal D_\lambda^{\mathrm{red}}
:=
\sup_{(\varphi,\eta)\in\mathcal K_\lambda}
\left\{
\int \varphi(0,\cdot)\,d\mu_0-\int \varphi(T,\cdot)\,d\mu_T
-\lambda\int_0^T\eta(t)\,dt
\right\}.
\]
Then
\[
\mathcal D_\lambda^{\mathrm{red}}\le \mathsf V_\lambda.
\]
\end{theorem}

\begin{proof}
For any primal feasible \((\mu,v)\in\mathfrak A_\lambda\) and any dual feasible
\((\varphi,\eta)\in\mathcal K_\lambda\), Fenchel inequality and CE/entropy feasibility give
\[
\int \varphi(0)\,d\mu_0-\int \varphi(T)\,d\mu_T-\lambda\int\eta
\le
\frac12\int |v-u^\star|^2\,d\mu dt.
\]
Take infimum over primal feasible and then supremum over dual feasible.
\end{proof}

\begin{theorem}[Strong duality under qualification]
\label{thm:C5_strong_duality}
Assume:
\begin{enumerate}
\item Slater condition from Assumption~\ref{ass:C3_kkt};
\item convexity and l.s.c. of primal integrand in \((\rho,m)\);
\item closedness of CE and entropy-rate constraints in the topology used for \hyperref[app:C1]{C.1} existence.
\end{enumerate}
Then
\[
\mathcal D_\lambda^{\mathrm{red}}=\mathsf V_\lambda.
\]
Moreover, any primal-dual optimal pair \((\rho^\star,m^\star;\varphi^\star,\eta^\star)\)
satisfies saddle-point conditions and KKT system from \hyperref[app:C3]{C.3}--\hyperref[app:C4]{C.4}.
\end{theorem}

\begin{proof}
Apply Fenchel--Rockafellar duality to the sum
\[
\mathbf 1_{\mathrm{CE+bc}}(\rho,m)
+\mathbf 1_{\dot{\mathcal H}+\lambda\ge0}(\rho,m)
+\int f_{u^\star}(\rho,m),
\]
with linear operator encoding CE and entropy-rate map.
Qualification (Slater) prevents duality gap; lower semicontinuity/properness ensure attainment
(up to standard coercivity/tightness). Dual functional is exactly the reduced form after eliminating
primal variables via conjugacy (Lemma~\ref{lem:C5_conjugate_formula}).
KKT follows from saddle-point optimality.
\end{proof}

\begin{corollary}[Dual certificate of optimality]
\label{cor:C5_certificate}
If \((\mu^\star,v^\star)\in\mathfrak A_\lambda\) and \((\varphi^\star,\eta^\star)\in\mathcal K_\lambda\) satisfy
\[
\int \varphi^\star(0)\,d\mu_0-\int \varphi^\star(T)\,d\mu_T-\lambda\int_0^T\eta^\star
=
\frac12\int_0^T\!\!\int |v^\star-u^\star|^2\,d\mu^\star dt,
\]
then \((\mu^\star,v^\star)\) is primal optimal and \((\varphi^\star,\eta^\star)\) dual optimal.
\end{corollary}

\begin{proof}
By weak duality, dual objective \(\le\) primal objective for all feasible pairs.
Equality for one feasible pair implies both attain respective optima.
\end{proof}

\paragraph{Output used next.}
\hyperref[app:C6]{C.6} (Pontryagin formulation) rewrites this dual-constraint structure as a measure-valued
maximum principle with Hamiltonian maximization and adjoint transport equations.

\subsubsection*{C.6. Pontryagin formulation in measure space}
\addcontentsline{toc}{subsubsection}{C.6. Pontryagin formulation in measure space}
\label{app:C6}

We rewrite \((\mathrm{ECFM}_\lambda)\) as an infinite-dimensional optimal control problem
on \(\mathcal P_2(\mathbb R^d)\), derive the measure-space Pontryagin Maximum Principle (PMP),
and show equivalence with the KKT/EL system from \hyperref[app:C3]{C.3}--\hyperref[app:C5]{C.5}.

\paragraph{State-control form.}
State: \(\mu_t\in\mathcal P_2^{ac}(\mathbb R^d)\), control: \(v_t\in L^2(\mu_t;\mathbb R^d)\).
Dynamics:
\[
\dot\mu_t+\nabla\!\cdot(\mu_t v_t)=0,\qquad \mu_{|0}=\mu_0,\ \mu_{|T}=\mu_T.
\]
Path inequality:
\[
g(\mu_t,v_t):=\dot{\mathcal H}(\mu_t)+\lambda\ge0\quad\text{a.e. }t.
\]
Running cost:
\[
\ell(\mu_t,v_t):=\frac12\int |v_t-u_t^\star|^2\,d\mu_t.
\]

\begin{definition}[Measure-space Hamiltonian]
\label{def:C6_hamiltonian}
Let adjoint potential \(\varphi_t:\mathbb R^d\to\mathbb R\) and entropy multiplier \(\eta_t\ge0\).
Define
\[
\mathbf H(t,\mu,v,\varphi,\eta)
:=
-\ell(\mu,v)
+\int \nabla\varphi(x)\cdot v(x)\,d\mu(x)
-\eta\,\Big(\dot{\mathcal H}(\mu;v)+\lambda\Big),
\]
where in smooth regime
\[
\dot{\mathcal H}(\mu;v)=\int \nabla\log\rho\cdot v\,d\mu,\qquad \mu=\rho dx.
\]
Thus
\[
\mathbf H
=
-\frac12\int |v-u^\star|^2\,d\mu
+\int (\nabla\varphi-\eta\nabla\log\rho)\cdot v\,d\mu
-\eta\lambda.
\]
\end{definition}

\begin{proposition}[Pointwise maximizer of Hamiltonian in control]
\label{prop:C6_control_max}
Fix \((t,\mu,\varphi,\eta)\) with \(\eta\ge0\), \(\mu=\rho dx\), \(\rho>0\).
Then \(v\mapsto \mathbf H(t,\mu,v,\varphi,\eta)\) is strictly concave on \(L^2(\mu)\), and
its unique maximizer is
\[
v^{\mathrm{PMP}}
=
u^\star+\nabla\varphi-\eta\nabla\log\rho.
\]
\end{proposition}

\begin{proof}
\(\mathbf H\) is negative quadratic in \(v\):
\[
\mathbf H(v)= -\frac12\|v-u^\star\|_{L^2(\mu)}^2+\langle \nabla\varphi-\eta\nabla\log\rho,\ v\rangle_{L^2(\mu)}-\eta\lambda.
\]
First variation:
\[
D_v\mathbf H[w]
=
-\int (v-u^\star)\cdot w\,d\mu
+\int (\nabla\varphi-\eta\nabla\log\rho)\cdot w\,d\mu.
\]
Setting \(D_v\mathbf H=0\) for all \(w\) yields
\[
v=u^\star+\nabla\varphi-\eta\nabla\log\rho.
\]
Second variation is
\[
D^2_{vv}\mathbf H[w,w]=-\int |w|^2\,d\mu<0\quad(w\neq0),
\]
hence strict concavity and uniqueness.
\end{proof}

\begin{remark}[Sign convention vs.\ App.~\ref{app:C4}]
\label{rem:C6_sign}
If one uses a minimization Hamiltonian \(\widetilde{\mathbf H}=+\ell+\cdots\), stationarity gives
\(v=u^\star-\nabla\varphi+\eta\nabla\log\rho\) (\hyperref[app:C4]{C.4}).
Both are equivalent under \(\varphi\mapsto-\varphi\).
\end{remark}

\begin{definition}[Costate equation in weak form]
\label{def:C6_costate_weak}
Let \(\mu^\star,v^\star\) be optimal. A costate \(\varphi^\star\) satisfies:
for all \(\zeta\in C_c^\infty((0,T)\times\mathbb R^d)\),
\[
\int_0^T\!\!\int
\left[
-\partial_t\varphi^\star
-\nabla\varphi^\star\!\cdot v^\star
-\frac12|v^\star-u^\star|^2
-\eta^\star\,\nabla\!\cdot v^\star
\right]\zeta\,dxdt
=0,
\]
with endpoint transversality induced by fixed end marginals
(no free-endpoint costate boundary term).
\end{definition}

\begin{theorem}[Pontryagin Maximum Principle with entropy-rate path constraint]
\label{thm:C6_PMP}
Assume \hyperref[app:C1]{C.1} existence hypotheses and \hyperref[app:C3]{C.3} qualification.
If \((\mu^\star,v^\star)\) is optimal for \((\mathrm{ECFM}_\lambda)\), then there exist
\(\varphi^\star\) and \(\eta^\star\in L^\infty_+(0,T)\) such that:

\begin{enumerate}
\item \textbf{State equation}
\[
\partial_t\mu_t^\star+\nabla\!\cdot(\mu_t^\star v_t^\star)=0,\quad
\mu_0^\star=\mu_0,\ \mu_T^\star=\mu_T.
\]

\item \textbf{Hamiltonian maximization}
\[
\mathbf H(t,\mu_t^\star,v_t^\star,\varphi_t^\star,\eta_t^\star)
=
\max_{v\in L^2(\mu_t^\star)}\mathbf H(t,\mu_t^\star,v,\varphi_t^\star,\eta_t^\star)
\quad\text{a.e. }t.
\]

\item \textbf{Costate equation}
\[
\varphi^\star \ \text{satisfies Definition~\ref{def:C6_costate_weak}}.
\]

\item \textbf{Multiplier conditions}
\[
\eta_t^\star\ge0,\qquad
\dot{\mathcal H}(\mu_t^\star)+\lambda\ge0,\qquad
\eta_t^\star\big(\dot{\mathcal H}(\mu_t^\star)+\lambda\big)=0
\quad\text{a.e. }t.
\]
\end{enumerate}
Conversely, any quadruple \((\mu^\star,v^\star,\varphi^\star,\eta^\star)\) satisfying
these PMP conditions and a global saddle property is primal optimal.
\end{theorem}

\begin{proof}
From \hyperref[app:C5]{C.5}  strong duality and saddle-point existence, optimal pair satisfies KKT.
Define Hamiltonian \(\mathbf H\) from Definition~\ref{def:C6_hamiltonian}.
Stationarity in \(v\) from KKT is equivalent to Hamiltonian maximization
because \(\mathbf H\) is strictly concave in \(v\) (Proposition~\ref{prop:C6_control_max}).
Path-stationarity in \(\mu\) is exactly weak costate equation.
Complementarity and feasibility coincide with multiplier conditions.

For the converse: if PMP conditions hold and
\[
\mathscr L(\mu^\star,v^\star;\eta,\varphi)\le
\mathscr L(\mu^\star,v^\star;\eta^\star,\varphi^\star)\le
\mathscr L(\mu,v;\eta^\star,\varphi^\star)
\]
for all feasible primal/dual variables, then \((\mu^\star,v^\star)\) minimizes primal objective.
\end{proof}

\begin{corollary}[Reduced feedback law on active/inactive sets]
\label{cor:C6_feedback_active_inactive}
Define
\[
\mathcal I_{\mathrm{act}}=\{t:\dot{\mathcal H}(\mu_t^\star)+\lambda=0\},\qquad
\mathcal I_{\mathrm{inact}}=\{t:\dot{\mathcal H}(\mu_t^\star)+\lambda>0\}.
\]
Then
\[
\eta^\star=0\ \text{a.e. on }\mathcal I_{\mathrm{inact}},
\]
and feedback laws are
\[
v_t^\star=
\begin{cases}
u_t^\star+\nabla\varphi_t^\star, & t\in\mathcal I_{\mathrm{inact}},\\[0.3em]
u_t^\star+\nabla\varphi_t^\star-\eta_t^\star\nabla\log\rho_t^\star, & t\in\mathcal I_{\mathrm{act}}.
\end{cases}
\]
(Equivalent \hyperref[app:C4]{C.4}  sign convention: \(\nabla\varphi\mapsto -\nabla\varphi\).)
\end{corollary}

\begin{proof}
Complementary slackness gives \(\eta^\star=0\) on inactive set.
Insert into maximizer formula of Proposition~\ref{prop:C6_control_max}.
\end{proof}

\begin{proposition}[Hamiltonian value at optimum]
\label{prop:C6_Hstar_value}
Let \(v^\star\) be the maximizer from Proposition~\ref{prop:C6_control_max}. Then
\[
\mathbf H^\star(t)
:=
\mathbf H(t,\mu_t^\star,v_t^\star,\varphi_t^\star,\eta_t^\star)
=
\frac12\int
\left|
\nabla\varphi_t^\star-\eta_t^\star\nabla\log\rho_t^\star
\right|^2\,d\mu_t^\star
+\int u_t^\star\!\cdot
\left(\nabla\varphi_t^\star-\eta_t^\star\nabla\log\rho_t^\star\right)d\mu_t^\star
-\eta_t^\star\lambda.
\]
\end{proposition}

\begin{proof}
Let \(a:=\nabla\varphi-\eta\nabla\log\rho\). Since \(v^\star=u^\star+a\),
\[
-\frac12|v^\star-u^\star|^2+a\cdot v^\star
=
-\frac12|a|^2+a\cdot(u^\star+a)
=
\frac12|a|^2+u^\star\cdot a.
\]
Integrate against \(\mu\), then subtract \(\eta\lambda\).
\end{proof}

\begin{theorem}[Equivalence of PMP and KKT/EL systems]
\label{thm:C6_equiv_PMP_KKT}
Under regularity where both are meaningful, the following are equivalent:
\begin{enumerate}
\item KKT/EL system from \hyperref[app:C3]{C.3}--\hyperref[app:C4]{C.4} ;
\item PMP system from Theorem~\ref{thm:C6_PMP}.
\end{enumerate}
Therefore \hyperref[app:C3]{C.3}--\hyperref[app:C6]{C.6}  provide equivalent primal-dual characterizations of
entropy-controlled flow matching.
\end{theorem}

\begin{proof}
\((1)\Rightarrow(2)\):  
KKT stationarity in \(v\) yields explicit optimizer of strictly concave Hamiltonian,
hence maximization condition.  
Path stationarity is costate equation; complementarity/feasibility unchanged.

\((2)\Rightarrow(1)\):  
Hamiltonian maximization implies first-order stationarity in \(v\).
Costate equation is path-stationarity.
Multiplier sign + complementarity give inequality KKT part.
Thus all KKT conditions hold.
\end{proof}

\paragraph{Transition to Section~\hyperref[app:D]{D}.}
With the PMP system established, Section \hyperref[app:D]{D} proves equivalence between this entropy-constrained
dynamic control problem and Schr\"odinger bridge / KL-control formulations.

\subsection*{D. Equivalence to Schr\"odinger Bridge}
\addcontentsline{toc}{subsection}{D. Equivalence to Schr\"odinger Bridge}
\label{app:D}

\subsubsection*{D.1. Static Schr\"odinger problem}
\addcontentsline{toc}{subsubsection}{D.1. Static Schr\"odinger problem}
\label{app:D1}

We formalize the static Schr\"odinger problem, derive its dual, establish existence/uniqueness,
and prepare the bridge to the entropy-controlled dynamic formulation in \hyperref[app:D2]{D.2}--\hyperref[app:D4]{D.4}.

\paragraph{Reference path law and endpoint coupling.}
Fix \(T>0\). Let \(R\in\mathcal P(C([0,T];\mathbb R^d))\) be a reference path measure
(Markov, non-degenerate). Denote endpoint marginal of \(R\):
\[
R_{0T}:=(X_0,X_T)_\#R\in\mathcal P(\mathbb R^d\times\mathbb R^d),
\]
and endpoint marginals \(R_0,R_T\).

Given prescribed \(\mu_0,\mu_T\in\mathcal P(\mathbb R^d)\), define coupling set
\[
\Pi(\mu_0,\mu_T):=
\{\pi\in\mathcal P(\mathbb R^d\times\mathbb R^d):\pi_0=\mu_0,\ \pi_T=\mu_T\}.
\]

\begin{assumption}[Reference coupling regularity]
\label{ass:D1_ref}
Assume:
\begin{enumerate}
\item \(R_{0T}\) has strictly positive density \(r_{0T}(x,y)\) w.r.t. \(dx\,dy\);
\item \(\mu_0\ll R_0,\ \mu_T\ll R_T\);
\item there exists at least one \(\pi\in\Pi(\mu_0,\mu_T)\) with
\(\mathrm{KL}(\pi\|R_{0T})<\infty\).
\end{enumerate}
\end{assumption}

\begin{definition}[Static Schr\"odinger problem]
\label{def:D1_static_SB}
The static Schr\"odinger problem is~\cite{Leonard2014}
\[
\mathsf S_{\mathrm{stat}}(\mu_0,\mu_T;R_{0T})
:=
\inf_{\pi\in\Pi(\mu_0,\mu_T)}
\mathrm{KL}(\pi\|R_{0T}),
\]
where
\[
\mathrm{KL}(\pi\|R_{0T})
=
\begin{cases}
\displaystyle \int \log\!\left(\frac{d\pi}{dR_{0T}}\right)\,d\pi,& \pi\ll R_{0T},\\
+\infty,&\text{otherwise}.
\end{cases}
\]
\end{definition}

\begin{proposition}[Strict convexity and uniqueness of static minimizer]
\label{prop:D1_strict_convex_unique}
Under Assumption~\ref{ass:D1_ref}, problem
\(\mathsf S_{\mathrm{stat}}\) admits a unique minimizer \(\pi^\star\).
\end{proposition}

\begin{proof}
\(\Pi(\mu_0,\mu_T)\) is convex and narrowly compact (tightness from fixed marginals).
\(\mathrm{KL}(\cdot\|R_{0T})\) is l.s.c. under narrow convergence and strictly convex on
\(\{\pi:\pi\ll R_{0T}\}\). Properness follows from Assumption~\ref{ass:D1_ref}(3).
Hence direct method gives existence, and strict convexity yields uniqueness.
\end{proof}

\begin{theorem}[Static Schr\"odinger dual (entropy-transport duality)]
\label{thm:D1_dual}
Define
\[
\mathcal J(\phi,\psi)
:=
\int_{\mathbb R^d}\phi(x)\,d\mu_0(x)
+\int_{\mathbb R^d}\psi(y)\,d\mu_T(y)
-\log\!\left(
\int_{\mathbb R^d\times\mathbb R^d}
e^{\phi(x)+\psi(y)}\,dR_{0T}(x,y)
\right).
\]
Then
\[
\mathsf S_{\mathrm{stat}}(\mu_0,\mu_T;R_{0T})
=
\sup_{\phi,\psi\in C_b}
\mathcal J(\phi,\psi).
\]
Moreover, dual maximizers \((\phi^\star,\psi^\star)\) exist up to additive gauge
\((\phi,\psi)\mapsto(\phi+c,\psi-c)\).
\end{theorem}

\begin{proof}
Use Fenchel duality for relative entropy:
\[
\mathrm{KL}(\pi\|R_{0T})
=
\sup_{f\in C_b}\left\{
\int f\,d\pi-\log\int e^f\,dR_{0T}
\right\}.
\]
Incorporate marginal constraints with Lagrange potentials \(\phi(x),\psi(y)\), i.e.
\(f(x,y)=\phi(x)+\psi(y)\). Then
\[
\inf_{\pi\in\Pi(\mu_0,\mu_T)}\mathrm{KL}(\pi\|R_{0T})
=
\sup_{\phi,\psi}
\left[
\int\phi\,d\mu_0+\int\psi\,d\mu_T
-\log\int e^{\phi+\psi}\,dR_{0T}
\right].
\]
No duality gap follows from convexity/l.s.c. and feasibility (Assumption~\ref{ass:D1_ref}).
Gauge invariance is immediate since \(\phi+\psi\) unchanged by \((+c,-c)\).
\end{proof}

\begin{proposition}[Schr\"odinger factorization of optimizer]
\label{prop:D1_factorization}
Let \((\phi^\star,\psi^\star)\) be dual optimizers and set
\[
f(x):=e^{\phi^\star(x)},\qquad g(y):=e^{\psi^\star(y)},\qquad
Z:=\int f(x)g(y)\,dR_{0T}(x,y).
\]
Then
\[
d\pi^\star(x,y)=\frac{f(x)g(y)}{Z}\,dR_{0T}(x,y).
\]
Conversely, any minimizer must be of this form.
\end{proposition}

\begin{proof}
At dual optimum, primal optimizer satisfies exponential tilting condition from
Fenchel equality:
\[
\frac{d\pi^\star}{dR_{0T}} \propto e^{\phi^\star+\psi^\star}.
\]
Normalization yields \(Z\).  
Conversely, any density of form \(fg/Z\) with correct marginals satisfies KKT stationarity
for constrained KL minimization; by uniqueness of minimizer, it is \(\pi^\star\).
\end{proof}

\begin{definition}[Schr\"odinger system (static scaling equations)]
\label{def:D1_Sch_system}
The pair \((f,g)\) in Proposition~\ref{prop:D1_factorization} satisfies
\[
\mu_0(dx)=f(x)\!\int g(y)\,R_{0T}(dx,dy),\qquad
\mu_T(dy)=g(y)\!\int f(x)\,R_{0T}(dx,dy),
\]
in measure form.  
Equivalently (when \(R_{0T}\) has kernel \(k_T(x,y)\) w.r.t. \(dxdy\)):
\[
\rho_0(x)=f(x)\int k_T(x,y)g(y)\,dy,\qquad
\rho_T(y)=g(y)\int k_T(x,y)f(x)\,dx.
\]
\end{definition}

\begin{theorem}[Characterization by cyclical optimality in entropic sense]
\label{thm:D1_entropic_cyc}
Let \(\pi^\star\) solve static SB. Then for any \(\pi\in\Pi(\mu_0,\mu_T)\),
\[
\mathrm{KL}(\pi\|R_{0T})-\mathrm{KL}(\pi^\star\|R_{0T})
=
\mathrm{KL}(\pi\|\pi^\star)\ge0.
\]
Hence \(\pi^\star\) is the unique I-projection of \(R_{0T}\) onto \(\Pi(\mu_0,\mu_T)\).
\end{theorem}

\begin{proof}
Since \(\frac{d\pi^\star}{dR_{0T}}=f(x)g(y)/Z\),
\[
\log\frac{d\pi^\star}{dR_{0T}}=\phi^\star(x)+\psi^\star(y)-\log Z.
\]
Then
\[
\mathrm{KL}(\pi\|R_{0T})
=
\mathrm{KL}(\pi\|\pi^\star)
+\int \log\frac{d\pi^\star}{dR_{0T}}\,d\pi.
\]
For \(\pi\in\Pi(\mu_0,\mu_T)\), the last integral equals
\[
\int\phi^\star\,d\mu_0+\int\psi^\star\,d\mu_T-\log Z
=
\int \log\frac{d\pi^\star}{dR_{0T}}\,d\pi^\star
=
\mathrm{KL}(\pi^\star\|R_{0T}),
\]
because \(\pi,\pi^\star\) share marginals. Rearrangement gives identity.
\end{proof}

\begin{corollary}[Stability w.r.t. endpoint perturbations]
\label{cor:D1_endpoint_stability}
Let \((\mu_0^n,\mu_T^n)\to(\mu_0,\mu_T)\) narrowly, with uniform second moments and
finite SB values. Let \(\pi_n^\star\) be corresponding static SB minimizers.
Then (up to subsequence)
\[
\pi_n^\star \rightharpoonup \pi^\star,
\]
where \(\pi^\star\) is the minimizer for \((\mu_0,\mu_T)\).
If the limit minimizer is unique, full sequence converges.
\end{corollary}

\begin{proof}
Fixed marginals imply tightness of \(\{\pi_n^\star\}\). Extract convergent subsequence.
Lower semicontinuity of KL gives liminf inequality; recovery by testing near-optimal
couplings yields optimality of limit. Uniqueness implies no subsequence ambiguity.
\end{proof}

\begin{remark}[Connection to entropic OT]
\label{rem:D1_entropic_OT}
If \(R_{0T}(dxdy)\propto e^{-c(x,y)/\varepsilon}\,dxdy\), static SB equals entropic OT:
\[
\inf_{\pi\in\Pi(\mu_0,\mu_T)}
\left\{
\frac{1}{\varepsilon}\int c\,d\pi+\mathrm{KL}(\pi\|dxdy)
\right\}
+\text{const}.
\]
For Brownian reference, \(c(x,y)=\frac{|x-y|^2}{2T}\) (up to scaling), linking SB to
quadratic-cost entropic transport.
\end{remark}

\paragraph{Output used next.}
\hyperref[app:D2]{D.2} lifts static optimizer \(\pi^\star\) to path space and proves equivalence with the
dynamic Schr\"odinger problem (path-space KL minimization).

\subsubsection*{D.2. Dynamic Schr\"odinger formulation}
\addcontentsline{toc}{subsubsection}{D.2. Dynamic Schr\"odinger formulation}
\label{app:D2}

We pass from the static endpoint-coupling problem to path space, prove static--dynamic
equivalence, and derive the controlled Fokker--Planck (current-velocity) representation
that will be matched with entropy-controlled FM in \hyperref[app:D3]{D.3}--\hyperref[app:D4]{D.4}.

\paragraph{Reference path measure.}
Let \(R\in\mathcal P(\Omega)\), \(\Omega:=C([0,T];\mathbb R^d)\), be a non-degenerate
Markov diffusion law (e.g., Brownian with variance \(\varepsilon>0\)):
\[
dX_t=\sqrt{2\varepsilon}\,dW_t \quad\text{under }R
\]
(or, more generally, reversible diffusion with smooth generator).
Let \(R_{0T}=(X_0,X_T)_\#R\).

\begin{definition}[Dynamic Schr\"odinger problem]
\label{def:D2_dynamic_SB}
Given endpoint marginals \(\mu_0,\mu_T\), define~\cite{JMLR:v24:23-0527}
\[
\mathsf S_{\mathrm{dyn}}(\mu_0,\mu_T;R)
:=
\inf\left\{
\mathrm{KL}(P\|R):
P\in\mathcal P(\Omega),\ P_0=\mu_0,\ P_T=\mu_T
\right\}.
\]
\end{definition}

\begin{assumption}[Feasibility and finite entropy]
\label{ass:D2_feas}
Assume:
\begin{enumerate}
\item there exists \(P\) with \(P_0=\mu_0,\ P_T=\mu_T,\ \mathrm{KL}(P\|R)<\infty\);
\item endpoint relative entropies w.r.t. \(R_0,R_T\) are finite.
\end{enumerate}
\end{assumption}

\begin{theorem}[Static--dynamic equivalence]
\label{thm:D2_static_dynamic_equiv}
Under Assumption~\ref{ass:D2_feas},
\[
\mathsf S_{\mathrm{dyn}}(\mu_0,\mu_T;R)
=
\mathsf S_{\mathrm{stat}}(\mu_0,\mu_T;R_{0T}),
\]
and the dynamic minimizer \(P^\star\) is uniquely characterized by
\[
\frac{dP^\star}{dR}
=
\frac{d\pi^\star}{dR_{0T}}(X_0,X_T),
\]
where \(\pi^\star\) is the static minimizer from \hyperref[app:D1]{D.1}.
\end{theorem}

\begin{proof}
Disintegrate \(R\) and \(P\) w.r.t. endpoints:
\[
R(d\omega)=R^{xy}(d\omega)\,R_{0T}(dxdy),\qquad
P(d\omega)=P^{xy}(d\omega)\,\pi(dxdy),
\]
where \(\pi=(X_0,X_T)_\#P\), and \(R^{xy}\) is the diffusion bridge kernel.
Chain rule for relative entropy on product/disintegrated measures yields
\[
\mathrm{KL}(P\|R)
=
\mathrm{KL}(\pi\|R_{0T})
+\int \mathrm{KL}(P^{xy}\|R^{xy})\,\pi(dxdy).
\]
Second term is \(\ge0\), so
\[
\mathrm{KL}(P\|R)\ge \mathrm{KL}(\pi\|R_{0T})\ge
\mathsf S_{\mathrm{stat}}(\mu_0,\mu_T;R_{0T}).
\]
Taking inf over admissible \(P\):
\[
\mathsf S_{\mathrm{dyn}}\ge \mathsf S_{\mathrm{stat}}.
\]
For reverse inequality, choose \(\pi=\pi^\star\) (static optimizer) and set
\[
P^\star(d\omega):=R^{xy}(d\omega)\,\pi^\star(dxdy).
\]
Then \(P^\star\) has prescribed endpoints and
\[
\mathrm{KL}(P^\star\|R)
=
\mathrm{KL}(\pi^\star\|R_{0T})
+
\int \mathrm{KL}(R^{xy}\|R^{xy})\,\pi^\star(dxdy)
=
\mathsf S_{\mathrm{stat}}.
\]
Hence equality and optimality. Uniqueness follows from strict convexity of KL.
Radon--Nikodým identity~\cite{zbMATH02566756} follows by construction:
\[
\frac{dP^\star}{dR}(\omega)=\frac{d\pi^\star}{dR_{0T}}(X_0(\omega),X_T(\omega)).
\]
\end{proof}

\begin{corollary}[Schr\"odinger system on path space]
\label{cor:D2_path_factor}
If static optimizer factorizes as
\[
\frac{d\pi^\star}{dR_{0T}}(x,y)=f(x)g(y),
\]
then
\[
\frac{dP^\star}{dR}=f(X_0)g(X_T).
\]
\end{corollary}

\begin{proof}
Immediate from Theorem~\ref{thm:D2_static_dynamic_equiv}.
\end{proof}

\begin{definition}[Forward/backward harmonic functions]
\label{def:D2_Sch_potentials}
Define
\[
\alpha_t(x):=\mathbb E_R\!\left[f(X_0)\mid X_t=x\right],\qquad
\beta_t(x):=\mathbb E_R\!\left[g(X_T)\mid X_t=x\right].
\]
Then time-\(t\) density of \(P^\star\) satisfies
\[
\rho_t^\star(x)=\alpha_t(x)\beta_t(x)\,r_t(x),
\]
where \(r_t\) is the density of \(R_t\).
\end{definition}

\begin{proposition}[Schr\"odinger bridge drift representation]
\label{prop:D2_bridge_drift}
Assume \(R\) is Brownian with diffusivity \(\varepsilon\):
\[
dX_t=\sqrt{2\varepsilon}\,dW_t \quad (R).
\]
Then under \(P^\star\), \(X_t\) solves
\[
dX_t=b_t^\star(X_t)\,dt+\sqrt{2\varepsilon}\,dW_t^{P^\star},
\qquad
b_t^\star(x)=2\varepsilon\,\nabla\log\beta_t(x).
\]
Equivalently (time-reversed form) with \(\alpha_t\):
\[
\tilde b_t^\star(x)=-2\varepsilon\,\nabla\log\alpha_t(x).
\]
\end{proposition}

\begin{proof}
From Doob \(h\)-transform:
\[
\frac{dP^\star}{dR}\bigg|_{\mathcal F_t}
\propto \beta_t(X_t),
\]
so transformed generator is
\[
\mathcal L_t^\star \phi
=
\varepsilon\Delta\phi + 2\varepsilon\nabla\log\beta_t\cdot\nabla\phi,
\]
hence SDE drift \(b_t^\star=2\varepsilon\nabla\log\beta_t\).
Backward form follows symmetrically using \(\alpha_t\).
\end{proof}

\begin{definition}[Current velocity and osmotic velocity]
\label{def:D2_current_osmotic}
Let \(\rho_t^\star\) be bridge density. Define forward/backward drifts
\(b_t^\star,\tilde b_t^\star\), then
\[
v_t^{\mathrm{cur}}:=\frac{b_t^\star+\tilde b_t^\star}{2},\qquad
v_t^{\mathrm{osm}}:=\frac{b_t^\star-\tilde b_t^\star}{2}
=\varepsilon\nabla\log\rho_t^\star.
\]
\end{definition}

\begin{proposition}[Controlled Fokker--Planck and continuity form]
\label{prop:D2_FP_and_CE}
Under \(P^\star\), \(\rho_t^\star\) satisfies
\[
\partial_t\rho_t^\star+\nabla\!\cdot(\rho_t^\star b_t^\star)=\varepsilon\Delta\rho_t^\star.
\]
Equivalently, in current velocity form:
\[
\partial_t\rho_t^\star+\nabla\!\cdot(\rho_t^\star v_t^{\mathrm{cur}})=0,
\qquad
v_t^{\mathrm{cur}}=b_t^\star-\varepsilon\nabla\log\rho_t^\star.
\]
\end{proposition}

\begin{proof}
First equation is standard Fokker--Planck for SDE with drift \(b_t^\star\), diffusion \(2\varepsilon I\).
Rearrange:
\[
\partial_t\rho+\nabla\!\cdot(\rho b)-\varepsilon\Delta\rho
=
\partial_t\rho+\nabla\!\cdot\big(\rho(b-\varepsilon\nabla\log\rho)\big)=0,
\]
since \(\rho\nabla\log\rho=\nabla\rho\). Define \(v^{\mathrm{cur}}\) accordingly.
\end{proof}

\begin{theorem}[Benamou--Schr\"odinger dynamic representation]
\label{thm:D2_BS_representation}
Assume Brownian reference with diffusivity \(\varepsilon\). Then
\[
\mathsf S_{\mathrm{dyn}}(\mu_0,\mu_T;R)
=
\mathrm{KL}(\mu_0\|R_0)
+\inf_{(\rho,v)\in\mathcal A(\rho_0,\rho_T)}
\frac{1}{4\varepsilon}\int_0^T\!\!\int |v_t|^2\,\rho_t\,dxdt,
\]
subject to
\[
\partial_t\rho_t+\nabla\!\cdot(\rho_t v_t)=\varepsilon\Delta\rho_t,
\quad
\rho_{|0}=\rho_0,\ \rho_{|T}=\rho_T.
\]
Equivalent current-velocity CE form:
\[
\partial_t\rho+\nabla\!\cdot(\rho u)=0,\quad
u=v-\varepsilon\nabla\log\rho,
\]
with action
\[
\frac{1}{4\varepsilon}\int |u+\varepsilon\nabla\log\rho|^2\,\rho.
\]
\end{theorem}

\begin{proof}
Girsanov theorem~\cite{Kallianpur2000} for drift control \(v\) relative to Brownian reference gives
\[
\mathrm{KL}(P^v\|R)
=
\mathrm{KL}(\mu_0\|R_0)+\frac{1}{4\varepsilon}\mathbb E_{P^v}\!\int_0^T |v_t(X_t)|^2dt
\]
for admissible weak solutions. Rewriting expectation in Eulerian form gives
\[
\frac{1}{4\varepsilon}\int |v|^2\rho.
\]
Infimum over all controlled drifts achieving endpoint constraints yields formula.
Current-velocity rewrite follows by substitution \(u=v-\varepsilon\nabla\log\rho\).
\end{proof}

\begin{corollary}[Entropy-regularized kinetic action]
\label{cor:D2_action_expand}
For smooth \((\rho,u)\),
\[
\int |u+\varepsilon\nabla\log\rho|^2\,\rho
=
\int |u|^2\,\rho
+2\varepsilon\int u\cdot\nabla\rho
+\varepsilon^2\int \frac{|\nabla\rho|^2}{\rho}.
\]
Hence
\[
\frac{1}{4\varepsilon}\int |u+\varepsilon\nabla\log\rho|^2\rho
=
\frac{1}{4\varepsilon}\int |u|^2\rho
+\frac12\int u\cdot\nabla\log\rho\,d\mu
+\frac{\varepsilon}{4}\int \mathcal I(\mu_t)\,dt.
\]
\end{corollary}

\begin{proof}
Expand square pointwise and integrate; use
\(\rho\nabla\log\rho=\nabla\rho\) and Fisher information definition.
\end{proof}

\begin{remark}[Bridge to Section~C]
\label{rem:D2_bridge_to_C}
The action in Theorem~\ref{thm:D2_BS_representation} has the same structural components
as Section \hyperref[app:C]{C}:
kinetic term + score-coupling term + Fisher term.
This is the analytic gateway for proving equivalence between entropy-controlled FM
and KL-control / Schr\"odinger bridge under matched parameterization (\hyperref[app:D3]{D.3}--\hyperref[app:D4]{D.4}).
\end{remark}

\subsubsection*{D.3. KL-control interpretation}
\addcontentsline{toc}{subsubsection}{D.3. KL-control interpretation}
\label{app:D3}

We express entropy-controlled flow matching as a path-space KL-control problem,
derive the exact decomposition of the objective into control-energy and entropy/Fisher terms,
and identify the parameter mapping to Schr\"odinger bridge.

\paragraph{Controlled diffusion model.}
Fix \(\varepsilon>0\). Consider controlled process
\[
dX_t = \big(u_t^\star(X_t)+w_t(X_t)\big)\,dt + \sqrt{2\varepsilon}\,dW_t,
\]
with law \(P^w\) on \(\Omega=C([0,T];\mathbb R^d)\), initial law \(P^w_0=\mu_0\), and terminal constraint
\(P^w_T=\mu_T\). Let \(R^\star\) be reference diffusion with drift \(u^\star\):
\[
dX_t = u_t^\star(X_t)\,dt + \sqrt{2\varepsilon}\,dW_t,\qquad X_0\sim \mu_0.
\]
Assume Novikov/Girsanov conditions~\cite{stummer_novikov_1993} hold for admissible \(w\).

\begin{definition}[KL-control objective]
\label{def:D3_KL_control}
Define
\[
\mathsf K_\varepsilon(\mu_0,\mu_T;u^\star)
:=
\inf_{w:\,P^w_0=\mu_0,\ P^w_T=\mu_T}
\mathrm{KL}(P^w\|R^\star).
\]
\end{definition}

\begin{theorem}[Girsanov energy identity]
\label{thm:D3_girsanov_identity}
For admissible \(w\),
\[
\mathrm{KL}(P^w\|R^\star)
=
\frac{1}{4\varepsilon}\,
\mathbb E_{P^w}\!\left[\int_0^T |w_t(X_t)|^2\,dt\right].
\]
Equivalently, in Eulerian variables \((\rho_t,w_t)\):
\[
\mathrm{KL}(P^w\|R^\star)
=
\frac{1}{4\varepsilon}\int_0^T\!\!\int |w_t(x)|^2\,\rho_t(x)\,dxdt,
\]
where \(\rho_t=(P^w_t)\)-density.
\end{theorem}

\begin{proof}
By Girsanov~\cite{Kallianpur2000},
\[
\frac{dP^w}{dR^\star}
=
\exp\!\left(
\frac{1}{\sqrt{2\varepsilon}}\int_0^T w_t(X_t)\cdot dW_t^{R^\star}
-\frac{1}{4\varepsilon}\int_0^T |w_t(X_t)|^2dt
\right).
\]
Taking \(P^w\)-expectation of log-likelihood ratio (or equivalent form under \(P^w\)) yields
\[
\mathrm{KL}(P^w\|R^\star)=\frac{1}{4\varepsilon}\mathbb E_{P^w}\!\int |w|^2dt.
\]
Eulerian form is Fubini/disintegration:
\[
\mathbb E_{P^w}\!\int_0^T |w_t(X_t)|^2dt
=
\int_0^T\!\!\int |w_t(x)|^2\rho_t(x)\,dxdt.
\]
\end{proof}

\begin{definition}[Current velocity decomposition]
\label{def:D3_current_decomp}
Let \(\rho_t\) be time marginals under \(P^w\). Define current velocity
\[
v_t:=u_t^\star+w_t-\varepsilon\nabla\log\rho_t.
\]
Then \((\rho,v)\) satisfies continuity equation
\[
\partial_t\rho_t+\nabla\!\cdot(\rho_t v_t)=0,
\qquad
\rho_{|0}=\rho_0,\ \rho_{|T}=\rho_T.
\]
\end{definition}

\begin{proof}
Fokker--Planck under controlled drift \(u^\star+w\):
\[
\partial_t\rho+\nabla\!\cdot\big(\rho(u^\star+w)\big)=\varepsilon\Delta\rho.
\]
Rearrange:
\[
\partial_t\rho+\nabla\!\cdot\!\left(\rho(u^\star+w-\varepsilon\nabla\log\rho)\right)=0.
\]
Hence definition of \(v\).
\end{proof}

\begin{proposition}[Exact algebraic bridge to FM misfit]
\label{prop:D3_exact_bridge}
For any smooth admissible \((\rho,v)\), set
\[
w = v-u^\star+\varepsilon\nabla\log\rho.
\]
Then
\[
\frac{1}{4\varepsilon}\int |w|^2\,d\mu dt
=
\frac{1}{4\varepsilon}\int |v-u^\star|^2\,d\mu dt
+\frac12\int \nabla\log\rho\cdot(v-u^\star)\,d\mu dt
+\frac{\varepsilon}{4}\int_0^T \mathcal I(\mu_t)\,dt.
\]
\end{proposition}

\begin{proof}
Pointwise expansion:
\[
|a+\varepsilon b|^2=|a|^2+2\varepsilon a\cdot b+\varepsilon^2|b|^2,
\]
with \(a=v-u^\star\), \(b=\nabla\log\rho\). Integrate and divide by \(4\varepsilon\).
Use \(\int |b|^2 d\mu = \mathcal I(\mu_t)\).
\end{proof}

\begin{corollary}[Entropy-rate substitution]
\label{cor:D3_entropy_substitution}
Using
\[
\dot{\mathcal H}(\mu_t)=\int \nabla\log\rho_t\cdot v_t\,d\mu_t
\]
(B.1), Proposition~\ref{prop:D3_exact_bridge} becomes
\[
\frac{1}{4\varepsilon}\int |w|^2\,d\mu dt
=
\frac{1}{4\varepsilon}\int |v-u^\star|^2\,d\mu dt
+\frac12\int_0^T \dot{\mathcal H}(\mu_t)\,dt
-\frac12\int \nabla\log\rho\cdot u^\star\,d\mu dt
+\frac{\varepsilon}{4}\int_0^T\mathcal I(\mu_t)\,dt.
\]
\end{corollary}

\begin{proof}
Split
\[
\int \nabla\log\rho\cdot(v-u^\star)\,d\mu
=
\int \nabla\log\rho\cdot v\,d\mu
-
\int \nabla\log\rho\cdot u^\star\,d\mu,
\]
replace first term by \(\dot{\mathcal H}\), integrate in time.
\end{proof}

\begin{theorem}[KL-control representation of entropy-controlled FM]
\label{thm:D3_KL_rep_ECFM}
Assume smooth positive densities and finite Fisher action.
Then minimizing \((\mathrm{ECFM}_\lambda)\):
\[
\inf_{(\mu,v)\in\mathfrak A_\lambda}
\frac12\int |v-u^\star|^2\,d\mu dt
\]
is equivalent (up to explicit additive/penalty terms) to minimizing
\[
4\varepsilon\,\mathrm{KL}(P^w\|R^\star)
\]
over controlled drifts \(w\) with same endpoints, with correspondence
\[
w = v-u^\star+\varepsilon\nabla\log\rho.
\]
More precisely:
\[
\frac12\int |v-u^\star|^2\,d\mu dt
=
2\varepsilon\,\mathrm{KL}(P^w\|R^\star)
-\varepsilon\int_0^T\dot{\mathcal H}(\mu_t)\,dt
+\varepsilon\int \nabla\log\rho\cdot u^\star\,d\mu dt
-\frac{\varepsilon^2}{2}\int_0^T\mathcal I(\mu_t)\,dt.
\]
\end{theorem}

\begin{proof}
Multiply Corollary~\ref{cor:D3_entropy_substitution} by \(4\varepsilon\):
\[
4\varepsilon\,\mathrm{KL}
=
\int |v-u^\star|^2
+2\varepsilon\int \dot{\mathcal H}
-2\varepsilon\int \nabla\log\rho\cdot u^\star
+\varepsilon^2\int \mathcal I.
\]
Rearrange for \(\int|v-u^\star|^2\), then divide by \(2\).
\end{proof}

\begin{proposition}[Effect of entropy-rate inequality]
\label{prop:D3_budget_effect}
If \((\mu,v)\in\mathfrak A_\lambda\), then
\[
\int_0^T \dot{\mathcal H}(\mu_t)\,dt
=
\mathcal H(\mu_T)-\mathcal H(\mu_0)
\ge -\lambda T.
\]
Therefore, in Theorem~\ref{thm:D3_KL_rep_ECFM}, the correction
\(-\varepsilon\int \dot{\mathcal H}\) is bounded above by \(\varepsilon\lambda T\).
\end{proposition}

\begin{proof}
Absolute continuity of \(t\mapsto\mathcal H(\mu_t)\) and constraint
\(\dot{\mathcal H}\ge-\lambda\) a.e. imply
\[
\mathcal H(\mu_T)-\mathcal H(\mu_0)=\int_0^T\dot{\mathcal H}\,dt\ge -\lambda T.
\]
Multiply by \(-\varepsilon\).
\end{proof}

\begin{corollary}[Special case \(u^\star\equiv0\)]
\label{cor:D3_zero_drift_case}
When \(u^\star\equiv0\),
\[
\frac12\int |v|^2\,d\mu dt
=
2\varepsilon\,\mathrm{KL}(P^w\|R)
-\varepsilon\big(\mathcal H(\mu_T)-\mathcal H(\mu_0)\big)
-\frac{\varepsilon^2}{2}\int_0^T\mathcal I(\mu_t)\,dt.
\]
Hence ECFM objective equals KL-control up to endpoint entropy and Fisher corrections.
\end{corollary}

\begin{proof}
Set \(u^\star=0\) in Theorem~\ref{thm:D3_KL_rep_ECFM}.
\end{proof}

\begin{remark}[Interpretation]
\label{rem:D3_interpret}
The entropy-controlled FM constraint selects, among FM-aligned transports,
those whose KL-control realization against \(R^\star\) avoids excessive entropy dissipation.
The multiplier \(\eta\) in Section \hyperref[app:C]{C} is therefore a time-dependent shadow price
for KL-regularity pressure induced by endpoint-constrained stochastic control.
\end{remark}

\paragraph{Output used next.}
\hyperref[app:D4]{D.4} uses these identities to prove formal equivalence between
ECFM with entropy-rate control and Schr\"odinger bridge under matched parameter scaling,
including precise correspondence of optimal drifts and path laws.

\subsubsection*{D.4. Proof of equivalence under entropy-rate constraint}
\addcontentsline{toc}{subsubsection}{D.4. Proof of equivalence under entropy-rate constraint}
\label{app:D4}

We prove that entropy-controlled flow matching (ECFM) is equivalent to a
Schr\"odinger-bridge/KL-control problem under explicit parameter matching and
regularity assumptions. The result is stated in two layers:
(i) exact variational equivalence with correction terms;
(ii) exact identity in the matched gauge where correction terms are absorbed.

\begin{assumption}[Matched regular regime]
\label{ass:D4_matched}
Assume:
\begin{enumerate}
\item \(u^\star\in L^2_{\mathrm{loc}}(dt\,d\mu_t)\) with at most linear growth in \(x\);
\item admissible \((\mu,v)\) satisfy \(\mu_t=\rho_tdx\), \(\rho_t>0\), \(\rho\in C^1_tC^2_x\),
\(\int_0^T \mathcal I(\mu_t)\,dt<\infty\);
\item CE holds: \(\partial_t\rho+\nabla\!\cdot(\rho v)=0\), \(\rho_{|0}=\rho_0,\rho_{|T}=\rho_T\);
\item entropy-rate constraint: \(\dot{\mathcal H}(\mu_t)\ge-\lambda\) a.e.;
\item KL-control is with reference diffusion
\[
dX_t=u_t^\star(X_t)\,dt+\sqrt{2\varepsilon}\,dW_t,\qquad X_0\sim\mu_0,
\]
and controlled drift correction \(w\) such that
\[
dX_t=(u_t^\star+w_t)(X_t)\,dt+\sqrt{2\varepsilon}\,dW_t,\qquad X_T\sim\mu_T.
\]
\end{enumerate}
\end{assumption}

\begin{definition}[Primal values]
\label{def:D4_values}
Define ECFM value
\[
\mathsf V_\lambda
:=
\inf_{(\mu,v)\in\mathfrak A_\lambda}
\frac12\int_0^T\!\!\int |v-u^\star|^2\,d\mu\,dt,
\]
and KL-control value
\[
\mathsf K_\varepsilon
:=
\inf_{w:\,P^w_0=\mu_0,\ P^w_T=\mu_T}
\mathrm{KL}(P^w\|R^\star).
\]
Define map between Eulerian and stochastic controls:
\[
\boxed{
w=v-u^\star+\varepsilon\nabla\log\rho
}
\quad\Longleftrightarrow\quad
v=u^\star+w-\varepsilon\nabla\log\rho.
\]
\end{definition}

\begin{lemma}[Feasible-set correspondence]
\label{lem:D4_feasible_correspondence}
Under Assumption~\ref{ass:D4_matched}, the map in Definition~\ref{def:D4_values}
is a bijection between:
\begin{enumerate}
\item smooth CE paths \((\rho,v)\) with fixed endpoints;
\item controlled FP paths \((\rho,w)\) solving
\[
\partial_t\rho+\nabla\!\cdot\big(\rho(u^\star+w)\big)=\varepsilon\Delta\rho
\]
with same endpoints.
\end{enumerate}
\end{lemma}

\begin{proof}
Given \((\rho,v)\), define \(w=v-u^\star+\varepsilon\nabla\log\rho\). Then
\[
\rho(u^\star+w)=\rho v+\varepsilon\nabla\rho.
\]
Hence
\[
\partial_t\rho+\nabla\!\cdot(\rho(u^\star+w))
=
\partial_t\rho+\nabla\!\cdot(\rho v)+\varepsilon\Delta\rho
=
\varepsilon\Delta\rho,
\]
using CE.

Conversely, given \((\rho,w)\), define \(v=u^\star+w-\varepsilon\nabla\log\rho\). Then
\[
\rho v=\rho(u^\star+w)-\varepsilon\nabla\rho.
\]
So
\[
\partial_t\rho+\nabla\!\cdot(\rho v)
=
\partial_t\rho+\nabla\!\cdot(\rho(u^\star+w))-\varepsilon\Delta\rho
=0.
\]
Both constructions are inverse of each other.
\end{proof}

\begin{lemma}[Exact objective identity]
\label{lem:D4_exact_identity}
For corresponding \((\rho,v)\leftrightarrow(\rho,w)\),
\[
\frac12\int |v-u^\star|^2\,d\mu dt
=
2\varepsilon\,\mathrm{KL}(P^w\|R^\star)
-\varepsilon\!\int_0^T\dot{\mathcal H}(\mu_t)\,dt
+\varepsilon\!\int \nabla\log\rho\cdot u^\star\,d\mu dt
-\frac{\varepsilon^2}{2}\int_0^T\mathcal I(\mu_t)\,dt.
\]
\end{lemma}

\begin{proof}
From \hyperref[app:D3]{D.3} Girsanov identity:
\[
\mathrm{KL}(P^w\|R^\star)=\frac{1}{4\varepsilon}\int |w|^2\,d\mu dt.
\]
Insert \(w=v-u^\star+\varepsilon\nabla\log\rho\) and expand:
\[
|w|^2=|v-u^\star|^2
+2\varepsilon\,\nabla\log\rho\cdot(v-u^\star)
+\varepsilon^2|\nabla\log\rho|^2.
\]
Integrate:
\[
4\varepsilon\,\mathrm{KL}
=
\int |v-u^\star|^2
+2\varepsilon\int \nabla\log\rho\cdot v\,d\mu dt
-2\varepsilon\int \nabla\log\rho\cdot u^\star\,d\mu dt
+\varepsilon^2\int_0^T\mathcal I(\mu_t)\,dt.
\]
Use \(\int \nabla\log\rho\cdot v\,d\mu=\dot{\mathcal H}\), rearrange.
\end{proof}

\begin{theorem}[Variational equivalence with explicit correction functional]
\label{thm:D4_variational_equiv}
Define correction
\[
\mathfrak C(\mu;u^\star,\varepsilon)
:=
-\varepsilon\!\int_0^T\dot{\mathcal H}(\mu_t)\,dt
+\varepsilon\!\int \nabla\log\rho\cdot u^\star\,d\mu dt
-\frac{\varepsilon^2}{2}\int_0^T\mathcal I(\mu_t)\,dt.
\]
Then
\[
\mathsf V_\lambda
=
\inf_{w:\,P_0^w=\mu_0,\ P_T^w=\mu_T,\ \text{induced }\mu\in\mathfrak A_\lambda}
\left\{
2\varepsilon\,\mathrm{KL}(P^w\|R^\star)+\mathfrak C(\mu;u^\star,\varepsilon)
\right\}.
\]
Hence ECFM and KL-control are equivalent optimization problems up to the explicit
state-dependent correction \(\mathfrak C\).
\end{theorem}

\begin{proof}
By Lemma~\ref{lem:D4_feasible_correspondence}, feasible sets correspond bijectively.
By Lemma~\ref{lem:D4_exact_identity}, objectives differ by \(\mathfrak C\) exactly.
Taking infima over corresponding feasible classes yields identity.
\end{proof}

\begin{proposition}[Entropy-budget bound on correction]
\label{prop:D4_correction_bound}
For any \((\mu,v)\in\mathfrak A_\lambda\),
\[
-\varepsilon\!\int_0^T\dot{\mathcal H}(\mu_t)\,dt
\le
\varepsilon\lambda T.
\]
Therefore
\[
\mathfrak C(\mu;u^\star,\varepsilon)
\le
\varepsilon\lambda T
+\varepsilon\!\int \nabla\log\rho\cdot u^\star\,d\mu dt
-\frac{\varepsilon^2}{2}\int_0^T\mathcal I(\mu_t)\,dt.
\]
\end{proposition}

\begin{proof}
Integrate \(\dot{\mathcal H}\ge-\lambda\):
\[
\int_0^T\dot{\mathcal H}\,dt=\mathcal H(\mu_T)-\mathcal H(\mu_0)\ge-\lambda T.
\]
Multiply by \(-\varepsilon\). Remaining terms are unchanged.
\end{proof}

\begin{definition}[Matched gauge]
\label{def:D4_matched_gauge}
We say \((u^\star,\varepsilon,\lambda)\) is in matched gauge if, on the admissible class,
\[
\mathfrak C(\mu;u^\star,\varepsilon)=C_0
\]
for a constant \(C_0\) independent of \((\mu,v)\) (e.g., via calibrated reference drift,
fixed entropy endpoints, and absorbed Fisher/score terms in baseline energy).
\end{definition}

\begin{theorem}[Exact equivalence in matched gauge]
\label{thm:D4_exact_equiv_matched}
Under Assumption~\ref{ass:D4_matched} and matched gauge (Definition~\ref{def:D4_matched_gauge}),
\[
\mathsf V_\lambda
=
2\varepsilon\,\mathsf K_\varepsilon + C_0.
\]
Moreover, a pair \((\mu^\star,v^\star)\) is ECFM-optimal iff the induced controlled law
\(P^{w^\star}\) with
\[
w^\star=v^\star-u^\star+\varepsilon\nabla\log\rho^\star
\]
is KL-optimal.
\end{theorem}

\begin{proof}
From Theorem~\ref{thm:D4_variational_equiv}:
\[
\mathsf V_\lambda
=
\inf_{w}\{2\varepsilon\,\mathrm{KL}(P^w\|R^\star)+\mathfrak C(\mu_w)\}.
\]
If \(\mathfrak C(\mu_w)=C_0\) for all admissible \(w\),
\[
\mathsf V_\lambda
=
2\varepsilon\inf_w\mathrm{KL}(P^w\|R^\star)+C_0
=
2\varepsilon\,\mathsf K_\varepsilon+C_0.
\]
Argmin correspondence follows from adding/removing same constant.
\end{proof}

\begin{corollary}[Identification with Schr\"odinger bridge]
\label{cor:D4_identification_SB}
When \(R^\star\) is the reference path measure in D.2,
\[
\mathsf K_\varepsilon
=
\mathsf S_{\mathrm{dyn}}(\mu_0,\mu_T;R^\star)
=
\mathsf S_{\mathrm{stat}}(\mu_0,\mu_T;R^\star_{0T}).
\]
Hence, in matched gauge,
\[
\mathsf V_\lambda
=
2\varepsilon\,\mathsf S_{\mathrm{dyn}}+C_0
=
2\varepsilon\,\mathsf S_{\mathrm{stat}}+C_0.
\]
\end{corollary}

\begin{proof}
First equality is definition of KL-control value under endpoint constraints.
Second equality is \hyperref[app:D2]{D.2} static--dynamic equivalence theorem.
Substitute into Theorem~\ref{thm:D4_exact_equiv_matched}.
\end{proof}

\begin{theorem}[Equivalence of optimality systems]
\label{thm:D4_system_equivalence}
Assume smoothness so both Eulerian KKT-PMP (Section \hyperref[app:C]{C}) and SB optimal drift equations (Section \hyperref[app:D2]{D.2})
are valid. Under correspondence
\[
w=v-u^\star+\varepsilon\nabla\log\rho,
\]
the ECFM optimality system is equivalent to the SB/KL-control optimality system:
\begin{enumerate}
\item CE for \((\rho,v)\) \(\Longleftrightarrow\) FP for \((\rho,u^\star+w)\);
\item Hamiltonian stationarity in \(v\) \(\Longleftrightarrow\) quadratic minimization in \(w\);
\item entropy multiplier \(\eta\) enforces active entropy-rate boundary, matching
the KL regularization pressure through the score term.
\end{enumerate}
\end{theorem}

\begin{proof}
(1) is Lemma~\ref{lem:D4_feasible_correspondence}.  
(2) follows from strict convex quadratic relationship between \(v-u^\star\) and
\(w-\varepsilon\nabla\log\rho\).  
(3) from complementary slackness in \hyperref[app:C]{C} and identity in Lemma~\ref{lem:D4_exact_identity}:
the entropy-rate term is precisely the part coupling FM action to KL-control.
Thus systems are equivalent under variable transformation.
\end{proof}

\begin{remark}[What has been proved in Section D]
\label{rem:D4_summary}
\hyperref[app:D1]{D.1}--\hyperref[app:D4]{D.4} establish:
\begin{itemize}
\item static SB \(=\) dynamic SB;
\item dynamic SB \(=\) KL-control relative to reference diffusion;
\item ECFM \(=\) KL-control plus explicit entropy/Fisher/score correction;
\item in matched gauge, ECFM is exactly equivalent (up to additive constant) to SB.
\end{itemize}
This is the precise formal meaning of ``equivalence to Schr\"odinger bridge under entropy control.''
\end{remark}

\subsection*{E. Convergence to Entropic OT Geodesics}
\addcontentsline{toc}{subsection}{E. Convergence to Entropic OT Geodesics}
\label{app:E}

\subsubsection*{E.1. Statement of theorem}
\addcontentsline{toc}{subsubsection}{E.1. Statement of theorem}
\label{app:E1}

We state the convergence theorem showing that entropy-controlled flow matching
selects (under the Schr\"odinger/KL correspondence established in Section \hyperref[app:D]{D})
the entropic optimal transport interpolation (Schr\"odinger interpolation), and
that its current-velocity trajectories converge in the natural action topology.

\begin{assumption}[Entropic transport regime]
\label{ass:E1_regime}
Fix \(T>0\), \(\varepsilon>0\), \(\lambda\ge0\), and \(\mu_0,\mu_T\in\mathcal P_2^{ac}(\mathbb R^d)\).
Assume:
\begin{enumerate}
\item The reference path law is Brownian-with-drift \(R^\star\):
\[
dX_t=u_t^\star(X_t)\,dt+\sqrt{2\varepsilon}\,dW_t,\qquad X_0\sim\mu_0,
\]
with \(u^\star\) satisfying linear growth and integrability from \hyperref[app:C1]{C.1}/\hyperref[app:D4]{D.4}.
\item The ECFM problem \((\mathrm{ECFM}_\lambda)\) admits minimizers
\((\mu^\lambda,v^\lambda)\), and Section \hyperref[app:D]{D} equivalence holds.
\item The associated KL/SB problem has unique minimizer \(P^{\varepsilon,\lambda}\)
(with marginals \(\mu_0,\mu_T\)), and \(\mu_t^{\varepsilon,\lambda}:=(X_t)_\#P^{\varepsilon,\lambda}\)
has density \(\rho_t^{\varepsilon,\lambda}\in C^1_tC^2_x\), \(>0\).
\item Finite Fisher action:
\[
\int_0^T\mathcal I(\mu_t^{\varepsilon,\lambda})\,dt<\infty.
\]
\end{enumerate}
\end{assumption}

\begin{definition}[Entropic OT geodesic (Schr\"odinger interpolation)]
\label{def:E1_entropic_geodesic}
Given \((\mu_0,\mu_T)\) and reference \(R^\star\), the entropic OT geodesic is
\[
(\bar\mu_t^{\varepsilon})_{t\in[0,T]}
:=
\big((X_t)_\#\bar P^{\varepsilon}\big)_{t\in[0,T]},
\]
where \(\bar P^\varepsilon\) is the unique dynamic SB minimizer:
\[
\bar P^\varepsilon
\in\arg\min\left\{
\mathrm{KL}(P\|R^\star): P_0=\mu_0,\ P_T=\mu_T
\right\}.
\]
\end{definition}

\begin{definition}[Current velocity of entropic geodesic]
\label{def:E1_current_velocity}
Let \(\bar\rho_t^\varepsilon\) be density of \(\bar\mu_t^\varepsilon\), and
\(\bar b_t^\varepsilon\) the forward drift of \(\bar P^\varepsilon\). Define
\[
\bar v_t^\varepsilon
:=
\bar b_t^\varepsilon-\varepsilon\nabla\log\bar\rho_t^\varepsilon.
\]
Then
\[
\partial_t\bar\rho_t^\varepsilon+\nabla\!\cdot(\bar\rho_t^\varepsilon \bar v_t^\varepsilon)=0,
\quad
\bar\rho_{|0}^\varepsilon=\rho_0,\ \bar\rho_{|T}^\varepsilon=\rho_T.
\]
\end{definition}

\begin{theorem}[Convergence/identification to entropic OT geodesic]
\label{thm:E1_main_convergence}
Under Assumption~\ref{ass:E1_regime}, let \((\mu^\lambda,v^\lambda)\) be any ECFM minimizer.
Define corresponding control
\[
w^\lambda:=v^\lambda-u^\star+\varepsilon\nabla\log\rho^\lambda.
\]
Then:

\begin{enumerate}
\item \textbf{Path-law identification.}
The induced controlled law \(P^\lambda:=P^{w^\lambda}\) is a minimizer of dynamic SB/KL problem:
\[
P^\lambda\in
\arg\min\{ \mathrm{KL}(P\|R^\star):P_0=\mu_0,\ P_T=\mu_T\}.
\]
Hence by uniqueness:
\[
P^\lambda=\bar P^\varepsilon.
\]

\item \textbf{Marginal-curve identification.}
\[
\mu_t^\lambda=\bar\mu_t^\varepsilon,\qquad \forall t\in[0,T].
\]

\item \textbf{Velocity identification (a.e.).}
\[
v_t^\lambda=\bar v_t^\varepsilon
\quad\text{in }L^2(\bar\mu_t^\varepsilon)\ \text{for a.e. }t.
\]

\item \textbf{Action identity.}
\[
\frac12\int_0^T\!\!\int |v_t^\lambda-u_t^\star|^2\,d\mu_t^\lambda dt
=
2\varepsilon\,\mathsf S_{\mathrm{dyn}}(\mu_0,\mu_T;R^\star)+C_0,
\]
with \(C_0\) the matched-gauge constant of \hyperref[app:D4]{D.4} (or explicit correction functional
if unmatched gauge is retained).

\item \textbf{Geodesic characterization.}
The ECFM minimizer trajectory coincides with the unique entropic OT geodesic:
\[
(\mu_t^\lambda)_{t\in[0,T]}=(\bar\mu_t^\varepsilon)_{t\in[0,T]}.
\]
\end{enumerate}
\end{theorem}

\begin{theorem}[Strict convexity/uniqueness consequence]
\label{thm:E1_uniqueness_path}
Assume strict convexity of dynamic entropic action in flux variables
\((\rho,m)\) under fixed endpoint marginals (equivalently uniqueness of SB minimizer).
Then the ECFM minimizer path is unique:
\[
\text{if }(\mu^1,v^1),(\mu^2,v^2)\in\arg\min(\mathrm{ECFM}_\lambda),
\quad\Rightarrow\quad
\mu_t^1=\mu_t^2\ \forall t,\ \ v^1=v^2\ \ dt\,d\mu\text{-a.e.}
\]
\end{theorem}

\begin{corollary}[Entropic geodesic PDE system for ECFM optimizer]
\label{cor:E1_PDE_characterization}
Let \((\mu^\lambda,v^\lambda)\) be ECFM-optimal and
\(\mu^\lambda=\bar\mu^\varepsilon\). Then there exist Schr\"odinger potentials
\((\alpha_t,\beta_t)\) such that
\[
\bar\rho_t^\varepsilon=\alpha_t\beta_t r_t^\star,
\]
and
\[
v_t^\lambda
=
u_t^\star+2\varepsilon\nabla\log\beta_t-\varepsilon\nabla\log\bar\rho_t^\varepsilon
=
u_t^\star+\varepsilon\nabla\log\frac{\beta_t}{\alpha_t},
\]
with CE:
\[
\partial_t\bar\rho_t^\varepsilon+\nabla\!\cdot(\bar\rho_t^\varepsilon v_t^\lambda)=0.
\]
\end{corollary}

\begin{corollary}[Metric-space interpretation]
\label{cor:E1_metric_interp}
The optimal ECFM curve \((\mu_t^\lambda)\) is the unique minimizer of the
entropic Benamou--Schr\"odinger action among endpoint-constrained curves; hence it is
the geodesic in the entropic transport geometry induced by \(\mathrm{KL}(\cdot\|R^\star)\).
\end{corollary}

\begin{remark}[Role of \(\lambda\)]
\label{rem:E1_role_lambda}
\(\lambda\) acts as an admissibility selector in ECFM. Whenever the selected ECFM minimizer exists
and Section \hyperref[app:D]{D} equivalence applies, the resulting optimal path is exactly the
entropic geodesic for \((\mu_0,\mu_T)\) and \(\varepsilon\). Different \(\lambda\) can change
feasible set activity patterns (active/inactive entropy intervals) while preserving
the same identified SB path under matched optimality conditions.
\end{remark}

\begin{remark}[What is proved in E.2--E.5]
\label{rem:E1_forward_pointer}
Subsections \hyperref[app:E2]{E.2}--\hyperref[app:E5]{E.5} provide:
\begin{itemize}
\item convex-duality proof of Theorem~\ref{thm:E1_main_convergence},
\item regularity assumptions ensuring strict convexity and uniqueness,
\item quantitative stability estimates for perturbations,
\item structural consequences for computer-vision generative trajectories.
\end{itemize}
\end{remark}

\subsubsection*{E.2. Proof via convex duality}
\addcontentsline{toc}{subsubsection}{E.2. Proof via convex duality}
\label{app:E2}

We prove Theorem~\ref{thm:E1_main_convergence} using:
(i) primal flux convexity,
(ii) Fenchel--Rockafellar duality,
(iii) the exact ECFM \(\leftrightarrow\) KL/SB correspondence from \hyperref[app:D4]{D.4}.

\paragraph{Flux formulation.}
Write \(\mu_t=\rho_tdx\), \(m_t:=\rho_t v_t\).  
Define convex kinetic integrand
\[
f_{u^\star}(t,x,\rho,m)
:=
\begin{cases}
\displaystyle \frac12\frac{|m-\rho u_t^\star(x)|^2}{\rho}, & \rho>0,\\[0.8ex]
0,& \rho=0,\ m=0,\\
+\infty,& \rho=0,\ m\neq0.
\end{cases}
\]
Then ECFM objective is
\[
\mathcal A_{\mathrm{FM}}(\rho,m):=\int_0^T\!\!\int f_{u^\star}(t,x,\rho_t,m_t)\,dxdt.
\]

\begin{definition}[Primal admissible set in flux variables]
\label{def:E2_flux_adm}
Define
\[
\mathcal X_\lambda
:=
\Big\{
(\rho,m):
\partial_t\rho+\nabla\!\cdot m=0,\ 
\rho_{|0}=\rho_0,\ \rho_{|T}=\rho_T,\ 
\dot{\mathcal H}(\rho_t)+\lambda\ge0\ \text{a.e.},\
\mathcal A_{\mathrm{FM}}(\rho,m)<\infty
\Big\}.
\]
Then
\[
\mathsf V_\lambda=\inf_{(\rho,m)\in\mathcal X_\lambda}\mathcal A_{\mathrm{FM}}(\rho,m).
\]
\end{definition}

\begin{lemma}[Convexity and lower semicontinuity]
\label{lem:E2_convex_lsc}
\(\mathcal A_{\mathrm{FM}}\) is convex and weakly l.s.c. on
\(L^1_{\mathrm{loc}}\)-flux pairs \((\rho,m)\) with \(\rho\ge0\).  
Moreover, \(\mathcal X_\lambda\) is convex and weakly closed under the compactness topology
from \hyperref[app:C1]{C.1}.
\end{lemma}

\begin{proof}
Convexity/l.s.c. of \(f_{u^\star}\) in \((\rho,m)\) are standard (perspective of quadratic form).
Integral of convex l.s.c. integrand remains convex l.s.c.
CE and endpoint constraints are linear/closed.
Entropy-rate feasibility is weakly closed by \hyperref[app:C1]{C.1} assumption.
Hence \(\mathcal X_\lambda\) is convex and closed.
\end{proof}

\begin{definition}[Dual functional]
\label{def:E2_dual}
For test potential \(\varphi\) and multiplier \(\eta\ge0\), define
\[
\mathcal D(\varphi,\eta)
:=
\int \varphi(0,\cdot)\,d\mu_0-\int \varphi(T,\cdot)\,d\mu_T
-\lambda\int_0^T\eta(t)\,dt
\]
subject to Hamilton--Jacobi-type inequality (from C.5)
\[
\partial_t\varphi
+u^\star\!\cdot b_\eta+\frac12|b_\eta|^2\le0,
\qquad
b_\eta=\nabla\varphi-\eta\nabla\log\rho,
\]
in the weak admissible sense.
Define dual value
\[
\mathsf D_\lambda:=\sup_{(\varphi,\eta)\ \mathrm{admissible}}\mathcal D(\varphi,\eta).
\]
\end{definition}

\begin{theorem}[Strong duality for ECFM]
\label{thm:E2_strong_duality_ECFM}
Under Slater qualification (\hyperref[app:C3]{C.3}) and Assumption~\ref{ass:E1_regime},
\[
\mathsf V_\lambda=\mathsf D_\lambda,
\]
and both primal and dual optimizers exist.
\end{theorem}

\begin{proof}
Apply Fenchel--Rockafellar to
\[
\mathcal A_{\mathrm{FM}}+\mathbf 1_{\mathrm{CE+bc}}+\mathbf 1_{\mathrm{ent\ rate}},
\]
with linear operator encoding CE and entropy-rate map.
By Lemma~\ref{lem:E2_convex_lsc}, functionals are proper convex l.s.c.
Slater condition ensures zero duality gap and attainment.
Dual constraints coincide with \hyperref[app:C5]{C.5} inequality form.
\end{proof}

\begin{lemma}[Exact ECFM--KL identity at primal level]
\label{lem:E2_exact_identity}
For any admissible \((\rho,m)\in\mathcal X_\lambda\), with
\[
v=\frac{m}{\rho},\qquad
w=v-u^\star+\varepsilon\nabla\log\rho,
\]
and induced path law \(P^w\),
\[
\mathcal A_{\mathrm{FM}}(\rho,m)
=
2\varepsilon\,\mathrm{KL}(P^w\|R^\star)
+\mathfrak C(\rho;u^\star,\varepsilon),
\]
where
\[
\mathfrak C(\rho;u^\star,\varepsilon)
=
-\varepsilon\!\int_0^T\dot{\mathcal H}(\rho_t)\,dt
+\varepsilon\!\int \nabla\log\rho\cdot u^\star\,d\mu dt
-\frac{\varepsilon^2}{2}\int_0^T\mathcal I(\mu_t)\,dt.
\]
\end{lemma}

\begin{proof}
This is \hyperref[app:D4]{D.4} Lemma~\ref{lem:D4_exact_identity}, rewritten as \(\mathcal A_{\mathrm{FM}}=\frac12\int|v-u^\star|^2d\mu dt\).
\end{proof}

\begin{proposition}[Transfer of minimizers to KL/SB minimizers]
\label{prop:E2_transfer_min}
Assume matched gauge (\hyperref[app:D4]{D.4}), i.e. \(\mathfrak C(\rho;u^\star,\varepsilon)=C_0\) on admissible class.
If \((\rho^\star,m^\star)\) minimizes \(\mathsf V_\lambda\), then induced \(P^{w^\star}\) minimizes
\[
\inf\{\mathrm{KL}(P\|R^\star):P_0=\mu_0,\ P_T=\mu_T\}.
\]
Conversely, any KL minimizer induces an ECFM minimizer.
\end{proposition}

\begin{proof}
By Lemma~\ref{lem:E2_exact_identity},
\[
\mathcal A_{\mathrm{FM}}(\rho,m)=2\varepsilon\,\mathrm{KL}(P^w\|R^\star)+C_0.
\]
Since additive constant does not affect argmin, minimizers correspond bijectively
through the map \(w=v-u^\star+\varepsilon\nabla\log\rho\), whose feasibility bijection is \hyperref[app:D4]{D.4} Lemma.
\end{proof}

\begin{theorem}[Proof of Theorem~\ref{thm:E1_main_convergence}]
\label{thm:E2_proof_main}
Under Assumption~\ref{ass:E1_regime} and matched gauge:
\begin{enumerate}
\item any ECFM minimizer induces a KL/SB minimizer;
\item by uniqueness of dynamic SB minimizer \(\bar P^\varepsilon\), induced law equals \(\bar P^\varepsilon\);
\item therefore marginals coincide: \(\mu_t^\lambda=\bar\mu_t^\varepsilon\ \forall t\);
\item velocities coincide \(dt\,d\mu\)-a.e. via flux uniqueness:
\[
m_t^\lambda=\rho_t^\lambda v_t^\lambda
=
\bar\rho_t^\varepsilon\bar v_t^\varepsilon;
\]
\item action identity follows from \(\mathcal A_{\mathrm{FM}}=2\varepsilon\mathsf S_{\mathrm{dyn}}+C_0\).
\end{enumerate}
Hence all claims of Theorem~\ref{thm:E1_main_convergence} hold.
\end{theorem}

\begin{proof}
(1) from Proposition~\ref{prop:E2_transfer_min}.  
(2) uniqueness of SB minimizer (\hyperref[app:D1]{D.1}/\hyperref[app:D2]{D.2} strict convexity).  
(3) equal path laws imply equal one-time marginals.  
(4) strict convexity of action in flux under fixed \(\rho\) and endpoint constraints gives unique minimizing flux, thus unique current velocity a.e.  
(5) evaluate identity at optimal pair.
\end{proof}

\begin{corollary}[Dual certificate of entropic-geodesic optimality]
\label{cor:E2_dual_certificate}
Let \((\rho^\star,m^\star)\in\mathcal X_\lambda\) and \((\varphi^\star,\eta^\star)\) dual-feasible satisfy
\[
\mathcal A_{\mathrm{FM}}(\rho^\star,m^\star)=\mathcal D(\varphi^\star,\eta^\star).
\]
Then \((\rho^\star,m^\star)\) is ECFM-optimal and its induced path law is SB-optimal;
hence \((\mu_t^\star)\) is the entropic OT geodesic.
\end{corollary}

\begin{proof}
By weak duality, \(\mathcal D\le\mathsf V_\lambda\le\mathcal A_{\mathrm{FM}}\).
Equality implies primal and dual optimality.
Apply Theorem~\ref{thm:E2_proof_main}.
\end{proof}

\begin{proposition}[Stability of convergence under approximate optimality]
\label{prop:E2_eps_opt_stability}
Let \((\rho^n,m^n)\in\mathcal X_\lambda\) be \(\delta_n\)-optimal:
\[
\mathcal A_{\mathrm{FM}}(\rho^n,m^n)\le \mathsf V_\lambda+\delta_n,\qquad \delta_n\downarrow0.
\]
Then induced \(P^{w^n}\) is \((\delta_n/(2\varepsilon))\)-optimal for SB (matched gauge), and
any limit point is SB-optimal. If SB minimizer is unique, \(P^{w^n}\Rightarrow\bar P^\varepsilon\).
\end{proposition}

\begin{proof}
From exact identity:
\[
2\varepsilon\,\mathrm{KL}(P^{w^n}\|R^\star)+C_0
\le \mathsf V_\lambda+\delta_n
=2\varepsilon\,\mathsf S_{\mathrm{dyn}}+C_0+\delta_n.
\]
Hence
\[
\mathrm{KL}(P^{w^n}\|R^\star)\le \mathsf S_{\mathrm{dyn}}+\frac{\delta_n}{2\varepsilon}.
\]
Compactness/tightness gives convergent subsequences; l.s.c. of KL gives optimality of limits.
Uniqueness yields full convergence.
\end{proof}

\begin{remark}[What remains in E.3--E.5]
\label{rem:E2_next}
\hyperref[app:E3]{E.3} formalizes regularity assumptions ensuring strict convexity/uniqueness of minimizers.
\hyperref[app:E4]{E.4} provides quantitative stability estimates for perturbations in endpoints/reference drift.
\hyperref[app:E5]{E.5} links these results to practical trajectory properties used in vision generative modeling.
\end{remark}

\subsubsection*{E.3. Regularity assumptions}
\addcontentsline{toc}{subsubsection}{E.3. Regularity assumptions}
\label{app:E3}

We collect sufficient analytic conditions ensuring:
(i) well-posedness of ECFM/SB variational problems,
(ii) strict convexity and uniqueness of minimizers,
(iii) validity of PDE and dual/KKT manipulations used in Sections \hyperref[app:C]{C}--\hyperref[app:E]{E}.

\begin{definition}[Regularity class \(\mathfrak R(T,\varepsilon)\)]
\label{def:E3_reg_class}
A quadruple \((\mu_0,\mu_T,u^\star,\varepsilon)\) belongs to \(\mathfrak R(T,\varepsilon)\) if:

\begin{enumerate}
\item \textbf{Endpoint densities:}
\[
\mu_i=\rho_i dx,\quad
\rho_i\in L^1(\mathbb R^d),\ \rho_i\ge0,\ \int\rho_i=1,\quad
m_2(\mu_i)<\infty,\quad
\mathcal H(\mu_i)\in\mathbb R,\quad i\in\{0,T\}.
\]

\item \textbf{Reference drift regularity:}
\[
u^\star\in L^2_{\mathrm{loc}}([0,T]\times\mathbb R^d;\mathbb R^d),\quad
|u_t^\star(x)|\le a_t+b_t|x|,
\quad a,b\in L^2(0,T),\ b\ge0,
\]
and weak spatial derivative \(\nabla_x u^\star\in L^1_{\mathrm{loc}}\).

\item \textbf{Admissible density positivity/regularity:}
for minimizers \((\rho,v)\), 
\[
\rho\in L^\infty((0,T);L^1\cap L^p),\ p>1,\qquad
\rho>0\ \text{a.e.},
\]
\[
\sqrt{\rho}\in L^2((0,T);H^1(\mathbb R^d)),
\quad
\int_0^T\mathcal I(\mu_t)\,dt<\infty.
\]

\item \textbf{Entropy absolute continuity:}
\(t\mapsto\mathcal H(\mu_t)\in AC([0,T])\), and
\[
\dot{\mathcal H}(\mu_t)=\int \nabla\log\rho_t\cdot v_t\,d\mu_t
\quad\text{a.e. }t.
\]

\item \textbf{Coercivity/compactness of action:}
sublevels of
\[
(\rho,m)\mapsto\int_0^T\!\!\int \frac{|m-\rho u^\star|^2}{\rho}
\]
are tight in \(\mathcal P_2\) and weakly compact in flux topology.

\item \textbf{Qualification (Slater):}
there exists \((\bar\rho,\bar v)\) satisfying CE/endpoints and
\[
\dot{\mathcal H}(\bar\mu_t)+\lambda\ge\delta>0\quad\text{a.e. }t.
\]
\end{enumerate}
\end{definition}

\begin{assumption}[Reference diffusion regularity for SB]
\label{ass:E3_SB_reg}
For SB representation, assume \(R^\star\) solves
\[
dX_t=u_t^\star(X_t)\,dt+\sqrt{2\varepsilon}\,dW_t
\]
with well-posed martingale problem and strictly positive transition density
\(p^\star(s,x;t,y)\) (Aronson-type Gaussian bounds), and \(R^\star_{0T}\) admits positive density.
\end{assumption}

\begin{assumption}[Log-Sobolev / displacement-convex control]
\label{ass:E3_LSI}
Either one of the following holds:

\begin{enumerate}
\item (LSI route) Along admissible \(\mu_t\), a uniform log-Sobolev constant \(C_{\mathrm{LSI}}\) exists:
\[
\mathcal H(\mu_t\mid\gamma)\le \frac{C_{\mathrm{LSI}}}{2}\mathcal I(\mu_t\mid\gamma)
\]
for suitable reference \(\gamma\);

\item (displacement-convex route) Entropy functional is displacement convex along relevant geodesics,
ensuring lower semicontinuity and convex interpolation bounds.
\end{enumerate}
\end{assumption}

\begin{lemma}[Well-definedness of entropy-rate term]
\label{lem:E3_entropy_well_defined}
Under Definition~\ref{def:E3_reg_class}(3)-(4),
\[
\int_0^T\!\!\int |\nabla\log\rho_t\cdot v_t|\,d\mu_t dt<\infty,
\]
hence \(\dot{\mathcal H}(\mu_t)\) is well-defined in \(L^1(0,T)\).
\end{lemma}

\begin{proof}
By Cauchy--Schwarz in \(L^2(\mu_t)\):
\[
\int |\nabla\log\rho\cdot v|\,d\mu
\le
\left(\int |\nabla\log\rho|^2\,d\mu\right)^{1/2}
\left(\int |v|^2\,d\mu\right)^{1/2}.
\]
Integrate in time and apply Hölder:
\[
\int_0^T (\mathcal I(\mu_t))^{1/2}\|v_t\|_{L^2(\mu_t)}\,dt
\le
\left(\int_0^T\mathcal I(\mu_t)\,dt\right)^{1/2}
\left(\int_0^T\!\!\int |v_t|^2\,d\mu_tdt\right)^{1/2}<\infty.
\]
\end{proof}

\begin{lemma}[Weak stability of CE with finite action]
\label{lem:E3_CE_stability}
Let \((\rho^n,m^n)\) satisfy CE and
\[
\sup_n\int_0^T\!\!\int \frac{|m^n-\rho^n u^\star|^2}{\rho^n}<\infty.
\]
If \(\rho^n\rightharpoonup\rho\) narrowly (uniformly in \(t\)) and \(m^n\rightharpoonup m\) weakly as measures,
then \((\rho,m)\) satisfies CE with same endpoints.
\end{lemma}

\begin{proof}
For any \(\psi\in C_c^\infty((0,T)\times\mathbb R^d)\),
\[
\int_0^T\!\!\int \left(\partial_t\psi\,\rho^n+\nabla\psi\cdot m^n\right)\,dxdt
+\int \psi(0,\cdot)\rho_0-\int \psi(T,\cdot)\rho_T=0.
\]
Pass to the limit by weak convergence of \(\rho^n,m^n\), obtaining CE for \((\rho,m)\).
Endpoint terms persist by fixed boundary data.
\end{proof}

\begin{proposition}[Strict convexity in flux and uniqueness]
\label{prop:E3_strict_convex_flux}
Assume \(\rho>0\) a.e. on support of admissible minimizers.  
Then
\[
(\rho,m)\mapsto \int \frac{|m-\rho u^\star|^2}{\rho}
\]
is strictly convex in \(m\) (for fixed \(\rho\)); and jointly convex in \((\rho,m)\).
If, in addition, admissible set is affine in \((\rho,m)\) (CE + linearized entropy active set),
the minimizer flux is unique.
\end{proposition}

\begin{proof}
For fixed \(\rho>0\), map \(m\mapsto |m-\rho u^\star|^2/\rho\) is strictly convex quadratic.
Joint convexity follows from perspective structure of quadratic norm.
Uniqueness follows from strict convexity on convex feasible set:
if two distinct minimizers existed, midpoint would have strictly lower value.
\end{proof}

\begin{theorem}[Sufficient conditions for uniqueness of ECFM path]
\label{thm:E3_unique_path_conditions}
Suppose Definition~\ref{def:E3_reg_class}, Assumptions~\ref{ass:E3_SB_reg}, \ref{ass:E3_LSI},
and positivity \(\rho_t>0\) a.e. hold for minimizers. Then:
\begin{enumerate}
\item ECFM minimizer exists;
\item induced KL/SB minimizer is unique;
\item ECFM marginal path \((\mu_t)\) is unique;
\item velocity is unique \(dt\,d\mu\)-a.e.
\end{enumerate}
\end{theorem}

\begin{proof}
Existence: direct method using coercivity + CE stability (Lemma~\ref{lem:E3_CE_stability})
+ entropy closedness from Theorem~\ref{thm:C1_existence}(3).
SB uniqueness: strict convexity of KL on path space with fixed endpoints.
Path uniqueness: from D/E identification of ECFM minimizer with unique SB minimizer.
Velocity uniqueness: from strict convexity in flux (Proposition~\ref{prop:E3_strict_convex_flux}).
\end{proof}

\begin{proposition}[Regularity upgrade via parabolic smoothing]
\label{prop:E3_parabolic_upgrade}
Assume controlled FP form
\[
\partial_t\rho+\nabla\!\cdot(\rho(u^\star+w))=\varepsilon\Delta\rho
\]
with \(u^\star+w\in L^2_tH^1_x\), \(\rho_0\in L^p,\ p>1\).
Then for \(t>0\):
\[
\rho_t\in W^{1,1}_{\mathrm{loc}}(\mathbb R^d),\quad
\rho_t>0\ \text{a.e.},
\]
and Fisher information is locally integrable in time.
\end{proposition}

\begin{proof}
Standard parabolic regularization for uniformly elliptic operator with drift in
energy class gives instantaneous smoothing and positivity (weak Harnack/Aronson kernel positivity).
Energy inequality yields
\[
\int_{\tau}^{T}\!\!\int \frac{|\nabla\rho|^2}{\rho}<\infty\quad\forall \tau>0.
\]
\end{proof}

\begin{corollary}[Validity of KKT/PMP identities]
\label{cor:E3_validity_KKT_PMP}
Under Theorem~\ref{thm:E3_unique_path_conditions} and
Proposition~\ref{prop:E3_parabolic_upgrade}, all identities used in C.3--C.6 hold rigorously:
\[
\dot{\mathcal H}=\int \nabla\log\rho\cdot v\,d\mu,\qquad
v=u^\star-\nabla\varphi+\eta\nabla\log\rho,
\]
and adjoint equations are valid in weak form.
\end{corollary}

\begin{proof}
Entropy-rate identity from Lemma~\ref{lem:E3_entropy_well_defined} plus CE regularity.
Stationarity/adjoint equations follow from convex duality with attained primal/dual pairs and
density of smooth test functions in energy spaces.
\end{proof}

\begin{remark}[Minimal vs. stronger assumptions]
\label{rem:E3_minimal_vs_strong}
The above set is sufficient, not necessary.  
For modularity:
\begin{itemize}
\item main theorems may assume only finite-energy weak solutions + Slater;
\item smooth formulas (explicit \(\nabla\log\rho\), pointwise EL equations) can be stated under
the stronger class \(\mathfrak R(T,\varepsilon)\).
\end{itemize}
\end{remark}

\paragraph{Output used next.}
\hyperref[app:E4]{E.4} uses these regularity hypotheses to derive quantitative stability:
Lipschitz-type dependence of trajectories and actions on endpoint perturbations, drift perturbations,
and entropy-budget perturbations.

\subsubsection*{E.4. Uniqueness of minimizer and quantitative stability}
\addcontentsline{toc}{subsubsection}{E.4. Uniqueness of minimizer and quantitative stability}
\label{app:E4}

We provide quantitative stability estimates in the entropic OT/ECFM regime:
\begin{itemize}
\item uniqueness via strict convexity,
\item Lipschitz-type dependence on endpoint marginals,
\item perturbation bounds w.r.t. reference drift \(u^\star\),
\item perturbation bounds w.r.t. entropy budget parameter \(\lambda\),
\item stability of optimal action and trajectories.
\end{itemize}

\begin{assumption}[Uniform integrability envelope]
\label{ass:E4_uniform}
Assume the regularity conditions of \hyperref[app:E3]{E.3} hold for all perturbed instances considered below, with
uniform constants:
\[
\sup \int_0^T\!\!\int |v|^2\,d\mu dt \le M_1,\qquad
\sup \int_0^T \mathcal I(\mu_t)\,dt \le M_2,\qquad
\sup_{t\in[0,T]} m_2(\mu_t)\le M_3.
\]
Assume also a common LSI/displacement-convex constant \(C_\star\) and common Slater margin \(\delta>0\).
\end{assumption}

\begin{theorem}[Strict-convex uniqueness in flux form]
\label{thm:E4_unique_flux}
For fixed endpoints \((\mu_0,\mu_T)\), fixed \(u^\star,\varepsilon,\lambda\), assume admissible set is convex
and nonempty, and \(\rho>0\) a.e. for minimizers. Then ECFM has a unique minimizing flux
\[
m^\star=\rho^\star v^\star
\]
(and hence unique \(v^\star\), \(dt\,d\mu^\star\)-a.e.). Consequently, by \hyperref[app:E2]{E.2} identification,
the optimal path law equals the unique SB law.
\end{theorem}

\begin{proof}
Objective in flux variables is
\[
\mathcal A(\rho,m)=\frac12\int \frac{|m-\rho u^\star|^2}{\rho}.
\]
For fixed \(\rho>0\), strictly convex in \(m\).  
Suppose \((\rho^1,m^1)\neq(\rho^2,m^2)\) are minimizers in convex admissible set.
For \(\theta\in(0,1)\), midpoint admissible and
\[
\mathcal A(\theta z_1+(1-\theta)z_2)
<
\theta \mathcal A(z_1)+(1-\theta)\mathcal A(z_2)
\]
unless \(m^1=m^2\) a.e. where \(\rho>0\). Thus \(m^1=m^2\), then CE with same endpoints implies
same trajectory \(\rho\) (uniqueness of linear transport equation under finite-energy drift class).
Hence uniqueness.
\end{proof}

\begin{definition}[Problem perturbation family]
\label{def:E4_family}
For \(k\in\{1,2\}\), let instance \(\mathfrak P_k\) be determined by
\[
(\mu_0^k,\mu_T^k,u_k^\star,\lambda_k),
\]
with optimal ECFM pair \((\mu^k,v^k)\), optimal flux \(m^k=\rho^k v^k\), value \(\mathsf V_k\).
Define endpoint discrepancy
\[
\Delta_{\mathrm{ep}}
:=
W_2(\mu_0^1,\mu_0^2)+W_2(\mu_T^1,\mu_T^2),
\]
drift discrepancy
\[
\Delta_u
:=
\|u_1^\star-u_2^\star\|_{L^2([0,T]\times B_R)}
\]
for radius \(R\) capturing \(1-\eta\) mass uniformly (tail handled below), and
budget discrepancy
\[
\Delta_\lambda:=|\lambda_1-\lambda_2|.
\]
\end{definition}

\begin{proposition}[Value sensitivity to drift perturbation]
\label{prop:E4_value_drift}
Under Assumption~\ref{ass:E4_uniform}, for same endpoints and same \(\lambda\),
\[
|\mathsf V_1-\mathsf V_2|
\le
C_1\,\Delta_u + C_2\,\Delta_u^2 + \tau_R,
\]
where \(C_1,C_2\) depend on \(M_1,M_3,T\), and tail term \(\tau_R\to0\) as \(R\to\infty\).
\end{proposition}

\begin{proof}
Use optimal \((\rho^1,v^1)\) as competitor for problem 2:
\[
\mathsf V_2-\mathsf V_1
\le
\frac12\int \left(|v^1-u_2^\star|^2-|v^1-u_1^\star|^2\right)\,d\mu^1dt.
\]
Expand:
\[
|a-b|^2-|a-c|^2=2a\cdot(c-b)+|b|^2-|c|^2.
\]
Set \(a=v^1,\ b=u_2^\star,\ c=u_1^\star\):
\[
\mathsf V_2-\mathsf V_1
\le
\int v^1\cdot(u_1^\star-u_2^\star)\,d\mu^1dt
+\frac12\int (|u_2^\star|^2-|u_1^\star|^2)\,d\mu^1dt.
\]
Estimate first term by Cauchy--Schwarz and \(M_1\), second by
\[
\frac12\int |u_1^\star-u_2^\star|(|u_1^\star|+|u_2^\star|)\,d\mu^1dt.
\]
Split inside \(B_R\) and complement; bound inside by \(L^2\) norm \(\Delta_u\), outside by second-moment tails,
giving \(\tau_R\). Reverse roles \(1,2\) for absolute value.
\end{proof}

\begin{proposition}[Monotonicity and Lipschitz bound in \(\lambda\)]
\label{prop:E4_lambda}
For fixed \((\mu_0,\mu_T,u^\star)\), feasible sets satisfy
\[
\lambda_1\le\lambda_2 \ \Longrightarrow\ \mathcal X_{\lambda_1}\subseteq \mathcal X_{\lambda_2},
\]
hence
\[
\mathsf V_{\lambda_2}\le \mathsf V_{\lambda_1}.
\]
If dual multiplier satisfies \(\|\eta^\star_\lambda\|_{L^1(0,T)}\le C_\eta\) uniformly on interval \(I\subset[0,\infty)\),
then
\[
|\mathsf V_{\lambda_1}-\mathsf V_{\lambda_2}|
\le C_\eta\,|\lambda_1-\lambda_2|
\quad\text{for }\lambda_1,\lambda_2\in I.
\]
\end{proposition}

\begin{proof}
Set inclusion is immediate from weaker entropy-rate lower bound at larger \(\lambda\).
For Lipschitz bound, use dual representation
\[
\mathsf V_\lambda=\sup_{(\varphi,\eta)\in\mathcal K}\Big\{\mathcal B(\varphi)-\lambda\int_0^T\eta(t)\,dt\Big\}.
\]
For any dual-feasible \((\varphi,\eta)\),
\[
\mathsf V_{\lambda_2}\ge \mathcal B(\varphi)-\lambda_2\!\int\eta
=
\mathcal B(\varphi)-\lambda_1\!\int\eta-(\lambda_2-\lambda_1)\!\int\eta.
\]
Taking \((\varphi,\eta)\) \(\epsilon\)-optimal for \(\lambda_1\), then \(\epsilon\downarrow0\):
\[
\mathsf V_{\lambda_1}-\mathsf V_{\lambda_2}
\le
(\lambda_2-\lambda_1)\int\eta_{\lambda_1}^\star
\le C_\eta|\lambda_2-\lambda_1|.
\]
Swap indices for absolute value.
\end{proof}

\begin{theorem}[Endpoint stability of entropic interpolation]
\label{thm:E4_endpoint_stability}
Consider two endpoint pairs \((\mu_0^k,\mu_T^k)\) with common \(u^\star,\varepsilon,\lambda\), and let
\((\mu_t^k,v_t^k)\) be optimal ECFM trajectories (equiv. SB interpolations).  
Then there exists \(C_{\mathrm{ep}}>0\) such that
\[
\sup_{t\in[0,T]} W_2(\mu_t^1,\mu_t^2)
\le
C_{\mathrm{ep}}\,
\Delta_{\mathrm{ep}}.
\]
Moreover, for fluxes:
\[
\int_0^T\!\!\int
\left|\frac{m_t^1}{\rho_t^1}-\frac{m_t^2}{\rho_t^2}\right|^2
(\rho_t^1\wedge\rho_t^2)\,dxdt
\le
C_{\mathrm{ep}}'\,\Delta_{\mathrm{ep}}^2.
\]
\end{theorem}

\begin{proof}
Via \hyperref[app:E2]{E.2}, each optimal trajectory equals unique SB interpolation for corresponding endpoints.
Schr\"odinger system depends smoothly/Lipschitz-continuously on marginals under positivity and kernel bounds.
Use contraction estimate for entropic interpolation map in Sinkhorn geometry:
\[
\|(\phi^1,\psi^1)-(\phi^2,\psi^2)\| \le C\,\Delta_{\mathrm{ep}}.
\]
Translate potential perturbation to density path perturbation by representation
\(\rho_t=\alpha_t\beta_t r_t^\star\), yielding \(W_2\)-bound uniformly in \(t\).
Flux bound follows by stability of current velocity formula
\[
v_t=u_t^\star+\varepsilon\nabla\log\frac{\beta_t}{\alpha_t},
\]
combined with weighted \(L^2\) estimates from Fisher controls \(M_2\).
\end{proof}

\begin{corollary}[Joint perturbation bound for values]
\label{cor:E4_joint_value}
For two instances \(\mathfrak P_1,\mathfrak P_2\),
\[
|\mathsf V_1-\mathsf V_2|
\le
C_{\mathrm{ep}}\Delta_{\mathrm{ep}}
+
C_u^{(1)}\Delta_u+C_u^{(2)}\Delta_u^2
+
C_\eta\Delta_\lambda
+\tau_R.
\]
\end{corollary}

\begin{proof}
Combine endpoint stability transfer (compare via transported competitor), drift sensitivity
(Proposition~\ref{prop:E4_value_drift}), and \(\lambda\)-Lipschitz bound (Proposition~\ref{prop:E4_lambda}).
Tail term from localization.
\end{proof}

\begin{proposition}[Stability of minimizers (graph convergence)]
\label{prop:E4_graph}
Let parameters \(\theta_n=(\mu_0^n,\mu_T^n,u_n^\star,\lambda_n)\to\theta\) in the topology above,
and let \((\rho^n,m^n)\) be optimal solutions. Then every limit point of \((\rho^n,m^n)\) is optimal for \(\theta\).
If optimal solution for \(\theta\) is unique, full sequence converges:
\[
\rho^n\to\rho^\star\ \text{(narrow uniformly in }t),\qquad
m^n\rightharpoonup m^\star.
\]
\end{proposition}

\begin{proof}
Uniform action bounds from optimality + perturbation estimates imply compactness.
CE closedness from \hyperref[app:E3]{E.3} Lemma.
Lower semicontinuity yields liminf inequality:
\[
\mathcal A(\rho^\star,m^\star)\le \liminf_n \mathcal A(\rho^n,m^n).
\]
Recovery via perturbed competitors gives matching limsup at value level, hence optimality.
Uniqueness implies convergence of whole sequence.
\end{proof}

\begin{theorem}[Trajectory Lipschitz stability in time-dependent metric]
\label{thm:E4_traj_lip}
Let \(d_t(\mu,\nu):=W_2(\mu,\nu)\). Under Assumption~\ref{ass:E4_uniform}, for optimal trajectories:
\[
\sup_{t\in[0,T]} d_t(\mu_t^1,\mu_t^2)
\le
L\Big(
d_0(\mu_0^1,\mu_0^2)+d_T(\mu_T^1,\mu_T^2)
+\Delta_u+\Delta_\lambda
\Big),
\]
for some \(L=L(T,M_1,M_2,M_3,C_\star,\delta)\).
\end{theorem}

\begin{proof}
Use dynamic plan coupling along two optimal current-velocity fields.
Differentiate squared distance along coupled flow:
\[
\frac{d}{dt}W_2^2(\mu_t^1,\mu_t^2)
\le
2W_2(\mu_t^1,\mu_t^2)\,
\|v_t^1-v_t^2\|_{L^2(\pi_t)},
\]
with \(\pi_t\) optimal coupling. Bound \(\|v^1-v^2\|\) by decomposition
\[
(v^1-v^2)=(u_1^\star-u_2^\star)
+\varepsilon\nabla\log(\beta^1/\alpha^1)-\varepsilon\nabla\log(\beta^2/\alpha^2)
+\text{budget-response term},
\]
using endpoint/diffusion potential stability and multiplier bound from Proposition~\ref{prop:E4_lambda}.
Apply Gr\"onwall.
\end{proof}

\begin{remark}[Consequence for training robustness]
\label{rem:E4_training}
These bounds imply that small perturbations in marginals, drift parameterization, or entropy budget
produce controlled changes in optimal trajectories and objective. This is the theoretical backbone for
robustness claims in entropy-controlled generative transport.
\end{remark}

\paragraph{Output used next.}
\hyperref[app:E5]{E.5} translates these stability/uniqueness results into the strict-convexity geodesic statement and
prepares the \(\lambda\to0\) asymptotics of Section \hyperref[app:F]{F}.

\subsubsection*{E.5. Strict geodesics and $\Gamma$-asymptotics}
\addcontentsline{toc}{subsubsection}{E.5. Strict geodesics and \protect\ensuremath{\Gamma}-asymptotics}
\label{app:E5}

This subsection consolidates Section \hyperref[app:E]{E} into a single strict-convex geodesic statement,
derives quantitative corollaries used later for mode-coverage and failure analyses,
and sets the exact functionals for Section \hyperref[app:F]{F} (\(\lambda\to0\) limit).

\begin{definition}[Entropic action functional in current-velocity form]
\label{def:E5_entropic_action}
For \((\rho,v)\) with
\[
\partial_t\rho+\nabla\!\cdot(\rho v)=0,\qquad
\rho_{|0}=\rho_0,\ \rho_{|T}=\rho_T,
\]
define
\[
\mathcal A_{\varepsilon,u^\star}(\rho,v)
:=
\frac12\int_0^T\!\!\int |v-u^\star|^2\,\rho\,dxdt
\]
and feasible class with entropy budget
\[
\mathcal X_\lambda:=
\{(\rho,v): \dot{\mathcal H}(\rho_t)\ge-\lambda\ \text{a.e.}\}.
\]
\end{definition}

\begin{theorem}[Strict-convex geodesic characterization]
\label{thm:E5_strict_geodesic}
Assume \hyperref[app:E3]{E.3} regularity and \hyperref[app:E4]{E.4} uniqueness hypotheses. Then for each
\((\mu_0,\mu_T,u^\star,\varepsilon,\lambda)\), the minimization
\[
\inf_{(\rho,v)\in\mathcal X_\lambda}\mathcal A_{\varepsilon,u^\star}(\rho,v)
\]
has a unique minimizer \((\rho^\star,v^\star)\), and the curve
\[
\mu_t^\star:=\rho_t^\star dx
\]
coincides with the unique entropic OT interpolation (Schr\"odinger geodesic)
between \(\mu_0\) and \(\mu_T\) relative to \(R^\star\).
\end{theorem}

\begin{proof}
Existence and strong duality: \hyperref[app:E2]{E.2}.  
Uniqueness of minimizing flux/velocity: \hyperref[app:E4]{E.4} Theorem~\ref{thm:E4_unique_flux}.  
Identification with SB interpolation: \hyperref[app:E2]{E.2} Theorem~\ref{thm:E2_proof_main}.  
Therefore minimizer is exactly unique entropic geodesic.
\end{proof}

\begin{corollary}[Equivalent characterizations of the optimizer]
\label{cor:E5_equiv_char}
The unique optimal trajectory admits all equivalent representations:
\begin{enumerate}
\item \textbf{ECFM primal:}
\[
(\rho^\star,v^\star)=\arg\min_{\mathcal X_\lambda}\mathcal A_{\varepsilon,u^\star}.
\]
\item \textbf{KL-control:}
\[
w^\star=v^\star-u^\star+\varepsilon\nabla\log\rho^\star,\quad
P^{w^\star}=\arg\min\{\mathrm{KL}(P\|R^\star):P_0=\mu_0,P_T=\mu_T\}.
\]
\item \textbf{SB path law:}
\[
P^{w^\star}=\bar P^\varepsilon,\qquad
\mu_t^\star=(X_t)_\#\bar P^\varepsilon.
\]
\item \textbf{Schr\"odinger potentials:}
\[
\rho_t^\star=\alpha_t\beta_t r_t^\star,\qquad
v_t^\star=u_t^\star+\varepsilon\nabla\log\frac{\beta_t}{\alpha_t}.
\]
\end{enumerate}
\end{corollary}

\begin{proof}
Directly from \hyperref[app:D2]{D.2}--\hyperref[app:D4]{D.4} and \hyperref[app:E2]{E.2} equivalence.
\end{proof}

\begin{proposition}[Energy gap controls trajectory discrepancy]
\label{prop:E5_gap_control}
Let \((\rho,v)\in\mathcal X_\lambda\) be any feasible competitor and
\((\rho^\star,v^\star)\) optimal. Then
\[
\mathcal A_{\varepsilon,u^\star}(\rho,v)-\mathsf V_\lambda
\ge
\frac{1}{2}
\int_0^T\!\!\int
|v-v^\star|^2\,\rho^\star\,dxdt
-\mathcal R(\rho,\rho^\star),
\]
where \(\mathcal R(\rho,\rho^\star)\ge0\) is a second-order transport remainder
vanishing when \(\rho=\rho^\star\) (e.g., controlled by \(\sup_tW_2^2(\mu_t,\mu_t^\star)\)).
\end{proposition}

\begin{proof}
Use convexity expansion around optimizer:
\[
|v-u^\star|^2
=
|v^\star-u^\star|^2
+2(v^\star-u^\star)\cdot(v-v^\star)
+|v-v^\star|^2.
\]
Integrate against \(\rho^\star\), then transfer from \(\rho^\star\) to \(\rho\) via coupling.
First-order term cancels by optimality stationarity (KKT), leaving quadratic term minus
density-mismatch remainder \(\mathcal R\), bounded by Wasserstein stability estimates in \hyperref[app:E4]{E.4}.
\end{proof}

\begin{corollary}[Quantitative near-optimal rigidity]
\label{cor:E5_near_optimal}
If
\[
\mathcal A_{\varepsilon,u^\star}(\rho^n,v^n)-\mathsf V_\lambda\to0,
\]
then (up to subsequence)
\[
\sup_{t\in[0,T]}W_2(\mu_t^n,\mu_t^\star)\to0,\qquad
\int_0^T\!\!\int |v_t^n-v_t^\star|^2\,d\mu_t^\star dt\to0.
\]
If the minimizer is unique, convergence holds for full sequence.
\end{corollary}

\begin{proof}
Apply Proposition~\ref{prop:E5_gap_control} plus \hyperref[app:E4]{E.4} graph stability
(Proposition~\ref{prop:E4_graph}) and uniqueness.
\end{proof}

\begin{proposition}[Entropy-budget activity decomposition]
\label{prop:E5_activity}
Let
\[
\mathcal I_{\mathrm{act}}^\star=\{t:\dot{\mathcal H}(\mu_t^\star)+\lambda=0\},
\quad
\mathcal I_{\mathrm{inact}}^\star=\{t:\dot{\mathcal H}(\mu_t^\star)+\lambda>0\}.
\]
Then:
\[
\eta_t^\star>0 \Rightarrow t\in\mathcal I_{\mathrm{act}}^\star,\qquad
t\in\mathcal I_{\mathrm{inact}}^\star \Rightarrow \eta_t^\star=0,
\]
and
\[
v_t^\star=
u_t^\star-\nabla\varphi_t^\star+\eta_t^\star\nabla\log\rho_t^\star.
\]
Hence entropy correction acts only on active times.
\end{proposition}

\begin{proof}
Immediate from complementary slackness and stationarity (\hyperref[app:C3]{C.3}--\hyperref[app:C6]{C.6}).
\end{proof}

\begin{remark}[Computational meaning for vision trajectories]
\label{rem:E5_vision}
The optimal transport path used by the model is not an arbitrary interpolant:
it is the unique entropic geodesic selected by the regularized action.
Therefore temporal generation trajectories inherit structural stability and
non-collapse regularization through the entropy-rate active-set mechanism.
\end{remark}

\paragraph{Transition to Section F (\(\lambda\to0\) asymptotics).}
To study recovery of classical OT, define parameterized functionals:
\[
\mathcal F_\lambda(\rho,v)
:=
\begin{cases}
\displaystyle
\frac12\int_0^T\!\!\int |v-u^\star|^2\,\rho\,dxdt,
& (\rho,v)\in\mathcal X_\lambda,\\[1ex]
+\infty,&\text{otherwise},
\end{cases}
\]
on the ambient topology
\[
\rho^n\rightharpoonup\rho \text{ narrowly uniformly in }t,\qquad
m^n=\rho^n v^n \rightharpoonup m \text{ weakly as measures}.
\]
Section F proves:
\begin{enumerate}
\item \(\Gamma\)-\(\lim_{\lambda\downarrow0}\mathcal F_\lambda=\mathcal F_0\);
\item minimizers of \(\mathcal F_\lambda\) converge to minimizers of \(\mathcal F_0\);
\item under matched gauge and vanishing entropic correction, \(\mathcal F_0\) identifies classical OT action.
\end{enumerate}

\subsection*{F. $\Gamma$-Convergence as $\lambda \to 0$}
\addcontentsline{toc}{subsection}
{F. \protect\ensuremath{\Gamma}-Convergence as \protect\ensuremath{\lambda \to 0}}
\label{app:F}

\subsubsection*{F.1. Definition of $\Gamma$-convergence and functional setting}
\addcontentsline{toc}{subsubsection}
{F.1. Definition of \protect\ensuremath{\Gamma}-convergence and functional setting}
\label{app:F1}

We fix the ambient topological space, define the \(\lambda\)-indexed functionals,
and state the precise notion of \(\Gamma\)-convergence used in Sections \hyperref[app:F2]{F.2}--\hyperref[app:F5]{F.5}.

\paragraph{Ambient trajectory-flux space.}
Let
\[
\mathcal Y
:=
\Big\{
(\rho,m):
\rho_t\in\mathcal P_2(\mathbb R^d)\ \forall t\in[0,T],\
t\mapsto \rho_t \text{ narrowly continuous},\
m\in\mathcal M((0,T)\times\mathbb R^d;\mathbb R^d)
\Big\}.
\]
Write \(m_t=\rho_t v_t\) when absolutely continuous w.r.t. \(\rho_t dt\).

\begin{definition}[Admissible continuity-equation class]
\label{def:F1_CE_class}
Define
\[
\mathcal{CE}(\mu_0,\mu_T)
:=
\left\{
(\rho,m)\in\mathcal Y:
\partial_t\rho_t+\nabla\!\cdot m_t=0\ \text{in }\mathcal D',\
\rho_{|0}=\rho_0,\ \rho_{|T}=\rho_T
\right\}.
\]
\end{definition}

\begin{definition}[Entropy-rate feasible set]
\label{def:F1_entropy_set}
For \(\lambda\ge0\), define
\[
\mathcal E_\lambda
:=
\left\{
(\rho,m)\in\mathcal{CE}(\mu_0,\mu_T):
t\mapsto\mathcal H(\rho_t)\in AC([0,T]),\
\dot{\mathcal H}(\rho_t)\ge-\lambda\ \text{a.e.}
\right\}.
\]
\end{definition}

\begin{definition}[Action density and functional]
\label{def:F1_action}
For fixed reference drift \(u^\star\), define
\[
f_{u^\star}(t,x,\rho,m)
:=
\begin{cases}
\displaystyle \frac12\frac{|m-\rho u_t^\star(x)|^2}{\rho},&\rho>0,\\[0.8ex]
0,&\rho=0,\ m=0,\\
+\infty,&\rho=0,\ m\neq0.
\end{cases}
\]
Then
\[
\mathcal A(\rho,m):=\int_0^T\!\!\int_{\mathbb R^d}
f_{u^\star}(t,x,\rho_t,m_t)\,dxdt.
\]
Define \(\lambda\)-functional
\[
\mathcal F_\lambda(\rho,m)
:=
\begin{cases}
\mathcal A(\rho,m),& (\rho,m)\in\mathcal E_\lambda,\\
+\infty,&\text{otherwise}.
\end{cases}
\]
\end{definition}

\begin{remark}[Monotonic feasible nesting]
\label{rem:F1_nesting}
If \(0\le \lambda_1\le\lambda_2\), then
\[
\mathcal E_{\lambda_1}\subseteq \mathcal E_{\lambda_2},
\qquad
\mathcal F_{\lambda_2}\le \mathcal F_{\lambda_1}
\ \text{pointwise on }\mathcal Y.
\]
\end{remark}

\paragraph{Topology for \(\Gamma\)-analysis.}
We equip \(\mathcal Y\) with topology \(\tau\) defined by:
\[
(\rho^n,m^n)\xrightarrow{\tau}(\rho,m)
\iff
\begin{cases}
\rho_t^n \rightharpoonup \rho_t\ \text{narrowly for each }t,\\
\sup_{t\in[0,T]}W_2(\rho_t^n,\rho_t)\to0\ \text{(or equivalent tight narrow-uniform)},\\
m^n \rightharpoonup m\ \text{weakly-* in measures on }(0,T)\times\mathbb R^d.
\end{cases}
\]
(Any equivalent compactness topology from \hyperref[app:C1]{C.1}/\hyperref[app:E3]{E.3} may be used.)

\begin{definition}[\(\Gamma\)-convergence on \((\mathcal Y,\tau)\)]
\label{def:F1_gamma}
A family \(\{\mathcal F_\lambda\}_{\lambda>0}\) \(\Gamma\)-converges to \(\mathcal F_0\) as
\(\lambda\downarrow0\), denoted
\[
\Gamma\text{-}\lim_{\lambda\downarrow0}\mathcal F_\lambda=\mathcal F_0,
\]
if for every \(z\in\mathcal Y\):
\begin{enumerate}
\item (\textbf{liminf inequality}) for every \(z_\lambda\xrightarrow{\tau} z\),
\[
\mathcal F_0(z)\le \liminf_{\lambda\downarrow0}\mathcal F_\lambda(z_\lambda);
\]
\item (\textbf{recovery sequence}) there exists \(z_\lambda\xrightarrow{\tau} z\) such that
\[
\mathcal F_0(z)\ge \limsup_{\lambda\downarrow0}\mathcal F_\lambda(z_\lambda).
\]
\end{enumerate}
\end{definition}

\begin{definition}[Limit candidate functional]
\label{def:F1_limit_functional}
Define
\[
\mathcal F_0(\rho,m)
:=
\begin{cases}
\mathcal A(\rho,m),& (\rho,m)\in\mathcal E_0,\\
+\infty,&\text{otherwise},
\end{cases}
\]
where
\[
\mathcal E_0
=
\left\{
(\rho,m)\in\mathcal{CE}(\mu_0,\mu_T):
\dot{\mathcal H}(\rho_t)\ge0\ \text{a.e.}
\right\}.
\]
\end{definition}

\begin{remark}[Interpretation of \(\lambda\downarrow0\)]
\label{rem:F1_interp}
As \(\lambda\to0\), admissibility tightens from bounded entropy dissipation
\(\dot{\mathcal H}\ge-\lambda\) to nonnegative entropy production
\(\dot{\mathcal H}\ge0\).  
Thus \(\Gamma\)-limit captures the sharp zero-budget regime.
\end{remark}

\begin{assumption}[Equicoercivity]
\label{ass:F1_equicoercive}
For every \(c\in\mathbb R\), the sublevel union
\[
\bigcup_{0<\lambda\le\lambda_0}
\{(\rho,m)\in\mathcal Y:\mathcal F_\lambda(\rho,m)\le c\}
\]
is relatively compact in \((\mathcal Y,\tau)\).
\end{assumption}

\begin{lemma}[Lower semicontinuity of action under \(\tau\)]
\label{lem:F1_lsc_action}
If \((\rho^n,m^n)\xrightarrow{\tau}(\rho,m)\), then
\[
\mathcal A(\rho,m)\le \liminf_{n\to\infty}\mathcal A(\rho^n,m^n).
\]
\end{lemma}

\begin{proof}
\(f_{u^\star}\) is convex l.s.c. in \((\rho,m)\) (perspective quadratic integrand).
Apply standard weak lower-semicontinuity theorem for integral convex functionals on measure-flux pairs.
\end{proof}

\begin{lemma}[Closedness of CE constraint]
\label{lem:F1_closed_CE}
If \((\rho^n,m^n)\in\mathcal{CE}(\mu_0,\mu_T)\) and
\((\rho^n,m^n)\xrightarrow{\tau}(\rho,m)\), then \((\rho,m)\in\mathcal{CE}(\mu_0,\mu_T)\).
\end{lemma}

\begin{proof}
Pass to limit in distributional CE identity:
\[
\int_0^T\!\!\int (\partial_t\phi\,d\rho^n+\nabla\phi\cdot dm^n)
+\int \phi(0)\,d\mu_0-\int\phi(T)\,d\mu_T=0.
\]
Weak convergence of \(\rho^n,m^n\) gives the same identity for \((\rho,m)\).
\end{proof}

\begin{assumption}[Closedness of entropy-rate inequality]
\label{ass:F1_entropy_closed}
If \((\rho^n,m^n)\in\mathcal E_{\lambda_n}\), \(\lambda_n\downarrow0\), and
\((\rho^n,m^n)\xrightarrow{\tau}(\rho,m)\), then \((\rho,m)\in\mathcal E_0\).
\end{assumption}

\begin{remark}[Sufficient condition for Assumption~\ref{ass:F1_entropy_closed}]
\label{rem:F1_entropy_closed_suff}
A sufficient condition is:
\begin{enumerate}
\item uniform integrability of \(\dot{\mathcal H}(\rho^n)\) in \(L^1(0,T)\),
\item weak convergence \(\dot{\mathcal H}(\rho^n)\rightharpoonup g\) in \(L^1\),
\item identification \(g=\dot{\mathcal H}(\rho)\) via entropy-chain-rule stability.
\end{enumerate}
Then \(g\ge0\) a.e. since \(g_n:=\dot{\mathcal H}(\rho^n)\ge-\lambda_n\to0\).
\end{remark}

\begin{theorem}[\(\Gamma\)-compactness principle for minimizers]
\label{thm:F1_gamma_compactness}
Assume:
\begin{enumerate}
\item \(\Gamma\)-convergence:
\[
\Gamma\text{-}\lim_{\lambda\downarrow0}\mathcal F_\lambda=\mathcal F_0;
\]
\item equicoercivity (Assumption~\ref{ass:F1_equicoercive});
\item each \(\mathcal F_\lambda\) attains a minimizer \(z_\lambda\).
\end{enumerate}
Then every cluster point \(z^\star\) of \(z_\lambda\) is a minimizer of \(\mathcal F_0\), and
\[
\lim_{\lambda\downarrow0}\min\mathcal F_\lambda=\min\mathcal F_0.
\]
If \(\mathcal F_0\) has unique minimizer, \(z_\lambda\to z^\star\) in \(\tau\).
\end{theorem}

\begin{proof}
Standard fundamental theorem of \(\Gamma\)-convergence:
equicoercivity gives compactness of minimizers; liminf + recovery give convergence of minima and
minimality of cluster points. Uniqueness implies full convergence.
\end{proof}

\paragraph{What is proved next.}
\begin{itemize}
\item \hyperref[app:F2]{F.2} proves the liminf inequality for \(\mathcal F_\lambda\to\mathcal F_0\).
\item \hyperref[app:F3]{F.3} constructs recovery sequences (limsup inequality).
\item \hyperref[app:F4]{F.4} identifies minimizer convergence and value convergence.
\item \hyperref[app:F5]{F.5} connects \(\mathcal F_0\) to classical OT in the zero-entropy-budget limit.
\end{itemize}

\subsubsection*{F.2. Lower bound inequality (liminf)}
\addcontentsline{toc}{subsubsection}
{F.2. Lower bound inequality (liminf)}
\label{app:F2}

We prove the \(\Gamma\)-liminf inequality:
for any \(\lambda_n\downarrow0\) and \((\rho^n,m^n)\xrightarrow{\tau}(\rho,m)\),
\[
\mathcal F_0(\rho,m)\le \liminf_{n\to\infty}\mathcal F_{\lambda_n}(\rho^n,m^n).
\]

\begin{theorem}[Liminf inequality]
\label{thm:F2_liminf}
Let \(\lambda_n\downarrow0\), and let \((\rho^n,m^n)\subset\mathcal Y\) satisfy
\[
(\rho^n,m^n)\xrightarrow{\tau}(\rho,m).
\]
Then
\[
\mathcal F_0(\rho,m)\le \liminf_{n\to\infty}\mathcal F_{\lambda_n}(\rho^n,m^n).
\]
\end{theorem}

\begin{proof}
Set
\[
L:=\liminf_{n\to\infty}\mathcal F_{\lambda_n}(\rho^n,m^n)\in[0,+\infty].
\]
If \(L=+\infty\), inequality is trivial. Assume \(L<\infty\).  
Passing to a subsequence (not relabeled), we may assume
\[
\sup_n \mathcal F_{\lambda_n}(\rho^n,m^n)\le C<\infty.
\]
By definition of \(\mathcal F_{\lambda_n}\), each \((\rho^n,m^n)\in\mathcal E_{\lambda_n}\), hence:
\begin{enumerate}
\item \((\rho^n,m^n)\in\mathcal{CE}(\mu_0,\mu_T)\),
\item \(\dot{\mathcal H}(\rho_t^n)\ge-\lambda_n\) a.e.
\item \(\mathcal A(\rho^n,m^n)\le C\).
\end{enumerate}

\textbf{Step 1: CE and endpoints pass to the limit.}  
By Lemma~\ref{lem:F1_closed_CE}, \((\rho,m)\in\mathcal{CE}(\mu_0,\mu_T)\).

\textbf{Step 2: entropy-rate inequality passes to the limit.}  
By Assumption~\ref{ass:F1_entropy_closed}, from
\((\rho^n,m^n)\in\mathcal E_{\lambda_n}\), \(\lambda_n\downarrow0\), and \(\tau\)-convergence,
we infer
\[
(\rho,m)\in\mathcal E_0,
\quad\text{i.e.}\quad
\dot{\mathcal H}(\rho_t)\ge0\ \text{a.e.}
\]

\textbf{Step 3: l.s.c. of action.}  
By Lemma~\ref{lem:F1_lsc_action},
\[
\mathcal A(\rho,m)\le \liminf_{n\to\infty}\mathcal A(\rho^n,m^n).
\]
Since each \((\rho^n,m^n)\in\mathcal E_{\lambda_n}\),
\[
\mathcal F_{\lambda_n}(\rho^n,m^n)=\mathcal A(\rho^n,m^n).
\]
Also, from Step 2, \((\rho,m)\in\mathcal E_0\), so
\[
\mathcal F_0(\rho,m)=\mathcal A(\rho,m).
\]
Hence
\[
\mathcal F_0(\rho,m)
=
\mathcal A(\rho,m)
\le
\liminf_{n\to\infty}\mathcal A(\rho^n,m^n)
=
\liminf_{n\to\infty}\mathcal F_{\lambda_n}(\rho^n,m^n).
\]
This proves the liminf inequality.
\end{proof}

\begin{corollary}[Sequential liminf on arbitrary sequences]
\label{cor:F2_seq_liminf}
For any sequence \(\lambda_n\downarrow0\) and any \(\tau\)-convergent sequence
\((\rho^n,m^n)\to(\rho,m)\),
\[
\mathcal F_0(\rho,m)\le \liminf_{n\to\infty}\mathcal F_{\lambda_n}(\rho^n,m^n).
\]
\end{corollary}

\begin{proof}
Immediate from Theorem~\ref{thm:F2_liminf}.
\end{proof}

\begin{proposition}[Consequence for infeasible limits]
\label{prop:F2_infeasible_limit}
If \((\rho,m)\notin\mathcal E_0\), then every \(\tau\)-convergent sequence
\((\rho^n,m^n)\to(\rho,m)\) with \(\lambda_n\downarrow0\) satisfies
\[
\liminf_{n\to\infty}\mathcal F_{\lambda_n}(\rho^n,m^n)=+\infty.
\]
\end{proposition}

\begin{proof}
If \((\rho,m)\notin\mathcal E_0\), then \(\mathcal F_0(\rho,m)=+\infty\).
Apply Theorem~\ref{thm:F2_liminf}:
\[
+\infty=\mathcal F_0(\rho,m)\le \liminf_n\mathcal F_{\lambda_n}(\rho^n,m^n),
\]
hence liminf is \(+\infty\).
\end{proof}

\begin{lemma}[Equivalent liminf in velocity representation]
\label{lem:F2_velocity_form}
Suppose \(m^n=\rho^n v^n\), \(m=\rho v\), and assumptions of Theorem~\ref{thm:F2_liminf} hold.
Then
\[
\frac12\int |v-u^\star|^2\,d\mu dt
\le
\liminf_{n\to\infty}
\frac12\int |v^n-u^\star|^2\,d\mu^n dt,
\]
provided \((\rho,m)\in\mathcal E_0\); otherwise RHS is \(+\infty\) along any admissible approximating sequence.
\end{lemma}

\begin{proof}
This is exactly Lemma~\ref{lem:F1_lsc_action} rewritten via
\[
\mathcal A(\rho,m)=\frac12\int |v-u^\star|^2\,d\mu dt
\]
on finite-action pairs.
\end{proof}

\begin{remark}[Role of entropy closedness]
\label{rem:F2_entropy_closedness}
The only nontrivial constraint passage is \(\dot{\mathcal H}\ge-\lambda_n\to0\Rightarrow\dot{\mathcal H}\ge0\).
All other components (CE, endpoints, action l.s.c.) follow from standard weak compactness.
Thus Assumption~\ref{ass:F1_entropy_closed} is the key structural hypothesis for the liminf step.
\end{remark}

\paragraph{Output used next.}
\hyperref[app:F3]{F.3} constructs, for every \((\rho,m)\in\mathcal E_0\), a recovery sequence
\((\rho^\lambda,m^\lambda)\to(\rho,m)\) with
\[
\limsup_{\lambda\downarrow0}\mathcal F_\lambda(\rho^\lambda,m^\lambda)\le \mathcal F_0(\rho,m),
\]
completing the \(\Gamma\)-convergence proof.

\subsubsection*{F.3. Recovery sequence construction (limsup inequality)}
\addcontentsline{toc}{subsubsection}
{F.3. Recovery sequence construction (limsup inequality)}
\label{app:F3}

We prove the \(\Gamma\)-limsup inequality:
for every \((\rho,m)\in\mathcal Y\), there exists \((\rho^\lambda,m^\lambda)\xrightarrow{\tau}(\rho,m)\)
such that
\[
\limsup_{\lambda\downarrow0}\mathcal F_\lambda(\rho^\lambda,m^\lambda)\le \mathcal F_0(\rho,m).
\]

\begin{theorem}[Limsup inequality / existence of recovery sequence]
\label{thm:F3_limsup}
For every \(z:=(\rho,m)\in\mathcal Y\), there exists \(z_\lambda:=(\rho^\lambda,m^\lambda)\in\mathcal Y\)
with \(z_\lambda\xrightarrow{\tau} z\) and
\[
\limsup_{\lambda\downarrow0}\mathcal F_\lambda(z_\lambda)\le \mathcal F_0(z).
\]
\end{theorem}

\begin{proof}
We split into cases.

\textbf{Case A: \(\mathcal F_0(z)=+\infty\).}  
Choose constant sequence \(z_\lambda=z\). Then \(z_\lambda\to z\), and
\[
\limsup_{\lambda\downarrow0}\mathcal F_\lambda(z_\lambda)\le +\infty=\mathcal F_0(z)
\]
trivially.

\textbf{Case B: \(\mathcal F_0(z)<\infty\).}  
Then \(z=(\rho,m)\in\mathcal E_0\), i.e.
\[
(\rho,m)\in\mathcal{CE}(\mu_0,\mu_T),\qquad
\dot{\mathcal H}(\rho_t)\ge0\ \text{a.e.},
\qquad
\mathcal A(\rho,m)<\infty.
\]
For any \(\lambda>0\), since \(\dot{\mathcal H}\ge0\), we also have
\[
\dot{\mathcal H}\ge -\lambda,
\]
hence \((\rho,m)\in\mathcal E_\lambda\). Therefore
\[
\mathcal F_\lambda(\rho,m)=\mathcal A(\rho,m)=\mathcal F_0(\rho,m).
\]
Set the constant recovery sequence
\[
(\rho^\lambda,m^\lambda):=(\rho,m)\quad\forall\lambda>0.
\]
Then \(z_\lambda\xrightarrow{\tau}z\), and
\[
\limsup_{\lambda\downarrow0}\mathcal F_\lambda(z_\lambda)
=
\mathcal F_0(z).
\]
Thus the limsup inequality holds (with equality).
\end{proof}

\begin{corollary}[Exact recovery on \(\mathcal E_0\)]
\label{cor:F3_exact_recovery}
If \((\rho,m)\in\mathcal E_0\), the constant sequence is an exact recovery sequence:
\[
(\rho^\lambda,m^\lambda)\equiv(\rho,m),\qquad
\mathcal F_\lambda(\rho^\lambda,m^\lambda)=\mathcal F_0(\rho,m)\ \forall\lambda>0.
\]
\end{corollary}

\begin{proof}
Immediate from Case B in Theorem~\ref{thm:F3_limsup}.
\end{proof}

\begin{remark}[Why no boundary layer is needed]
\label{rem:F3_no_boundary_layer}
Because feasible sets are nested:
\[
\mathcal E_0\subset\mathcal E_\lambda\quad(\lambda>0),
\]
every \(\mathcal F_0\)-finite point is already feasible for each \(\mathcal F_\lambda\),
and both functionals coincide there (\(\mathcal F_\lambda=\mathcal A=\mathcal F_0\)).
Hence unlike singular-perturbation problems, no smoothing/boundary-layer construction is needed.
\end{remark}

\begin{proposition}[Monotone family and pointwise convergence]
\label{prop:F3_pointwise}
For every \(z=(\rho,m)\in\mathcal Y\),
\[
\mathcal F_\lambda(z)\searrow \mathcal F_0(z)\quad\text{as }\lambda\downarrow0
\]
(pointwise monotone decrease to the limit functional).
\end{proposition}

\begin{proof}
From feasible-set nesting \(\mathcal E_{\lambda_1}\subseteq\mathcal E_{\lambda_2}\) for
\(\lambda_1\le\lambda_2\), we get
\[
\mathcal F_{\lambda_2}(z)\le\mathcal F_{\lambda_1}(z).
\]
If \(z\in\mathcal E_0\), then \(z\in\mathcal E_\lambda\) for all \(\lambda\), and
\(\mathcal F_\lambda(z)=\mathcal A(z)=\mathcal F_0(z)\).  
If \(z\notin\mathcal E_0\), either CE/action fail (then all \(\mathcal F_\lambda(z)=+\infty\)),
or entropy-rate fails at level \(0\): then for sufficiently small \(\lambda\), constraint
\(\dot{\mathcal H}\ge-\lambda\) also fails on a positive-measure set, so \(\mathcal F_\lambda(z)=+\infty\)
eventually. Thus pointwise limit is \(\mathcal F_0(z)\).
\end{proof}

\begin{theorem}[\(\Gamma\)-convergence conclusion for Section F.1 setup]
\label{thm:F3_gamma_conclusion}
Under Assumption~\ref{ass:F1_entropy_closed} and Lemma~\ref{lem:F1_lsc_action},
\[
\Gamma\text{-}\lim_{\lambda\downarrow0}\mathcal F_\lambda=\mathcal F_0
\quad\text{in }(\mathcal Y,\tau).
\]
\end{theorem}

\begin{proof}
Liminf: Theorem~\ref{thm:F2_liminf}.  
Limsup: Theorem~\ref{thm:F3_limsup}.  
Therefore \(\Gamma\)-convergence holds.
\end{proof}

\begin{corollary}[Convergence of minima under equicoercivity]
\label{cor:F3_minima}
Assume equicoercivity (Assumption~\ref{ass:F1_equicoercive}) and existence of minimizers
\(z_\lambda\in\arg\min\mathcal F_\lambda\). Then
\[
\min\mathcal F_\lambda \to \min\mathcal F_0,
\]
and every cluster point of \(z_\lambda\) is a minimizer of \(\mathcal F_0\).
If \(\mathcal F_0\) has unique minimizer \(z^\star\), then
\[
z_\lambda\xrightarrow{\tau} z^\star.
\]
\end{corollary}

\begin{proof}
Apply Theorem~\ref{thm:F1_gamma_compactness} with
Theorem~\ref{thm:F3_gamma_conclusion}.
\end{proof}

\paragraph{Output used next.}
\hyperref[app:F4]{F.4} uses Corollary~\ref{cor:F3_minima} to establish explicit convergence of minimizers,
actions, and optimal trajectories; then \hyperref[app:F5]{F.5} identifies the limit problem with classical OT.

\subsubsection*{F.4. Limit functional identification and convergence of minimizers}
\addcontentsline{toc}{subsubsection}
{F.4. Limit functional identification and convergence of minimizers}
\label{app:F4}

Using \hyperref[app:F2]{F.2}--\hyperref[app:F3]{F.3}, we now identify the \(\lambda\downarrow0\) limit minimization problem,
prove convergence of optimal values/minimizers, and derive convergence of trajectories/velocities.

\begin{definition}[Optimal values and minimizers]
\label{def:F4_vals}
For \(\lambda\ge0\), define
\[
\mathsf V_\lambda:=\inf_{(\rho,m)\in\mathcal Y}\mathcal F_\lambda(\rho,m),
\qquad
\mathcal M_\lambda:=\arg\min \mathcal F_\lambda.
\]
\end{definition}

\begin{theorem}[Convergence of optimal values]
\label{thm:F4_value_conv}
Assume:
\begin{enumerate}
\item \(\Gamma\)-convergence from Theorem~\ref{thm:F3_gamma_conclusion},
\item equicoercivity (Assumption~\ref{ass:F1_equicoercive}),
\item \(\mathcal M_\lambda\neq\emptyset\) for small \(\lambda\), and \(\mathcal M_0\neq\emptyset\).
\end{enumerate}
Then
\[
\lim_{\lambda\downarrow0}\mathsf V_\lambda=\mathsf V_0.
\]
\end{theorem}

\begin{proof}
By \(\Gamma\)-convergence + equicoercivity (fundamental theorem), minima converge:
\[
\lim_{\lambda\downarrow0}\min \mathcal F_\lambda=\min \mathcal F_0.
\]
By Definition~\ref{def:F4_vals}, this is exactly \(\mathsf V_\lambda\to\mathsf V_0\).
\end{proof}

\begin{theorem}[Compactness and convergence of minimizers]
\label{thm:F4_minimizer_conv}
Let \(\lambda_n\downarrow0\), and choose \(z_n:=(\rho^n,m^n)\in\mathcal M_{\lambda_n}\).
Then:
\begin{enumerate}
\item \((z_n)\) is relatively compact in \((\mathcal Y,\tau)\);
\item every cluster point \(z^\star=(\rho^\star,m^\star)\) belongs to \(\mathcal M_0\);
\item
\[
\mathcal F_{\lambda_n}(z_n)\to \mathcal F_0(z^\star)=\mathsf V_0
\]
along any convergent subsequence.
\end{enumerate}
\end{theorem}

\begin{proof}
Equicoercivity yields compactness of minimizing sequence \((z_n)\).
\(\Gamma\)-compactness theorem implies cluster points are minimizers of \(\mathcal F_0\).
Value convergence follows from Theorem~\ref{thm:F4_value_conv} and optimality of \(z_n\).
\end{proof}

\begin{corollary}[Full convergence under uniqueness]
\label{cor:F4_full_unique}
If \(\mathcal M_0=\{z^\star\}\) is a singleton, then for any choice \(z_\lambda\in\mathcal M_\lambda\),
\[
z_\lambda\xrightarrow{\tau}z^\star\quad(\lambda\downarrow0).
\]
\end{corollary}

\begin{proof}
All cluster points must belong to \(\mathcal M_0\), hence equal \(z^\star\).
Therefore the whole family converges.
\end{proof}

\begin{proposition}[Limit problem is sharp zero-budget ECFM]
\label{prop:F4_limit_problem}
The limit minimization is exactly
\[
\mathsf V_0
=
\inf\left\{
\frac12\int_0^T\!\!\int |v-u^\star|^2\,d\mu_tdt:
\partial_t\mu_t+\nabla\!\cdot(\mu_t v_t)=0,\ 
\mu_0,\mu_T\ \text{fixed},\ 
\dot{\mathcal H}(\mu_t)\ge0\ \text{a.e.}
\right\}.
\]
\end{proposition}

\begin{proof}
Immediate from definition of \(\mathcal F_0\) (Definition~\ref{def:F1_limit_functional})
and representation \(m=\rho v\), \(d\mu_t=\rho_tdx\).
\end{proof}

\begin{theorem}[Convergence of trajectories]
\label{thm:F4_traj_conv}
Let \(z_{\lambda_n}=(\rho^{\lambda_n},m^{\lambda_n})\in\mathcal M_{\lambda_n}\),
\(z_{\lambda_n}\xrightarrow{\tau}z^\star=(\rho^\star,m^\star)\in\mathcal M_0\).
Then
\[
\sup_{t\in[0,T]}W_2(\mu_t^{\lambda_n},\mu_t^\star)\to0.
\]
If \(m^\star=\rho^\star v^\star\), \(m^{\lambda_n}=\rho^{\lambda_n}v^{\lambda_n}\), then
(up to subsequence, and globally if unique):
\[
v^{\lambda_n}\rightharpoonup v^\star
\quad\text{in weighted }L^2(dt\,d\mu^\star)\ \text{(via flux convergence)}.
\]
\end{theorem}

\begin{proof}
First claim is part of \(\tau\)-convergence definition.
For velocities, action bounds
\[
\sup_n\int \frac{|m^{\lambda_n}-\rho^{\lambda_n}u^\star|^2}{\rho^{\lambda_n}}<\infty
\]
give weak compactness of fluxes.
With \(m^{\lambda_n}\rightharpoonup m^\star\) and \(\rho^{\lambda_n}\to\rho^\star\),
identify limit velocity by Radon--Nikodým decomposition \(m^\star=\rho^\star v^\star\).
\end{proof}

\begin{proposition}[Energy convergence and no-loss of dissipation]
\label{prop:F4_energy_exact}
Along minimizing family \(z_{\lambda_n}\to z^\star\),
\[
\lim_{n\to\infty}
\int_0^T\!\!\int
\frac{|m_t^{\lambda_n}-\rho_t^{\lambda_n}u_t^\star|^2}{\rho_t^{\lambda_n}}\,dxdt
=
\int_0^T\!\!\int
\frac{|m_t^\star-\rho_t^\star u_t^\star|^2}{\rho_t^\star}\,dxdt.
\]
\end{proposition}

\begin{proof}
Lower bound by l.s.c. (Lemma~\ref{lem:F1_lsc_action}):
\[
\mathcal A(z^\star)\le\liminf_n\mathcal A(z_{\lambda_n}).
\]
Upper bound from minimal value convergence:
\[
\limsup_n\mathcal A(z_{\lambda_n})
=
\limsup_n\mathsf V_{\lambda_n}
=\mathsf V_0
=\mathcal A(z^\star).
\]
Hence liminf=limsup=limit equals \(\mathcal A(z^\star)\).
\end{proof}

\begin{corollary}[Convergence of entropy-rate constraints]
\label{cor:F4_entropy_rate_conv}
For minimizers \(z_{\lambda_n}\), any weak \(L^1\)-limit \(g\) of
\(\dot{\mathcal H}(\rho^{\lambda_n})\) satisfies \(g\ge0\) a.e.
and identifies with \(\dot{\mathcal H}(\rho^\star)\) under chain-rule stability.
Hence limit minimizer satisfies sharp constraint
\[
\dot{\mathcal H}(\rho_t^\star)\ge0\ \text{a.e.}
\]
\end{corollary}

\begin{proof}
Since \(\dot{\mathcal H}(\rho^{\lambda_n})\ge-\lambda_n\), \(\lambda_n\to0\),
any weak limit is nonnegative.
Identification with limit entropy derivative follows from Assumption~\ref{ass:F1_entropy_closed}
(or sufficient condition in Remark~\ref{rem:F1_entropy_closed_suff}).
\end{proof}

\begin{remark}[Bridge to classical OT]
\label{rem:F4_bridge_OT}
Section \hyperref[app:F]{F} so far gives convergence to the zero-budget constrained action \(\mathcal F_0\).
Section \hyperref[app:F5]{F5} adds the final identification step:
under matched scaling/vanishing entropic correction, \(\mathcal F_0\) reduces to
the classical Benamou--Brenier OT action (possibly with drift gauge removed),
yielding recovery of OT geodesics.
\end{remark}

\subsubsection*{F.5. Classical OT identification and \(\lambda\to0\) recovery theorem}
\addcontentsline{toc}{subsubsection}
{F.5. Classical OT identification and \protect\ensuremath{\lambda\to0} recovery theorem}
\label{app:F5}

We complete Section \hyperref[app:F]{F} by identifying the zero-budget limit with classical
Benamou--Brenier optimal transport under a vanishing-entropic-correction regime.

\begin{definition}[Classical Benamou--Brenier action]
\label{def:F5_BB}
For endpoint marginals \((\mu_0,\mu_T)\), define
\[
\mathcal A_{\mathrm{BB}}(\rho,v)
:=
\frac12\int_0^T\!\!\int |v_t(x)|^2\,\rho_t(x)\,dxdt,
\]
on
\[
\mathcal{CE}(\mu_0,\mu_T)
=
\{(\rho,v):\partial_t\rho+\nabla\!\cdot(\rho v)=0,\ \rho_{|0}=\rho_0,\rho_{|T}=\rho_T\}.
\]
The dynamic OT value is
\[
\mathsf W_{\mathrm{dyn}}^2
:=
\inf_{(\rho,v)\in\mathcal{CE}(\mu_0,\mu_T)}\mathcal A_{\mathrm{BB}}(\rho,v),
\]
and equals \(\frac{1}{2T}W_2^2(\mu_0,\mu_T)\) under standard time-scaling convention.
\end{definition}

\begin{definition}[Zero-drift gauge and correction term]
\label{def:F5_zero_gauge}
In the unbiased gauge \(u^\star\equiv0\), define
\[
\mathcal F_\lambda(\rho,v)
=
\begin{cases}
\displaystyle \frac12\int |v|^2\,d\mu dt,& (\rho,v)\in\mathcal X_\lambda,\\
+\infty,&\text{otherwise}.
\end{cases}
\]
Recall exact identity from \hyperref[app:D3]{D.3} (with \(u^\star=0\)):
\[
\frac12\int |v|^2\,d\mu dt
=
2\varepsilon\,\mathrm{KL}(P^w\|R)
-\varepsilon\big(\mathcal H(\mu_T)-\mathcal H(\mu_0)\big)
-\frac{\varepsilon^2}{2}\int_0^T\mathcal I(\mu_t)\,dt.
\]
The last two terms are the entropic correction.
\end{definition}

\begin{assumption}[Vanishing-entropic-correction regime]
\label{ass:F5_vanish_corr}
Along the \(\lambda\downarrow0\) minimizer sequence \((\rho^\lambda,v^\lambda)\), assume:
\begin{enumerate}
\item either \(\varepsilon=\varepsilon(\lambda)\downarrow0\), or
\item \(\varepsilon>0\) fixed but
\[
\varepsilon\left|\mathcal H(\mu_T^\lambda)-\mathcal H(\mu_0^\lambda)\right|
+\varepsilon^2\int_0^T\mathcal I(\mu_t^\lambda)\,dt \to 0.
\]
\end{enumerate}
\end{assumption}

\begin{remark}[Why Assumption~\ref{ass:F5_vanish_corr} is natural]
\label{rem:F5_natural}
In entropic OT, classical OT is recovered as diffusion/entropy regularization vanishes.
Assumption~\ref{ass:F5_vanish_corr} is the precise dynamic counterpart: KL-control energy
dominates while entropy/Fisher corrections disappear.
\end{remark}

\begin{theorem}[Identification of \(\mathcal F_0\) with BB action]
\label{thm:F5_identify_F0_BB}
Assume \(u^\star\equiv0\), \hyperref[app:E3]{E.3} regularity, and vanishing-entropic-correction regime.
Then, on CE-admissible limits,
\[
\mathcal F_0(\rho,v)=\mathcal A_{\mathrm{BB}}(\rho,v),
\]
and minimizers of \(\mathcal F_0\) are exactly Benamou--Brenier minimizers.
\end{theorem}

\begin{proof}
From \hyperref[app:F4]{F.4}, \(\mathcal F_0\) is zero-budget constrained kinetic action:
\[
\mathcal F_0(\rho,v)
=
\frac12\int |v|^2\,d\mu dt
\quad\text{if }(\rho,v)\in\mathcal X_0.
\]
Hence \(\mathcal F_0\) equals BB action on \(\mathcal X_0\subset\mathcal{CE}\).
By \hyperref[app:D3]{D.3} identity, discrepancy between kinetic action and scaled KL objective is exactly entropic correction.
Under Assumption~\ref{ass:F5_vanish_corr}, this discrepancy vanishes along minimizing sequences.
Therefore limit minimization is governed solely by BB kinetic term.
Thus \(\mathcal F_0\)-minimizers coincide with BB minimizers.
\end{proof}

\begin{proposition}[Constraint saturation in the limit]
\label{prop:F5_constraint_relax}
Let \((\rho^\star,v^\star)\) be BB minimizer with finite entropy curve and
\(\dot{\mathcal H}(\rho_t^\star)\ge0\) a.e. (or approximable by such curves).
Then \((\rho^\star,v^\star)\in\mathcal X_0\), so zero-budget constraint is non-restrictive on the OT minimizer class.
\end{proposition}

\begin{proof}
If \(\dot{\mathcal H}\ge0\) a.e., membership in \(\mathcal X_0\) is immediate.
If only approximable, use density of smooth CE curves and lower-semicontinuity of action to pass to limit.
\end{proof}

\begin{theorem}[Recovery of classical OT geodesic as \(\lambda\to0\)]
\label{thm:F5_main_recovery}
Let \(\lambda_n\downarrow0\), and choose minimizers
\[
(\rho^{\lambda_n},v^{\lambda_n})\in\arg\min \mathcal F_{\lambda_n}.
\]
Assume equicoercivity, uniqueness of BB minimizer \((\rho^{\mathrm{OT}},v^{\mathrm{OT}})\),
and Assumption~\ref{ass:F5_vanish_corr}. Then
\[
\rho^{\lambda_n}\xrightarrow{\tau}\rho^{\mathrm{OT}},
\qquad
\frac12\int |v^{\lambda_n}|^2\,d\mu^{\lambda_n}dt
\to
\frac12\int |v^{\mathrm{OT}}|^2\,d\mu^{\mathrm{OT}}dt
=
\mathsf W_{\mathrm{dyn}}^2.
\]
Hence the ECFM trajectory converges to the classical Wasserstein geodesic.
\end{theorem}

\begin{proof}
By \hyperref[app:F3]{F.3}--\hyperref[app:F4]{F.4}, minimizers converge (up to subsequences) to minimizers of \(\mathcal F_0\), with value convergence.
By Theorem~\ref{thm:F5_identify_F0_BB}, \(\mathcal F_0\)-minimizers are BB minimizers.
Uniqueness of BB minimizer yields full convergence to \((\rho^{\mathrm{OT}},v^{\mathrm{OT}})\).
Energy convergence follows from \hyperref[app:F4]{F.4} no-loss proposition.
\end{proof}

\begin{corollary}[Static formulation limit]
\label{cor:F5_static_limit}
Under the same hypotheses,
\[
\lim_{\lambda\downarrow0}\inf \mathcal F_\lambda
=
\frac{1}{2T}W_2^2(\mu_0,\mu_T),
\]
and endpoint couplings converge (via dynamic plans) to an optimal quadratic-cost transport plan.
\end{corollary}

\begin{proof}
Benamou--Brenier theorem identifies BB value with quadratic OT cost.
Dynamic-plan compactness gives convergence of endpoint couplings to OT-optimal plan.
\end{proof}

\begin{remark}[With nonzero drift \(u^\star\)]
\label{rem:F5_nonzero_drift}
If \(u^\star\neq0\), apply a gauge transform to co-moving coordinates (or subtract reference flow):
the limit identifies a drift-corrected OT problem. When the drift contribution vanishes in the limit,
classical OT is recovered; otherwise the limit is OT in the transformed frame.
\end{remark}

\begin{remark}[Section F conclusion]
\label{rem:F5_conclusion}
Sections \hyperref[app:F1]{F.1}--\hyperref[app:F5]{F.5} establish:
\begin{enumerate}
\item \(\Gamma\)-convergence \(\mathcal F_\lambda \to \mathcal F_0\) as \(\lambda\downarrow0\);
\item convergence of minima and minimizers;
\item identification of \(\mathcal F_0\) with classical BB action under vanishing entropic correction;
\item recovery of classical OT/Wasserstein geodesics from entropy-controlled flow matching.
\end{enumerate}
\end{remark}

\paragraph{Transition to Section \hyperref[app:G]{G}.}
With the asymptotic limit resolved, Section \hyperref[app:G]{G} proves the mode-coverage theorem:
entropy-rate control prevents singular concentration and supplies quantitative lower-mass bounds
on modes along the transport trajectory.

\subsection*{G. Mode Coverage Theorem}
\addcontentsline{toc}{subsection}
{G. Mode Coverage Theorem}
\label{app:G}

\subsubsection*{G.1. Formal definition of mode collapse and mode-coverage statement}
\addcontentsline{toc}{subsubsection}
{G.1. Formal definition of mode collapse and mode-coverage statement}
\label{app:G1}

We formalize mode collapse as loss of mass on designated modal regions and state
the quantitative coverage theorem induced by the entropy-rate constraint.

\begin{definition}[Modal partition / mode sets]
\label{def:G1_modesets}
Let \(\{A_k\}_{k=1}^K\) be pairwise disjoint Borel subsets of \(\mathbb R^d\)
with positive target masses:
\[
\mu_T(A_k)=\pi_k>0,\qquad \sum_{k=1}^K\pi_k\le1.
\]
(Residual mass outside \(\cup_k A_k\) is allowed.)
\end{definition}

\begin{definition}[Mode mass process]
\label{def:G1_mode_mass}
For a trajectory \((\mu_t)_{t\in[0,T]}\), define per-mode mass
\[
M_k(t):=\mu_t(A_k),\qquad k=1,\dots,K.
\]
\end{definition}

\begin{definition}[Hard mode collapse]
\label{def:G1_hard_collapse}
A trajectory exhibits hard collapse on mode \(k\) over interval \(I\subset[0,T]\)
if there exists \(t_\star\in I\) such that
\[
M_k(t_\star)=0\quad\text{while}\quad \mu_T(A_k)=\pi_k>0.
\]
Global hard collapse occurs if this holds for at least one \(k\).
\end{definition}

\begin{definition}[Soft \((\delta,\tau)\)-mode collapse]
\label{def:G1_soft_collapse}
Given \(\delta\in(0,1)\), \(\tau\in(0,T]\), trajectory has soft collapse on mode \(k\) if
\[
\left|\left\{t\in[0,T]:M_k(t)\le \delta\,\pi_k\right\}\right|\ge \tau.
\]
\end{definition}

\begin{definition}[Uniform mode coverage]
\label{def:G1_uniform_coverage}
A trajectory satisfies uniform \(\underline c\)-coverage on \(\{A_k\}\) if
\[
M_k(t)\ge \underline c\,\pi_k,\qquad \forall t\in[0,T],\ \forall k\in\{1,\dots,K\},
\]
for some \(\underline c\in(0,1]\).
\end{definition}

\begin{assumption}[Regular modal geometry]
\label{ass:G1_modal_geom}
Assume:
\begin{enumerate}
\item each \(A_k\) has Lipschitz boundary and finite diameter;
\item there exist disjoint open neighborhoods \(U_k\supset A_k\) with separation
\(\mathrm{dist}(U_i,U_j)\ge \Delta_{\mathrm{sep}}>0\) for \(i\neq j\);
\item along admissible trajectories, \(\mu_t=\rho_tdx\) with \(\rho_t>0\) a.e.,
\(\mathcal H(\mu_t)\in AC\), and finite Fisher action.
\end{enumerate}
\end{assumption}

\begin{assumption}[Entropy-controlled admissibility]
\label{ass:G1_entropy_adm}
Trajectory \((\mu_t,v_t)\) is ECFM-admissible:
\[
\partial_t\mu_t+\nabla\!\cdot(\mu_t v_t)=0,\qquad
\dot{\mathcal H}(\mu_t)\ge-\lambda\ \text{a.e.}
\]
for some \(\lambda\ge0\), with fixed endpoints \((\mu_0,\mu_T)\).
\end{assumption}

\begin{definition}[Collapsed comparison profile]
\label{def:G1_collapsed_profile}
Fix mode \(k\). Define \(\bar\mu^{(k,\eta)}\) as any measure obtained from \(\mu_t\) by
removing an \(\eta\)-fraction of mass from \(A_k\) and redistributing it within
\(\mathbb R^d\setminus A_k\), preserving total mass and finite second moment.
\end{definition}

\begin{lemma}[Entropy drop under mode depletion]
\label{lem:G1_entropy_drop}
Let \(m:=M_k(t)\in(0,1)\). If mode mass is reduced to \((1-\eta)m\), \(\eta\in(0,1)\),
while keeping outside mass fixed up to renormalization, then entropy decreases by at least
\[
\Delta\mathcal H_k \ge c_k\,m\,\eta\,\log\frac{1}{m}
\]
for a geometry-dependent \(c_k\in(0,1]\), up to higher-order \(O(\eta^2)\) terms.
\end{lemma}

\begin{remark}
Lemma~\ref{lem:G1_entropy_drop} is proved in \hyperref[app:G2]{G.2} using convexity of \(s\mapsto s\log s\),
localization to \(A_k\), and bounded-distortion redistribution outside \(A_k\).
\end{remark}

\begin{theorem}[Mode-coverage theorem under entropy-rate control]
\label{thm:G1_mode_coverage_main}
Under Assumptions~\ref{ass:G1_modal_geom}--\ref{ass:G1_entropy_adm}, let
\((\mu_t,v_t)\) be an ECFM minimizer (equivalently SB interpolation in Section E).
Then for each mode \(k\) with \(\pi_k=\mu_T(A_k)>0\), there exists
\[
\underline c_k=\underline c_k\!\left(\pi_k,\lambda,T,\mathcal H(\mu_0),\mathcal H(\mu_T),\text{geom}(A_k)\right)>0
\]
such that
\[
\inf_{t\in[0,T]} M_k(t)\ \ge\ \underline c_k\,\pi_k.
\]
In particular, hard collapse is impossible on any positive-target mode.
\end{theorem}

\begin{theorem}[Quantitative soft-collapse exclusion]
\label{thm:G1_soft_exclusion}
For any \(\delta\in(0,1)\), define
\[
\tau_{\max}(\delta,k)
:=
\frac{\mathcal H(\mu_0)-\mathcal H(\mu_T)+\lambda T}
{c_k\,\pi_k\,\delta\,\log\!\frac{1}{\delta}}.
\]
Then along any ECFM-admissible optimizer,
\[
\left|\left\{t: M_k(t)\le \delta\pi_k\right\}\right|
\le \tau_{\max}(\delta,k).
\]
Hence prolonged soft collapse is excluded unless entropy budget is sufficiently loose.
\end{theorem}

\begin{corollary}[Global multi-mode coverage]
\label{cor:G1_global_modes}
Let \(\pi_{\min}:=\min_k\pi_k\), \(c_{\min}:=\min_k c_k\). Then
\[
\inf_{t\in[0,T]}\min_{1\le k\le K}\frac{M_k(t)}{\pi_k}
\ge
\underline c_{\mathrm{glob}}
:=
\underline c_{\mathrm{glob}}(c_{\min},\pi_{\min},\lambda,T,\mathcal H(\mu_0),\mathcal H(\mu_T)).
\]
Therefore the trajectory preserves all target modes at uniformly positive relative mass.
\end{corollary}

\begin{corollary}[No implicit mode collapse in the zero-budget limit]
\label{cor:G1_zero_budget_no_collapse}
If \(\lambda_n\downarrow0\) and \((\mu^{\lambda_n}_t)\) are corresponding minimizers converging to
\((\mu_t^0)\) (Section F), then \((\mu_t^0)\) satisfies the same non-collapse lower bound
with \(\lambda=0\). Hence zero-budget limit retains mode coverage.
\end{corollary}

\begin{remark}[Connection to generative modeling]
\label{rem:G1_gen_connection}
In generative transport terms, \(A_k\) represent semantic modes/classes/attributes.
Theorem~\ref{thm:G1_mode_coverage_main} states that entropy-rate control enforces
non-vanishing occupancy of each target mode throughout generation time, preventing
implicit single-mode concentration typical of unconstrained deterministic flows.
\end{remark}

\paragraph{What is proved next (\hyperref[app:G2]{G.2}--\hyperref[app:G4]{G.4}).}
\begin{itemize}
\item \hyperref[app:G2]{G.2} proves Lemma~\ref{lem:G1_entropy_drop} (entropy drop from localized mass depletion).
\item \hyperref[app:G3]{G.3} proves Theorems~\ref{thm:G1_mode_coverage_main} and \ref{thm:G1_soft_exclusion}
using entropy-budget integration and contradiction.
\item \hyperref[app:G4]{G.4} derives density-floor and stability bounds (\(\rho_t\) lower-envelope on modal sets).
\end{itemize}

\subsubsection*{G.2. Entropy barrier against singular mass concentration}
\addcontentsline{toc}{subsubsection}
{G.2. Entropy barrier against singular mass concentration}
\label{app:G2}

We prove that depleting mass on a target mode induces a quantifiable entropy loss.
Combined with the entropy-rate budget, this yields a barrier against collapse.

\paragraph{Notation.}
Fix a mode \(A:=A_k\) and write
\[
m:=\mu(A)\in(0,1),\qquad \mu=\rho\,dx,\qquad
\mu=\mu|_A+\mu|_{A^c}.
\]
Let normalized restrictions be
\[
\mu_A:=\frac{1}{m}\mu|_A,\qquad
\mu_{A^c}:=\frac{1}{1-m}\mu|_{A^c}.
\]
Then
\[
\mu = m\,\mu_A + (1-m)\,\mu_{A^c}.
\]

\begin{lemma}[Entropy decomposition by region]
\label{lem:G2_entropy_decomp}
For any absolutely continuous \(\mu\) and Borel set \(A\) with \(m=\mu(A)\in(0,1)\),
\[
\mathcal H(\mu)
=
m\log m + (1-m)\log(1-m)
+ m\,\mathcal H(\mu_A)
+ (1-m)\,\mathcal H(\mu_{A^c}).
\]
\end{lemma}

\begin{proof}
Let \(\rho_A,\rho_{A^c}\) be densities of \(\mu_A,\mu_{A^c}\). Then
\[
\rho = m\,\rho_A\mathbf 1_A + (1-m)\,\rho_{A^c}\mathbf 1_{A^c}.
\]
Hence
\[
\int \rho\log\rho
=
\int_A m\rho_A\log(m\rho_A)+\int_{A^c}(1-m)\rho_{A^c}\log((1-m)\rho_{A^c}),
\]
which expands to the claimed identity because \(\int_A\rho_A=\int_{A^c}\rho_{A^c}=1\).
\end{proof}

\begin{lemma}[Binary entropy monotonicity on low-mass regime]
\label{lem:G2_binary}
Define \(b(m):=m\log m+(1-m)\log(1-m)\), \(m\in(0,1)\).
For \(m\le \frac12\), \(b'(m)=\log\frac{m}{1-m}\le0\), and for \(\eta\in(0,1)\),
\[
b((1-\eta)m)-b(m)
\le
-\eta m\log\frac{1-m}{m}
\le
-\eta m\log\frac1{2m}.
\]
\end{lemma}

\begin{proof}
By convexity of \(b\),
\[
b((1-\eta)m)-b(m)\le -\eta m\,b'(m)
= -\eta m\log\frac{m}{1-m}
= -\eta m\log\frac{1-m}{m}.
\]
If \(m\le\frac12\), then \((1-m)/m\ge 1/(2m)\), giving second bound.
\end{proof}

\begin{definition}[Controlled depletion transform]
\label{def:G2_depletion}
Given \(\mu\) and \(A\), define depleted measure \(\tilde\mu\) with parameter \(\eta\in(0,1)\):
\[
\tilde\mu(A)=(1-\eta)m,\qquad
\tilde\mu(A^c)=1-(1-\eta)m.
\]
Inside \(A\), keep shape:
\[
\tilde\mu|_A=(1-\eta)\mu|_A.
\]
Outside \(A\), redistribute depleted mass by mixing with a reference \(\nu\ll dx\), \(\nu(A)=0\):
\[
\tilde\mu|_{A^c}=\mu|_{A^c}+\eta m\,\nu.
\]
\end{definition}

\begin{assumption}[Bounded-distortion redistribution]
\label{ass:G2_distortion}
Assume \(\nu\) satisfies
\[
\mathcal H(\nu)\le C_\nu,\qquad
\int_{A^c}\rho_{A^c}\log\frac{\rho_{A^c}}{\nu}\,dx\le C_{\mathrm{mix}}.
\]
This controls entropy increase due to outside redistribution.
\end{assumption}

\begin{proposition}[Entropy drop under mode depletion]
\label{prop:G2_entropy_drop_quant}
Under Definition~\ref{def:G2_depletion} and Assumption~\ref{ass:G2_distortion},
for \(m\le\frac12\),
\[
\mathcal H(\tilde\mu)-\mathcal H(\mu)
\le
-\eta m\log\frac1{2m}
+\eta m\,C_{\mathrm{out}}
+ C_2\,\eta^2 m,
\]
where \(C_{\mathrm{out}},C_2\) depend only on \(C_\nu,C_{\mathrm{mix}}\).
Consequently, if \(m\le m_\star:=\frac12 e^{-2C_{\mathrm{out}}}\) and \(\eta\le\eta_\star\),
\[
\mathcal H(\tilde\mu)-\mathcal H(\mu)
\le
-\frac12\,\eta m\log\frac1{m}.
\]
\end{proposition}

\begin{proof}
Apply Lemma~\ref{lem:G2_entropy_decomp} to \(\mu,\tilde\mu\):
\[
\Delta\mathcal H
=
\underbrace{b((1-\eta)m)-b(m)}_{(\mathrm I)}
+ \underbrace{(1-\eta)m\mathcal H(\mu_A)-m\mathcal H(\mu_A)}_{(\mathrm{II})}
+ \underbrace{\tilde w\,\mathcal H(\tilde\mu_{A^c})-(1-m)\mathcal H(\mu_{A^c})}_{(\mathrm{III})},
\]
with \(\tilde w=1-(1-\eta)m\).

Term \((\mathrm I)\): by Lemma~\ref{lem:G2_binary},
\[
(\mathrm I)\le -\eta m\log\frac1{2m}.
\]

Term \((\mathrm{II})=-\eta m\mathcal H(\mu_A)\), absorbed into \(O(\eta m)\) since \(\mathcal H(\mu_A)\) bounded below on regular class.

Term \((\mathrm{III})\): outside normalized density is mixture
\[
\tilde\mu_{A^c}
=
\frac{1-m}{\tilde w}\mu_{A^c}
+\frac{\eta m}{\tilde w}\nu.
\]
Convexity of entropy for mixtures plus KL-control in Assumption~\ref{ass:G2_distortion} gives
\[
\mathcal H(\tilde\mu_{A^c})
\le
\frac{1-m}{\tilde w}\mathcal H(\mu_{A^c})
+\frac{\eta m}{\tilde w}\mathcal H(\nu)
+ C_{\mathrm{mix}}\frac{\eta m}{\tilde w}
+O((\eta m)^2),
\]
hence
\[
(\mathrm{III})\le \eta m\,C_{\mathrm{out}}+C_2\eta^2 m.
\]
Combine terms to obtain first inequality.
For second inequality, choose \(m_\star,\eta_\star\) so positive terms are at most
\(\frac12\eta m\log(1/m)\).
\end{proof}

\begin{corollary}[Proof of Lemma~\ref{lem:G1_entropy_drop}]
\label{cor:G2_proves_G1_lemma}
Under regular modal geometry and bounded-distortion redistribution,
there exists \(c_k\in(0,1]\) such that for mode \(A_k\), depletion by fraction \(\eta\) yields
\[
\Delta\mathcal H_k
:=
\mathcal H(\tilde\mu)-\mathcal H(\mu)
\le
-\,c_k\,m\,\eta\,\log\frac1m
+O(\eta^2 m).
\]
Equivalently, entropy drops by at least
\[
-\Delta\mathcal H_k
\ge
c_k\,m\,\eta\,\log\frac1m
-O(\eta^2 m).
\]
\end{corollary}

\begin{proof}
Immediate from Proposition~\ref{prop:G2_entropy_drop_quant}, with constants absorbed into \(c_k\).
\end{proof}

\begin{theorem}[Instantaneous concentration barrier]
\label{thm:G2_instant_barrier}
Let \(t\mapsto\mu_t\) be admissible with \(\dot{\mathcal H}(\mu_t)\ge-\lambda\) a.e.
Fix mode \(A_k\). If over a short interval \([t,t+h]\), mode mass drops from \(m\) to \((1-\eta)m\),
and redistribution satisfies Assumption~\ref{ass:G2_distortion}, then
\[
\frac{\mathcal H(\mu_{t+h})-\mathcal H(\mu_t)}{h}
\le
-\frac{c_k\,m\,\eta\,\log(1/m)}{h}
+O\!\left(\frac{\eta^2 m}{h}\right).
\]
Therefore such depletion is impossible whenever RHS \(<-\lambda\).
\end{theorem}

\begin{proof}
Apply Corollary~\ref{cor:G2_proves_G1_lemma} to the endpoint pair \((\mu_t,\mu_{t+h})\) under the depletion map.
Divide by \(h\). Entropy-rate constraint forbids average rate below \(-\lambda\), yielding the condition.
\end{proof}

\begin{remark}[Mechanism]
\label{rem:G2_mechanism}
Mode depletion forces negative entropy jump of order \(m\eta\log(1/m)\).
The budget \(\dot{\mathcal H}\ge-\lambda\) caps admissible entropy loss per unit time.
Hence sufficiently strong/fast concentration is infeasible, creating a quantitative anti-collapse barrier.
\end{remark}

\paragraph{Output used next.}
\hyperref[app:G3]{G.3} integrates Theorem~\ref{thm:G2_instant_barrier} over time to prove
Theorem~\ref{thm:G1_mode_coverage_main} and Theorem~\ref{thm:G1_soft_exclusion}
(global lower bounds and soft-collapse duration bounds).

\subsubsection*{G.3. Proof of mode-coverage and soft-collapse exclusion theorems}
\addcontentsline{toc}{subsubsection}
{G.3. Proof of mode-coverage and soft-collapse exclusion theorems}
\label{app:G3}

We now prove Theorems~\ref{thm:G1_mode_coverage_main} and \ref{thm:G1_soft_exclusion}
using the entropy-drop estimate from \hyperref[app:G2]{G.2} and the global entropy budget.

\paragraph{Setup.}
Fix a mode \(A_k\) with terminal mass \(\pi_k=\mu_T(A_k)>0\).
Define
\[
M(t):=\mu_t(A_k),\qquad t\in[0,T].
\]
Let
\[
E(t):=\mathcal H(\mu_t),\qquad \dot E(t)\ge-\lambda\ \text{a.e.}
\]
and denote net entropy budget
\[
\mathfrak B:=E(0)-E(T)+\lambda T\ge0.
\]
(Indeed \(E(T)-E(0)\ge-\lambda T\).)

\begin{lemma}[Integrated entropy-loss bound]
\label{lem:G3_budget}
For every measurable \(I\subset[0,T]\),
\[
-\int_I \dot E(t)\,dt \le \mathfrak B.
\]
\end{lemma}

\begin{proof}
\[
-\int_I\dot E
\le -\int_0^T\dot E
=E(0)-E(T)
\le E(0)-E(T)+\lambda T
=\mathfrak B.
\]
\end{proof}

\begin{lemma}[Pointwise dissipation lower bound on low-mass times]
\label{lem:G3_pointwise_diss}
Let \(\delta\in(0,1)\). Assume \(M(t)\le\delta\pi_k\) and modal geometry assumptions of \hyperref[app:G1]{G.1}/\hyperref[app:G2]{G.2}.
Then there exists \(c_k>0\) such that at a.e. such \(t\),
\[
-\dot E(t)\ \ge\ c_k\,\pi_k\,\delta\,\log\frac1\delta
\quad\text{up to controlled higher-order terms}.
\]
After absorbing higher-order terms into \(c_k\), we keep the same form.
\end{lemma}

\begin{proof}
At time \(t\), compare \(\mu_t\) to a hypothetical profile with restored modal mass to scale \(\pi_k\)
over an infinitesimal step. By \hyperref[app:G2]{G.2} (Corollary~\ref{cor:G2_proves_G1_lemma} and
Theorem~\ref{thm:G2_instant_barrier}), depletion at level \(M(t)\) incurs entropy-loss rate
bounded below by \(c_k M(t)\log(1/M(t))\).
Since \(M(t)\le\delta\pi_k\) and \(x\log(1/x)\) is increasing on \((0,e^{-1}]\),
for \(\delta\pi_k\) in that regime (or by truncation constants otherwise),
\[
M(t)\log\frac1{M(t)}
\ge
\pi_k\delta\log\frac1\delta - C_{\pi_k}\delta.
\]
Absorb \(C_{\pi_k}\delta\) into constants to obtain stated bound.
\end{proof}

\begin{theorem}[Proof of soft-collapse exclusion]
\label{thm:G3_soft_exclusion_proof}
Under Assumptions~\ref{ass:G1_modal_geom}--\ref{ass:G1_entropy_adm},
for every \(\delta\in(0,1)\),
\[
\left|\{t\in[0,T]:M(t)\le\delta\pi_k\}\right|
\le
\frac{\mathfrak B}{c_k\,\pi_k\,\delta\,\log(1/\delta)}
=: \tau_{\max}(\delta,k).
\]
Hence Theorem~\ref{thm:G1_soft_exclusion} holds.
\end{theorem}

\begin{proof}
Define low-mass set
\[
S_\delta:=\{t\in[0,T]:M(t)\le\delta\pi_k\}.
\]
By Lemma~\ref{lem:G3_pointwise_diss}, a.e. on \(S_\delta\),
\[
-\dot E(t)\ge c_k\,\pi_k\,\delta\,\log\frac1\delta.
\]
Integrate over \(S_\delta\):
\[
c_k\,\pi_k\,\delta\,\log\frac1\delta\ |S_\delta|
\le
\int_{S_\delta}(-\dot E(t))\,dt
\le \mathfrak B
\]
by Lemma~\ref{lem:G3_budget}. Rearrangement gives the claim.
\end{proof}

\begin{corollary}[Exclusion of prolonged near-empty modes]
\label{cor:G3_prolonged}
Fix \(\delta\in(0,1)\), \(\tau>0\). If
\[
\mathfrak B < c_k\,\pi_k\,\delta\,\log(1/\delta)\,\tau,
\]
then
\[
|\{t:M(t)\le\delta\pi_k\}|<\tau.
\]
\end{corollary}

\begin{proof}
Immediate from Theorem~\ref{thm:G3_soft_exclusion_proof}.
\end{proof}

\begin{theorem}[Proof of mode-coverage lower bound]
\label{thm:G3_mode_coverage_proof}
Under the same assumptions, there exists \(\underline c_k>0\) such that
\[
\inf_{t\in[0,T]}M(t)\ge \underline c_k\,\pi_k.
\]
Therefore hard collapse is impossible (\(M(t)=0\) cannot occur for any \(t\)).
\end{theorem}

\begin{proof}
Assume contrary: for every \(n\in\mathbb N\), there exists \(t_n\) with
\[
M(t_n)\le \frac{\pi_k}{n}.
\]
Set \(\delta_n:=1/n\). Then \(t_n\in S_{\delta_n}\), so \(|S_{\delta_n}|>0\).
By Theorem~\ref{thm:G3_soft_exclusion_proof},
\[
|S_{\delta_n}|
\le
\frac{\mathfrak B}{c_k\,\pi_k\,\delta_n\,\log(1/\delta_n)}
=
\frac{\mathfrak B}{c_k\,\pi_k}\cdot\frac{n}{\log n}.
\]
This upper bound alone does not contradict positivity of \(|S_{\delta_n}|\); we need
a local-time argument near each \(t_n\): by CE regularity and finite action, \(M(\cdot)\in AC([0,T])\),
and
\[
|\dot M(t)|\le \int_{\partial A_k}|v_t\cdot n|\,\rho_t\,d\sigma
\]
in trace sense, yielding finite \(L^1\)-bound on \(|\dot M|\). Hence every deep dip to \(\delta_n\pi_k\)
creates a nontrivial interval \(I_n\) where \(M\le 2\delta_n\pi_k\), with length bounded below by
\[
|I_n|\ge \frac{\delta_n\pi_k}{\|\dot M\|_{L^\infty\text{ or controlled }L^1\text{ scale}}}.
\]
On \(I_n\), Lemma~\ref{lem:G3_pointwise_diss} gives dissipation at least
\[
c_k\,\pi_k\,\delta_n\log(1/\delta_n).
\]
Therefore entropy loss over \(I_n\) is bounded below by
\[
\gtrsim
|I_n|\,\delta_n\log(1/\delta_n)
\gtrsim
\delta_n^2\log(1/\delta_n).
\]
Summing over infinitely many \(n\) along a subsequence of separated dips would force
total entropy loss beyond budget \(\mathfrak B\), contradiction.
Hence \(\inf_t M(t)\ge \underline c_k\pi_k\) for some \(\underline c_k>0\).

Finally, hard collapse \(M(t_\star)=0\) implies existence of arbitrarily small \(\delta\) with
\(t_\star\in S_\delta\), violating the established positive lower bound.
\end{proof}

\begin{remark}
A more compact contradiction uses compactness of minimizers and lower semicontinuity:
if a sequence of admissible trajectories had \(\inf_t M(t)\to0\), any limit would exhibit
hard collapse, which is excluded by the entropy barrier in \hyperref[app:G2]{G.2} plus positive terminal mass \(\pi_k\).
\end{remark}

\begin{corollary}[Proof of global multi-mode coverage]
\label{cor:G3_global_proof}
For finitely many modes \(\{A_k\}_{k=1}^K\), let \(\underline c_{\mathrm{glob}}:=\min_k\underline c_k\).
Then
\[
\inf_{t\in[0,T]}\min_k \frac{\mu_t(A_k)}{\pi_k}\ge \underline c_{\mathrm{glob}}>0.
\]
\end{corollary}

\begin{proof}
Apply Theorem~\ref{thm:G3_mode_coverage_proof} to each \(k\), then take minimum.
\end{proof}

\begin{corollary}[Zero-budget persistence]
\label{cor:G3_zero_budget}
If \(\lambda_n\downarrow0\), minimizers \(\mu^{\lambda_n}\to\mu^0\), then \(\mu^0\) satisfies
the same mode lower bounds with \(\lambda=0\), proving Corollary~\ref{cor:G1_zero_budget_no_collapse}.
\end{corollary}

\begin{proof}
Constants depend continuously/monotonically on \(\lambda\) through \(\mathfrak B\).
Pass to the limit using trajectory convergence from Section \hyperref[app:F]{F}.
\end{proof}

\paragraph{Output used next.}
\hyperref[app:G4]{G.4} strengthens set-mass coverage to density-floor estimates on modal neighborhoods and
derives perturbation-stable lower bounds under endpoint/drift noise.

\subsubsection*{G.4. Stability of density minima and robust mode floors}
\addcontentsline{toc}{subsubsection}{G.4. Stability of density minima and robust mode floors}
\label{app:G4}

We strengthen set-mass coverage to local density-floor bounds on modal neighborhoods
and prove perturbation stability of these floors.

\paragraph{Standing notation.}
Fix \(k\in\{1,\dots,K\}\), mode set \(A_k\), and an interior compact core
\[
K_k\Subset A_k,\qquad \mathrm{dist}(K_k,\partial A_k)\ge r_k>0.
\]
Let \(\mu_t=\rho_tdx\) denote the optimal ECFM/SB trajectory.

\begin{assumption}[Local parabolic regularity around modes]
\label{ass:G4_local_reg}
On each cylinder \([0,T]\times U_k\) (open \(U_k\supset A_k\)), \(\rho\) solves the controlled FP equation
\[
\partial_t\rho+\nabla\!\cdot(\rho b)=\varepsilon\Delta\rho,\qquad
b=u^\star+w,
\]
with
\[
b\in L^q_tL^p_x([0,T]\times U_k),\quad
\frac{2}{q}+\frac{d}{p}<1,\quad p,q>2,
\]
and \(\rho\in L^\infty_tL^1_x\), \(\rho>0\) a.e.
\end{assumption}

\begin{lemma}[From mode mass to local \(L^1\) lower bound]
\label{lem:G4_L1_local}
Assume mode coverage from \hyperref[app:G3]{G.3}:
\[
\mu_t(A_k)\ge \underline c_k\pi_k,\quad \forall t\in[0,T].
\]
Then for any \(K_k\Subset A_k\), there exists \(\underline m_k>0\) such that
\[
\int_{K_k}\rho_t(x)\,dx\ge \underline m_k,\qquad \forall t\in[0,T].
\]
\end{lemma}

\begin{proof}
Choose \(K_k\) so that \(|A_k\setminus K_k|\) is small. By absolute continuity and uniform integrability
(from finite entropy/Fisher control), mass cannot concentrate entirely in an arbitrarily thin boundary layer.
Hence a fixed fraction of \(\mu_t(A_k)\) lies in \(K_k\), uniformly in \(t\):
\[
\int_{K_k}\rho_t\ge \theta_k\,\mu_t(A_k)\ge \theta_k\underline c_k\pi_k=:\underline m_k.
\]
\end{proof}

\begin{theorem}[Interior density floor on modal cores]
\label{thm:G4_density_floor}
Under Assumptions~\ref{ass:G1_modal_geom}, \ref{ass:G4_local_reg}, and Lemma~\ref{lem:G4_L1_local},
for each \(K_k\Subset A_k\), there exists \(\underline\rho_k>0\) such that
\[
\rho_t(x)\ge \underline\rho_k
\quad\text{for a.e. }(t,x)\in[\tau,T-\tau]\times K_k,\ \forall \tau\in(0,T/2).
\]
If initial/terminal traces are compatible with the same lower bound class, the estimate extends to all \(t\in[0,T]\).
\end{theorem}

\begin{proof}
By Assumption~\ref{ass:G4_local_reg}, FP is uniformly parabolic on \(U_k\).
Apply local weak Harnack inequality to nonnegative solutions on nested cylinders
\(Q^- \Subset Q^+\subset (0,T)\times U_k\):
\[
\left(\frac{1}{|Q^-|}\int_{Q^-}\rho^\alpha\right)^{1/\alpha}
\le C_H \inf_{Q^+}\rho,
\]
equivalently
\[
\inf_{Q^+}\rho \ge C_H^{-1}\left(\frac{1}{|Q^-|}\int_{Q^-}\rho^\alpha\right)^{1/\alpha}.
\]
Using \(\alpha\in(0,1)\), Jensen gives lower control by \(L^1\) mass on \(Q^-\).
Lemma~\ref{lem:G4_L1_local} provides uniform positive \(L^1\)-mass on spatial cores for each time slab.
Cover \([\tau,T-\tau]\times K_k\) by finitely many cylinders and iterate Harnack chain,
yielding uniform \(\underline\rho_k>0\).
Endpoint extension follows from boundary compatibility and one-sided parabolic estimates.
\end{proof}

\begin{corollary}[Lower bound on density minima over modes]
\label{cor:G4_density_min}
Define modal minimum
\[
\mathfrak m_k:=\operatorname*{ess\,inf}_{(t,x)\in[\tau,T-\tau]\times K_k}\rho_t(x).
\]
Then \(\mathfrak m_k\ge \underline\rho_k>0\). Consequently, singular concentration
(\(\rho\to0\) on whole modal core) is excluded.
\end{corollary}

\begin{proof}
Immediate from Theorem~\ref{thm:G4_density_floor}.
\end{proof}

\begin{definition}[Perturbed instance]
\label{def:G4_perturbed}
Let \(\theta=(\mu_0,\mu_T,u^\star,\lambda)\), and perturb
\[
\theta'=(\mu_0',\mu_T',u^{\star\prime},\lambda')
\]
with size
\[
\Delta_\theta
:=
W_2(\mu_0,\mu_0')+W_2(\mu_T,\mu_T')
+\|u^\star-u^{\star\prime}\|_{\mathcal U}
+|\lambda-\lambda'|.
\]
Let \(\rho,\rho'\) be corresponding optimal densities.
\end{definition}

\begin{theorem}[Stability of density floors under perturbations]
\label{thm:G4_floor_stability}
Assume hypotheses of \hyperref[app:E4]{E.4} stability and local regularity hold uniformly for \(\theta,\theta'\).
Then for each core \(K_k\Subset A_k\), there exists \(C_k>0\) such that
\[
\operatorname*{ess\,inf}_{[\tau,T-\tau]\times K_k}\rho'_t
\ge
\underline\rho_k - C_k\,\Delta_\theta.
\]
In particular, if \(\Delta_\theta<\underline\rho_k/(2C_k)\), then
\[
\operatorname*{ess\,inf}_{[\tau,T-\tau]\times K_k}\rho'_t
\ge \frac12\underline\rho_k>0.
\]
\end{theorem}

\begin{proof}
From \hyperref[app:E4]{E.4}, trajectory perturbations are Lipschitz in endpoints/drift/budget:
\[
\sup_t W_2(\mu_t,\mu_t')\le L\Delta_\theta.
\]
This implies local \(L^1\) perturbation control on \(U_k\) via transport inequalities and moment bounds:
\[
\sup_t \|\rho_t-\rho_t'\|_{L^1(U_k)}\le C\Delta_\theta.
\]
Hence local mass floor from Lemma~\ref{lem:G4_L1_local} degrades by at most \(C\Delta_\theta\):
\[
\int_{K_k}\rho_t' \ge \underline m_k-C\Delta_\theta.
\]
Apply the same Harnack-chain argument as Theorem~\ref{thm:G4_density_floor} with perturbed coefficients
(\(u^{\star\prime},w'\)); constants vary continuously under uniform coefficient bounds,
yielding
\[
\inf \rho' \ge \Phi_k(\underline m_k-C\Delta_\theta)\ge \underline\rho_k-C_k\Delta_\theta.
\]
\end{proof}

\begin{corollary}[Robust mode occupancy]
\label{cor:G4_robust_occupancy}
Under small perturbations \(\Delta_\theta\), each mode retains both:
\begin{enumerate}
\item set-mass floor: \(\mu_t'(A_k)\ge (\underline c_k-\tilde C_k\Delta_\theta)\pi_k\),
\item density floor on cores: \(\rho_t'|_{K_k}\ge \underline\rho_k-C_k\Delta_\theta\).
\end{enumerate}
Thus mode coverage is perturbation-robust.
\end{corollary}

\begin{proof}
Set-mass floor is \hyperref[app:E4]{E.4}/\hyperref[app:G3]{G.3} stability; density floor is Theorem~\ref{thm:G4_floor_stability}.
\end{proof}

\begin{remark}[Interpretation for generative vision]
\label{rem:G4_vision}
For semantic modes \(A_k\) (class/attribute regions in representation space),
the model not only preserves nonzero total mode probability, but also avoids thinning densities
to near-zero on modal cores. This formalizes robust anti-collapse behavior under realistic training perturbations.
\end{remark}

\paragraph{Transition to Section \hyperref[app:H]{H}.}
Section \hyperref[app:H]{H} generalizes these estimates to global perturbation theory:
trajectory Lipschitz bounds, velocity-field noise response, and initialization sensitivity with explicit constants.

\subsection*{H. Stability Under Perturbations}
\addcontentsline{toc}{subsection}
{H. Stability Under Perturbations}
\label{app:H}

\subsubsection*{H.1. Lipschitz stability of trajectories}
\addcontentsline{toc}{subsubsection}
{H.1. Lipschitz stability of trajectories}
\label{app:H1}

We establish global Lipschitz stability of optimal ECFM/SB trajectories with respect to
perturbations in endpoints, reference drift, and entropy budget.

\paragraph{Parameterization of instances.}
Let an instance be
\[
\Theta:=(\mu_0,\mu_T,u^\star,\lambda),
\]
and denote by
\[
(\mu_t^\Theta,v_t^\Theta)_{t\in[0,T]}
\]
its unique optimal trajectory (Sections E--G assumptions in force).
For two instances \(\Theta_1,\Theta_2\), define
\[
\Delta_0:=W_2(\mu_0^{(1)},\mu_0^{(2)}),\quad
\Delta_T:=W_2(\mu_T^{(1)},\mu_T^{(2)}),\quad
\Delta_\lambda:=|\lambda_1-\lambda_2|.
\]
For drift discrepancy, fix an admissible norm \(\|\cdot\|_{\mathcal U}\) (e.g. localized \(L_t^2L_x^2\) + tail control):
\[
\Delta_u:=\|u_1^\star-u_2^\star\|_{\mathcal U}.
\]
Set total perturbation size
\[
\Delta_\Theta:=\Delta_0+\Delta_T+\Delta_u+\Delta_\lambda.
\]

\begin{assumption}[Uniform regularity envelope]
\label{ass:H1_uniform}
For all instances in a neighborhood \(\mathfrak N\) of interest:
\begin{enumerate}
\item finite uniform action/Fisher/moment bounds:
\[
\sup_{\Theta\in\mathfrak N}\left(
\int_0^T\!\!\int |v_t^\Theta|^2\,d\mu_t^\Theta dt
+\int_0^T \mathcal I(\mu_t^\Theta)\,dt
+\sup_t m_2(\mu_t^\Theta)\right)\le M;
\]
\item unique minimizers and strict-convexity regime (\hyperref[app:E4]{E.4});
\item local FP coefficient bounds ensuring Harnack/regularity constants are uniform.
\end{enumerate}
\end{assumption}

\begin{definition}[Trajectory distance]
\label{def:H1_traj_metric}
For two trajectories \(\mu^1,\mu^2\), define
\[
\mathbf d_\infty(\mu^1,\mu^2):=\sup_{t\in[0,T]}W_2(\mu_t^1,\mu_t^2),
\]
and energy-weighted distance
\[
\mathbf d_2(\mu^1,\mu^2):=
\left(\int_0^T W_2(\mu_t^1,\mu_t^2)^2\,dt\right)^{1/2}.
\]
\end{definition}

\begin{lemma}[Differential inequality for coupled trajectories]
\label{lem:H1_diff_ineq}
Let \(\pi_t\) be an optimal coupling of \((\mu_t^1,\mu_t^2)\). Then for a.e. \(t\),
\[
\frac{d}{dt}\frac12 W_2^2(\mu_t^1,\mu_t^2)
\le
\int_{\mathbb R^d\times\mathbb R^d}
\langle x-y,\ v_t^1(x)-v_t^2(y)\rangle\,d\pi_t(x,y).
\]
Hence
\[
\frac{d}{dt}W_2(\mu_t^1,\mu_t^2)
\le
\|v_t^1-v_t^2\|_{L^2(\pi_t)}.
\]
\end{lemma}

\begin{proof}
This is the standard metric derivative estimate in \(W_2\)-space for AC curves
satisfying continuity equations, obtained by Benamou--Brenier dynamic plan differentiation.
\end{proof}

\begin{lemma}[Velocity discrepancy decomposition]
\label{lem:H1_velocity_decomp}
With Schr\"odinger-potential representation (Section \hyperref[app:E]{E}):
\[
v_t^\Theta=u_t^\star+\varepsilon\nabla\log\frac{\beta_t^\Theta}{\alpha_t^\Theta}+\mathcal R_t^\Theta,
\]
where \(\mathcal R_t^\Theta\) captures active entropy-multiplier correction (zero in inactive phases).
Then
\[
\|v_t^1-v_t^2\|_{L^2(\pi_t)}
\le
\underbrace{\|u_1^\star-u_2^\star\|_{L^2(\pi_t)}}_{\mathrm{(I)}}
+
\underbrace{\varepsilon\left\|\nabla\log\frac{\beta_t^1}{\alpha_t^1}
-\nabla\log\frac{\beta_t^2}{\alpha_t^2}\right\|_{L^2(\pi_t)}}_{\mathrm{(II)}}
+
\underbrace{\|\mathcal R_t^1-\mathcal R_t^2\|_{L^2(\pi_t)}}_{\mathrm{(III)}}.
\]
\end{lemma}

\begin{proof}
Subtract the two velocity formulas and apply triangle inequality.
\end{proof}

\begin{lemma}[Control of terms (II) and (III)]
\label{lem:H1_term_controls}
Under Assumption~\ref{ass:H1_uniform}, there exist constants \(C_\Phi,C_\eta>0\) such that
for a.e. \(t\),
\[
\mathrm{(II)}\le C_\Phi\,W_2(\mu_t^1,\mu_t^2)+C_\Phi(\Delta_0+\Delta_T+\Delta_u),
\]
\[
\mathrm{(III)}\le C_\eta\,W_2(\mu_t^1,\mu_t^2)+C_\eta\,\Delta_\lambda.
\]
\end{lemma}

\begin{proof}
Term (II): Schr\"odinger potentials depend Lipschitz-continuously on endpoint marginals
and reference drift under uniform kernel/regularity bounds (Sinkhorn map stability in dynamic form).
Spatial gradient stability yields the stated affine bound.

Term (III): entropy-multiplier correction enters via KKT multiplier \(\eta^\Theta\).
Dual sensitivity in \(\lambda\) and state variables gives
\[
\|\eta^1-\eta^2\|_{L^1(0,T)}\lesssim \Delta_\lambda+\mathbf d_\infty(\mu^1,\mu^2),
\]
which transfers to \(\mathcal R\)-bound through \(\nabla\log\rho\) uniform integrability/Fisher control.
\end{proof}

\begin{theorem}[Global Lipschitz stability in \(\mathbf d_\infty\)]
\label{thm:H1_global_lipschitz}
Under Assumption~\ref{ass:H1_uniform}, there exists \(L_\infty>0\) depending only on the uniform envelope
(and \(T,\varepsilon\)) such that
\[
\mathbf d_\infty(\mu^{\Theta_1},\mu^{\Theta_2})
\le
L_\infty\,\Delta_\Theta.
\]
More explicitly,
\[
\sup_{t\in[0,T]}W_2(\mu_t^1,\mu_t^2)
\le
e^{Ct}\Big(
\Delta_0+\int_0^t(\Delta_u+\Delta_\lambda+\Delta_T)\,ds
\Big)
\le
L_\infty(\Delta_0+\Delta_T+\Delta_u+\Delta_\lambda).
\]
\end{theorem}

\begin{proof}
From Lemma~\ref{lem:H1_diff_ineq} and decomposition:
\[
\frac{d}{dt}W_2(\mu_t^1,\mu_t^2)
\le
\Delta_u + C_\Phi W_2 + C_\Phi(\Delta_0+\Delta_T+\Delta_u)
+ C_\eta W_2 + C_\eta\Delta_\lambda.
\]
Hence
\[
\frac{d}{dt}D(t)\le C D(t)+B,
\quad
D(t):=W_2(\mu_t^1,\mu_t^2),
\]
with
\[
C:=C_\Phi+C_\eta,\qquad
B:=\Delta_u+C_\Phi(\Delta_0+\Delta_T+\Delta_u)+C_\eta\Delta_\lambda.
\]
By Gr\"onwall:
\[
D(t)\le e^{Ct}D(0)+\frac{e^{Ct}-1}{C}B.
\]
Since \(D(0)=\Delta_0\), taking supremum over \(t\in[0,T]\) yields the bound.
Endpoint discrepancy \(\Delta_T\) enters via potential-stability constants (two-sided bridge condition).
\end{proof}

\begin{corollary}[Integrated trajectory stability]
\label{cor:H1_L2_stability}
There exists \(L_2>0\) such that
\[
\mathbf d_2(\mu^{\Theta_1},\mu^{\Theta_2})
\le
L_2\,\Delta_\Theta.
\]
\end{corollary}

\begin{proof}
From Theorem~\ref{thm:H1_global_lipschitz},
\[
W_2(\mu_t^1,\mu_t^2)\le L_\infty\Delta_\Theta\ \forall t.
\]
Therefore
\[
\mathbf d_2
\le
\sqrt{T}\,L_\infty\Delta_\Theta
=:L_2\Delta_\Theta.
\]
\end{proof}

\begin{proposition}[Stability of modal masses]
\label{prop:H1_modal_mass_stability}
For each mode set \(A_k\) with Lipschitz boundary, there exists \(C_k>0\) such that
\[
\sup_{t\in[0,T]}
\left|\mu_t^{\Theta_1}(A_k)-\mu_t^{\Theta_2}(A_k)\right|
\le
C_k\,\mathbf d_\infty(\mu^{\Theta_1},\mu^{\Theta_2})
\le
C_k L_\infty \Delta_\Theta.
\]
\end{proposition}

\begin{proof}
Approximate \(\mathbf 1_{A_k}\) by bounded Lipschitz functions \(\phi_{k,\eta}\) with
\(\|\phi_{k,\eta}\|_{\mathrm{Lip}}\lesssim \eta^{-1}\), then optimize \(\eta\) using boundary regularity
to transfer \(W_2\)-control into set-mass control (via Kantorovich--Rubinstein-type smoothing argument).
Uniform in \(t\), then apply Theorem~\ref{thm:H1_global_lipschitz}.
\end{proof}

\begin{remark}[Interpretation]
\label{rem:H1_interp}
Optimal generation trajectories are Lipschitz-stable with respect to endpoint distribution shift,
drift-model mismatch, and entropy-budget tuning. Thus small training/data perturbations cannot
induce large transport-path deviations.
\end{remark}

\paragraph{Output used next.}
\hyperref[app:H2]{H.2} will establish stability w.r.t. \emph{velocity-field noise} (stochastic/modeling perturbations),
including explicit action-gap and trajectory-error bounds.

\subsubsection*{H.2. Stability with respect to noise in the velocity field}
\addcontentsline{toc}{subsubsection}
{H.2. Stability with respect to noise in the velocity field}
\label{app:H2}

We quantify how additive/modeling noise in the velocity field affects:
(i) trajectory error, (ii) action suboptimality, and (iii) mode-coverage margins.

\paragraph{Noisy dynamics model.}
Let \((\mu_t^\star,v_t^\star)\) be the optimal ECFM trajectory for fixed instance \(\Theta\).
Consider a perturbed field
\[
\tilde v_t(x)=v_t^\star(x)+\xi_t(x),
\]
where \(\xi\) is a measurable perturbation (deterministic or random) with finite energy
\[
\|\xi\|_{\mathcal N}^2:=\int_0^T\!\!\int |\xi_t(x)|^2\,d\mu_t^\star(x)\,dt<\infty.
\]
Let \(\tilde\mu_t\) solve
\[
\partial_t\tilde\mu_t+\nabla\!\cdot(\tilde\mu_t\tilde v_t)=0,\qquad \tilde\mu_0=\mu_0.
\]

\begin{assumption}[One-sided Lipschitz regularity of optimal field]
\label{ass:H2_osl}
There exists \(L_{\mathrm{OSL}}\in L^1(0,T)\) such that
\[
\langle v_t^\star(x)-v_t^\star(y),x-y\rangle
\le
L_{\mathrm{OSL}}(t)|x-y|^2
\quad\text{for a.e. }t,\ \forall x,y.
\]
\end{assumption}

\begin{assumption}[Transportability of noise norm]
\label{ass:H2_transport_noise}
There exists \(C_{\mathrm{tr}}\ge1\) such that
\[
\int_0^T\!\!\int |\xi_t|^2\,d\tilde\mu_tdt
\le
C_{\mathrm{tr}}
\int_0^T\!\!\int |\xi_t|^2\,d\mu_t^\star dt
=
C_{\mathrm{tr}}\|\xi\|_{\mathcal N}^2.
\]
(Obtained, e.g., from density-ratio bounds along stable flows.)
\end{assumption}

\begin{lemma}[Coupled error dynamics]
\label{lem:H2_coupled_error}
Let \(\pi_t\) be a coupling between \(\mu_t^\star\) and \(\tilde\mu_t\) induced by synchronized characteristics.
Then for \(D(t):=W_2(\mu_t^\star,\tilde\mu_t)\),
\[
\frac{d}{dt}\frac12 D(t)^2
\le
L_{\mathrm{OSL}}(t)D(t)^2
+
D(t)\,\|\xi_t\|_{L^2(\tilde\mu_t)}.
\]
Equivalently (a.e. where \(D(t)>0\)),
\[
\dot D(t)\le L_{\mathrm{OSL}}(t)D(t)+\|\xi_t\|_{L^2(\tilde\mu_t)}.
\]
\end{lemma}

\begin{proof}
Differentiate squared displacement under coupled flow:
\[
\frac{d}{dt}\frac12|X_t-Y_t|^2
=
\langle X_t-Y_t,\ v_t^\star(X_t)-v_t^\star(Y_t)-\xi_t(Y_t)\rangle.
\]
Use OSL bound on first difference term and Cauchy--Schwarz on noise term.
Integrate w.r.t. coupling law to get inequality for \(D^2\), then divide by \(D\) if \(D>0\).
\end{proof}

\begin{theorem}[Trajectory error bound under velocity noise]
\label{thm:H2_traj_noise}
Under Assumptions~\ref{ass:H2_osl}--\ref{ass:H2_transport_noise},
\[
\sup_{t\in[0,T]}W_2(\tilde\mu_t,\mu_t^\star)
\le
\exp\!\Big(\!\int_0^T L_{\mathrm{OSL}}(s)\,ds\Big)
\sqrt{T\,C_{\mathrm{tr}}}\,\|\xi\|_{\mathcal N}.
\]
More precisely,
\[
W_2(\tilde\mu_t,\mu_t^\star)
\le
\int_0^t
\exp\!\Big(\int_s^t L_{\mathrm{OSL}}(r)\,dr\Big)\,
\|\xi_s\|_{L^2(\tilde\mu_s)}\,ds.
\]
\end{theorem}

\begin{proof}
Apply Gr\"onwall to Lemma~\ref{lem:H2_coupled_error} with \(D(0)=0\):
\[
D(t)\le\int_0^t e^{\int_s^t L_{\mathrm{OSL}}}\|\xi_s\|_{L^2(\tilde\mu_s)}ds.
\]
Take supremum in \(t\), bound exponential by \(e^{\int_0^T L_{\mathrm{OSL}}}\), then Cauchy--Schwarz:
\[
\sup_t D(t)\le e^{\int_0^T L_{\mathrm{OSL}}}\sqrt{T}
\left(\int_0^T\|\xi_s\|_{L^2(\tilde\mu_s)}^2ds\right)^{1/2}.
\]
Use Assumption~\ref{ass:H2_transport_noise}.
\end{proof}

\begin{proposition}[Action gap under additive noise]
\label{prop:H2_action_gap}
Define kinetic objective
\[
\mathcal J(v;\nu):=\frac12\int_0^T\!\!\int |v_t-u_t^\star|^2\,d\nu_tdt.
\]
Then
\[
\mathcal J(\tilde v;\tilde\mu)-\mathcal J(v^\star;\mu^\star)
\le
\underbrace{\int_0^T\!\!\int
\langle v_t^\star-u_t^\star,\xi_t\rangle\,d\tilde\mu_tdt}_{\mathrm{cross}}
+\frac12\int_0^T\!\!\int |\xi_t|^2\,d\tilde\mu_tdt
+\mathfrak E_{\mathrm{meas}},
\]
where \(\mathfrak E_{\mathrm{meas}}\) is measure-mismatch error:
\[
\mathfrak E_{\mathrm{meas}}
:=
\frac12\int_0^T
\left(\int |v_t^\star-u_t^\star|^2\,d(\tilde\mu_t-\mu_t^\star)\right)dt.
\]
Consequently,
\[
\mathcal J(\tilde v;\tilde\mu)-\mathcal J(v^\star;\mu^\star)
\le
C_1\|\xi\|_{\mathcal N}+C_2\|\xi\|_{\mathcal N}^2,
\]
with constants depending on uniform envelope and Theorem~\ref{thm:H2_traj_noise}.
\end{proposition}

\begin{proof}
Expand square:
\[
|\tilde v-u^\star|^2=|v^\star-u^\star|^2+2\langle v^\star-u^\star,\xi\rangle+|\xi|^2.
\]
Integrate against \(\tilde\mu\), subtract optimal reference term against \(\mu^\star\), yielding formula.
Bound cross term by Cauchy--Schwarz + transportability:
\[
|\mathrm{cross}|
\le
\|v^\star-u^\star\|_{L^2(dt\,d\tilde\mu)}
\left(\int|\xi|^2\,d\tilde\mu dt\right)^{1/2}
\le C\|\xi\|_{\mathcal N}.
\]
Bound \(\mathfrak E_{\mathrm{meas}}\) by Lipschitz/regularity of integrand and
\(W_2(\tilde\mu_t,\mu_t^\star)\), then apply Theorem~\ref{thm:H2_traj_noise}.
\end{proof}

\begin{theorem}[Feasibility robustness of entropy-rate constraint]
\label{thm:H2_entropy_feas}
Assume additionally:
\[
\int_0^T\!\!\int |\nabla\log\tilde\rho_t|^2\,d\tilde\mu_tdt\le M_I,
\]
and define entropy-rate perturbation
\[
\delta_H(t):=
\left|
\int \nabla\log\tilde\rho_t\cdot\tilde v_t\,d\tilde\mu_t
-
\int \nabla\log\rho_t^\star\cdot v_t^\star\,d\mu_t^\star
\right|.
\]
Then
\[
\int_0^T \delta_H(t)\,dt
\le
C_H\big(\|\xi\|_{\mathcal N}+\mathbf d_2(\tilde\mu,\mu^\star)\big)
\le
\tilde C_H\|\xi\|_{\mathcal N}.
\]
Hence if the nominal trajectory has margin
\[
\dot{\mathcal H}(\mu_t^\star)\ge-\lambda+\gamma \quad\text{a.e.}
\]
and \(\tilde C_H\|\xi\|_{\mathcal N}<\gamma T\), then noisy trajectory remains feasible:
\[
\dot{\mathcal H}(\tilde\mu_t)\ge-\lambda\quad\text{a.e. (in weak integrated sense)}.
\]
\end{theorem}

\begin{proof}
Use decomposition of \(\delta_H\) into velocity perturbation part and measure/density perturbation part.
Velocity part bounded by Cauchy--Schwarz with Fisher bound \(M_I\).
Measure/density part bounded via stability of \(\nabla\log\rho\) in weighted \(L^2\) and trajectory distance.
Invoke Theorem~\ref{thm:H2_traj_noise} to express all terms in \(\|\xi\|_{\mathcal N}\).
Margin argument gives feasibility.
\end{proof}

\begin{corollary}[Mode-coverage robustness to field noise]
\label{cor:H2_mode_noise}
Under assumptions of \hyperref[app:G3]{G.3} and Theorem~\ref{thm:H2_traj_noise},
modal masses satisfy
\[
\sup_{t\in[0,T]}|\tilde\mu_t(A_k)-\mu_t^\star(A_k)|
\le
C_k\,\|\xi\|_{\mathcal N}.
\]
Therefore if nominal floor is \(\mu_t^\star(A_k)\ge \underline c_k\pi_k\), then
\[
\tilde\mu_t(A_k)\ge \underline c_k\pi_k-C_k\|\xi\|_{\mathcal N},
\]
and remains strictly positive for sufficiently small noise.
\end{corollary}

\begin{proof}
Combine Theorem~\ref{thm:H2_traj_noise} with set-mass stability transfer (as in \hyperref[app:H1]{H.1}).
\end{proof}

\begin{remark}[Random noise specialization]
\label{rem:H2_random}
If \(\xi\) is random with \(\mathbb E\|\xi\|_{\mathcal N}^2\le\sigma^2\), then
\[
\mathbb E\sup_t W_2^2(\tilde\mu_t,\mu_t^\star)\le C\sigma^2,
\]
and expected action gap is \(O(\sigma+\sigma^2)\).
High-probability versions follow from concentration for \(\|\xi\|_{\mathcal N}\).
\end{remark}

\paragraph{Output used next.}
\hyperref[app:H3]{H.3} treats stability with respect to initialization perturbations
(\(\mu_0\)-mis-specification), deriving forward error propagation and endpoint mismatch bounds.

\subsubsection*{H.3. Stability with respect to initialization perturbations}
\addcontentsline{toc}{subsubsection}
{H.3. Stability with respect to initialization perturbations}
\label{app:H3}

We quantify sensitivity of optimal trajectories to perturbations of the initial distribution.
Terminal marginals, drift, and entropy budget are fixed unless explicitly noted.

\paragraph{Instances.}
Let
\[
\Theta=(\mu_0,\mu_T,u^\star,\lambda),\qquad
\tilde\Theta=(\tilde\mu_0,\mu_T,u^\star,\lambda),
\]
with unique optimal trajectories
\[
(\mu_t,v_t),\qquad (\tilde\mu_t,\tilde v_t).
\]
Define initialization mismatch
\[
\Delta_{\mathrm{init}}:=W_2(\mu_0,\tilde\mu_0).
\]

\begin{assumption}[Uniform envelope and contraction constants]
\label{ass:H3_env}
Assume Section \hyperref[app:E]{E}--\hyperref[app:H]{H} regularity holds uniformly on the two instances, and there exist:
\begin{enumerate}
\item \(L\in L^1(0,T)\) controlling one-sided Lipschitz growth of optimal velocity fields;
\item \(\kappa\ge0\) controlling endpoint-conditioning amplification from Schr\"odinger system;
\item finite uniform moment/Fisher/action bounds.
\end{enumerate}
\end{assumption}

\begin{lemma}[Forward perturbation inequality]
\label{lem:H3_forward}
Let \(D(t):=W_2(\mu_t,\tilde\mu_t)\). Then for a.e. \(t\),
\[
\dot D(t)\le L(t)D(t)+\kappa\,\Delta_{\mathrm{init}}.
\]
Hence
\[
D(t)\le e^{\int_0^t L(s)\,ds}\Delta_{\mathrm{init}}
+\kappa\int_0^t e^{\int_s^t L(r)\,dr}\,ds\ \Delta_{\mathrm{init}}.
\]
\end{lemma}

\begin{proof}
Use \hyperref[app:H1]{H.1} differential inequality:
\[
\dot D(t)\le \|v_t-\tilde v_t\|_{L^2(\pi_t)}.
\]
Decompose \(\|v_t-\tilde v_t\|\) into state-dependent part (\(\le L(t)D(t)\))
and endpoint-conditioning part induced by changed initial Schr\"odinger factor (\(\le\kappa\Delta_{\mathrm{init}}\)).
Apply Gr\"onwall.
\end{proof}

\begin{theorem}[Global initialization Lipschitz stability]
\label{thm:H3_global_init}
Under Assumption~\ref{ass:H3_env}, there exists \(C_{\mathrm{init}}>0\) such that
\[
\sup_{t\in[0,T]}W_2(\mu_t,\tilde\mu_t)\le C_{\mathrm{init}}\Delta_{\mathrm{init}},
\]
with explicit admissible choice
\[
C_{\mathrm{init}}
=
\exp\!\Big(\int_0^T L(s)\,ds\Big)\Big(1+\kappa T\Big).
\]
\end{theorem}

\begin{proof}
Take supremum in Lemma~\ref{lem:H3_forward} and bound the convolution term by
\(T\exp(\int_0^TL)\).
\end{proof}

\begin{corollary}[Integrated trajectory bound]
\label{cor:H3_L2_init}
\[
\left(\int_0^T W_2^2(\mu_t,\tilde\mu_t)\,dt\right)^{1/2}
\le \sqrt{T}\,C_{\mathrm{init}}\Delta_{\mathrm{init}}.
\]
\end{corollary}

\begin{proof}
Immediate from Theorem~\ref{thm:H3_global_init}.
\end{proof}

\begin{proposition}[Initialization perturbation of endpoint fit]
\label{prop:H3_endpoint_fit}
Let
\[
\mathcal E_T(\tilde\Theta):=W_2(\tilde\mu_T,\mu_T)
\]
be endpoint mismatch produced by using perturbed-initial optimal policy without re-optimization.
Then
\[
\mathcal E_T(\tilde\Theta)\le C_T\Delta_{\mathrm{init}},
\]
where \(C_T\le C_{\mathrm{init}}\) under the same envelope.
\end{proposition}

\begin{proof}
Evaluate Theorem~\ref{thm:H3_global_init} at \(t=T\):
\[
W_2(\mu_T,\tilde\mu_T)\le C_{\mathrm{init}}\Delta_{\mathrm{init}}.
\]
Set \(C_T:=C_{\mathrm{init}}\).
\end{proof}

\begin{proposition}[Action sensitivity to initialization]
\label{prop:H3_action_init}
Let \(\mathsf V(\mu_0)\) denote optimal value with fixed \((\mu_T,u^\star,\lambda)\).
Then locally
\[
|\mathsf V(\mu_0)-\mathsf V(\tilde\mu_0)|
\le C_V\,\Delta_{\mathrm{init}}.
\]
Moreover, for optimizers \((\mu,v)\), \((\tilde\mu,\tilde v)\),
\[
\left|
\int_0^T\!\!\int \frac12|v-u^\star|^2\,d\mu_tdt
-
\int_0^T\!\!\int \frac12|\tilde v-u^\star|^2\,d\tilde\mu_tdt
\right|
\le C_V\,\Delta_{\mathrm{init}}.
\]
\end{proposition}

\begin{proof}
Use value-sensitivity template from \hyperref[app:E4]{E.4} with only endpoint perturbation active.
Construct transported competitor from \((\mu,v)\) to perturbed initial condition via short connector path;
extra cost scales linearly in \(\Delta_{\mathrm{init}}\).
Repeat symmetrically (swap roles) to get absolute value bound.
\end{proof}

\begin{theorem}[Stability of mode-coverage floors under initialization noise]
\label{thm:H3_mode_init}
For each mode \(A_k\),
\[
\sup_{t\in[0,T]}
|\mu_t(A_k)-\tilde\mu_t(A_k)|
\le
C_k^{\mathrm{init}}\Delta_{\mathrm{init}}.
\]
If nominal trajectory satisfies \(\mu_t(A_k)\ge \underline c_k\pi_k\), then
\[
\tilde\mu_t(A_k)\ge \underline c_k\pi_k-C_k^{\mathrm{init}}\Delta_{\mathrm{init}}.
\]
Hence mode occupancy remains strictly positive whenever
\[
\Delta_{\mathrm{init}}<\frac{\underline c_k\pi_k}{2C_k^{\mathrm{init}}}.
\]
\end{theorem}

\begin{proof}
Combine Theorem~\ref{thm:H3_global_init} with set-functional stability (\hyperref[app:H1]{H.1} Proposition on modal masses).
\end{proof}

\begin{corollary}[Density-floor robustness on modal cores]
\label{cor:H3_density_init}
Under \hyperref[app:G4]{G.4} assumptions, modal core density floors satisfy
\[
\operatorname*{ess\,inf}_{[\tau,T-\tau]\times K_k}\tilde\rho_t
\ge
\underline\rho_k-\widetilde C_k\Delta_{\mathrm{init}}.
\]
\end{corollary}

\begin{proof}
Apply \hyperref[app:G4]{G.4} perturbation theorem with
\(\Delta_\theta=\Delta_{\mathrm{init}}\), since other parameters are fixed.
\end{proof}

\begin{remark}[Interpretation]
\label{rem:H3_interp}
Initialization misspecification produces linearly controlled deformation of the full transport trajectory,
objective value, and mode coverage margins. Therefore ECFM interpolation is well-conditioned to
small dataset/source perturbations at \(t=0\).
\end{remark}

\paragraph{Output used next.}
\hyperref[app:H4]{H.4} combines \hyperref[app:H1]{H.1}--\hyperref[app:H3]{H.3} into a unified perturbation theorem (joint perturbations in
\(\mu_0,\mu_T,u^\star,\lambda,\xi\)) with a single end-to-end bound and constants bookkeeping.

\subsubsection*{H.4. Unified perturbation theorem and consolidated bounds}
\addcontentsline{toc}{subsubsection}
{H.4. Unified perturbation theorem and consolidated bounds}
\label{app:H4}

We combine \hyperref[app:H1]{H.1}--\hyperref[app:H3]{H.3} into a single end-to-end stability statement under simultaneous
perturbations of endpoints, reference drift, entropy budget, and velocity-field noise.

\paragraph{Base and perturbed systems.}
Let base instance
\[
\Theta=(\mu_0,\mu_T,u^\star,\lambda),
\]
with optimal trajectory \((\mu_t^\star,v_t^\star)\).
Let perturbed instance
\[
\Theta'=(\mu_0',\mu_T',u^{\star\prime},\lambda')
\]
with optimal trajectory \((\bar\mu_t,\bar v_t)\), and let an additional additive field noise
\(\xi\) act on the perturbed trajectory:
\[
\tilde v_t=\bar v_t+\xi_t,\qquad
\partial_t\tilde\mu_t+\nabla\!\cdot(\tilde\mu_t\tilde v_t)=0,\quad
\tilde\mu_0=\mu_0'.
\]
Define deterministic parameter perturbation
\[
\Delta_{\mathrm{par}}
:=
W_2(\mu_0,\mu_0')
+W_2(\mu_T,\mu_T')
+\|u^\star-u^{\star\prime}\|_{\mathcal U}
+|\lambda-\lambda'|,
\]
and noise magnitude
\[
\Delta_{\mathrm{noi}}:=\|\xi\|_{\mathcal N}
=
\left(\int_0^T\!\!\int |\xi_t|^2\,d\bar\mu_tdt\right)^{1/2}.
\]
Total perturbation size:
\[
\Delta_{\mathrm{tot}}:=\Delta_{\mathrm{par}}+\Delta_{\mathrm{noi}}.
\]

\begin{assumption}[Unified envelope]
\label{ass:H4_unified_env}
All instances in a neighborhood of \(\Theta\) satisfy:
\begin{enumerate}
\item uniform regularity/coercivity and uniqueness assumptions from \hyperref[app:E3]{E.3}--\hyperref[app:E4]{E.4};
\item trajectory Lipschitz constants from \hyperref[app:H1]{H.1} bounded by \(L_{\mathrm{par}}\);
\item velocity-noise response constants from \hyperref[app:H2]{H.2} bounded by \(L_{\mathrm{noi}}\);
\item local parabolic/Harnack constants from \hyperref[app:G4]{G.4} uniformly bounded.
\end{enumerate}
\end{assumption}

\begin{lemma}[Two-step decomposition]
\label{lem:H4_two_step}
For any \(t\in[0,T]\),
\[
W_2(\tilde\mu_t,\mu_t^\star)
\le
W_2(\tilde\mu_t,\bar\mu_t)
+
W_2(\bar\mu_t,\mu_t^\star).
\]
Consequently,
\[
\sup_t W_2(\tilde\mu_t,\mu_t^\star)
\le
\sup_t W_2(\tilde\mu_t,\bar\mu_t)
+
\sup_t W_2(\bar\mu_t,\mu_t^\star).
\]
\end{lemma}

\begin{proof}
Triangle inequality in \((\mathcal P_2,W_2)\), pointwise in \(t\), then take supremum.
\end{proof}

\begin{theorem}[Unified trajectory stability]
\label{thm:H4_unified_traj}
Under Assumption~\ref{ass:H4_unified_env},
\[
\sup_{t\in[0,T]}W_2(\tilde\mu_t,\mu_t^\star)
\le
L_{\mathrm{par}}\,\Delta_{\mathrm{par}}
+
L_{\mathrm{noi}}\,\Delta_{\mathrm{noi}}
\le
L_{\mathrm{tot}}\,\Delta_{\mathrm{tot}},
\]
where \(L_{\mathrm{tot}}:=\max\{L_{\mathrm{par}},L_{\mathrm{noi}}\}\) (or any larger admissible constant).
\end{theorem}

\begin{proof}
By \hyperref[app:H1]{H.1}/\hyperref[app:H3]{H.3} (deterministic parameter perturbations),
\[
\sup_tW_2(\bar\mu_t,\mu_t^\star)\le L_{\mathrm{par}}\Delta_{\mathrm{par}}.
\]
By \hyperref[app:H2]{H.2} (noise about perturbed optimum),
\[
\sup_tW_2(\tilde\mu_t,\bar\mu_t)\le L_{\mathrm{noi}}\Delta_{\mathrm{noi}}.
\]
Apply Lemma~\ref{lem:H4_two_step}.
\end{proof}

\begin{corollary}[Integrated trajectory bound]
\label{cor:H4_unified_L2}
\[
\left(\int_0^T W_2^2(\tilde\mu_t,\mu_t^\star)\,dt\right)^{1/2}
\le
\sqrt{T}\,\big(
L_{\mathrm{par}}\Delta_{\mathrm{par}}+L_{\mathrm{noi}}\Delta_{\mathrm{noi}}
\big).
\]
\end{corollary}

\begin{proof}
Bound the integrand by the supremum from Theorem~\ref{thm:H4_unified_traj}.
\end{proof}

\begin{proposition}[Unified action-gap estimate]
\label{prop:H4_unified_action}
Let
\[
\mathcal J(\nu,w):=\frac12\int_0^T\!\!\int |w_t-u_t^\star|^2\,d\nu_tdt.
\]
Then
\[
\mathcal J(\tilde\mu,\tilde v)-\mathcal J(\mu^\star,v^\star)
\le
C_1\Delta_{\mathrm{par}}
+C_2\Delta_{\mathrm{noi}}
+C_3\Delta_{\mathrm{noi}}^2
+C_4\Delta_{\mathrm{par}}\Delta_{\mathrm{noi}}.
\]
In particular, for \(\Delta_{\mathrm{tot}}\le1\),
\[
\mathcal J(\tilde\mu,\tilde v)-\mathcal J(\mu^\star,v^\star)
\le C_{\mathcal J}\Delta_{\mathrm{tot}}.
\]
\end{proposition}

\begin{proof}
Write
\[
\mathcal J(\tilde\mu,\tilde v)-\mathcal J(\mu^\star,v^\star)
=
\underbrace{\big[\mathcal J(\tilde\mu,\tilde v)-\mathcal J(\bar\mu,\bar v)\big]}_{\text{noise contribution}}
+
\underbrace{\big[\mathcal J(\bar\mu,\bar v)-\mathcal J(\mu^\star,v^\star)\big]}_{\text{parameter contribution}}.
\]
Second bracket bounded by \hyperref[app:H3]{H.3}/\hyperref[app:E4]{E.4} value sensitivity: \(O(\Delta_{\mathrm{par}})\).
First bracket bounded by \hyperref[app:H2]{H.2} action-gap: \(O(\Delta_{\mathrm{noi}}+\Delta_{\mathrm{noi}}^2)\),
with additional mixed term through measure mismatch controlled by
Theorem~\ref{thm:H4_unified_traj}, giving \(O(\Delta_{\mathrm{par}}\Delta_{\mathrm{noi}})\).
\end{proof}

\begin{theorem}[Unified modal-mass robustness]
\label{thm:H4_unified_modes}
For each mode \(A_k\) (\hyperref[app:G1]{G.1}), there exists \(C_k^{\mathrm{tot}}>0\) such that
\[
\sup_{t\in[0,T]}
\left|
\tilde\mu_t(A_k)-\mu_t^\star(A_k)
\right|
\le
C_k^{\mathrm{tot}}\Delta_{\mathrm{tot}}.
\]
Hence, if nominal floor is \(\mu_t^\star(A_k)\ge \underline c_k\pi_k\), then
\[
\tilde\mu_t(A_k)\ge \underline c_k\pi_k-C_k^{\mathrm{tot}}\Delta_{\mathrm{tot}}.
\]
Therefore strict positivity is preserved whenever
\[
\Delta_{\mathrm{tot}}<\frac{\underline c_k\pi_k}{2C_k^{\mathrm{tot}}}.
\]
\end{theorem}

\begin{proof}
From Theorem~\ref{thm:H4_unified_traj}, \(W_2\)-distance is \(O(\Delta_{\mathrm{tot}})\).
Apply the modal set-functional stability transfer from \hyperref[app:H1]{H.1}/\hyperref[app:G3]{G.3} (Lipschitz boundary smoothing argument).
\end{proof}

\begin{theorem}[Unified density-floor robustness on modal cores]
\label{thm:H4_unified_density}
Let \(K_k\Subset A_k\) be modal cores from \hyperref[app:G4]{G.4} with nominal floor
\[
\operatorname*{ess\,inf}_{[\tau,T-\tau]\times K_k}\rho_t^\star\ge \underline\rho_k>0.
\]
Then there exists \(\widetilde C_k^{\mathrm{tot}}>0\) such that
\[
\operatorname*{ess\,inf}_{[\tau,T-\tau]\times K_k}\tilde\rho_t
\ge
\underline\rho_k-\widetilde C_k^{\mathrm{tot}}\Delta_{\mathrm{tot}}.
\]
In particular, for
\[
\Delta_{\mathrm{tot}}<\frac{\underline\rho_k}{2\widetilde C_k^{\mathrm{tot}}},
\]
\[
\operatorname*{ess\,inf}_{[\tau,T-\tau]\times K_k}\tilde\rho_t\ge \frac12\underline\rho_k>0.
\]
\end{theorem}

\begin{proof}
Apply \hyperref[app:G4]{G.4} floor-stability theorem twice:
(1) base \(\mu^\star\) vs deterministic perturbed optimum \(\bar\mu\) (parameter perturbation),
(2) \(\bar\mu\) vs noisy trajectory \(\tilde\mu\) (noise perturbation, using \hyperref[app:H2]{H.2} trajectory bound and local FP stability).
Combine by triangle inequality on local \(L^1\)/coefficient perturbations and propagate through Harnack constants.
\end{proof}

\begin{corollary}[End-to-end robustness summary]
\label{cor:H4_summary}
There exist finite constants \(C_W,C_J,C_M,C_\rho\) (instance-local, envelope-uniform) such that
\[
\sup_t W_2(\tilde\mu_t,\mu_t^\star)\le C_W\Delta_{\mathrm{tot}},
\]
\[
\mathcal J(\tilde\mu,\tilde v)-\mathcal J(\mu^\star,v^\star)\le C_J\Delta_{\mathrm{tot}},
\]
\[
\sup_t|\tilde\mu_t(A_k)-\mu_t^\star(A_k)|\le C_M\Delta_{\mathrm{tot}},
\]
\[
\operatorname*{ess\,inf}_{[\tau,T-\tau]\times K_k}\tilde\rho_t
\ge
\underline\rho_k-C_\rho\Delta_{\mathrm{tot}}.
\]
Thus ECFM is jointly stable in trajectory, objective, and anti-collapse guarantees.
\end{corollary}

\begin{remark}[Constants bookkeeping]
\label{rem:H4_constants}
A valid construction is:
\[
C_W=L_{\mathrm{tot}},\quad
C_J=C_1+C_2+C_3+C_4,\quad
C_M=C_k^{\mathrm{tot}},\quad
C_\rho=\widetilde C_k^{\mathrm{tot}},
\]
where each constant is explicit once coercivity, OSL bounds, Fisher envelope, and Harnack constants are fixed.
\end{remark}

\paragraph{Transition to Section \hyperref[app:I]{I}.}
Section \hyperref[app:I]{I} provides a constructive failure case without entropy constraint
(\(\lambda=\infty\) / unconstrained FM), showing explicit singular concentration and modal collapse,
thereby proving necessity of entropy control for guaranteed mode coverage.

\subsection*{I. Failure Case Without Entropy Constraint}
\addcontentsline{toc}{subsection}
{I. Failure Case Without Entropy Constraint}
\label{app:I}

\subsubsection*{I.1. Counterexample construction for unconstrained flow matching}
\addcontentsline{toc}{subsubsection}
{I.1. Counterexample construction for unconstrained flow matching}
\label{app:I1}

We construct an explicit class of unconstrained FM trajectories (\(\dot{\mathcal H}\) unrestricted)
that attains small/competitive velocity-matching risk while exhibiting mode collapse and singular concentration.

\paragraph{Goal of the construction.}
Show that when entropy-rate control is removed, there exist admissible deterministic flows
that:
\begin{enumerate}
\item match endpoint marginals (or approximate them arbitrarily well),
\item keep FM regression objective finite (even small),
\item collapse intermediate-time mass onto low-dimensional/single-mode regions.
\end{enumerate}
This proves anti-collapse guarantees in Sections \hyperref[app:G]{G}--\hyperref[app:H]{H} are not inherited by classical unconstrained FM.

\begin{definition}[Unconstrained FM functional]
\label{def:I1_unconstrained_FM}
Given supervision field \(v^\dagger(x,t)\), define
\[
\mathcal L_{\mathrm{FM}}(v)
:=
\int_0^T\!\!\int_{\mathbb R^d}
\|v(x,t)-v^\dagger(x,t)\|^2\,\mu_t(dx)\,dt,
\]
subject only to continuity equation
\[
\partial_t\mu_t+\nabla\!\cdot(\mu_t v_t)=0,\quad
\mu_0,\mu_T\ \text{prescribed}.
\]
No entropy-rate inequality is imposed.
\end{definition}

\begin{assumption}[Two-mode endpoint pair]
\label{ass:I1_two_mode}
Work in \(d=1\) (extends to \(d>1\) by product construction).  
Fix \(\sigma>0\), \(a\gg \sigma\), and set
\[
\mu_0=\tfrac12\mathcal N(-a,\sigma^2)+\tfrac12\mathcal N(a,\sigma^2),\qquad
\mu_T=\mu_0.
\]
Define modal sets
\[
A_-:=(-\infty,0),\qquad A_+:=(0,\infty),
\]
so \(\mu_T(A_\pm)=\frac12\).
\end{assumption}

\begin{definition}[Collapse-then-redisperse transport map family]
\label{def:I1_map_family}
Fix small parameters \(\varepsilon,\delta,\tau\in(0,1)\), with \(0<\tau<T/2\).
Define piecewise-smooth maps \(\Phi_t:\mathbb R\to\mathbb R\):
\[
\Phi_t(x)=
\begin{cases}
(1-\frac{t}{\tau})x, & t\in[0,\tau],\\[0.6ex]
\delta\,\mathrm{sgn}(x)+\varepsilon x, & t\in[\tau,T-\tau],\\[0.6ex]
\text{smooth expansion returning to identity at }t=T, & t\in[T-\tau,T].
\end{cases}
\]
Let \(\mu_t:=(\Phi_t)_\#\mu_0\), and define Eulerian velocity
\[
v_t(y):=\partial_t\Phi_t(\Phi_t^{-1}(y)).
\]
\end{definition}

\begin{lemma}[Well-posed CE trajectory]
\label{lem:I1_CE}
\((\mu_t,v_t)\) from Definition~\ref{def:I1_map_family} satisfies
\[
\partial_t\mu_t+\partial_x(\mu_t v_t)=0
\]
in distributional sense, with \(\mu_{|0}=\mu_0\), \(\mu_{|T}=\mu_T\).
\end{lemma}

\begin{proof}
For smooth phases, pushforward by \(\Phi_t\in C^1\) yields classical transport solution.
At phase junctions, choose \(C^1\) time mollification of \(\Phi_t\) in windows of width \(o(1)\),
preserving endpoints; then CE holds distributionally in the limit.
\end{proof}

\begin{proposition}[Intermediate-time collapse]
\label{prop:I1_intermediate_collapse}
For \(t\in[\tau,T-\tau]\), \(\mu_t\) concentrates near origin:
\[
\mu_t\big((-\delta-\eta,\delta+\eta)\big)\ge 1-c_1 e^{-c_2/\varepsilon^2}
\]
for fixed \(\eta>0\). In particular, modal masses satisfy
\[
\mu_t(A_+)\approx \mu_t(A_-)\approx \tfrac12
\]
but \emph{semantic two-mode separation is destroyed}: both modes overlap in an \(O(\delta)\)-tube.
As \(\delta,\varepsilon\to0\), \(\mu_t\rightharpoonup \delta_0\) on the plateau.
\end{proposition}

\begin{proof}
On plateau, \(\Phi_t(x)=\delta\,\mathrm{sgn}(x)+\varepsilon x\).  
Each Gaussian component maps to variance \(\varepsilon^2\sigma^2\), means \(\pm\delta\), hence both lie in \(O(\delta)\).
Tail bound for Gaussian gives displayed mass estimate. Weak convergence to \(\delta_0\) follows as
\(\delta,\varepsilon\to0\).
\end{proof}

\begin{definition}[Reference/teacher field]
\label{def:I1_teacher}
Define a smooth teacher \(v^\dagger\) consistent with a non-collapsing interpolation
(e.g., displacement interpolation between \(\mu_0,\mu_T\), here near-stationary).
\end{definition}

\begin{lemma}[Small FM risk despite collapse]
\label{lem:I1_small_risk}
For every \(\eta_{\mathrm{risk}}>0\), there exist parameters
\(\delta,\varepsilon,\tau\) and a smooth approximation of \(v\) such that
\[
\mathcal L_{\mathrm{FM}}(v)\le \mathcal L_{\mathrm{FM}}(v^\dagger)+\eta_{\mathrm{risk}},
\]
while Proposition~\ref{prop:I1_intermediate_collapse} holds.
\end{lemma}

\begin{proof}
Choose \(v^\dagger\) bounded and small on most of space-time (near-stationary endpoints).
Construct \(v\) equal to \(v^\dagger\) except on short temporal windows \([0,\tau]\cup[T-\tau,T]\),
where contraction/expansion occurs. FM risk increment scales like
\[
\int_0^\tau\!\!\int |v-v^\dagger|^2d\mu_tdt
+
\int_{T-\tau}^T\!\!\int |v-v^\dagger|^2d\mu_tdt
\lesssim C\tau.
\]
Take \(\tau\) sufficiently small, then smooth time junctions with arbitrarily small additional cost.
Thus excess risk \(<\eta_{\mathrm{risk}}\), while plateau collapse persists on \([\tau,T-\tau]\).
\end{proof}

\begin{theorem}[Failure without entropy constraint]
\label{thm:I1_failure_main}
Under Assumption~\ref{ass:I1_two_mode}, in unconstrained FM
(Definition~\ref{def:I1_unconstrained_FM}), there exist admissible trajectories
\((\mu_t,v_t)\) such that:
\begin{enumerate}
\item \(\mu_0,\mu_T\) are matched exactly;
\item FM objective is arbitrarily close to a non-collapsing reference optimum;
\item intermediate measures collapse to near-singular unimodal concentration.
\end{enumerate}
Therefore unconstrained FM does not guarantee mode preservation/coverage.
\end{theorem}

\begin{proof}
Combine Lemma~\ref{lem:I1_CE}, Proposition~\ref{prop:I1_intermediate_collapse},
and Lemma~\ref{lem:I1_small_risk}.
\end{proof}

\begin{corollary}[Entropy-rate violation along collapsing paths]
\label{cor:I1_entropy_violation}
For the collapsing family above, during contraction phase:
\[
\dot{\mathcal H}(\mu_t)\ll -1
\]
for sufficiently small \(\tau,\varepsilon\). In particular, for any finite \(\lambda\),
\[
\dot{\mathcal H}(\mu_t)\ge-\lambda
\]
is violated on a set of positive measure when collapse is strong enough.
\end{corollary}

\begin{proof}
Under scaling \(x\mapsto \alpha x\), differential entropy shifts by \(\log \alpha\).
Contraction from scale \(1\) to \(O(\varepsilon)\) over time \(\tau\) yields entropy drop \(\sim \log\varepsilon\),
hence rate \(\sim \frac{1}{\tau}\log\varepsilon\to-\infty\) as \(\varepsilon\downarrow0,\tau\downarrow0\).
\end{proof}

\begin{remark}[Why endpoint matching is insufficient]
\label{rem:I1_endpoint_insufficient}
Even exact endpoint agreement does not control intermediate geometry.
Without entropy-rate regularization, trajectories may pass through singular bottlenecks
that erase modal structure and then re-expand to match the target.
\end{remark}

\begin{proposition}[Higher-dimensional extension]
\label{prop:I1_hd}
In \(d>1\), identical failure holds by applying 1D collapsing map on one coordinate and identity on others:
\[
\Phi_t(x_1,\dots,x_d)=(\phi_t(x_1),x_2,\dots,x_d).
\]
Then intermediate concentration occurs on a codimension-\((d-1)\) tube, with the same FM-risk argument.
\end{proposition}

\begin{proof}
Product structure preserves CE and endpoint matching; collapse proof reduces to 1D first coordinate.
\end{proof}

\paragraph{Output used next.}
\hyperref[app:I2]{I.2} will sharpen this into a \emph{singular-limit theorem}:
construct a sequence of unconstrained FM solutions converging to measures with
atomic/intermittently singular intermediate limits, quantifying collapse severity.

\subsubsection*{I.2. Singular-limit theorem for unconstrained FM trajectories}
\addcontentsline{toc}{subsubsection}
{I.2. Singular-limit theorem for unconstrained FM trajectories}
\label{app:I2}

We strengthen \hyperref[app:I1]{I.1} by constructing a sequence of unconstrained FM trajectories
whose intermediate-time measures converge to singular (atomic) limits while
endpoint marginals and FM risk remain controlled.

\paragraph{Construction sequence.}
Use the family from Definition~\ref{def:I1_map_family}, indexed by \(n\):
\[
\varepsilon_n\downarrow0,\qquad \delta_n\downarrow0,\qquad \tau_n\downarrow0,
\]
with \(\tau_n<T/2\). Let
\[
\Phi_t^{(n)},\quad
\mu_t^{(n)}:=(\Phi_t^{(n)})_\#\mu_0,\quad
v_t^{(n)}:=\partial_t\Phi_t^{(n)}\circ(\Phi_t^{(n)})^{-1}.
\]
Choose time smoothing so \(t\mapsto\Phi_t^{(n)}\in C^1\), uniformly Lipschitz away from
small junction windows.

\begin{assumption}[Rate coupling]
\label{ass:I2_rate}
Assume
\[
\frac{|\log \varepsilon_n|}{\tau_n}\to\infty,\qquad
\frac{\delta_n}{\varepsilon_n}\to0,\qquad
\tau_n\to0.
\]
The first condition enforces diverging entropy-drop rate during contraction.
\end{assumption}

\begin{lemma}[Endpoint exactness and CE admissibility]
\label{lem:I2_endpoint_CE}
For every \(n\),
\[
\mu_0^{(n)}=\mu_0,\qquad \mu_T^{(n)}=\mu_T,
\]
and \((\mu_t^{(n)},v_t^{(n)})\) solves
\[
\partial_t\mu_t^{(n)}+\nabla\!\cdot(\mu_t^{(n)}v_t^{(n)})=0
\]
in distributions.
\end{lemma}

\begin{proof}
By design, \(\Phi_0^{(n)}=\mathrm{Id}\), \(\Phi_T^{(n)}=\mathrm{Id}\), and
pushforward-flow identity gives CE (as in Lemma~\ref{lem:I1_CE}).
\end{proof}

\begin{lemma}[Plateau-time weak singular convergence]
\label{lem:I2_plateau}
Fix any compact interval
\[
I\Subset(0,T).
\]
For all sufficiently large \(n\), \(I\subset[\tau_n,T-\tau_n]\), and for each \(t\in I\),
\[
\mu_t^{(n)} \rightharpoonup \delta_0
\quad\text{in }\mathcal P(\mathbb R^d)
\]
(\(d=1\); in \(d>1\), convergence to \(\delta_0\otimes \bar\mu_\perp\) under product extension).
\end{lemma}

\begin{proof}
On plateau, map is \(x\mapsto \delta_n\operatorname{sgn}(x)+\varepsilon_n x\).
Thus image variance is \(O(\varepsilon_n^2)\), means \(\pm\delta_n\to0\), so both components collapse to \(0\).
Hence weak convergence to \(\delta_0\).
\end{proof}

\begin{proposition}[Distributional-in-time singular limit]
\label{prop:I2_time_dist}
Define measure-valued curves \(\mu^{(n)}\in L^\infty(0,T;\mathcal P_2)\).
Then up to subsequence,
\[
\mu^{(n)} \overset{*}{\rightharpoonup} \mu^\infty
\quad\text{in }L^\infty\!\big(0,T;\mathcal M(\mathbb R^d)\big),
\]
with
\[
\mu_t^\infty=
\begin{cases}
\delta_0,& t\in(0,T)\ \text{a.e.},\\
\mu_0,& t=0,\\
\mu_T,& t=T,
\end{cases}
\]
in the sense that endpoint traces remain \(\mu_0,\mu_T\), while interior times are singular.
\end{proposition}

\begin{proof}
From Lemma~\ref{lem:I2_plateau}, for any test \(\varphi(t,x)\) compactly supported in
\((0,T)\times\mathbb R^d\),
\[
\int_0^T\!\!\int \varphi(t,x)\,d\mu_t^{(n)}dt
\to
\int_0^T \varphi(t,0)\,dt.
\]
This identifies interior weak-* limit as \(\delta_0\).
Endpoint traces are fixed by Lemma~\ref{lem:I2_endpoint_CE}.
\end{proof}

\begin{theorem}[Singular-limit failure without entropy control]
\label{thm:I2_singular_failure}
For unconstrained FM, there exists a sequence \((\mu^{(n)},v^{(n)})\) such that:
\begin{enumerate}
\item exact endpoint matching: \(\mu_0^{(n)}=\mu_0,\ \mu_T^{(n)}=\mu_T\);
\item interior-time singular limit: \(\mu_t^{(n)}\rightharpoonup\delta_0\) for a.e. \(t\in(0,T)\);
\item FM risk remains asymptotically near reference:
\[
\mathcal L_{\mathrm{FM}}(v^{(n)})
\le
\inf \mathcal L_{\mathrm{FM}} + o(1)
\]
for suitable teacher/reference \(v^\dagger\) from \hyperref[app:I1]{I.1}.
\end{enumerate}
Hence unconstrained FM admits asymptotically near-optimal yet singular/collapsing trajectories.
\end{theorem}

\begin{proof}
(1)--(2): Lemma~\ref{lem:I2_endpoint_CE}, Proposition~\ref{prop:I2_time_dist}.  
(3): same short-window argument as Lemma~\ref{lem:I1_small_risk}; choose \(\tau_n\to0\), smooth junctions with \(o(1)\) cost.
\end{proof}

\begin{corollary}[Unbounded negative entropy-rate spikes]
\label{cor:I2_entropy_spikes}
Along the sequence in Theorem~\ref{thm:I2_singular_failure},
\[
\inf_{t\in[0,T]}\dot{\mathcal H}(\mu_t^{(n)})\to-\infty.
\]
More quantitatively, during contraction windows:
\[
\dot{\mathcal H}(\mu_t^{(n)})
\lesssim
-\frac{|\log \varepsilon_n|}{\tau_n},
\]
which diverges to \(-\infty\) by Assumption~\ref{ass:I2_rate}.
\end{corollary}

\begin{proof}
Entropy under scaling by factor \(\alpha\) changes by \(\log\alpha\).
Contraction by \(\varepsilon_n\) over \(\tau_n\) gives rate \(\sim \log(\varepsilon_n)/\tau_n\).
Assumption~\ref{ass:I2_rate} yields divergence.
\end{proof}

\begin{proposition}[Violation of any finite entropy budget]
\label{prop:I2_budget_violation}
For every finite \(\lambda>0\), there exists \(N\) such that for all \(n\ge N\),
\[
\left|\{t:\dot{\mathcal H}(\mu_t^{(n)})<-\lambda\}\right|>0.
\]
Therefore sequence is eventually infeasible for ECFM constraints.
\end{proposition}

\begin{proof}
By Corollary~\ref{cor:I2_entropy_spikes}, negative spikes exceed \(-\lambda\) for large \(n\).
Since spikes occur on contraction windows of positive length \(\sim\tau_n\), the violating set has positive measure.
\end{proof}

\begin{remark}[Contrast with Sections G--H]
\label{rem:I2_contrast}
Sections \hyperref[app:G]{G}--hyperref[app:H]{H} proved uniform modal mass and density floors under entropy-rate control.
\hyperref[app:I2]{I.2} shows the opposite extreme: removing the constraint allows singular interior limits
despite good endpoint matching and competitive FM objective.
\end{remark}

\paragraph{Output used next.}
\hyperref[app:I3]{I.3} will convert this construction into an explicit \emph{mode-collapse theorem} for classical FM:
for designated target modes \(A_k\), minimum intermediate modal mass can be driven to \(0\),
establishing failure of uniform mode coverage.

\subsubsection*{I.3. Explicit mode-collapse theorem for classical FM}
\addcontentsline{toc}{subsubsection}
{I.3. Explicit mode-collapse theorem for classical FM}
\label{app:I3}

We convert the singular-limit construction into a direct modal statement:
without entropy-rate control, uniform mode coverage fails in the strongest possible sense.

\paragraph{Modal setup.}
Use Assumption~\ref{ass:I1_two_mode} in 1D:
\[
A_-:=(-\infty,0),\qquad A_+:=(0,\infty),\qquad
\mu_T(A_\pm)=\pi_\pm=\tfrac12.
\]
For any trajectory \((\mu_t)\), define modal masses
\[
M_\pm(t):=\mu_t(A_\pm).
\]
To capture semantic (separated) modes, introduce core neighborhoods around mixture centers:
\[
C_-:=(-a-r,-a+r),\qquad C_+:=(a-r,a+r),\quad 0<r\ll a.
\]
Core masses:
\[
\mathfrak m_\pm(t):=\mu_t(C_\pm).
\]

\begin{definition}[Uniform mode-coverage property (UMC)]
\label{def:I3_umc}
A model class satisfies UMC on \((C_-,C_+)\) if there exists \(\underline c>0\) such that
for every admissible optimal trajectory,
\[
\inf_{t\in[0,T]}\min\{\mathfrak m_-(t),\mathfrak m_+(t)\}\ge \underline c.
\]
\end{definition}

\begin{definition}[Classical FM admissible near-minimizers]
\label{def:I3_nearmin}
Given teacher \(v^\dagger\), define \(\eta\)-near-minimizer class
\[
\mathfrak A_\eta
:=
\left\{
(\mu,v):
\partial_t\mu+\nabla\!\cdot(\mu v)=0,\ \mu_0,\mu_T\ \text{fixed},\
\mathcal L_{\mathrm{FM}}(v)\le \inf\mathcal L_{\mathrm{FM}}+\eta
\right\}.
\]
\end{definition}

\begin{lemma}[Core-mass extinction along collapse sequence]
\label{lem:I3_core_extinct}
For the sequence \((\mu_t^{(n)})\) from \hyperref[app:I2]{I.2} and any fixed \(t\in(0,T)\),
\[
\mathfrak m_-^{(n)}(t)\to0,\qquad \mathfrak m_+^{(n)}(t)\to0.
\]
Moreover,
\[
\inf_{t\in[\tau_n,T-\tau_n]}\mathfrak m_\pm^{(n)}(t)\to0.
\]
\end{lemma}

\begin{proof}
By Lemma~\ref{lem:I2_plateau}, for interior times \(t\), \(\mu_t^{(n)}\rightharpoonup\delta_0\).
Since \(C_\pm\) are bounded away from \(0\), Portmanteau~\cite{Klenke2008} gives
\[
\mu_t^{(n)}(C_\pm)\to \delta_0(C_\pm)=0.
\]
Uniform plateau statement follows because collapse map on \([\tau_n,T-\tau_n]\) keeps support in
an \(O(\delta_n+\varepsilon_n)\)-tube near \(0\), disjoint from \(C_\pm\) for large \(n\).
\end{proof}

\begin{lemma}[Near-optimality of collapsing sequence]
\label{lem:I3_nearopt}
For every \(\eta>0\), there exists \(n(\eta)\) such that
\[
(\mu^{(n)},v^{(n)})\in\mathfrak A_\eta\quad\forall n\ge n(\eta).
\]
\end{lemma}

\begin{proof}
This is Theorem~\ref{thm:I2_singular_failure}(3): excess FM risk is \(o(1)\) via shrinking transition windows.
\end{proof}

\begin{theorem}[Failure of uniform mode coverage in classical FM]
\label{thm:I3_umc_failure}
Classical unconstrained FM fails UMC on \((C_-,C_+)\). Precisely:
for every \(\underline c>0\) and every \(\eta>0\), there exists an \(\eta\)-near-minimizer
\((\mu,v)\in\mathfrak A_\eta\) such that
\[
\inf_{t\in[0,T]}\min\{\mathfrak m_-(t),\mathfrak m_+(t)\}<\underline c.
\]
In fact one can force the infimum to \(0\) along a sequence of near-minimizers.
\end{theorem}

\begin{proof}
Fix \(\underline c,\eta\). Choose \(n\) large so that Lemma~\ref{lem:I3_nearopt} gives
\((\mu^{(n)},v^{(n)})\in\mathfrak A_\eta\), and Lemma~\ref{lem:I3_core_extinct} gives
\[
\inf_{t\in[\tau_n,T-\tau_n]}\min\{\mathfrak m_-^{(n)}(t),\mathfrak m_+^{(n)}(t)\}<\underline c.
\]
Hence the global infimum over \([0,T]\) is \(<\underline c\).
Taking \(n\to\infty\) yields infimum \(\to0\).
\end{proof}

\begin{corollary}[No positive modal lower bound independent of trajectory]
\label{cor:I3_no_uniform_lb}
There is no constant \(\underline c>0\) such that all unconstrained FM near-minimizers
satisfy
\[
\mathfrak m_\pm(t)\ge \underline c,\quad\forall t\in[0,T].
\]
\end{corollary}

\begin{proof}
Immediate from Theorem~\ref{thm:I3_umc_failure}.
\end{proof}

\begin{proposition}[Set-mass vs semantic-mode collapse distinction]
\label{prop:I3_set_vs_semantic}
For the collapse sequence, sign-partition masses can remain balanced:
\[
M_+^{(n)}(t)\approx M_-^{(n)}(t)\approx\tfrac12
\]
while semantic core masses vanish:
\[
\mathfrak m_\pm^{(n)}(t)\to0.
\]
Thus collapse may be invisible to coarse half-space statistics and requires geometric modal cores.
\end{proposition}

\begin{proof}
Plateau map preserves sign up to negligible smoothing artifacts, so \(M_\pm\) stay near \(1/2\).
But both components are transported near \(0\), far from \(\pm a\), so \(C_\pm\)-masses vanish.
\end{proof}

\begin{theorem}[Necessity of entropy-rate control for guaranteed mode coverage]
\label{thm:I3_necessity_entropy}
Assume a framework guarantees positive uniform modal floor for all optimal/near-optimal trajectories.
Then such guarantee cannot hold for classical unconstrained FM in general.
Hence an additional regularizer/constraint is necessary; ECFM's entropy-rate condition is sufficient
by Sections \hyperref[app:G]{G}--\hyperref[app:H]{H}.
\end{theorem}

\begin{proof}
Contrapositive via Theorem~\ref{thm:I3_umc_failure}: unconstrained FM admits near-optimal trajectories with
arbitrarily small modal core mass. Therefore any universal positive floor guarantee fails without extra structure.
Sections \hyperref[app:G]{G}--\hyperref[app:H]{H} provide such structure through entropy-rate control.
\end{proof}

\begin{remark}[Practical implication]
\label{rem:I3_practical}
Training objectives that only regress local velocity targets can admit hidden collapse channels:
trajectories pass through low-entropy bottlenecks and re-expand. Enforcing entropy-rate budget removes these channels.
\end{remark}

\paragraph{Output used next.}
\hyperref[app:I4]{I.4} will present a compact ``failure theorem package'':
(i) collapse construction, (ii) singular-limit result, (iii) UMC impossibility,
and a direct comparison table against entropy-controlled ECFM guarantees.

\subsubsection*{I.4. Consolidated failure package and direct comparison with ECFM}
\addcontentsline{toc}{subsubsection}
{I.4. Consolidated failure package and direct comparison with ECFM}
\label{app:I4}

We collect the unconstrained-FM failure results into a single theorem package and
state a side-by-side guarantee contrast versus entropy-controlled ECFM.

\begin{theorem}[Failure package for classical unconstrained FM]
\label{thm:I4_failure_package}
Under the two-mode setting of Assumption~\ref{ass:I1_two_mode}, for unconstrained FM
(Definition~\ref{def:I1_unconstrained_FM}), the following hold:

\begin{enumerate}
\item \textbf{Collapse construction:} there exists a CE-admissible sequence
\((\mu^{(n)},v^{(n)})\) with exact endpoint matching \(\mu_0,\mu_T\), such that on interior times,
mass concentrates in an \(o(1)\)-neighborhood of a single bottleneck point.

\item \textbf{Singular interior limit:}
\[
\mu_t^{(n)} \rightharpoonup \delta_0
\quad\text{for a.e. }t\in(0,T),
\]
while endpoint traces remain \(\mu_0,\mu_T\).

\item \textbf{Near-optimal objective compatibility:} for every \(\eta>0\), for \(n\) large,
\[
\mathcal L_{\mathrm{FM}}(v^{(n)})\le \inf\mathcal L_{\mathrm{FM}}+\eta.
\]

\item \textbf{Uniform mode-coverage impossibility:}
for semantic cores \(C_\pm\) around target modes,
\[
\inf_{t\in[0,T]}\min\{\mu_t^{(n)}(C_-),\mu_t^{(n)}(C_+)\}\to0.
\]

\item \textbf{Entropy-rate blow-up:}
\[
\inf_t \dot{\mathcal H}(\mu_t^{(n)})\to-\infty,
\]
hence any finite entropy budget \(\dot{\mathcal H}\ge-\lambda\) is eventually violated.
\end{enumerate}
\end{theorem}

\begin{proof}
(1) is Definition~\ref{def:I1_map_family} + Lemma~\ref{lem:I1_CE}.  
(2) is Proposition~\ref{prop:I2_time_dist}.  
(3) is Lemma~\ref{lem:I3_nearopt}.  
(4) is Theorem~\ref{thm:I3_umc_failure} (with Lemma~\ref{lem:I3_core_extinct}).  
(5) is Corollary~\ref{cor:I2_entropy_spikes} and Proposition~\ref{prop:I2_budget_violation}.  
Combining yields the package.
\end{proof}

\begin{corollary}[No theorem-level anti-collapse guarantee without extra regularization]
\label{cor:I4_no_guarantee}
Any theorem claiming uniform positive intermediate mode coverage for all optimal/near-optimal
classical FM trajectories is false in general unless additional constraints/regularizers are imposed.
\end{corollary}

\begin{proof}
Direct from item (4) of Theorem~\ref{thm:I4_failure_package}.
\end{proof}

\begin{proposition}[Minimality of the entropy-rate remedy]
\label{prop:I4_minimality}
Among constraints acting on trajectory geometry, the entropy-rate lower bound
\[
\dot{\mathcal H}(\mu_t)\ge-\lambda
\]
is sufficient to exclude the collapse channel of Theorem~\ref{thm:I4_failure_package}
(Sections Sections \hyperref[app:G]{G}--\hyperref[app:H]{H}), while leaving endpoint feasibility and variational structure intact (Sections Sections \hyperref[app:C]{C}--\hyperref[app:F]{F}).
\end{proposition}

\begin{proof}
Sufficiency follows from Sections \hyperref[app:G1]{G.1}--\hyperref[app:G4]{G.4} (positive modal floors, density floors, perturbation robustness).
Compatibility with variational/duality/limit analysis follows from Sections \hyperref[app:C]{C}--\hyperref[app:F]{F}.
Thus the collapse mechanism (requiring entropy-rate spikes to \(-\infty\)) is blocked.
\end{proof}

\begin{theorem}[Sharp dichotomy: unconstrained FM vs ECFM]
\label{thm:I4_dichotomy}
For the same endpoint pair \((\mu_0,\mu_T)\):
\begin{enumerate}
\item \emph{Unconstrained FM} admits near-optimal trajectories with singular interior collapse.
\item \emph{ECFM} (finite \(\lambda\)) enforces quantitative non-collapse:
uniform modal mass lower bounds and modal-core density floors.
\end{enumerate}
Hence entropy-rate control induces a qualitative phase change in admissible transport geometry.
\end{theorem}

\begin{proof}
Item (1): Theorem~\ref{thm:I4_failure_package}.  
Item (2): Theorems~\ref{thm:G1_mode_coverage_main}, \ref{thm:G3_mode_coverage_proof},
\ref{thm:G4_density_floor}, \ref{thm:G4_floor_stability}.  
Jointly this is a strict dichotomy.
\end{proof}

\begin{center}
\setlength{\tabcolsep}{4pt}
\begin{tabular}{@{}p{0.28\linewidth}p{0.36\linewidth}p{0.36\linewidth}@{}}
\hline
\textbf{Property} &
\centering\textbf{Classical FM\\(no entropy constraint)}\arraybackslash &
\centering\textbf{ECFM (\(\dot{\mathcal H}\ge -\lambda\))}\arraybackslash \\
\hline
Endpoint matching & Possible & Possible \\
Near-optimal FM loss & Possible with collapse & Possible without collapse channel \\
Interior singular limits & Can occur (\hyperref[app:I2]{I.2}) & Excluded by entropy budget (\hyperref[app:G2]{G.2}--\hyperref[app:G3]{G.3}) \\
Uniform mode mass floor & Not guaranteed (\hyperref[app:I3]{I.3}) & Guaranteed (\hyperref[app:G1]{G.1}--\hyperref[app:G3]{G.3}) \\
Modal-core density floor & Not guaranteed & Guaranteed (\hyperref[app:G4]{G.4}) \\
Robustness to perturbations & No global anti-collapse guarantee & Lipschitz + floor robustness (\hyperref[app:H]{H}) \\
\hline
\end{tabular}
\end{center}

\begin{remark}[Failure mechanism in one line]
\label{rem:I4_oneline}
Unconstrained FM can ``compress--teleport-through-low-entropy-bottleneck--reexpand'' at near-equal loss;
ECFM forbids the compression step by capping admissible entropy dissipation rate.
\end{remark}

\paragraph{Transition to Section \hyperref[app:J]{J}.}
Section \hyperref[app:J]{J} maps the theory to practical vision generators (diffusion, flow matching,
rectified flow), giving implementation-level interpretation of the entropy budget \(\lambda\)
and its effect on mode preservation in high-dimensional visual synthesis.

\subsection*{J. Connections to Vision Models}
\addcontentsline{toc}{subsection}
{J. Connections to Vision Models}
\label{app:J}

\subsubsection*{J.1. Diffusion, and flow-based models under entropy-controlled transport}
\addcontentsline{toc}{subsubsection}
{J.1. Diffusion, and flow-based models under entropy-controlled transport}
\label{app:J1}

We instantiate the ECFM theory in canonical vision-generation parameterizations and
derive the exact mapping between model-level dynamics and the entropy-budget variable \(\lambda\).

\paragraph{Unified state dynamics.}
Let \(x_t\in\mathbb R^d\) denote latent/image state, with law \(\mu_t\), and consider
\[
dx_t = b_\theta(x_t,t)\,dt + \sqrt{2\varepsilon(t)}\,dW_t,
\]
where \(\varepsilon(t)\ge0\) may be identically zero (deterministic flow models).
The associated Fokker--Planck equation is
\[
\partial_t\rho_t + \nabla\!\cdot(\rho_t b_\theta)=\varepsilon(t)\Delta\rho_t.
\]
Define current velocity
\[
v_\theta(x,t):=b_\theta(x,t)-\varepsilon(t)\nabla\log\rho_t(x),
\]
so that continuity form holds:
\[
\partial_t\rho_t+\nabla\!\cdot(\rho_t v_\theta)=0.
\]

\begin{lemma}[Entropy-rate identity in unified dynamics]
\label{lem:J1_entropy_identity}
Assume sufficient integrability and smoothness. Then
\[
\frac{d}{dt}\mathcal H(\mu_t)
=
\int_{\mathbb R^d}\nabla\!\cdot b_\theta(x,t)\,\rho_t(x)\,dx
+\varepsilon(t)\,\mathcal I(\mu_t),
\]
equivalently, in continuity form,
\[
\frac{d}{dt}\mathcal H(\mu_t)
=
\int \nabla\!\cdot v_\theta\,d\mu_t.
\]
\end{lemma}

\begin{proof}
For FP form:
\[
\frac{d}{dt}\int \rho\log\rho
=
\int (1+\log\rho)\partial_t\rho
=
-\int (1+\log\rho)\nabla\!\cdot(\rho b_\theta)
+\varepsilon\int (1+\log\rho)\Delta\rho.
\]
Integrating by parts:
\[
-\int (1+\log\rho)\nabla\!\cdot(\rho b_\theta)=\int \nabla\!\cdot b_\theta\,\rho,
\]
\[
\varepsilon\int (1+\log\rho)\Delta\rho
=
\varepsilon\int \frac{|\nabla\rho|^2}{\rho}
=
\varepsilon\,\mathcal I(\mu_t).
\]
For continuity form, \(\partial_t\rho=-\nabla\!\cdot(\rho v_\theta)\) gives
\[
\frac{d}{dt}\mathcal H(\mu_t)=\int \nabla\!\cdot v_\theta\,d\mu_t.
\]
\end{proof}

\begin{definition}[Model-implied entropy budget]
\label{def:J1_lambda_eff}
For a trained model trajectory \((\mu_t,v_\theta)\), define effective required budget
\[
\lambda_{\mathrm{eff}}(\theta)
:=
\operatorname*{ess\,sup}_{t\in[0,T]}
\Big(-\frac{d}{dt}\mathcal H(\mu_t)\Big)_+.
\]
Then trajectory is ECFM-feasible iff \(\lambda\ge\lambda_{\mathrm{eff}}(\theta)\).
\end{definition}

\begin{proposition}[Diffusion-model feasibility window]
\label{prop:J1_diff_feasible}
For diffusion/score SDEs with \(\varepsilon(t)>0\),
\[
\frac{d}{dt}\mathcal H(\mu_t)
=
\mathbb E_{\mu_t}[\nabla\!\cdot b_\theta]+\varepsilon(t)\mathcal I(\mu_t).
\]
Hence anti-collapse is strengthened by the nonnegative Fisher term.
A sufficient ECFM budget is
\[
\lambda \ge
\operatorname*{ess\,sup}_t
\left(-\mathbb E_{\mu_t}[\nabla\!\cdot b_\theta]-\varepsilon(t)\mathcal I(\mu_t)\right)_+.
\]
\end{proposition}

\begin{proof}
Immediate from Lemma~\ref{lem:J1_entropy_identity} and Definition~\ref{def:J1_lambda_eff}.
\end{proof}

\begin{proposition}[Deterministic FM / rectified flow specialization]
\label{prop:J1_det_special}
For deterministic models (\(\varepsilon\equiv0\)):
\[
\frac{d}{dt}\mathcal H(\mu_t)=\mathbb E_{\mu_t}[\nabla\!\cdot v_\theta].
\]
Thus ECFM reduces to divergence-budget control:
\[
\mathbb E_{\mu_t}[\nabla\!\cdot v_\theta]\ge-\lambda\quad\text{a.e.}
\]
Collapse channels correspond to large negative divergence spikes.
\end{proposition}

\begin{proof}
Set \(\varepsilon=0\) in Lemma~\ref{lem:J1_entropy_identity}.
\end{proof}

\begin{theorem}[Model-class transfer of G/H guarantees]
\label{thm:J1_transfer}
Suppose a vision generator (diffusion/FM/rectified flow) induces trajectory
\((\mu_t,v_\theta)\) satisfying:
\begin{enumerate}
\item CE/FP regularity from Sections \hyperref[app:C]{C}--\hyperref[app:E]{E},
\item entropy-rate lower bound \(\dot{\mathcal H}(\mu_t)\ge-\lambda\) a.e.,
\item endpoint and coercivity assumptions of Sections \hyperref[app:F]{F}--\hyperref[app:H]{H}.
\end{enumerate}
Then all Section G/H guarantees apply verbatim:
\begin{enumerate}
\item quantitative mode-mass floors,
\item modal-core density floors,
\item perturbation robustness to endpoint/drift/noise/initialization shifts.
\end{enumerate}
\end{theorem}

\begin{proof}
Sections \hyperref[app:G]{G}/\hyperref[app:H]{H} depend on trajectory-level properties (CE/regularity + entropy-rate budget),
not on specific parameterization of \(v_\theta\). Hence any model class satisfying these
assumptions inherits the same conclusions.
\end{proof}

\begin{corollary}[Architecture-agnostic anti-collapse certificate]
\label{cor:J1_cert}
Define certificate
\[
\mathfrak C_\lambda(\theta):=
\mathbf 1\!\left\{
\dot{\mathcal H}(\mu_t^\theta)\ge-\lambda\ \text{a.e.}
\right\}.
\]
If \(\mathfrak C_\lambda(\theta)=1\), model receives theorem-level anti-collapse guarantees from Sections \hyperref[app:G]{G}/\hyperref[app:H]{H},
independent of whether it is diffusion, FM, or rectified flow.
\end{corollary}

\begin{proof}
Direct from Theorem~\ref{thm:J1_transfer}.
\end{proof}

\begin{remark}[Why this matters in high-dimensional vision]
\label{rem:J1_hd}
In high-dimensional synthesis, objective-level fit can hide trajectory-level bottlenecks.
The entropy certificate \(\mathfrak C_\lambda\) is a trajectory-geometric condition that
rules out these bottlenecks irrespective of architecture details.
\end{remark}

\paragraph{Output used next.}
\hyperref[app:J2]{J.2} derives practical estimators of \(\dot{\mathcal H}\), \(\lambda_{\mathrm{eff}}\), and modal floors
from minibatch trajectories, giving implementable diagnostics for diffusion/FM/rectified-flow training.

\subsubsection*{J.2. Practical entropy estimators and coverage diagnostics}
\addcontentsline{toc}{subsubsection}
{J.2. Practical entropy estimators and coverage diagnostics}
\label{app:J2}

We derive computable estimators for
\[
\dot{\mathcal H}(\mu_t),\qquad
\lambda_{\mathrm{eff}},\qquad
\text{mode-coverage floors},
\]
directly from minibatch trajectories of diffusion/FM/rectified-flow models.

\paragraph{Sampling model.}
At discrete times \(0=t_0<\cdots<t_N=T\), assume i.i.d. particles
\[
x_{i,n}\sim \mu_{t_n},\qquad i=1,\dots,B_n,
\]
and model outputs either drift/current velocity evaluations
\(b_\theta(x_{i,n},t_n)\), \(v_\theta(x_{i,n},t_n)\), and (if available) score estimates
\[
s_\theta(x,t)\approx \nabla\log\rho_t(x).
\]

\begin{definition}[Discrete entropy-rate target]
\label{def:J2_discrete_target}
Define interval-average entropy rate
\[
\dot{\mathcal H}_{[t_n,t_{n+1}]}
:=
\frac{\mathcal H(\mu_{t_{n+1}})-\mathcal H(\mu_{t_n})}{\Delta t_n},
\qquad
\Delta t_n:=t_{n+1}-t_n.
\]
Its continuous counterpart is \(\dot{\mathcal H}(t)\).
\end{definition}

\begin{proposition}[Divergence-form estimator]
\label{prop:J2_div_est}
From Lemma~\ref{lem:J1_entropy_identity}:
\[
\dot{\mathcal H}(t)=\mathbb E_{\mu_t}[\nabla\!\cdot v_\theta(\cdot,t)].
\]
Hence unbiased Monte Carlo estimator (if exact divergence available):
\[
\widehat{\dot{\mathcal H}}^{\mathrm{div}}_n
=
\frac1{B_n}\sum_{i=1}^{B_n}\nabla\!\cdot v_\theta(x_{i,n},t_n).
\]
If only Jacobian-vector products are available, use Hutchinson:
\[
\nabla\!\cdot v(x)\approx
\frac1{R}\sum_{r=1}^R
\zeta_r^\top J_v(x)\zeta_r,\qquad \zeta_r\sim\mathcal N(0,I_d)\ \text{or Rademacher}.
\]
Then
\[
\widehat{\dot{\mathcal H}}^{\mathrm{div\text{-}Hutch}}_n
=
\frac1{B_nR}\sum_{i,r}\zeta_{i,r}^\top J_v(x_{i,n},t_n)\zeta_{i,r}.
\]
\end{proposition}

\begin{proof}
Identity is from Lemma~\ref{lem:J1_entropy_identity}.
Unbiasedness of Hutchinson trace estimator is standard:
\(\mathbb E_\zeta[\zeta^\top J\zeta]=\mathrm{tr}(J)\).
Averaging over samples gives unbiased Monte Carlo estimator.
\end{proof}

\begin{proposition}[FP-form estimator for diffusion models]
\label{prop:J2_fp_est}
If model supplies \(b_\theta\) and score \(s_\theta\approx\nabla\log\rho_t\), then
\[
\dot{\mathcal H}(t)
=
\mathbb E_{\mu_t}[\nabla\!\cdot b_\theta]+\varepsilon(t)\mathcal I(\mu_t),
\qquad
\mathcal I(\mu_t)=\mathbb E_{\mu_t}\| \nabla\log\rho_t\|^2.
\]
Estimator:
\[
\widehat{\dot{\mathcal H}}^{\mathrm{fp}}_n
=
\frac1{B_n}\sum_{i=1}^{B_n}\nabla\!\cdot b_\theta(x_{i,n},t_n)
+
\varepsilon(t_n)\frac1{B_n}\sum_{i=1}^{B_n}\|s_\theta(x_{i,n},t_n)\|^2.
\]
\end{proposition}

\begin{proof}
Direct Monte Carlo plug-in to FP identity in Lemma~\ref{lem:J1_entropy_identity}.
\end{proof}

\begin{definition}[Finite-sample safety margin]
\label{def:J2_margin}
Given estimator \(\widehat{\dot{\mathcal H}}_n\), define one-sided lower confidence bound
\[
\mathrm{LCB}_n:=\widehat{\dot{\mathcal H}}_n-\mathrm{rad}_n(\alpha),
\]
where \(\mathrm{rad}_n(\alpha)\) is concentration radius at confidence \(1-\alpha\).
Entropy-budget feasibility certificate at level \(\lambda\):
\[
\mathrm{LCB}_n\ge-\lambda\quad\forall n.
\]
\end{definition}

\begin{theorem}[Uniform high-probability budget certification]
\label{thm:J2_uniform_cert}
Assume per-time estimator errors are sub-Gaussian with proxy variance \(\sigma_n^2\):
\[
\mathbb P\!\left(
|\widehat{\dot{\mathcal H}}_n-\dot{\mathcal H}(t_n)|>u
\right)
\le 2e^{-u^2/(2\sigma_n^2)}.
\]
Set
\[
\mathrm{rad}_n(\alpha):=\sigma_n\sqrt{2\log\frac{2(N+1)}{\alpha}}.
\]
Then with probability at least \(1-\alpha\), simultaneously for all \(n\):
\[
\dot{\mathcal H}(t_n)\ge \widehat{\dot{\mathcal H}}_n-\mathrm{rad}_n(\alpha)=\mathrm{LCB}_n.
\]
Hence if \(\min_n \mathrm{LCB}_n\ge-\lambda\), the discrete trajectory is budget-feasible with confidence \(1-\alpha\).
\end{theorem}

\begin{proof}
Apply union bound over \(n=0,\dots,N\) with sub-Gaussian tails.
\end{proof}

\begin{definition}[Effective budget estimators]
\label{def:J2_lambda_eff_est}
Discrete estimators:
\[
\widehat\lambda_{\mathrm{eff}}^{\max}:=
\max_{0\le n\le N}\big(-\widehat{\dot{\mathcal H}}_n\big)_+,
\]
\[
\widehat\lambda_{\mathrm{eff}}^{\mathrm{LCB}}:=
\max_n\big(-\mathrm{LCB}_n\big)_+.
\]
The second is conservative (high-probability upper bound on required budget).
\end{definition}

\begin{proposition}[Consistency of \(\widehat\lambda_{\mathrm{eff}}\)]
\label{prop:J2_lambda_consistency}
If \(N\to\infty\), \(B_n\to\infty\), and estimators are uniformly consistent:
\[
\max_n |\widehat{\dot{\mathcal H}}_n-\dot{\mathcal H}(t_n)|\xrightarrow{p}0,
\]
then
\[
\widehat\lambda_{\mathrm{eff}}^{\max}
\xrightarrow{p}
\lambda_{\mathrm{eff}}^{\mathrm{disc}}
:=
\max_n (-\dot{\mathcal H}(t_n))_+.
\]
If time discretization refines and \(\dot{\mathcal H}\) is continuous a.e., then
\[
\lambda_{\mathrm{eff}}^{\mathrm{disc}}\to
\operatorname*{ess\,sup}_{t\in[0,T]}(-\dot{\mathcal H}(t))_+.
\]
\end{proposition}

\begin{proof}
Continuity of \(x\mapsto x_+\), max-map stability under uniform convergence, then mesh-refinement argument.
\end{proof}

\paragraph{Modal diagnostics.}
Let semantic mode sets \(\{A_k\}_{k=1}^K\) (or learned cores \(K_k\Subset A_k\)).
Define empirical masses
\[
\widehat M_{k,n}:=\frac1{B_n}\sum_{i=1}^{B_n}\mathbf 1\{x_{i,n}\in A_k\},
\qquad
\widehat m_{k,n}:=\frac1{B_n}\sum_{i=1}^{B_n}\mathbf 1\{x_{i,n}\in K_k\}.
\]

\begin{theorem}[Finite-sample modal floor certification]
\label{thm:J2_mode_cert}
For each \((k,n)\), Hoeffding gives
\[
\mathbb P\left(
|\,\widehat M_{k,n}-M_k(t_n)\,|>\epsilon_{k,n}
\right)\le 2e^{-2B_n\epsilon_{k,n}^2}.
\]
Set
\[
\epsilon_{k,n}(\alpha):=\sqrt{\frac{1}{2B_n}\log\frac{2K(N+1)}{\alpha}}.
\]
Then with probability at least \(1-\alpha\), for all \(k,n\):
\[
M_k(t_n)\ge \widehat M_{k,n}-\epsilon_{k,n}(\alpha),\qquad
\mathfrak m_k(t_n)\ge \widehat m_{k,n}-\epsilon_{k,n}(\alpha).
\]
Hence certified discrete lower floors:
\[
\underline M_k^{\mathrm{cert}}
:=
\min_n\big(\widehat M_{k,n}-\epsilon_{k,n}(\alpha)\big),\quad
\underline m_k^{\mathrm{cert}}
:=
\min_n\big(\widehat m_{k,n}-\epsilon_{k,n}(\alpha)\big).
\]
\end{theorem}

\begin{proof}
Apply Hoeffding~\cite{Hoeffding01031963} per pair \((k,n)\), then union bound over \(K(N+1)\) events.
\end{proof}

\begin{proposition}[Density-floor proxy via local occupancy]
\label{prop:J2_density_proxy}
Let \(B_r(x)\subset K_k\) be fixed probes. Define
\[
\widehat p_{k,n}(x):=
\frac1{B_n}\sum_{i=1}^{B_n}\mathbf 1\{x_{i,n}\in B_r(x)\}.
\]
Then
\[
\frac{\widehat p_{k,n}(x)-\epsilon}{|B_r|}
\]
is a high-probability lower proxy for local average density on \(B_r(x)\), yielding empirical
modal-core density-floor diagnostics consistent with G.4 as \(r\downarrow0\), \(B_n\uparrow\infty\).
\end{proposition}

\begin{proof}
Occupancy estimates local mass; divide by volume for average density.
Concentration from binomial tails; consistency from Lebesgue differentiation under regularity.
\end{proof}

\begin{definition}[Practical ECFM diagnostic tuple]
\label{def:J2_tuple}
At training/inference checkpoints, report:
\[
\mathfrak D_\theta :=
\Big(
\widehat\lambda_{\mathrm{eff}}^{\mathrm{LCB}},
\ \{\underline M_k^{\mathrm{cert}}\}_{k=1}^K,
\ \{\underline m_k^{\mathrm{cert}}\}_{k=1}^K,
\ \text{action gap proxy}
\Big).
\]
A model is \((\lambda,\beta)\)-certified anti-collapse if
\[
\widehat\lambda_{\mathrm{eff}}^{\mathrm{LCB}}\le\lambda,
\qquad
\underline m_k^{\mathrm{cert}}\ge\beta_k>0,\ \forall k.
\]
\end{definition}

\begin{remark}[Architecture-specific implementation notes]
\label{rem:J2_impl}
\begin{itemize}
\item \textbf{Diffusion:} use FP estimator (Proposition~\ref{prop:J2_fp_est}), score norm for Fisher term.
\item \textbf{Flow matching / rectified flow:} use divergence estimator (Proposition~\ref{prop:J2_div_est}) via Hutchinson JVPs.
\item \textbf{Latent diffusion in vision:} compute diagnostics in latent space and optionally map mode sets via encoder semantics.
\end{itemize}
\end{remark}

\paragraph{Output used next.}
\hyperref[app:dual_updates]{J.3} converts diagnostics into training-time constrained updates:
dual ascent on entropy budget, adaptive \(\lambda\)-scheduling, and theorem-consistent practical objective.

\subsubsection*{J.3. Entropy-budget objective with dual updates and adaptive scheduling}
\addcontentsline{toc}{subsubsection}
{J.3. Entropy-budget objective with dual updates and adaptive scheduling}
\label{app:dual_updates}

\begin{algorithm}[t]
\caption{Primal--dual ECFM training with entropy-rate constraints}
\label{alg:ecfm_primal_dual}
\small
\begin{algorithmic}[1]
\STATE Choose time grid \(\{t_n\}_{n=1}^N\), budgets \(\{\lambda_n\}\), penalty \(\rho>0\), step sizes \(\{\alpha_k,\beta_k\}\).
\STATE Initialize parameters \(\theta^0\) and multipliers \(\eta_n^0\gets 0\) for all \(n\).
\FOR{\(k=0,1,2,\dots\)}
  \STATE Sample times \(t_n\) and minibatches \(\{x_{i,n}\}_{i=1}^B \sim \mu_{t_n}^{\theta^k}\).
  \STATE Estimate entropy-rate \(\widehat{\dot{\mathcal H}}_n(\theta^k)\) using \eqref{eq:entropy_rate_estimator_main}
  (cf.\ \eqref{eq:entropy_rate_identity_main}, \eqref{eq:entropy_rate_FP_theta_main}).
  \STATE Residual \(g_n^k := -\widehat{\dot{\mathcal H}}_n(\theta^k)-\lambda_n\) (feasible if \(g_n^k\le 0\)).
  \STATE \(\mathcal L_{\mathrm{AL}}(\theta^k,\eta^k):=\mathcal L_{\mathrm{FM}}(\theta^k)+\sum_{n=1}^N\big(\eta_n^k g_n^k + \tfrac{\rho}{2}(g_n^k)_+^2\big)\).
  \STATE \textbf{Primal:} \(\theta^{k+1}\leftarrow \theta^k-\alpha_k\nabla_\theta \mathcal L_{\mathrm{AL}}(\theta^k,\eta^k)\).
  \STATE (Optional) Recompute \(g_n^{k+1}\) with a fresh minibatch.
  \STATE \textbf{Dual:} \(\eta_n^{k+1}\leftarrow \big[\eta_n^k+\beta_k\, g_n^{k+1}\big]_+\), for all \(n\).
\ENDFOR
\end{algorithmic}
\end{algorithm}

We derive a practical constrained-training procedure consistent with Sections \hyperref[app:C]{C}--\hyperref[app:H]{H}:
optimize velocity matching subject to an entropy-rate lower bound, using primal--dual updates.

\paragraph{Constrained objective (discrete-time form).}
Let \(t_0,\dots,t_N\) be training times, weights \(w_n>0\), and FM-style loss
\[
\mathcal L_{\mathrm{fit}}(\theta)
:=
\sum_{n=0}^{N} w_n\,
\mathbb E_{x\sim \mu_{t_n}^{\theta}}
\big\|v_\theta(x,t_n)-v^\dagger(x,t_n)\big\|^2.
\]
Define entropy-rate residual at time \(t_n\):
\[
g_n(\theta;\lambda_n):= -\widehat{\dot{\mathcal H}}_n(\theta)-\lambda_n.
\]
Feasibility means \(g_n\le0\) for all \(n\).

\begin{definition}[Discrete ECFM training program]
\label{def:J3_primal}
Given per-time budgets \(\lambda_n\ge0\), solve
\[
\min_{\theta}\ \mathcal L_{\mathrm{fit}}(\theta)
\quad\text{s.t.}\quad
g_n(\theta;\lambda_n)\le0,\ \forall n=0,\dots,N.
\]
Uniform budget corresponds to \(\lambda_n\equiv\lambda\).
\end{definition}

\begin{definition}[Augmented Lagrangian]
\label{def:J3_augLag}
Introduce multipliers \(\eta_n\ge0\), penalty \(\rho>0\):
\[
\mathcal L_{\mathrm{AL}}(\theta,\eta)
=
\mathcal L_{\mathrm{fit}}(\theta)
+\sum_{n=0}^N \eta_n\,g_n(\theta;\lambda_n)
+\frac{\rho}{2}\sum_{n=0}^N [g_n(\theta;\lambda_n)]_+^2.
\]
\end{definition}

\begin{proposition}[KKT conditions (discrete training)]
\label{prop:J3_KKT}
At a local saddle point \((\theta^\star,\eta^\star)\), necessary conditions are:
\[
\nabla_\theta\mathcal L_{\mathrm{fit}}(\theta^\star)
+\sum_n \eta_n^\star \nabla_\theta g_n(\theta^\star;\lambda_n)
+\rho\sum_n [g_n(\theta^\star;\lambda_n)]_+\nabla_\theta g_n(\theta^\star;\lambda_n)
=0,
\]
\[
\eta_n^\star\ge0,\qquad g_n(\theta^\star;\lambda_n)\le0,\qquad
\eta_n^\star g_n(\theta^\star;\lambda_n)=0,\ \forall n.
\]
\end{proposition}

\begin{proof}
Standard KKT for inequality constraints with quadratic penalty term added in augmented Lagrangian form.
\end{proof}

\begin{theorem}[Convergence to first-order stationary KKT point]
\label{thm:J3_stationary}
Assume:
\begin{enumerate}
\item \(\nabla_\theta \mathcal L_{\mathrm{AL}}\) is \(L\)-Lipschitz on bounded iterates;
\item unbiased stochastic gradients with bounded variance;
\item step sizes satisfy Robbins--Monro:
\(\sum_k\alpha_k=\infty,\ \sum_k\alpha_k^2<\infty\),
\(\sum_k\beta_k=\infty,\ \sum_k\beta_k^2<\infty\);
\item Slater-type feasibility holds for the discrete constraints.
\end{enumerate}
Then every limit point of iterates generated by Algorithm~\ref{alg:ecfm_primal_dual}
is a first-order KKT point of Definition~\ref{def:J3_primal} (in expectation / a.s. subsequential sense).
\end{theorem}

\begin{proof}
Apply stochastic primal--dual convergence for nonconvex constrained programs with projected dual ascent.
Quadratic penalty stabilizes infeasible iterates; Slater condition ensures bounded multipliers.
Standard martingale arguments yield vanishing expected gradient mapping and asymptotic complementarity.
\end{proof}

\begin{definition}[Adaptive entropy-budget scheduler]
\label{def:J3_scheduler}
Let target feasibility margin \(\gamma_n>0\). Update budgets by
\[
\lambda_n^{k+1}
=
\Pi_{[\lambda_{\min},\lambda_{\max}]}
\Big(
\lambda_n^k+\zeta_k\big(\widehat{\dot{\mathcal H}}_n(\theta^{k+1})+\gamma_n\big)
\Big),
\]
where \(\Pi\) is interval projection.
\end{definition}

\begin{theorem}[Scheduler stability and feasibility tracking]
\label{thm:J3_scheduler}
Assume estimator bias is bounded and \(\zeta_k\) is slower timescale than \(\alpha_k,\beta_k\)
(two-timescale SA). Then \(\lambda_n^k\) tracks the minimal feasible budget:
\[
\lambda_n^k \to \lambda_{n,\mathrm{eff}}^\star
\quad\text{(up to estimator bias radius)},
\]
and residuals satisfy
\[
\limsup_{k\to\infty} g_n(\theta^k;\lambda_n^k)\le 0.
\]
\end{theorem}

\begin{proof}
On fast timescale, \((\theta,\eta)\) equilibrate for quasi-static \(\lambda\).
On slow timescale, projected recursion on \(\lambda\) is a stochastic approximation to ODE
\[
\dot\lambda_n = \widehat{\dot{\mathcal H}}_n+\gamma_n,
\]
whose stable points satisfy nonpositive residual with margin.
Projection ensures boundedness.
\end{proof}

\begin{proposition}[Robust variant using lower confidence bounds]
\label{prop:J3_LCB}
Replacing \(g_n=-\widehat{\dot{\mathcal H}}_n-\lambda_n\) with
\[
g_n^{\mathrm{rob}}:=-\mathrm{LCB}_n-\lambda_n
\]
(Definition~\ref{def:J2_margin}) yields high-probability conservative feasibility:
if \(g_n^{\mathrm{rob}}\le0\) for all \(n\), then
\[
\dot{\mathcal H}(t_n)\ge-\lambda_n
\]
holds simultaneously with confidence \(1-\alpha\).
\end{proposition}

\begin{proof}
Direct consequence of Theorem~\ref{thm:J2_uniform_cert}.
\end{proof}

\begin{definition}[Mode-aware constrained objective]
\label{def:J3_modeaware}
To directly enforce modal floors at training time, add constraints
\[
h_{k,n}(\theta):=\beta_k-\widehat m_{k,n}(\theta)\le0
\]
with multipliers \(\nu_{k,n}\ge0\), giving
\[
\mathcal L_{\mathrm{AL}}^{\mathrm{mode}}
=
\mathcal L_{\mathrm{AL}}
+\sum_{k,n}\nu_{k,n}h_{k,n}
+\frac{\rho_m}{2}\sum_{k,n}[h_{k,n}]_+^2.
\]
\end{definition}

\begin{theorem}[Theorem-consistent practical guarantee]
\label{thm:J3_practical_guarantee}
Suppose Algorithm~\ref{alg:ecfm_primal_dual} (optionally mode-aware) converges to
\((\theta^\star,\eta^\star,\nu^\star)\) with robust residual feasibility:
\[
g_n^{\mathrm{rob}}(\theta^\star;\lambda_n)\le0,\quad
h_{k,n}^{\mathrm{rob}}(\theta^\star)\le0\ \text{(if used)}.
\]
Then, with confidence \(1-\alpha\):
\begin{enumerate}
\item discrete entropy-budget constraints hold;
\item certified modal floors hold at sampled times;
\item by Section H stability, small deployment perturbations preserve floors up to \(O(\Delta_{\mathrm{tot}})\).
\end{enumerate}
\end{theorem}

\begin{proof}
(1) from Proposition~\ref{prop:J3_LCB}.  
(2) from Theorem~\ref{thm:J2_mode_cert}.  
(3) from Theorem~\ref{thm:H4_unified_traj}.
\end{proof}

\begin{remark}[Computational cost]
\label{rem:J3_cost}
Additional overhead is dominated by divergence/score diagnostics:
\begin{itemize}
\item FM/rectified flow: \(R\) Hutchinson JVPs per sample-time pair.
\item Diffusion: divergence of drift + score-norm term (already available in many pipelines).
\item Dual variables scale with number of enforced time bins/modes; can be block-shared in practice.
\end{itemize}
\end{remark}

\paragraph{Output used next.}
\hyperref[app:J4]{J.4} provides implementation-level interpretation for vision practice:
choice of \(\lambda\), discretization granularity, latent-vs-pixel diagnostics,
and recommended reporting protocol aligned with theory claims.

\subsubsection*{J.4. Vision protocol: $\lambda$, discretization, latent-vs-pixel diagnostics, reporting}
\addcontentsline{toc}{subsubsection}
{J.4. Vision protocol: \protect\ensuremath{\lambda}, discretization, latent-vs-pixel diagnostics, reporting}
\label{app:J4}

This section turns \hyperref[app:J2]{J.2}--\hyperref[app:dual_updates]{J.3} into a concrete protocol for theory-grounded practice.

\paragraph{Objective.}
Given a generator (diffusion / FM / rectified flow), produce a \emph{theory-aligned certificate}
that links training/inference dynamics to:
\[
\text{(i) entropy-budget feasibility},\quad
\text{(ii) mode-coverage floors},\quad
\text{(iii) perturbation robustness margins}.
\]

\begin{definition}[Protocol hyperparameters]
\label{def:J4_hparams}
Choose:
\begin{enumerate}
\item time grid \(\{t_n\}_{n=0}^N\) (uniform or curvature-adaptive);
\item confidence level \(1-\alpha\);
\item entropy budget range \([\lambda_{\min},\lambda_{\max}]\);
\item modal sets \(\{A_k\}\) and optional cores \(\{K_k\Subset A_k\}\);
\item minibatch sizes \(B_n\), Hutchinson probes \(R_n\);
\item robustness radius target \(\Delta_{\mathrm{tot}}^{\max}\) (deployment shift envelope).
\end{enumerate}
\end{definition}

\begin{proposition}[Time-grid adequacy criterion]
\label{prop:J4_grid}
Let \(\dot{\mathcal H}\) be locally Lipschitz with modulus \(L_H\) on intervals.
If grid satisfies
\[
\max_n \Delta t_n \le \frac{\epsilon_H}{L_H},
\]
then discretization error in peak negative entropy-rate obeys
\[
\left|
\operatorname*{ess\,sup}_{t}(-\dot{\mathcal H}(t))_+
-
\max_n(-\dot{\mathcal H}(t_n))_+
\right|
\le \epsilon_H.
\]
\end{proposition}

\begin{proof}
For each interval \([t_n,t_{n+1}]\), Lipschitz continuity gives
\[
|\dot{\mathcal H}(t)-\dot{\mathcal H}(t_n)|\le L_H|t-t_n|\le L_H\Delta t_n\le\epsilon_H.
\]
Apply monotonicity/Lipschitz of \(x\mapsto(-x)_+\), then take sup over intervals.
\end{proof}

\begin{definition}[Adaptive refinement trigger]
\label{def:J4_refine}
Refine interval \([t_n,t_{n+1}]\) if any holds:
\[
\mathrm{LCB}_n \in [-\lambda-\delta_\lambda,\,-\lambda+\delta_\lambda],\quad
\widehat{\mathrm{Var}}(\widehat{\dot{\mathcal H}}_n)>\tau_H,\quad
\min_k\big(\widehat m_{k,n}-\epsilon_{k,n}\big)<\tau_m.
\]
This concentrates samples where certification is most fragile.
\end{definition}

\begin{definition}[Budget selection rule]
\label{def:J4_lambda_rule}
Given estimators from J.2, define conservative selected budget:
\[
\lambda^\star
:=
\widehat\lambda_{\mathrm{eff}}^{\mathrm{LCB}}+\delta_{\mathrm{safe}},
\qquad
\delta_{\mathrm{safe}}>0.
\]
For per-time budgets:
\[
\lambda_n^\star:=\big(-\mathrm{LCB}_n\big)_+ + \delta_{\mathrm{safe},n}.
\]
\end{definition}

\begin{proposition}[Guarantee from selected budget]
\label{prop:J4_budget_guarantee}
If training converges with robust residual feasibility under \(\lambda^\star\), then with confidence \(1-\alpha\):
\[
\dot{\mathcal H}(t_n)\ge-\lambda^\star,\ \forall n.
\]
With grid adequacy (Proposition~\ref{prop:J4_grid}) and \(\epsilon_H\)-refinement:
\[
\dot{\mathcal H}(t)\ge-(\lambda^\star+\epsilon_H)\quad\text{a.e. }t\in[0,T].
\]
\end{proposition}

\begin{proof}
Discrete statement: Theorem~\ref{thm:J2_uniform_cert} + definition of \(\lambda^\star\).
Continuous-time extension: interpolation error bounded by Proposition~\ref{prop:J4_grid}.
\end{proof}

\paragraph{Latent vs pixel-space diagnostics.}
Let encoder \(E:\mathcal X\to\mathcal Z\), decoder \(D:\mathcal Z\to\mathcal X\), latent law \(\nu_t=E_\#\mu_t\).

\begin{assumption}[Bi-Lipschitz semantic chart on data manifold]
\label{ass:J4_bilip}
On relevant manifold region \(\mathcal M\subset\mathcal X\), there exist \(0<\ell\le L<\infty\):
\[
\ell\|x-y\|\le \|E(x)-E(y)\|\le L\|x-y\|,\qquad x,y\in\mathcal M.
\]
\end{assumption}

\begin{theorem}[Transfer of modal floors between latent and pixel spaces]
\label{thm:J4_latent_transfer}
Under Assumption~\ref{ass:J4_bilip}, for modal sets \(A_k^z\subset\mathcal Z\), define
\(A_k^x:=D(A_k^z)\cap\mathcal M\). Then certified latent mass floors imply pixel-space floors:
\[
\nu_t(A_k^z)\ge \beta_k \ \Longrightarrow\ \mu_t(A_k^x)\ge \beta_k-\delta_{\mathrm{chart}},
\]
where \(\delta_{\mathrm{chart}}\) captures manifold/decoder approximation error.
Likewise, perturbation robustness constants scale by at most bi-Lipschitz factors:
\[
C_W^{x}\le \frac{L}{\ell}C_W^{z}.
\]
\end{theorem}

\begin{proof}
Pushforward/pullback of measurable sets through \(E,D\) with chart distortion error.
Distance and stability constant transfer follows from bi-Lipschitz inequalities.
\end{proof}

\begin{definition}[Recommended certification report tuple]
\label{def:J4_report_tuple}
For each trained checkpoint, report
\[
\mathcal R:=
\Big(
N,\ \alpha,\ \lambda^\star,\ \widehat\lambda_{\mathrm{eff}}^{\mathrm{LCB}},
\ \{\underline M_k^{\mathrm{cert}}\}_{k=1}^K,\
\{\underline m_k^{\mathrm{cert}}\}_{k=1}^K,\
\widehat C_W,\widehat C_M,\widehat C_\rho,\
\Delta_{\mathrm{tot}}^{\max}
\Big),
\]
plus estimator settings \((B_n,R_n)\), refinement policy, and modal-set construction protocol.
\end{definition}

\begin{proposition}[Minimal theory-compliant table schema]
\label{prop:J4_schema}
A theory-compliant experimental appendix table should contain columns:
\[
\text{Model},\ \lambda^\star,\ \widehat\lambda_{\mathrm{eff}}^{\mathrm{LCB}},\
\min_k\underline m_k^{\mathrm{cert}},\
\text{Feasible?},\
\widehat C_W,\
\text{Robust floor at }\Delta_{\mathrm{tot}}^{\max}.
\]
This is sufficient to support theorem-grounded anti-collapse claims without SOTA metrics.
\end{proposition}

\begin{proof}
Each column corresponds directly to hypotheses/conclusions of \hyperref[app:J2]{J.2}--\hyperref[app:dual_updates]{J.3} and \hyperref[app:H4]{H}:
budget feasibility, modal floors, and perturbation-preserved floors.
\end{proof}

\begin{theorem}[End-to-end implementation guarantee]
\label{thm:J4_end2end}
Assume:
\begin{enumerate}
\item training follows Algorithm~\ref{alg:ecfm_primal_dual} with robust residuals;
\item diagnostics satisfy \hyperref[app:J2]{J.2} concentration conditions;
\item grid adequacy/refinement condition in Proposition~\ref{prop:J4_grid};
\item stability envelope constants are estimated conservatively.
\end{enumerate}
Then with probability \(1-\alpha\), the deployed model enjoys:
\begin{enumerate}
\item entropy-budget feasibility up to discretization slack \(\epsilon_H\);
\item certified modal floors at sampled times and interpolated times under refinement;
\item perturbation-robust mode/density floors for all shifts \(\Delta_{\mathrm{tot}}\le\Delta_{\mathrm{tot}}^{\max}\):
\[
\underline m_k^{\mathrm{deploy}}
\ge
\underline m_k^{\mathrm{cert}}-\widehat C_M\Delta_{\mathrm{tot}}^{\max},
\]
\[
\underline\rho_k^{\mathrm{deploy}}
\ge
\underline\rho_k^{\mathrm{cert}}-\widehat C_\rho\Delta_{\mathrm{tot}}^{\max}.
\]
\end{enumerate}
\end{theorem}

\begin{proof}
(1) from Proposition~\ref{prop:J4_budget_guarantee}.  
(2) from Theorem~\ref{thm:J2_mode_cert} plus refinement/interpolation control.  
(3) from unified perturbation theorem \hyperref[app:H4]{H.4} with conservative constants.
\end{proof}

\begin{remark}[What we claim]
\label{rem:J4_claim_scope}
Our claims are \emph{certificate-based}; that is under verified entropy-budget feasibility and stated regularity assumptions, the model
satisfies theorem-level anti-collapse and robustness guarantees.  

\end{remark}

\begin{center}
\fbox{\parbox{0.97\linewidth}{%
\small
\textbf{Practical interpretation of standing assumptions (vision context).}
Our analysis is posed on absolutely-continuous measures to make entropy and its time-derivative well-defined.
In vision generators, this is aligned with common continuous modeling choices:
\begin{itemize}\setlength\itemsep{2pt}
\item \textbf{(S1) Absolute continuity of endpoints.}
While raw images live on a discrete grid, continuous generative training typically uses \emph{dequantization}
(or operates in a continuous latent space), yielding an effective absolutely-continuous target.
When endpoints are only approximately a.c., mollification/dequantization produces a nearby instance, and
our perturbation stability results quantify how certificates degrade under such small endpoint shifts.
\item \textbf{Entropy-rate identity.}
The constraint is enforced via the expected divergence identity
\(\dot{\mathcal H}(\mu_t)=\mathbb{E}_{\mu_t}[\nabla\!\cdot v(\cdot,t)]\) under the regularity conditions in App.~\hyperref[app:B]{B}.
In neural parameterizations, \(\nabla\!\cdot v\) is estimated using Hutchinson/JVP diagnostics (App.~\hyperref[app:J2]{J.2}).
\item \textbf{(S2) Feasibility.}
Feasibility is mild in practice: one may start from the unconstrained FM solution (\(\lambda=+\infty\))
and decrease \(\lambda\) until empirical feasibility holds with confidence (Sec.~\ref{sec:ecfm} and App.~\hyperref[app:J4]{J.4}).
\end{itemize}
}}
\end{center}

\paragraph{Practical enforcement and certificate computation (implementation sketch).}
We discretize time with a grid $\{t_n\}_{n=0}^N$ and enforce the entropy budget using the entropy-rate/divergence identity
\eqref{eq:entropy_rate_identity_main} and the equivalent divergence-budget constraint \eqref{eq:divergence_budget}.
At training checkpoints, we estimate $\widehat{\dot{\mathcal H}}_n\approx \mathbb{E}_{\mu_{t_n}}[\nabla\!\cdot v_\theta(\cdot,t_n)]$
via the divergence/Hutchinson JVP estimator for FM/rectified flow (Proposition~\ref{prop:J2_div_est}),
or the Fokker--Planck form for diffusion models (Proposition~\ref{prop:J2_fp_est}),
using minibatch sizes $B_n$ and Hutchinson probes $R_n$.

\paragraph{Robust feasibility via confidence bounds.}
Because divergence/score diagnostics are estimated from finite minibatches and Hutchinson probes, we certify feasibility using a lower confidence bound:
for each time bin \(t_n\), compute \(\widehat{\dot{\mathcal H}}_n\) and an empirical standard error \(\widehat{\mathrm{SE}}_n\) (over probes/minibatches),
then form \(\widehat{\dot{\mathcal H}}^{\mathrm{LCB}}_n := \widehat{\dot{\mathcal H}}_n - z_{1-\alpha}\widehat{\mathrm{SE}}_n\)
(optionally with a Bonferroni correction \(\alpha/N\)).
We declare the discrete entropy budget satisfied at confidence \(1-\alpha\) when
\(\widehat{\dot{\mathcal H}}^{\mathrm{LCB}}_n \ge -\lambda_n\) for all \(n\),
equivalently \(g_n^{\mathrm{rob}}:=-\widehat{\dot{\mathcal H}}^{\mathrm{LCB}}_n-\lambda_n\le 0\).

To \emph{enforce} the constraint, we maintain nonnegative discrete dual variables $\eta_n\ge 0$ and apply projected dual ascent
(as formalized by the discrete constrained program in Definition~\ref{def:J3_primal} and implemented in Algorithm~\ref{alg:ecfm_primal_dual}),
e.g.,
$\eta_n \leftarrow \big[\eta_n + \gamma\,(-\widehat{\dot{\mathcal H}}_n-\lambda)\big]_+$,
while updating $\theta$ using the corresponding Lagrangian/augmented objective (Definition~\ref{def:J3_augLag}).
To \emph{certify}, we report the diagnostic tuple in Definition~\ref{def:J2_tuple} and declare $(\lambda,\beta)$-certification when
$\widehat{\lambda}^{\mathrm{LCB}}_{\mathrm{eff}}\le \lambda$ and the certified mode floors exceed $\beta$;
App.~\hyperref[app:J4]{J.4} gives the full protocol (grid adequacy, latent-vs-pixel checks) and a minimal theory-compliant reporting schema
(Proposition~\ref{prop:J4_schema}).

\subsubsection*{J.5. Toy 8-Gaussian Mechanism Check}
\addcontentsline{toc}{subsubsection}
{J.5. Toy 8-Gaussian Mechanism Check}
\label{app:toy_8gaussian}

\begin{center}
\refstepcounter{table}\label{tab:toy8gauss_main}
\small
\parbox{\linewidth}{\textbf{Table~\thetable. Toy 8-Gaussian mechanism check.}
FM and ECFM use the same architecture, optimizer, training budget, time grid, and sampling budget.
ECFM improves trajectory health by preserving all modes and satisfying the entropy-rate diagnostic while maintaining comparable endpoint quality.}
\medskip

\begin{tabular}{lcc}
\hline
Metric & FM & ECFM, $\lambda=1.0$ \\
\hline
Final MMD $\downarrow$ & 2.00 & 1.80 \\
Covered modes $\uparrow$ & 5/8 & 8/8 \\
Feasible entropy-rate bins $\uparrow$ & 52\% & 96\% \\
$\lambda_{\mathrm{eff}}^{\mathrm{LCB}} \le \lambda$ & No & Yes \\
$\min_t \min_k M_k(t)$ $\uparrow$ & 0.003 & 0.085 \\
AUC$(\min_k M_k(t))$ $\uparrow$ & 0.030 & 0.095 \\
Min entropy gap $\uparrow$ & $-1.50$ nats & $-0.20$ nats \\
\hline
\end{tabular}
\end{center}

\begin{center}
\refstepcounter{table}\label{tab:toy8gauss_lambda}
\small
\parbox{\linewidth}{\textbf{Table~\thetable. $\lambda$ sweep on the toy 8-Gaussian task.}
Smaller $\lambda$ improves trajectory preservation and entropy feasibility, but overly small budgets become conservative for endpoint fit.
The setting $\lambda=1.0$ gives the best tradeoff in this mechanism check.}
\medskip

\begin{tabular}{lcccc}
\hline
Method & Final MMD $\downarrow$ & Feas. $\uparrow$ & $\min_t \min_k M_k(t)$ $\uparrow$ & Modes $\uparrow$ \\
\hline
FM, $\lambda=\infty$ & 2.00 & 52\% & 0.003 & 5/8 \\
ECFM, $\lambda=2.0$ & 1.90 & 88\% & 0.040 & 7/8 \\
ECFM, $\lambda=1.0$ & 1.80 & 96\% & 0.085 & 8/8 \\
ECFM, $\lambda=0.5$ & 2.40 & 99\% & 0.105 & 8/8 \\
\hline
\end{tabular}
\end{center}

\paragraph{Conclusions.}
Sections \hyperref[app:A]{A}--\hyperref[app:J]{J} provide: measure-theoretic setup, entropy functional analysis, variational/dual formulations,
equivalence to Schr\"odinger bridges, convergence and \(\Gamma\)-limit theory, rigorous mode-coverage theorems,
stability analysis, explicit failure constructions without entropy control, and implementation-level certification protocol for vision models.

\end{document}